%% file: main.tex
\documentclass[11pt]{report}
\usepackage[
  dissertation
 ,final
 ,raggedbottom
,tocbold        % uncomment to enable bold chapter titles in the ToC
]{USCthesis}

\usepackage[
  showframe = false,% draw a border around textwidth
  pass      = true, % force 8.5"x11" pagesize
]{geometry}

\usepackage[utf8]{inputenc} % load inputenc before csquotes
\usepackage{microtype}
\usepackage[tt=false]{libertine} % libertine's \ttfamily isn't that great
\usepackage[T1]{fontenc} % load fonts before fontenc

\usepackage{natbib}
\usepackage{array}
\usepackage{amsmath, amssymb, amsthm, bm, bbm, amsfonts, mathalfa}
\usepackage{mathtools}
\usepackage[mathscr]{eucal}
\usepackage{dsfont}

\usepackage[export]{adjustbox} % for frame option in \includegraphics
\usepackage{graphicx}
\usepackage{subcaption}

\usepackage{algorithm, algorithmic}
\usepackage{listings}

\usepackage{booktabs}
\usepackage{multirow}
\newcolumntype{M}[1]{>{\centering\arraybackslash}m{#1}}
\newcolumntype{L}[1]{>{\raggedright\arraybackslash}m{#1}}

\usepackage{appendix}
\usepackage{spverbatim}
\usepackage{csquotes}
\usepackage{efbox}
\usepackage{enumitem}
\usepackage[shortcuts]{extdash} % use `\-/' to hyphenate words/phrases that have a dash in them
\usepackage[symbol]{footmisc}
\usepackage{hyphenat}
\usepackage{ifthen}
\usepackage{ragged2e}
\usepackage{url}

\usepackage[usenames, dvipsnames, svgnames,table]{xcolor}
\definecolor{lightblue}{RGB}{60,60,200}
\definecolor{darkblue}{RGB}{1,1,255}
\definecolor{mplblue}{RGB}{31, 119, 180}
\definecolor{cskyblue}{rgb}{0.01,0.39,0.75}

\usepackage[textsize=tiny,textwidth=2cm]{todonotes}

\usepackage[
  breaklinks    = true,
  colorlinks    = true,
  hypertexnames = false,
  pdfpagelabels = false,
  citecolor     = {blue!80!black},
  linkcolor     = {blue!80!black},
  urlcolor      = {blue!80!black},
]{hyperref} % load hyperref as the last package

\usepackage[capitalize,nameinlink,noabbrev]{cleveref}
\crefformat{equation}{Eq.~(#2#1#3)}
\crefmultiformat{equation}{(#2#1#3)}%
{ and~(#2#1#3)}{, (#2#1#3)}{ and~(#2#1#3)}
\crefformat{theorem}{Theorem~#2#1#3}
\crefformat{proposition}{Proposition~#2#1#3}
\crefformat{lemma}{Lemma~#2#1#3}
\crefformat{remark}{Remark~#2#1#3}
\crefformat{section}{Section~#2#1#3}
\crefformat{assumption}{Assumption~#2#1#3}
\crefformat{example}{Example~#2#1#3}
\crefformat{figure}{Figure~#2#1#3}
\crefformat{algorithm}{Algorithm~#2#1#3}
\crefrangeformat{section}{Section~#3#1#4--#5#2#6}
\crefformat{assumption}{Assumption~#2#1#3}
\crefrangeformat{equation}{~(#3#1#4--#5#2#6)}
\crefrangeformat{figure}{Figure~#3#1#4--#5#2#6}
\crefrangeformat{assumption}{Assumptions~(#3#1#4--#5#2#6)}
\crefformat{appendix}{Appendix~#2#1#3}

\theoremstyle{definition}  % remove this for going back to the default italic theorem, defintions, lemmas, etc
\newtheorem{definition}{Definition}[section]
\newtheorem{theorem}{Theorem}[section]
\newtheorem{lemma}[theorem]{Lemma}
\newtheorem{proposition}[theorem]{Proposition}
\newtheorem{remark}{Remark}[section]
\newtheorem{corollary}[theorem]{Corollary}
\newtheorem{fact}[theorem]{Fact}

\usepackage{etoolbox}
\newtoggle{draft}
\settoggle{draft}{true} % change toggle for draft or final versions

\iftoggle{draft}{
  \overfullrule=10pt                       % highlight overfull hboxes
}{
  \PassOptionsToPackage{final}{showlabels} % hide labels on figures, etc
}

\newtoggle{thesis}
\settoggle{thesis}{true}

\newcommand{\predict}{LIMIT}

\DeclareMathOperator{\USI}{SI}
\DeclareMathOperator{\FSI}{F-SI}

\def \ofcmi {f\text{-}\mathrm{CMI}}
\def \fcmi {f\text{-}\mathrm{CMI}_P}

\input{useful-commands}
\input{math-notation}

\begin{document}

% title should be all caps
\title{{On Information Captured by Neural Networks:\\Connections with Memorization and Generalization}}

\author{Hrayr Harutyunyan}

% major should be all caps
\majorfield{COMPUTER SCIENCE}

% date should be May, August, or December (when degrees are conferred)
\submitdate{August 2023}

%%% preface %%%%%%%%%%%%%%%%%%%%%%%%%%%%%%%%%%%%%%%%%%%%%%%%%%%%%%%%%%%%
\begin{preface}
  % \prefacesection{Dedication}
  % \input{dedication.tex}

  \prefacesection{Acknowledgements}
  \input{acknowledgements.tex}

  {
  \hypersetup{hidelinks} % color all links black in the preface
  \tableofcontents
  \listoftables
  \listoffigures
  }

  \prefacesection{Abstract}
  \input{abstract.tex}
\end{preface}

%%% introduction %%%%%%%%%%%%%%%%%%%%%%%%%%%%%%%%%%%%%%%%%%%%%%%%%%%%%%%
\chapter{Introduction}\label{ch:introduction}
\input{introduction.tex}

\chapter{Improving Generalization by Controlling Label Noise Information in Neural Network Weights}\label{ch:label-noise}
\input{label-noise/main.tex}

\chapter{Estimating Informativeness of Samples with Smooth Unique Information}\label{ch:unique-info}
\input{unique-info/main.tex}

\chapter{Information-theoretic generalization bounds for black-box learning algorithms}\label{ch:sample-info}
\input{sample-info/main.tex}

\chapter{Formal limitations of sample-wise information-theoretic generalization bounds}\label{ch:limitations}
\input{limitations/main.tex}

\chapter{Supervision Complexity and its Role in Knowledge Distillation}\label{ch:sup-complexity}
\input{sup-complexity/main.tex}

%%% bibliography %%%%%%%%%%%%%%%%%%%%%%%%%%%%%%%%%%%%%%%%%%%%%%%%%%%%%%%
%  * bibliography should be an un-numbered chapter, and still have a
%    pdfbookmark and a line in the table of contents
%
%  * bibliography contents should be singlespace, and optionally a smaller
%    font
%
%  * first line of this "chapter" should be in the same spot as the first
%    line of preface sections (e.g., acknowledgement)
%
%  * we use \raggedright so things like URLs and DOIs aren't stretched out.
%
\clearpage
\begin{singlespace}
\bibliographystyle{plainnat}
\addcontentsline{toc}{chapter}{Bibliography}
\bibliography{main}
\end{singlespace}

% \begin{singlespace}
%   % increase penalty such that we don't break entries over pages
%   % source: https://tex.stackexchange.com/a/43275
%   \patchcmd{\bibsetup}{\interlinepenalty=5000}{\interlinepenalty=10000}{}{}

%   % reduce spacing between each bibentry
%   \setlength\bibitemsep{0.9\baselineskip}

%   % don't justify-align entries: this prevents stretching out each line
%   \raggedright
%   \printbibliography[
%     heading = none
%   ]
% \end{singlespace}

% %count figure numbers using letters
% \renewcommand{\thefigure}{A\arabic{figure}}
% \setcounter{figure}{-1}

% \clearpage
% \appendix

% Appendices
\crefalias{subsection}{appendix}
\crefalias{section}{appendix}

\phantomsection
\addcontentsline{toc}{chapter}{Appendices}%
\chapter*{Appendices}
\renewcommand\thesection{\Alph{section}}
\renewcommand*{\thesubsection}{\Alph{section}.\arabic{subsection}}
\begingroup
\numberwithin{equation}{section}

%%%%%%%%%%%% Proofs %%%%%%%%%%%% 
\section{Proofs}\label{app:proofs}
This appendix presents the missing proofs.
\input{label-noise/sections/proofs}
\input{unique-info/sections/proofs}
\input{sample-info/sections/proofs}
\input{limitations/proofs}
\input{sup-complexity/proofs}

%%%%%%%%%%% Derivations %%%%%%%%%% 
\section{Derivations}\label{app:derivations}
\input{unique-info/sections/derivations}

%%%%%%%%%%% Additional experimental details %%%%%%%%%%%
\section{Additional experimental details}\label{app:add-exp-details}
This appendix presents some of the experimental details that were not included in the main text for better readability.
\input{sample-info/sections/experimental-details}

%%%%%%%%%%% Additional results %%%%%%%%%% 
\section{Additional results}\label{app:add-results}
\input{label-noise/sections/additional_results}
\input{sup-complexity/additional_results}

\endgroup

\end{document}

%% file: useful-commands.tex
\newcommand{\rbr}[1]{\left(#1\right)}
\newcommand{\bbr}[1]{\Big(#1\Big)}
\newcommand{\sbr}[1]{\left[#1\right]}
\newcommand{\cbr}[1]{\left\{#1\right\}}
\newcommand{\nbr}[1]{\left\|#1\right\|}
\newcommand{\abs}[1]{\left|#1\right|}

\newcommand{\bs}[1]{{\boldsymbol{#1}}}  % shorthand for boldsymbol

\newcommand{\mathset}[1]{\cbr{#1}}  % for sets

\providecommand{\ind}[1]{{\bf 1}{\cbr{#1}}}  % indicator function

\newcommand\numberthis{\addtocounter{equation}{1}\tag{\theequation}}  % for numbering just one equation in a list of equations

\newcommand{\der}[2]{\frac{\text{d} #1}{\text{d} #2} }

\newcommand{\indep}{\protect\mathpalette{\protect\independenT}{\perp}}  % independence sign
\def\independenT#1#2{\mathrel{\rlap{$#1#2$}\mkern2mu{#1#2}}}

\DeclareMathOperator*{\argmin}{\operatorname{argmin}}

\DeclareMathOperator{\sign}{\operatorname{sign}}

\DeclareMathOperator{\trace}{\operatorname{Tr}}

\DeclareMathOperator*{\var}{\operatorname{Var}}
\DeclareMathOperator*{\cov}{\operatorname{Cov}}
\DeclareMathOperator{\E}{\mathbb{E}}

\newcommand{\KL}[2]{\operatorname{KL}\rbr{#1\,\|\,#2}}
\newcommand{\KLdis}[3]{\operatorname{KL}^{#1}\rbr{#2\,\|\,#3}}

\newcommand{\Choose}[2]{\binom{#1}{#2}}

\newcommand{\floor}[1]{\lfloor #1 \rfloor}

\newcommand{\thesame}[1]{#1}

\newcommand{\best}[1]{\cellcolor{gray!25}{#1}}

%% file: math-notation.tex
\def \HH {\mathcal{H}}

\def \XX {\mathcal{X}}
\def \YY {\mathcal{Y}}
\def \ZZ {\mathcal{Z}}
\def \LL {\mathcal{L}}

\def \FF {\mathcal{F}}
\def \GG {\mathcal{G}}

\def \WW {\mathcal{W}}

\def \bR {\mathbb{R}}

%% file: acknowledgements.tex
First and foremost, I would like to express my deepest gratitude to my advisors, Aram Galstyan and Greg Ver Steeg.
This work would not exist without their invaluable guidance and unwavering support.
I am immensely thankful to them for cultivating an environment with complete academic freedom.
I would also like to thank Bistra Dilkina, Haipeng Luo, and Mahdi Soltanolkotabi for their insightful comments as members of various committees during my Ph.D. journey.

I extend my heartfelt appreciation to my fellow former or present Ph.D. students Sami Abu-El-Haija, Shushan Arakelyan, Rob Brekelmans, Aaron Ferber, Palash Goyal, Umang Gupta, David Kale, Neal Lawton, Myrl Marmarelis, Daniel Moyer, Kyle Reing, for their friendship, stimulating discussions, and collaborative spirit.
I would like to thank my friends and collaborators at YerevaNN, who warmly welcomed me in their office whenever I was in Armenia.

I would like to express my sincere appreciation to my collaborators Alessandro Achille, Rahul Bhotika, Sanjiv Kumar, Orchid Majumder, Aditya Krishna Menon, Giovanni Paolini, Maxim Raginsky, Avinash Ravichandran, Ankit Singh Rawat, Stefano Soatto, and Seungyeon Kim for their invaluable contributions and support throughout our collaborative endeavors.
I am thankful to Alessandro Achille and Avinash Ravichandran for hosting two fruitful summer internships at Amazon Web Services; and to Ankit Singh Rawat and Aditya Krishna Menon for the productive internship at Google Research.

I am indebted to the family of Avanesyans, who treated me like a family member and provided much-needed family support when I was away from my home country.
Finally, I would like to express my deepest gratitude to my parents and brother, whose unconditional love and sacrifices have been the foundation of my academic pursuits.

I acknowledge support from the USC Annenberg Fellowship.
Works described in \cref{ch:label-noise,ch:sample-info} were based partly on research sponsored by Air Force Research Laboratory under agreement number FA8750-19-1-1000.

%% file: abstract.tex
Despite the popularity and success of deep learning, there is limited understanding of when, how, and why neural networks generalize to unseen examples.
Since learning can be seen as extracting information from data, we formally study information captured by neural networks during training.
Specifically, we start with viewing learning in presence of noisy labels from an information-theoretic perspective and derive a learning algorithm that limits label noise information in weights.
We then define a notion of unique information that an individual sample provides to the training of a deep network, shedding some light on the behavior of neural networks on examples that are atypical, ambiguous, or belong to underrepresented subpopulations.
We relate example informativeness to generalization by deriving nonvacuous generalization gap bounds.
Finally, by studying knowledge distillation, we highlight the important role of data and label complexity in generalization.
Overall, our findings contribute to a deeper understanding of the mechanisms underlying neural network generalization.

%% file: introduction.tex
Over the past decade, deep learning has achieved remarkable success in a wide range of applications, including computer vision, natural language processing, speech recognition, robotics, and generative modeling. %~\citep{krizhevsky2017imagenet,he2015delving,graves2013speech,vaswani2017attention,rcnn}.
Large neural networks trained with variants of stochastic gradient descent demonstrate excellent generalization capability, despite having enough capacity to memorize their training set~\citep{zhang2016understanding}.
Although some progress has been made toward understanding deep learning, a comprehensive understanding of \emph{when}, \emph{why}, and \emph{how} neural networks generalize remains elusive.

\section{Memorization in deep learning}
One aspect of deep learning that needs to be understood better is \emph{memorization}.
In a broad sense, it is unclear what information neural networks memorize; what types of memorization occur during training; which are harmful and which are helpful to generalization; and how to measure and control various types of memorization.
The fact that the term ``memorization'' has many definitions and interpretations indicates that memorization can come in different flavors.
The simplest form of memorization is memorizing incorrect labels or label noise, which has been the subject of many studies~\citep{survey,song2022learning}.
In specific learning scenarios, memorizing noisy labels does not significantly affect test error~\citep{liang2020just, bartlett2020benign,hastie2022surprises,frei2022benign,cao2022benign}.
However, label noise memorization significantly degrades test performance in typical deep learning settings~\citep{zhang2016understanding,chen2019understanding,mallinar2022benign}.
Given that real-world labeled datasets often have mislabeled examples due to ambiguities, labeling errors, or measurement errors, there is a great need for methods of measuring memorization and for training algorithms that are robust to label noise.

Memorization of noisy labels is also important because it provides insights into understanding the behavior of neural networks.
It has been observed that neural networks learn simple
and generalized patterns first~\citep{arpit2017closer} and generalize well in the early training epochs~\citep{li2020gradient}.
Even when trained on a dataset with noisy labels, neural networks still learn useful representations, especially in early layers~\citep{dosovitskiy2014discriminative,pondenkandath2018leveraging,maennel2020neural,anagnostidis2023the}.

Besides label noise, neural networks can memorize other information about their training set.
For example, an image classifier can memorize individual training examples~\citep{feldman2020neural}.
Given a white-box or black-box access to a neural network, it is possible to extract membership information, as shown by
\citet{shokri2017membership} and \citet{nasr2019comprehensive}.
Neural networks trained for language modeling memorize sensitive personal information~\citep{carlini2019secret}, factual knowledge~\citep{petroni-etal-2019-language}, long sequences of words from training data verbatim~\citep{carlini2021extracting,tirumala2022memorization,carlini2023quantifying}, and idioms~\citep{haviv-etal-2023-understanding}.
A diffusion model can memorize specific training examples~\citep{carlini2023extracting}.
Some of these instances of memorization are undesirable due to privacy concerns.
Nevertheless, from the generalization perspective, some types of memorization can be beneficial.
Indeed, memorization is an optimal strategy in some settings~\citep{feldman2020does,brown2021memorization}.
There is currently limited knowledge regarding when neural networks transition from generalizing to memorizing~\citep{cohen2018dnn,Zhang2020Identity}.

\section{The generalization puzzle}
Arguably the most important question in deep learning is that of generalization.
As demonstrated by \citet{zhang2016understanding}, overparameterized neural networks generalize well despite being able to memorize the training set mechanically.
Classical learning theory results based on various notions of hypothesis set complexity do not explain this phenomenon --
the same class of neural networks can generalize well for one training data but fail for another.
Furthermore, the modern practice of training neural networks goes against the conventional bias-variance trade-off wisdom in that neural networks are often trained to interpolation without explicit generalization~\citep{belkin2019reconciling}.

These observations have sparked a search for effective notions of complexity, implicit biases, and for data- and algorithm-dependent generalization bounds~\citep{bartlett2021deep,Jiang*2020Fantastic,dziugaite2020search}.
It has been well-established that number of parameters does not serve as a good notion of complexity, and other alternatives were proposed, such as size-based measures~\citep{661502, neyshabur2015norm, bartlett2017spectrally,golowich2018size}.
The good generalization of stochastic gradient descent has been attributed to implicit biases such as finding flat minima~\citep{keskar2017on}, finding minimum norm solutions~\citep{soudry2018implicit}, spectral bias~\citep{rahaman2019spectral}, simplicity bias~\citep{Nakkiran:2019}, representation compression~\citep{shwartzziv2017opening}, 
and stability~\citep{hardt2016train}, among others.
Nevertheless, it has been challenging to find generalization bounds that produce good quantitative results~\citep{neyshabur2017exploring,nagarajan2019uniform,dziugaite2020search}.
Exceptions include some PAC-Bayes bounds~\citep{DBLP:conf/uai/DziugaiteR17,zhou2018nonvacuous} that apply to modified learning algorithms.

While the role of learning algorithms and architectures have been studied extensively, the same cannot be said for the role of data distribution.
Evidently, a good learning algorithm and neural network architecture alone do not necessarily result in good generalization.
Therefore, an adequate generalization theory should also make some assumptions about the data distribution.
For most data-dependent generalization bounds, this is done by assuming access to training data, limiting the explanatory power.
Ideally, a theory should predict good generalization without making strong assumptions about the data distribution.
For example, \citet{yang2022does} show that the eigenspectrum of the input correlation matrix of typical datasets has a certain property that effectively induces a capacity control for a neural network.
\citet{arora2019fine} and \citet{ortiz2021can} show that a good alignment between the neural tangent kernel and labels ensures fast learning and good generalization.
Finding sufficient data properties that are characteristic of real-world datasets and enable good generalization is thus a key direction in explaining the success of deep learning.

Overall, besides satisfying scientific curiosity and providing generalization guarantees, understanding generalization in deep learning is also a practically fruitful research direction.
We expect that an adequate generalization theory will lead to improved training algorithms, neural network architectures, and better scaling with respect to data and model size.
Some instances of confirming this expectation already exist, such as training algorithms Entropy-SGD~\citep{chaudhari2019entropy} and sharpness-aware minimization~\citep{foret2020sharpness}.
We present another such development in this dissertation (\cref{ch:sup-complexity}).

\section{Our contributions}
The results presented in this dissertation can be seen as contributions to the problems of memorization and generalization in deep learning.
We view training a neural network as extracting information from samples in a dataset and storing it in the weights of the network so that it may be used in future inference or prediction.
Consequently, we formally study information captured by a neural network during training to understand the complex interplay between data, learning algorithm, and hypothesis class.
As we shall see, this allows us to define principled notions of memorization, shed some light on the behavior of neural networks, derive nonvacuous generalization gap bounds, and guide the training algorithm design process.
\cref{tab:contributions} presents the mapping between the chapters and papers.

\begin{table}[t]
    \small
    \caption{A mapping from chapters to papers.}
    \centering
    \begin{tabular}{cc}
    \toprule
    Chapter & Paper\\
    \midrule
    \cref{ch:label-noise}   & \citet{harutyunyan2020improving} \\
    \cref{ch:unique-info}   & \citet{harutyunyan2021estimating} \\
    \cref{ch:sample-info}   & \citet{harutyunyan2021informationtheoretic} \\
    \cref{ch:limitations}   & \citet{harutyunyan2022formal} \\
    \cref{ch:sup-complexity}   & \citet{harutyunyan2023supervision} \\
    \bottomrule
    \end{tabular}
    \label{tab:contributions}
\end{table}

\subsection*{Label noise memorization}
In \cref{ch:label-noise}, we start with studying the simplest form of memorization conceptually: memorizing label noise.
As a fundamental measure of label noise memorization, we consider the Shannon mutual information $I(W; \bs{Y} \mid \bs{Y})$ between the learned weights $W$ and the training labels $\bs{Y}$, given the training inputs $\bs{X}$.
We show that reducing the training error beyond the noise level results in a large value of $I(W; \bs{Y} \mid \bs{X})$.
Furthermore, we show that learning with constrained label noise information provides a certain level of label noise robustness.
Starting with this constrained problem, we derive a new learning algorithm in which the main classifier is trained with gradient updates predicted by another neural network that does not access the dataset labels.
In addition to good results on standard benchmark tasks, the proposed algorithm provides a partial justification for the well-known co-teaching approach~\citep{han2018co}.

\subsection*{A more general notion of example memorization}
In \cref{ch:unique-info}, we go beyond label noise memorization and aim to quantify how much information about a particular example is captured during the training of a neural network.
Due to combinatorial challenges related to redundancies, synergies, and high-order dependencies, we focus on measuring \emph{unique} information, which can be seen as a measure of example memorization (not necessarily harmful to generalization).
We define, both in weight space and function space, a notion of unique information that an example provides to the training of a neural network.
While rooted in information theory, these quantities capture some aspects of stability theory and influence functions.
The proposed unique information measures apply to even deterministic training algorithms and can be approximated efficiently for wide or pretrained neural networks using a linearization of the model.

Apart from having important applications, such as data valuation, active learning, data summarization, and guiding the data collection process, measuring unique information provides insights into \emph{how} neural networks generalize.
We find that typically only a small portion of examples are informative. These are usually atypical, hard, ambiguous, mislabeled, or underrepresented examples.
In some cases, one can remove up to 90\% of uninformative examples without degrading test set performance.
Conversely, removing highly memorized examples decreases test accuracy substantially, indicating that memorization is sometimes needed for good generalization.
Furthermore, we find that some uninformative examples can become highly informative when some other examples are removed from the training set.
Our findings add to the increasing amount of research focused on revealing the function of example memorization in generalization~\citep{koh2017understanding, feldman2020does, feldman2020neural, paul2021deep, sorscher2022beyond, carlini2022privacy}.

\subsection*{Information-theoretic generalization bounds}
Information about the training set captured by a neural network is also helpful for studying generalization.
In their seminal work, \citet{xu2017information} introduce a generalization gap bound depending on the Shannon information $I(W; S)$ between the learned weights $W$ and the training set $S$.
This result confirms the intuition that a learner will generalize well if it does not memorize the training set.

Unfortunately, the bound is vacuous in realistic settings and gives only qualitative insights.
Many better bounds were introduced subsequently, but two problems remained: (a) the bounds were vacuous in practical settings, and (b) they were hard to estimate due to challenges in estimating mutual information between high-dimensional variables.
In \cref{ch:sample-info}, building on the conditional mutual information bounds of \citet{steinke2020reasoning} and the sample-wise bounds of \citet{bu2020tightening}, we derive expected generalization bounds for supervised learning algorithms based on information contained in predictions rather than in the output of the training algorithm.
These bounds improve over the existing information-theoretic bounds, apply to a wider range of algorithms, give meaningful results for deterministic algorithms, and are significantly easier to estimate.
Furthermore, some classical learning theory results, such as expected generalization gap bounds based on Vapnik–Chervonenkis (VC) dimension~\citep{vapnik1998statistical} and algorithmic stability~\citep{bousquet2002stability}, follow directly from our results.
More importantly, the bounds are nonvacuous in practical scenarios for deep learning.
In one case, for a neural network with 3M parameters that reaches 9\% test error with just 75 MNIST training examples, the estimated test error bound is 22\%.

An essential ingredient in recent improvements of information-theoretic generalization bounds (including in our work presented in \cref{ch:sample-info}) is the introduction of sample-wise information bounds by \citet{bu2020tightening} that depend on the average amount of information the learned hypothesis has about a single training example.
In particular, for a training set of $n$ examples $S=(Z_1,\ldots,Z_n)$, their bound depends on $1/n \sum_{i=1}^n I(W;Z_i)$ rather than $I(W;S)/n$.
However, these sample-wise bounds were derived only for the \emph{expected} generalization gap, where the expectation is taken over both the training set and stochasticity of the training algorithm.
While PAC-Bayes and information-theoretic bounds are intimately related, the same technique did not work for deriving sample-wise single-draw or PAC-Bayes generalization gap bounds.
In \cref{ch:limitations},we show that sample-wise bounds are generally possible only for the \emph{expected} generalization gap (the weakest form of generalization guarantee).
In other words, sample-wise single-draw and PAC-Bayes generalization bounds are impossible unless additional assumptions are made.
Surprisingly, we also find that single-draw and PAC-Bayes bounds with information captured in pairs of examples, $\frac{1}{n(n-1)}\sum_{i\neq j} I(W;Z_i,Z_j)$, are possible without additional assumptions.

\subsection*{The role of supervision complexity in generalization}
One drawback of information-theoretic generalization bounds is that they depend too much on training data and data distribution. While these bounds can give concrete generalization guarantees and help design better learning algorithms, they do not specify what data properties are desirable for good generalization.
Their strong data-dependent nature enables differentiating cases like learning with ground truth labels and learning with random labels but comes at the cost of reduced explanatory power.
There are cases when subtle differences in data result in significantly different generalization performances.
One can argue that even the tightest information-theoretic bounds do not explain why these subtle differences have such effects.

One such prominent case arises in knowledge distillation~\citep{Bucilla:2006, hinton2015distilling}, which is a popular method of transferring knowledge from a large ``teacher'' model to a more compact ``student'' model.
In the most basic form of knowledge distillation, the student is trained to fit the teacher's predicted \emph{label distribution} (also called \emph{soft labels}) for each training example.
It has been well-established that distilled students usually perform better than students trained on raw dataset labels~\citep{hinton2015distilling, furlanello2018born, stanton2021does, gou2021knowledge}.
Notably, the teacher itself is usually trained on the same inputs but with original hard labels.
From a purely information-theoretic perspective, this phenomenon is quite surprising because, by the data processing inequality, the teacher predictions do not add any new information that was not presented in the original training dataset.
Clearly, the distillation dataset must satisfy some desired property that enables better generalization.

Several works have attempted to uncover \emph{why} knowledge distillation can improve the student performance.
Some prominent observations are that (self-)distillation induces certain favorable optimization biases in the training objective~\citep{phuong19understand,ji2020knowledge}, 
lowers variance of the objective~\citep{menon2021statistical,Dao:2021,ren2022better}, 
increases regularization towards learning ``simpler'' functions~\citep{mobahi2020self}, 
transfers information from different data views~\citep{allen-zhu2023understanding}, and 
scales per-example gradients based on the teacher's confidence~\citep{furlanello2018born,tang2020understanding}.
Nevertheless, there are still no compelling answers to why knowledge distillation works, what the exact role of temperature scaling is, what effects the teacher-student capacity gap has, and what makes a good teacher ultimately.

In \cref{ch:sup-complexity}, we provide a new perspective on knowledge distillation through the lens of \emph{supervision complexity}.
To put it concisely, supervision complexity quantifies why certain targets 
(e.g., temperature-scaled teacher probabilities) may be ``easier'' for a student model to learn compared to others (e.g., raw one-hot labels), 
owing to better alignment with the student's \emph{neural tangent kernel} (\emph{NTK})~\citep{jacot2018ntk,lee2019wide}.
We derive a new generalization bound for distillation that highlights how student generalization is controlled by a balance of 
the \emph{teacher generalization}, the student's \emph{margin} with respect to the soft labels, and the supervision complexity of the soft labels.
We show that both temperature scaling and early stopping
reduce the supervision complexity at the expense of lowering the classification margin.
Based on our analysis, we advocate using a simple \emph{online distillation} algorithm, wherein the student receives progressively more complex soft labels corresponding to teacher predictions at various checkpoints during its training.
Online distillation improves significantly over standard distillation and is especially successful for students with weak inductive biases, for which the final teacher predictions are often as complex as dataset labels, particularly during the early stages of training.

\section{Notation and preliminaries}
Before proceeding to our contributions, we first introduce some basic notation, describe an abstract learning setting, and provide an overview of some information-theoretic quantities used in this dissertation.

\paragraph{Notation.}
We use capital letters ($X$, $Y$, $Z$, etc) for random variables, corresponding lowercase letters ($x$, $y$, $z$, etc) for their values, and calligraphic letters for their domains ($\XX$, $\YY$, $\ZZ$, etc).
We use $\E_{P_X}\sbr{f(X)} = \int f \text{d} P_X$ to denote expectations.
Whenever the distribution over which the expectation is taken is clear from the context, we simply write $\E_{X}\sbr{f(X)}$ or $\E\sbr{f(X)}$.
A random variable $X$ is called $\sigma$-subgaussian if $ \E \exp(t (X - \E X)) \le \exp(\sigma^2 t^2 / 2), \ \forall t \in \mathbb{R}$.
For example, a random variable that takes values in $[a,b]$ almost surely, is $(b-a)/2$-subgaussian.

Throughout this dissertation, $[n]$ denotes the set $\mathset{1,2,\ldots,n}$.
If $A = (a_1,\ldots,a_n)$ is a collection, then $A_{-i} \triangleq  (a_1,\ldots,a_{i-1},a_{i+1},\ldots,a_n)$ and $A_{1:k} \triangleq \mathset{a_1,\ldots,a_k}$.
For $J\in\mathset{0,1}^n$, $\bar{J} \triangleq (1-J_1,\ldots,1 - J_n)$ is the negation of $J$.
When there can be confusion about whether a variable refers to a collection of items or a single item, we use bold symbols to denote the former.
For a pair of integers $n \ge m \ge 0$, $\binom{n}{m} = \frac{n!}{m!(n-m)!}$ denotes the binomial coefficient.
If $y(x)\in\mathbb{R}^m$ and $x\in\mathbb{R}^n$, then the Jacobian $\frac{\partial y}{\partial x}$ is an $m \times n$ matrix.
The gradient  $\nabla_x y$ is an $n \times m$ matrix denoting the transpose of the Jacobian.
This convention in convenient when working with gradient-descent-like algorithms.
In particular, when $w \in \mathbb{R}^d$ is a column vector and $\mathcal{L}(w)$ is a scalar, $\nabla_w \mathcal{L}(w)$ is also a column vector.

\paragraph{Learning setup.}
In the subsequent chapters, we consider the following abstract learning setting or an instance of it.
We will consider a standard learning setting.
There is an unknown data distribution $P_Z$ on an input space $\ZZ$.
The learner observes a collection of $n$ i.i.d examples $S=(Z_1,\ldots,Z_n)$ sampled from $P_Z$ and outputs a hypothesis (possibly random) belonging to a hypothesis space $\mathcal{W}$.
We will treat the learning algorithm as a probability kernel $Q_{W|S}$, which given a training set $s$ outputs a hypothesis $W$ sampled from the distribution $Q_{W|S=s}$.
For deterministic algorithms, $Q_{W|S=s}$ is a point mass distribution.
Together with $P_S$, the algorithm $Q_{W|S}$ induces a joint probability distribution $P_{W,S}=P_S Q_{W|S}$ on $\mathcal{W} \times \mathcal{Z}^n$.

The performance of a hypothesis $w\in\mathcal{W}$ on an example $z\in\ZZ$ is measured with a loss function $\ell : \mathcal{W} \times \mathcal{Z} \rightarrow \bR$.
For a hypothesis $w\in\mathcal{W}$, the population risk $R(w)$ is defined as $\E_{Z'\sim P_Z}\sbr{\ell(w,Z')}$, while the empirical risk is defined as $r_S(w) = 1/n \sum_{i=1}^n \ell(w,Z_i)$.
Note that the empirical risk is a random variable depending on $S$.
We will often study the difference between population and empirical risks, $R(W) - r_S(W)$, which is called generalization gap or generalization error.
Note that generalization gap is a random variable depending on both training set $S$ and randomness of the training algorithm $Q$.

Often we will be interested in supervised learning problems, where $\ZZ = \XX \times \YY$ and $Z_i = (X_i, Y_i)$.
In such cases, we define $\bs{X}\triangleq(X_1,\ldots,X_n)$ and $\bs{Y} \triangleq (Y_1,\ldots,Y_n)$.

\paragraph{Information-theoretic concepts.} 
The entropy of a discrete random variable $X$ with probability mass function $p(x)$ is defined as $H(X) = -\sum_{x \in \XX} p(x) \log p(x)$.
Analogously, the differential entropy of a continuous random variable $X$ with probability density $p(x)$ is defined as $H(X) = -\int_\XX p(x) \log p(x) \text{d}x $.
Given two probability measures $P$ and $Q$ defined on the same measurable space, such that $P$ is absolutely continuous with respect to $Q$, the Kullback–Leibler (KL) divergence from $P$ to $Q$ is defined as $\KL{P}{Q} = \int  \log \frac{\text{d} P}{\text{d} Q} \text{d}P$, where $\frac{\text{d} P}{\text{d} Q}$ is the Radon-Nikodym derivative of $P$ with respect to $Q$.
When $X$ and $Y$ are random variables defined on the same probability space, we sometimes shorten $\KL{P_X}{P_Y}$ to $\KL{X}{Y}$.
The Shannon mutual information between random variables $X$ and $Y$ is $I(X; Y) = \KL{P_{X,Y}}{P_X \otimes P_Y}$.

The conditional variants of entropy, differential entropy, KL divergence, and Shannon mutual information entail an expectation over the random variable in the condition.
For example, $I(X; Y \mid Z) \allowbreak = \int_\ZZ \KL{P_{X,Y | Z}}{P_{X|Z}\otimes P_{Y|Z}} \text{d} P_Z$.
The disintegrated variants of these quantities are denoted with a superscript indicating the condition.
For example, $I^Z(X;Y) \triangleq \KL{P_{X,Y | Z}}{P_{X|Z}\otimes P_{Y|Z}}$ denotes the disintegrated mutual information~\citep{negrea2019information}.
Note that the disintegrated mutual information is a random variable depending on $Z$.
In this dissertation, all information-theoretic quantities are measured in nats, unless specified otherwise.
Please refer to \citep{cover} for more in-detail description of the aforementioned concepts.

%% file: label-noise/main.tex
\input{label-noise/sections/intro.tex}

\input{label-noise/sections/label-noise-info.tex}

\input{label-noise/sections/method.tex}

\input{label-noise/sections/experiments.tex}

\input{label-noise/sections/related.tex}

\section{Conclusion}
Several theoretical works have highlighted the importance of the information about the training data that is memorized in the weights.
We distinguished two components of it and demonstrated that the conditional mutual information of weights and labels given inputs is closely related to memorization of labels and generalization performance.
By bounding this quantity in terms of information in gradients, we were able to derive the first practical schemes for controlling label information in the weights and demonstrated that this outperforms approaches for learning with noisy labels.

%% file: label-noise/sections/intro.tex
\section{Introduction}
Despite having millions of parameters, modern neural networks generalize surprisingly well.
However, their training is particularly susceptible to noisy labels, as shown by \citet{zhang2016understanding} in their analysis of generalization error.
In the presence of noisy or incorrect labels, networks start to memorize the training labels, which degrades the generalization performance~\citep{chen2019understanding}.
At the extreme, standard architectures have the capacity to achieve 100\% classification accuracy on training data, even when labels are assigned at random~\citep{zhang2016understanding}.
Furthermore, standard explicit or implicit regularization techniques such as dropout, weight decay or data augmentation do not directly address nor completely prevent label memorization~\citep{zhang2016understanding, arpit2017closer}.

Poor generalization due to label memorization is a significant problem because many large, real-world datasets are imperfectly labeled.
Label noise may be introduced when building datasets from unreliable sources of information or using crowd-sourcing resources like Amazon Mechanical Turk.
A practical solution to the memorization problem is likely to be algorithmic as sanitizing labels in large datasets is costly and time consuming.
Existing approaches for addressing the problem of label noise and generalization performance include deriving robust loss functions~\citep{natarajan2013learning, mae, gce, dmi}, loss correction techniques~\citep{Sukhbaatar2014TrainingCN,xiao2015learning,goldberger2016training, patrini2017making}, re-weighting samples~\citep{jiang2018mentornet, ren2018learning}, detecting incorrect samples and relabeling them~\citep{Reed2014TrainingDN, tanaka2018joint, ma2018dimensionality}, and employing two networks that select training examples for each other~\citep{han2018co,yu2019does}.

We propose an information-theoretic approach that directly addresses the root of the problem. If a classifier is able to correctly predict a training label that is actually random, it must have somehow stored information about this label in the parameters of the model. 
To quantify this information, \citet{achille2018emergence} consider weights as a random variable, $W$, that depends on stochasticity in training data and parameter initialization. 
The entire training dataset is considered a random variable consisting of a vector of inputs, $\bs{X}$, and a vector of labels for each input, $\bs{Y}$.
The amount of label memorization is then given by the Shannon mutual information between weights and labels conditioned on inputs, $I(W ; \bs{Y} \mid \bs{X})$. \citet{achille2018emergence} show that this term appears in a decomposition of the commonly used expected cross-entropy loss, along with three other individually meaningful terms. Surprisingly, cross-entropy rewards large values of $I(W ; \bs{Y} \mid \bs{X})$, which may promote memorization if labels contain information beyond what can be inferred from $\bs{X}$.
Such a result highlights that in addition to the network's representational capabilities, the loss function -- or more generally, the learning algorithm -- plays an important role in memorization. To this end, we wish to study the utility of limiting $I(W ; \bs{Y} \mid \bs{X})$, and how it can be used to modify training algorithms to reduce memorization.

\begin{figure}[!t]
    \centering
    \begin{subfigure}{0.48\textwidth}
    \includegraphics[width=\textwidth]{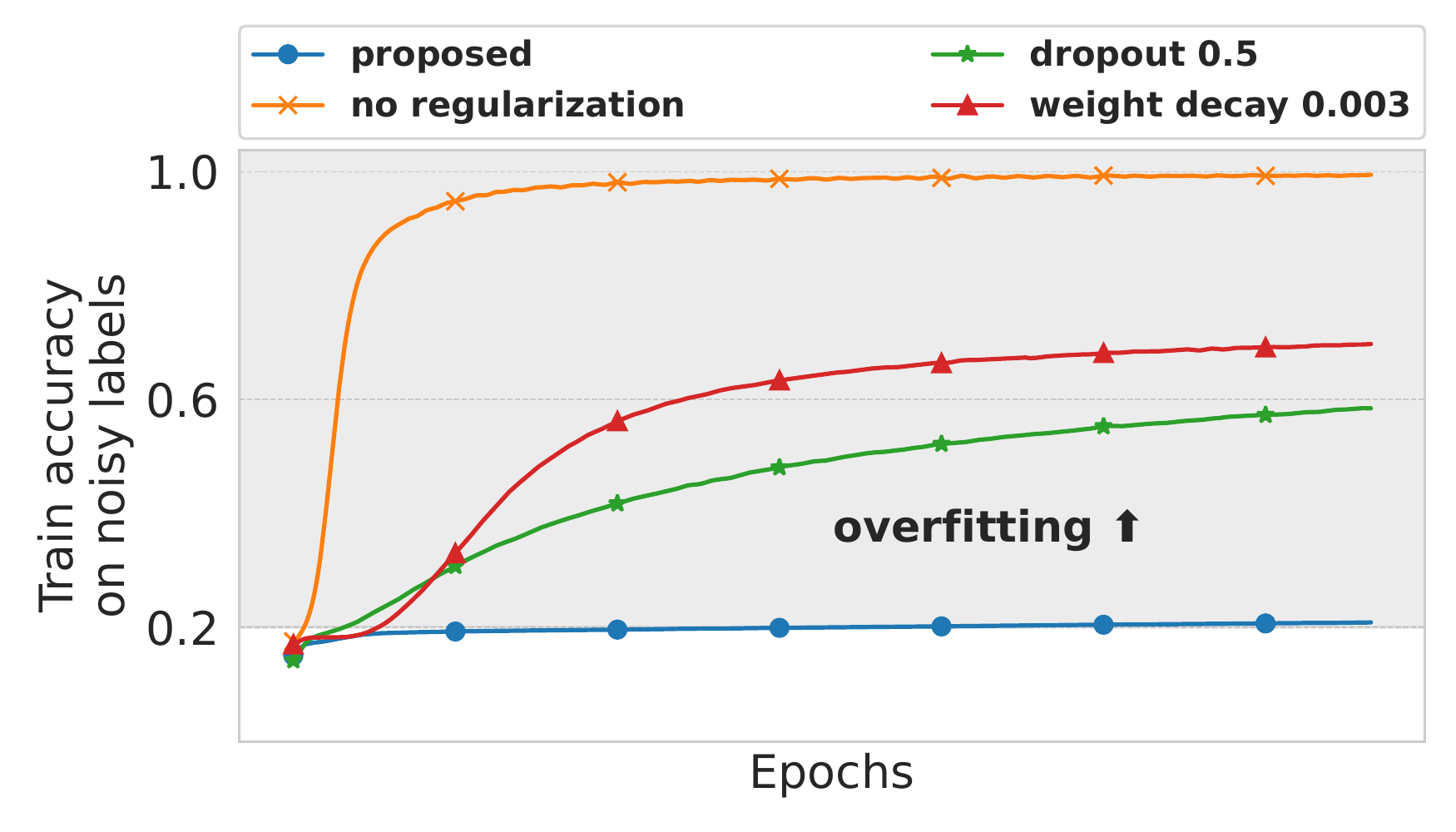}
    \end{subfigure}%
    \begin{subfigure}{0.48\textwidth}
    \includegraphics[width=\textwidth]{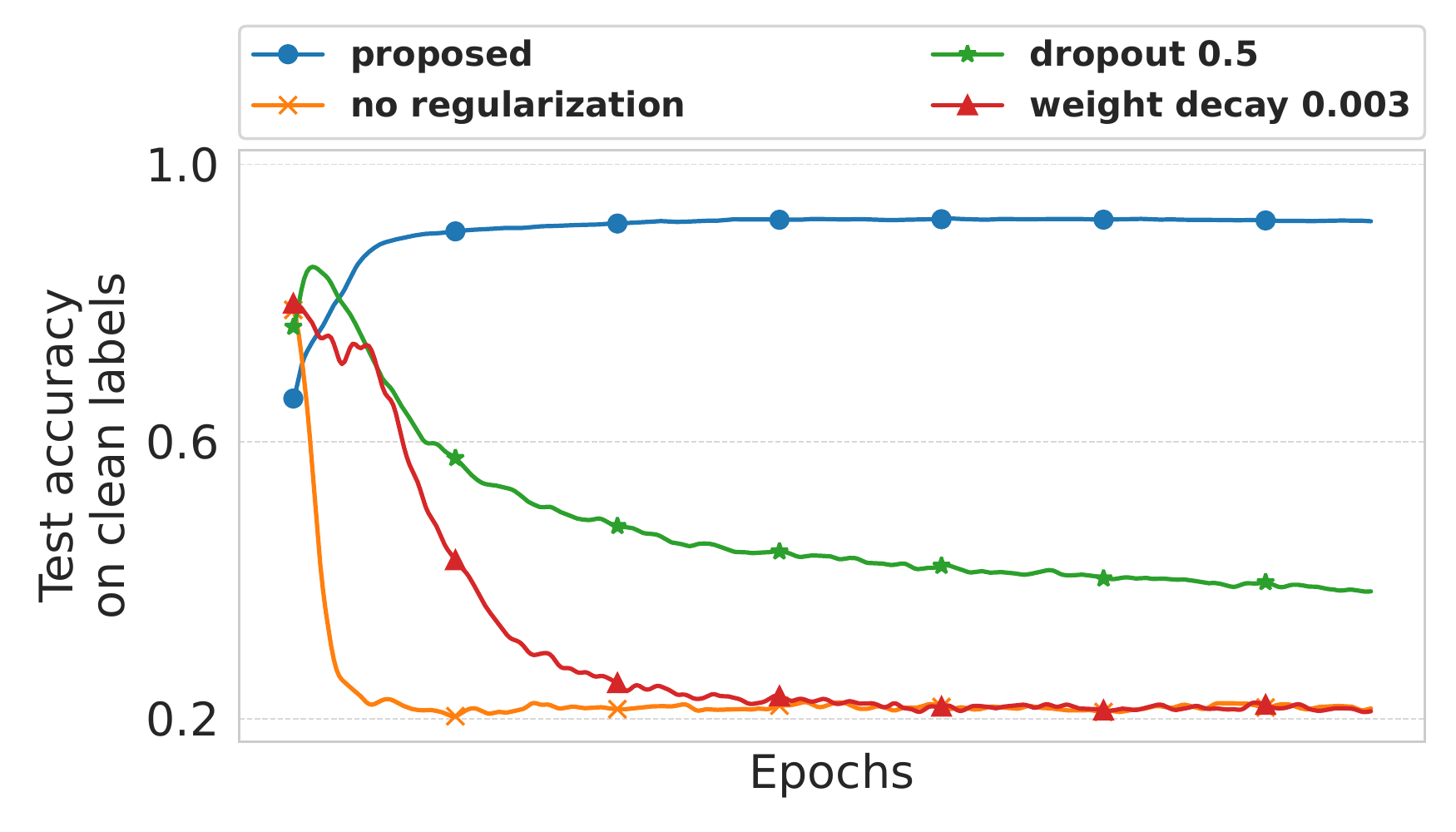}
    \end{subfigure}
    \caption{Neural networks tend to memorize labels when trained with noisy labels (80\% noise in this case), even when dropout or weight decay are applied.
    Our training approach limits label noise information in neural network weights, avoiding memorization of labels and improving generalization. See \cref{subsec:fano} for more details.}
    \label{fig:preventing_memorization}
\end{figure}

Our main contributions towards this goal are as follows: 1) We show that low values of $I(W ; \bs{Y} \mid \bs{X})$ correspond to reduction in memorization of label noise, and lead to better generalization gap bounds. 2) We propose training methods that control memorization by regularizing label noise information in weights.
When the training algorithm is a variant of stochastic gradient descent, one can achieve this by controlling label noise information in gradients. A promising way of doing this is through an additional network that tries to predict the classifier gradients without using label information. We experiment with two training procedures that incorporate gradient prediction in different ways: one which uses the auxiliary network to penalize the classifier, and another which uses predicted gradients to train it. In both approaches, we employ a regularization that penalizes the L2 norm of predicted gradients to control their capacity.
The latter approach can be viewed as a search over training algorithms, as it implicitly looks for a loss function that balances training performance with label memorization.
3) Finally, we show that the auxiliary network can be used to detect incorrect or misleading labels.
To illustrate the effectiveness of the proposed approaches, we apply them on corrupted versions of MNIST, CIFAR-10, CIFAR-100 with various label noise models, and on the Clothing1M dataset, which already contains noisy labels. We show that methods based on gradient prediction yield drastic improvements over standard training algorithms (like cross-entropy loss), and outperform  competitive approaches designed for learning with noisy labels.

%% file: label-noise/sections/label-noise-info.tex
\section{Label noise information in weights}
We begin by formally introducing a measure of label noise information in weights, and discuss its connections to memorization and generalization.
Consider a labeled training set $S = (Z_1, \ldots, Z_n)$ consisting of $n$ i.i.d. examples from $P_Z$, with $Z_i=(X_i,Y_i)$.
Let $\bs{X}\triangleq(X_1,\ldots,X_n)$ and $\bs{Y}\triangleq (Y_1,\ldots,Y_n)$.
Consider a neural network $f_w(y \mid x)$ with parameters $w$ that models the conditional distribution of labels.
In practice such neural networks are trained by minimizing the empirical negative log-likelihood loss:
\begin{equation}
    r_S(w) = -\frac{1}{n}\sum_{i=1}^n \log f_w(Y_i \mid X_i).
\end{equation}
For a given training algorithm $Q_{W|S}$, the expected value of this empirical risk can be decomposed as follows~\citep{achille2018emergence}:
\begin{align}
\E_{P_S}\E_{W\sim Q_{W|S}}\sbr{r_S(W)} &= H(\bs{Y} \mid \bs{X}) + \mathbb{E}_{P_{\bs{X},W}}\sbr{ \KL{p(\bs{Y} \mid \bs{X})}{f_W(\bs{Y} \mid \bs{X})}} - I(W ; \bs{Y} \mid \bs{X}).
\label{eq:ce-decomp}
\end{align}
The first term measures inherent uncertainty of labels, which is independent of the model.
The second term measures how well the learned model $f_W(y \mid x)$ approximates the ground truth conditional distribution $p(y \mid x)$ on training inputs.
The third term is the label noise information, and enters the equation with a negative sign.

The problem of minimizing this expected cross-entropy is equivalent to selecting an appropriate training algorithm.
If the labels contain information beyond what can be inferred from inputs (i.e., $H(\bs{Y} \mid \bs{X}) > 0$), such an algorithm may do well by memorizing the labels through the third term of \cref{eq:ce-decomp}.
Indeed, minimizing the empirical cross-entropy loss $Q^\mathrm{ERM}_{W|S} = \delta(w^*(S))$, where $w^*(S) \in \argmin_{w} r_S(w)$, does exactly that~\citep{zhang2016understanding}.

\subsection{Decreasing label noise information reduces memorization}\label{subsec:fano}
To demonstrate that $I(W ; \bs{Y} \mid \bs{X})$ is directly linked to memorization, we prove that any algorithm with small $I(W ; \bs{Y} \mid \bs{X})$ overfits less to label noise in the training set.
\begin{theorem}
Consider a dataset $S=(Z_1,\ldots,Z_n)$ of $n$ i.i.d. examples, with $Z_i=(X_i, Y_i)$. Assume that the domain of labels, $\YY$, be a finite set of at least two elements.
Let $Q_{W|S}$ be possibly stochastic learning algorithm producing weights for a classifier.
Let $\widehat{Y}_i$ denote the prediction of the classifier on the $i$-th example and let $E_i \triangleq \ind{\widehat{Y}_i \neq Y_i}$ be a random variable corresponding to predicting $Y_i$ incorrectly.
Then, the following inequality holds:
\begin{equation}
\E\sbr{\sum_{i=1}^n E_i} \ge \frac{H(\bs{Y} \mid \bs{X}) - I(W ; \bs{Y} \mid \bs{X}) - \sum_{i=1}^n H(E_i)}{ \log \left(\lvert \YY \rvert - 1\right)}.
\end{equation}
\label{thm:fano}
\end{theorem}
This result establishes a lower bound on the expected number of prediction errors on the training set, which increases as $I(W; \bs{Y} \mid \bs{X})$ decreases.
For example, consider a corrupted version of the MNIST dataset where each label is changed with probability $0.8$ to a uniformly random incorrect label.
By the above bound, every algorithm for which $I(W; \bs{Y} \mid \bs{X}) = 0$ will make at least $80\%$ prediction errors on the training set in expectation.
In contrast, if the weights retain $1$ bit of label noise information per example, the classifier will make at least 40.5\% errors in expectation.
Below we discuss the dependence of error probability on $I(W; \bs{Y} \mid \bs{X})$.

\begin{remark}
Let $k = |\YY|$, $r = \frac{1}{n}\mathbb{E}\sbr{\sum_{i=1}^n E_i}$ denote the expected training error rate, and $h(r) = -r \log r - (1-r)\log(1-r)$ be the binary entropy function.
Then we can simplify the results of \cref{thm:fano} as follows:
\begin{align}
r &\ge \frac{H(Y_1 \mid X_1) - I(W ; \bs{Y} \mid \bs{X}) / n - \frac{1}{n}\sum_{i=1}^n h\rbr{P(E_i=1)}}{\log(k - 1)}\\
&\ge \frac{H(Y_1 \mid X_1) - I(W ; \bs{Y} \mid \bs{X}) / n - h\rbr{\frac{1}{n}\sum_{i=1}^n P(E_i=1)}}{\log(k - 1)}&&\text{(by Jensen's inequality)}\\
&= \frac{H(Y_1 \mid X_1) - I(W ; \bs{Y} \mid \bs{X}) / n - h(r)}{\log(k - 1)}.\numberthis\label{eq:fano_r}
\end{align}
Solving this inequality for $r$ is challenging. 
One can simplify the right hand side further by bounding $H(E_1) \le 1$ (assuming that entropies are measured in bits).
However, this will loosen the bound.
Alternatively, we can find the smallest $r_0$ for which \cref{eq:fano_r} holds and claim that $r \ge r_0$.
\end{remark}

\begin{remark}
If $|\YY| = 2$, then $\log(|\YY| - 1) = 0$, putting which in \cref{eq:original-fano} of \cref{app:proofs}:
\begin{equation}
h(r) \ge H(Y_1 \mid X_1) - I(W ; \bs{Y} \mid \bs{X}) / n.
\end{equation}

\begin{figure}[t]
    \centering
    \includegraphics[width=0.5\textwidth]{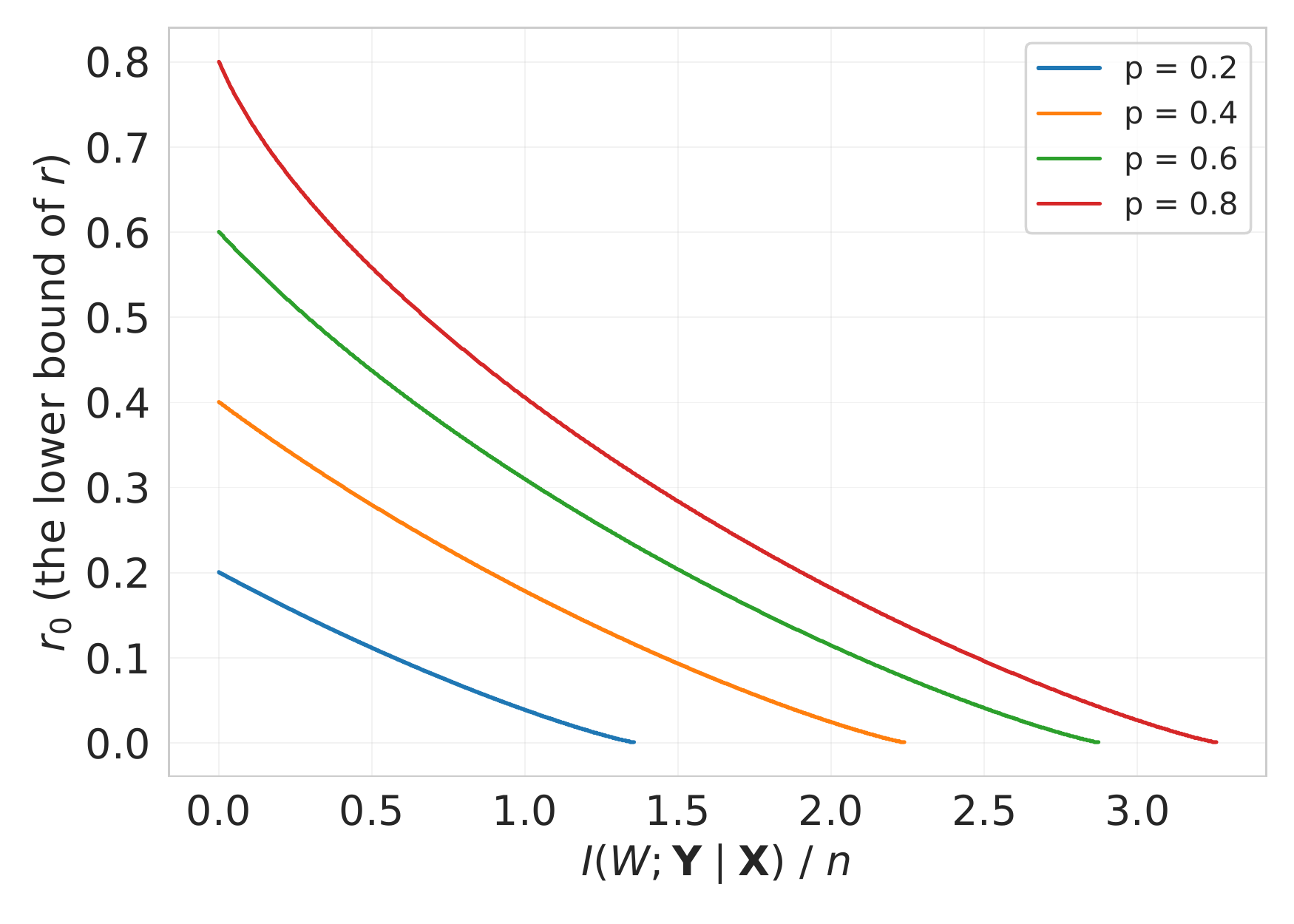}
    \caption{The lower bound $r_0$ on the rate of training errors $r$ \cref{thm:fano} establishes for varying values of $I(W; \bs{Y} \mid \bs{X})$, in the case when label noise is uniform and probability of a label being incorrect is $p$.}
    \label{fig:fano}
\end{figure}
\end{remark}

\begin{remark}
When we have uniform label noise where a label is incorrect with probability $p$ ($0 \le p < \frac{k-1}{k}$) and $I(W ; \bs{Y} \mid \bs{X}) = 0$, the bound of \cref{eq:fano_r} is tight, i.e., implies that $r \ge p$. 
To see this, we note that $H(Y_1 \mid X_1) = h(p) + p \log (k-1)$, putting which in \cref{eq:fano_r} gives us:
\begin{equation}
    r \ge \frac{h(p) + p \log(k-1) - h(r)}{\log(k-1)} = p + \frac{h(p) - h(r)}{\log(k-1)}.
    \label{eq:fano_uniform}
\end{equation}
Therefore, when $r=p$, the inequality holds, implying that $r_0 \le p$. To show that $r_0 = p$, we need to show that for any $0 \le r < p$, the \cref{eq:fano_uniform} does not hold.
Let $r \in [0, p)$ and assume that \cref{eq:fano_uniform} holds. Then
\begin{align}
    r &\ge p + \frac{h(p) - h(r)}{\log(k-1)}\\
      &\ge p + \frac{h(p) - \left(h(p) + (r-p) h'(p)\right)}{\log(k-1)}\\
      &\ge p + \frac{-(r-p) \log(k-1)}{\log(k-1)} = 2p - r,\numberthis\label{eq:fano-zero-mi-case}
\end{align}
where the second line above follows from concavity of $h(x)$, and the third line follows from the fact that $h'(p) > -\log(k-1)$ when $0\le p < (k-1)/k$.
\cref{eq:fano-zero-mi-case} directly contradicts with $r<p$. Therefore, \cref{eq:fano_uniform} cannot hold for any $r < p$.
\end{remark}

When $I(W ; \bs{Y} \mid \bs{X}) > 0$, we can find the smallest $r_0$ by a numerical method.
\cref{fig:fano} plots $r_0$ vs $I(W ; \bs{Y} \mid \bs{X})$ when the label noise is uniform.
When the label noise is not uniform, the bound of \cref{eq:fano_r} becomes loose as Fano's inequality becomes loose.

\cref{thm:fano} provides theoretical guarantees that memorization of noisy labels is prevented when $I(W; \bs{Y} \mid \bs{X})$ is small, in contrast to standard regularization techniques -- such as dropout, weight decay, and data augmentation -- which only slow it down~\citep{zhang2016understanding, arpit2017closer}.
To demonstrate this empirically, we compare an algorithm that controls $I(W ; \bs{Y} \mid \bs{X})$ (presented in \cref{sec:label-noise-method}) against these regularization techniques on the aforementioned corrupted MNIST setup. We see in \cref{fig:preventing_memorization} that explicitly preventing memorization of label noise information leads to optimal training performance (20\% training accuracy) and good generalization on a non-corrupted validation set. Other approaches quickly exceed 20\% training accuracy by incorporating label noise information, and generalize poorly as a consequence.
The classifier here is a fully connected neural network with 4 hidden layers each having 512 ReLU units.
The rates of dropout and weight decay were selected according to the performance on a validation set.

\subsection{Decreasing label noise information improves generalization}
The information that weights contain about a training dataset $S$ has previously been linked to generalization~\citep{xu2017information}.
While we will explore connections between information in weights and generalization in detail in \cref{ch:sample-info,ch:limitations}, it is instructive to consider here the seminal result of \citet{xu2017information}.
They prove that when $\ell(w, Z')$, with $Z'\sim P_Z$, is $\sigma$-subgaussian, we have that
\begin{equation}
    \underbrace{\E_{P_S}\E_{W\sim Q_{W|S}} \abs{\E\sbr{R(W) - r_S(W)}}}_{\text{expected generalization gap}} \le \sqrt{\frac{2\sigma^2}{n} I(W; S)}.
    \label{eq:xu-raginsky}
\end{equation}

For good test performance, learning algorithms need to have both a small generalization gap, and good training performance.
The latter may require retaining more information about the training set, meaning there is a natural conflict between increasing training performance and decreasing the generalization gap bound of \cref{eq:xu-raginsky}.
Furthermore, information in weights can be decomposed as follows: $I(W ; S) = I(W ; \bs{X}) + I(W ; \bs{Y} \mid \bs{X})$.
We claim that one needs to prioritize reducing $I(W ; \bs{Y} \mid \bs{X})$ over $I(W ; \bs{X})$ for the following reason.
When noise is present in the training labels, fitting this noise implies a non-zero value of $I(W ; \bs{Y} \mid \bs{X})$, which grows linearly with the number of samples $n$.
In such cases, the generalization gap bound of \cref{eq:xu-raginsky} does not improve as $n$ increases.
To get meaningful generalization bounds via \cref{eq:xu-raginsky} one needs to limit $I(W ; \bs{Y} \mid \bs{X})$.
We hypothesize that for efficient learning algorithms, this condition might be also sufficient.

%% file: label-noise/sections/method.tex
\section{Methods limiting label information}\label{sec:label-noise-method}
We now consider how to design training algorithms that control $I(W ; \bs{Y} \mid \bs{X})$.
We assume $f_w(y \mid x) = \mathrm{Multinoulli}(y; \mathrm{softmax}(a))$, where $a$ is the output of a neural network $h_w(x)$.
We consider the case when $h_w(x)$ is trained with a variant of stochastic gradient descent for $T$ iterations.
The inputs and labels of a mini-batch at iteration $t$ are denoted by $X_t$ and $Y_t$ respectively, and are selected using a deterministic procedure (such as cycling through the dataset or using pseudo-randomness).
To keep the notation simple, we treat ($X_t$, $Y_t$) as a single example.
The derivations below can be straightforwardly extended to the case when $X_t$ and $Y_t$ are a mini-batch of inputs and labels.

Let $W_0$ denote the random weights at initialization, and $W_t$ denote the weights after iteration $t$.
Let $\ell(w, (x,y))$ be some classification loss function (e.g, the cross-entropy loss).
Let $G^\ell_t \triangleq \nabla_{w} \ell(W_{t-1}, (X_t, Y_t))$ be the gradient at iteration $t$.
Let the update rule be $W_t = \Psi(W_0, G_{1:t})$, where $G_t$ denote the gradients used to update the weights, possibly different from $G^\ell_t$.
The final weights $W_T$ are the output of the algorithm (i.e., $W=W_T$).
In the simplest case, $W_T = W_0 -\sum_{t=1}^T \eta_t G_t$, with some learning rate schedule $\eta_t$.

To limit $I(W; \bs{Y} \mid \bs{X})$ the following sections will discuss two approximations which relax the computational difficulty, while still providing meaningful bounds: 1) first, we show that the information in weights can be replaced by information in the gradients; 2) we introduce a variational bound on the information in gradients.
The bound employs an auxiliary network that predicts gradients of the original loss without label information.
We then explore two ways of incorporating predicted gradients: (a) using them in a regularization term for gradients of the original loss, and (b) using them to train the classifier. 

\subsection{Penalizing information in gradients}

Looking at \cref{eq:ce-decomp} it is tempting to add $I(W ; \bs{Y} \mid \bs{X})$ as a regularization to $\E_{P_{S,W}}\sbr{r_S(W)}$ and minimize over all training algorithms:
\begin{equation}
\min_{Q_{W|S}} \E_{P_{S,W}}\sbr{r_S(W)} + I(W ; \bs{Y} \mid \bs{X}).
\label{eq:opt-over-algos}
\end{equation}
This will become equivalent to minimizing $\E_{P_{\bs{X},W}} \KL{p(\bs{Y} \mid \bs{X})}{f_W(\bs{Y} \mid \bs{X})}$.
Unfortunately, the optimization problem \cref{eq:opt-over-algos} is hard to solve for two major reasons.
First, the optimization is over training algorithms (rather than over the weights of a classifier, as in the standard supervised learning setting).
Second, the penalty $I(W ; \bs{Y} \mid \bs{X})$ is hard to compute.

To simplify the problem of \cref{eq:opt-over-algos}, we relate information in weights to information in gradients as follows:
\begin{align}
I(W; \bs{Y} \mid \bs{X}) &\le I(G_{1:T} ; \bs{Y} \mid \bs{X})=\sum_{t=1}^T I(G_t ; \bs{Y} \mid \bs{X}, G_{<t}),\label{eq:chain-rule}
\end{align}
where $G_{<t}$ is a shorthand for the set $\mathset{G_1,\ldots,G_{t-1}}$.
Hereafter, we focus on constraining $I(G_t ; \bs{Y} \mid \bs{X}, G_{<t})$ at each iteration.
Our task becomes choosing a loss function to optimize such that $I(G_t ; \bs{Y} \mid \bs{X}, G_{<t})$ is small and $f_{W_t}(y \mid x)$ is a good classifier.
One key observation is that if our task is to minimize label noise information in gradients it may be helpful to consider gradients with respect to the last layer only and compute the remaining gradients using back-propagation.
As these steps of back-propagation do not use labels, by data processing inequality, subsequent gradients would have at most as much label information as the last layer gradient.

To simplify information-theoretic quantities, we add a small independent Gaussian noise to the gradients of the original loss: $\tilde{G}^\ell_t \triangleq G^\ell_t + \xi_t$, where $\xi_t \sim \mathcal{N}(0, \sigma_\xi^2 I)$ and $\sigma_\xi$ is small enough to have no significant effect on training (e.g., less than $10^{-9}$).
With this convention, we formulate the following regularized objective function:
\begin{equation}
\min_w \ell(W_{t-1}, (X_t, Y_t)) + \lambda I(\tilde{G}^\ell_t; \bs{Y} \mid \bs{X}, G_{<t}),
\label{eq:penalize-formulation}
\end{equation}
where $\lambda > 0$ is a regularization coefficient.
The term $I(\tilde{G}^\ell_t ; \bs{Y} \mid \bs{X}, G_{<t})$ is a function $\Phi(W_{t-1}; X_t)$ of $W_{t-1}$ and $X_t$.
Computing this function would allow the optimization of \cref{eq:penalize-formulation} through gradient descent: $G_t = G^\ell_t + \xi_t + \nabla_w \Phi(W_{t-1}; X_t)$.
Importantly, label noise information is equal in both $G_t$ and $\tilde{G}^\ell_t$, as the gradient from the regularization is constant given $\bs{X}$ and $G_{<t}$:
\begin{align}
I(G_t ; \bs{Y} \mid \bs{X}, G_{<t}) &= I(G^\ell_t + \xi_t + \nabla_w \Phi(W_{t-1}; X_t) ; \bs{Y} \mid \bs{X}, G_{<t})\\
&= I(G^\ell_t + \xi_t ; \bs{Y} \mid \bs{X}, G_{<t})\\
&= I(\tilde{G}^\ell_t ; \bs{Y} \mid \bs{X}, G_{<t}).
\end{align}
Therefore, by minimizing $I(\tilde{G}^\ell_t ; \bs{Y} \mid \bs{X}, G_{<t})$ in \cref{eq:penalize-formulation} we minimize $I(G_t ; \bs{Y} \mid \bs{X}, G_{<t})$, which is used to upper bound $I(W ; \bs{Y} \mid \bs{X})$ in \cref{eq:chain-rule}.
We rewrite this regularization in terms of entropy and discard the constant term, $H(\xi_t)$:
\begin{align*}
I(\tilde{G}^\ell_t ; \bs{Y} \mid \bs{X}, G_{<t}) &= H(\tilde{G}^\ell_t \mid \bs{X}, G_{<t}) - H(\tilde{G}^\ell_t \mid \bs{X}, \bs{Y}, G_{<t})\\
&= H(\tilde{G}^\ell_t \mid \bs{X}, G_{<t}) - H(\xi_t)\numberthis\label{eq:entropy-of-grad}.
\end{align*}

\subsection{Variational bounds on gradient information}
The first term in \cref{eq:entropy-of-grad} is still challenging to compute, as we typically only have one sample from the unknown distribution $p(y_t \mid x_t)$.
Nevertheless, we can upper bound it with the cross-entropy $H_{p,q} = -\E_{\tilde{G}^\ell_t}\sbr{ \log q_\phi(\tilde{G}^\ell_t \mid \bs{X}, G_{<t})}$, where $q_\phi(\cdot \mid \bs{x}, g_{<t})$ is a variational approximation for $p(\tilde{g}^\ell_t \mid \bs{x}, g_{<t})$:
\begin{align}
H(\tilde{G}^\ell_t \mid \bs{X}, G_{<t}) \le -\E_{\tilde{G}^\ell_t}\sbr{\log q_\phi(\tilde{G}^\ell_t \mid \bs{X}, G_{<t})}.
\end{align}
This bound is correct when $\phi$ is a constant or a random variable that depends only on $\bs{X}$.
With this upper bound, \cref{eq:penalize-formulation} reduces to:
\begin{equation}
    \min_{w, \phi} \ell(W_{t-1}, (X_t, Y_t)) - \lambda \E_{\tilde{G}^\ell_t}\sbr{ \log q_\phi(\tilde{G}^\ell_t \mid \bs{X}, G_{<t})}.
    \label{eq:penalize-cross-entropy}
\end{equation}
This formalization introduces a soft constraint on the classifier by attempting to make its gradients predictable without labels $\bs{Y}$, effectively reducing $I(G_t ; \bs{Y} \mid \bs{X}, G_{<t})$.

Assuming $\widehat{y} = \mathrm{softmax}(h_w(x))$ denotes the predicted class probabilities of the classifier and $\ell$ is the cross-entropy loss, the gradient with respect to logits $a=h_w(x)$ is $\widehat{y} - y$ (assuming $y$ has a one-hot encoding).
We have that $I(\widehat{Y_t} - Y_t + \xi_t ; \bs{Y} \mid \bs{X}, G_{<t})$ = $I(Y_t + \xi_t ; \bs{Y} \mid \bs{X}, G_{<t})$.
Therefore, if we only consider gradients with respect to logits in the penalty of \cref{eq:penalize-cross-entropy}, the resulting penalty would not serve as a meaningful regularizer, as it has no dependent on $W_{t-1}$.
Instead, we descend an additional level to consider gradients of the final layer parameters.
When the final layer of $h_w(x)$ is fully connected with inputs $z$ and weights $U$ (i.e., $a = U z$), the gradients with respect to its parameters is equal to $(\widehat{y} - y)z^T$.
If we only consider this gradient in \cref{eq:penalize-cross-entropy}, we have that $I(\tilde{G}^\ell_t ; \bs{Y} \mid \bs{X}, G_{<t}) = I((\widehat{Y}_t - Y_t)Z_t^T + \xi_t ; \bs{Y} \mid \bs{X}, G_{<t}) = I(Y_t Z_t^T + \xi_t ; \bs{Y} \mid \bs{X}, G_{<t})$.
There is now dependence on $W_{t-1}$ through $Z_t$.
We choose to parametrize $q_\phi(\cdot \mid \bs{x}, g_{<t})$ as a Gaussian distribution with mean $\mu_t = (\mathrm{softmax}(a_t) - \mathrm{softmax}(r_\phi(x_t))) z_t^T$ and fixed covariance $\sigma_q I$, where $r_\phi(\cdot)$ is another neural network.
Under this assumption, $H_{p,q}$ becomes proportional to:
\begin{align}
\E\sbr{\left\lVert (\widehat{Y}_t - Y_t) Z_t^T + \xi_t - \mu_t\right\rVert^2_2} &= \E\sbr{\xi_t^2} + \E\sbr{\left\lVert Z_t \right\rVert_2^2 \left\lVert Y_t - \mathrm{softmax}(r_\phi(X_t))\right\rVert^2_2}.
\end{align}
Ignoring constants and approximating the expectation above with one Monte Carlo sample computed using the label $Y_t$, the objective of \cref{eq:penalize-cross-entropy} becomes:
\begin{equation}
\min_w \ell(W_{t-1}, (X_t, Y_t)) + \lambda \rbr{\left\lVert Z_t \right\rVert_2^2 \left\lVert Y_t - \mathrm{softmax}(r_\phi(X_t))\right\rVert^2_2}.
\label{eq:penalize-final}
\end{equation}

While this may work in principle, in practice the dependence on $w$ is only through the norm of $Z_t$, making it weak to have much effect on the overall objective.
We confirm this experimentally in \cref{sec:label-noise-experiments}.
To introduce more complex dependencies on $w$, one would need to model the gradients of deeper layers.

\subsection{Predicting gradients without label information}
An alternative approach is to use gradients predicted by $q_\phi(\cdot \mid \bs{X}, G_{<t})$ to update classifier weights, i.e., sample $G_t \sim q_\phi(\cdot \mid \bs{X}, G_{<t})$.
This is a much stricter condition, as it implies $I(G_t ; \bs{Y} \mid \bs{X}, G_{<t})=0$ (again assuming $\phi$ is a constant or a random variable that depends only on $\bs{X}$).
Note that minimizing $H_{p,q}$ makes the predicted gradient $G_t$ a good substitute for the cross-entropy gradients $\tilde{G}^\ell_t$.
Therefore, we write down the following objective function:
\begin{equation}
    \min_{w, \phi} \tilde{\ell}(W_{t-1}, \phi, X_t) - \lambda \E_{\tilde{G}^\ell_t}\sbr{ \log q_\phi(\tilde{G}^\ell_t \mid \bs{X}, G_{<t})},
    \label{eq:predict-grad}
\end{equation}
where $\tilde{\ell}(w, \phi, x)$  is a probabilistic function defined implicitly such that $\nabla_w \tilde{\ell}(w, \phi, x) \sim q_\phi(\cdot \mid x, w)$.
We found that this approach performs significantly better than the penalizing approach of \cref{eq:penalize-cross-entropy}.
To update $w$ only with predicted gradients, we disable the dependence of the second term of \cref{eq:predict-grad} on $w$ in the implementation.
Additionally, one can set $\lambda=1$ above as the first term depends only on $w$, while the second term depends only on $\phi$.

We choose to predict the gradients with respect to logits only and compute the remaining gradients using backpropagation.
We consider two distinct parameterizations for $q_\phi$ -- \textbf{Gaussian:} $q_\phi(\cdot \mid x, w) = \mathcal{N}\left(\mu, \sigma_q^2 I\right)$, and \textbf{Laplace:} $q_\phi(\cdot \mid x, w) = \prod_j \mathrm{Lap}\left(\mu_j, \sigma_q /\sqrt{2})\right)$,
with $\mu = \mathrm{softmax}(a) - \mathrm{softmax}(r_\phi(x))$ and $r_\phi(\cdot)$ being an auxiliary neural network as before.
Under these Gaussian and Laplace parameterizations, $H_{p,q}$ becomes proportional to  $\E\nbr{\mu_t - \tilde{G}^\ell_t}_2^2$ and $\E\nbr{\mu_t - \tilde{G}^\ell_t}_1$ respectively.
In the Gaussian case $\phi$ is updated with a mean square error loss (MSE) function, while in the Laplace case it is updated with a mean absolute error loss (MAE).
The former is expected to be faster at learning, but less robust to noise~\citep{mae}.

\subsection{Reducing overfitting in gradient prediction}
In both approaches of \cref{eq:penalize-cross-entropy,eq:predict-grad}, the classifier can still overfit if $q_\phi(\cdot \mid x, w)$ overfits.
There are multiple ways to prevent this.
One can choose $q_\phi$ to be parametrized with a small network, or pretrain and freeze some of its layers in an unsupervised fashion.
In this work, we choose to control the L2 norm of the mean of predicted gradients, $\lVert \mu \rVert^2_2$, while keeping the variance $\sigma_q^2$ fixed.
This can be viewed as limiting the capacity of gradients $G_t$.
\begin{proposition}
If $G_t = \mu_t + \epsilon_t$, where $\epsilon_t \sim \mathcal{N}(0, \sigma_q^2 I_d)$ is independent noise, and $\E\sbr{ \mu^T_t \mu_t } \le L^2$, then the following inequality holds:
\begin{align}
I(G_t ; \bs{Y} \mid \bs{X}, G_{<t}) \le \frac{d}{2}\log\left(1 + \frac{L^2}{d\sigma_q^2}\right).
\end{align}
\label{prop:gradient_capacity}
\end{proposition}
The same bound holds when $\epsilon_t$ is sampled from a product of $d$ univariate zero-mean Laplace distributions with variance $\sigma_q^2$, since the proof relies only on $\epsilon_t$ being zero-mean and having variance $\sigma^2_q$.
The final objective of our main method becomes:
\begin{equation}
    \min_{w, \phi} \tilde{\ell}(W_{t-1}, \phi, X_t) - \lambda \E_{\tilde{G}^\ell_t}\sbr{ \log q_\phi(\tilde{G}^\ell_t \mid \bs{X}, G_{<t})} + \beta \lVert \mu_t \rVert_2^2.
\end{equation}
As before, to update $w$ only with predicted gradients, we disable the dependence of the second and third terms above on $w$ in the implementation.
We name this final approach \textbf{LIMIT} -- \underline{\textbf{l}}imiting label \underline{\textbf{i}}nformation \underline{\textbf{m}}emorization \underline{\textbf{i}}n \underline{\textbf{t}}raining.
We denote the variants with Gaussian and Laplace distributions as LIMIT$_\mathcal{G}$ and LIMIT$_\mathcal{L}$ respectively.
The pseudocode of LIMIT is presented \cref{alg:training-loop}.
Note that in contrast to the previous approach of \cref{eq:penalize-formulation}, this follows the spirit of \cref{eq:opt-over-algos}, in the sense that the optimization over $\phi$ can be seen as optimizing over training algorithms; namely, learning a loss function implicitly through gradients.
With this interpretation, the gradient norm penalty can be viewed as a way to smooth the learned loss, which is a good inductive bias and facilitates learning.

\begin{algorithm}[t]
\small
\caption{LIMIT: limiting label information memorization in training.\\Our implementation is available at \url{https://github.com/hrayrhar/limit-label-memorization}.}
\begin{algorithmic}[1]
    \STATE {\bfseries Input:} Training dataset $S$.
    \STATE {\bfseries Input:} Gradient norm regularization coefficient $\beta$.
    \quad\COMMENT{$\lambda$ is set to $1$}
    \STATE Initialize the classifier $f_w(y \mid x)$ and gradient predictor $q_\phi(\cdot \mid \bs{x}, g_{<t})$.
    \FOR{$t=1..T$}
        \STATE Fetch the next batch $(X_t, Y_t)$ and compute the predicted logits $a_t$.
        \STATE Compute the cross-entropy gradient, $G^\ell_t \leftarrow \mathrm{softmax}(a_t) - Y_t$.
        \IF{sampling of gradients is enabled}
            \STATE $G_t \sim q_\phi(\cdot \mid \bs{X}, G_{<t})$.
        \ELSE
            \STATE $G_t \leftarrow \mu_t$ \quad\COMMENT{the mean of predicted gradient}
        \ENDIF
        \STATE Starting with $G_t$, backpropagate to compute the gradient with respect to $w$.
        \STATE Update $W_{t-1}$ to $W_t$.
        \STATE Update $\phi$ using the gradient of the following loss: $-\log q_\phi(\tilde{G}^\ell_t \mid \bs{X}, G_{<t}) + \beta \lVert \mu_t \rVert_2^2$.
    \ENDFOR
\end{algorithmic}
\label{alg:training-loop}
\end{algorithm}

%% file: label-noise/sections/experiments.tex
\section{Experiments}\label{sec:label-noise-experiments}
We set up experiments with noisy datasets to see how well the proposed methods perform for different types and amounts of label noise.
The simplest baselines in our comparison are  standard cross-entropy (CE) and mean absolute error (MAE) loss functions.
The next baseline is the forward correction approach (FW) proposed by~\citet{patrini2017making}, where the label noise transition matrix is estimated and used to correct the loss function.
Finally, we include the determinant mutual information (DMI) loss, which is the log-determinant of the confusion matrix between predicted and given labels~\citep{dmi}.
Both FW and DMI baselines require initialization with the best result of the CE baseline.
To avoid small experimental differences, we implement all baselines, closely following the original implementations of FW and DMI.

We train all baselines except DMI using the ADAM optimizer~\citep{kingma2014adam} with learning rate $\alpha = 10^{-3}$ and $\beta_1 = 0.9$.
As DMI is very sensitive to the learning rate, we tune it by choosing the best from the following grid of values $\mathset{10^{-3},10^{-4},10^{-5},10^{-6}}$.
The soft regularization approach of \cref{eq:penalize-final} has two hyperparameters: $\lambda$ and $\beta$.
We select $\lambda$ from $\mathset{0.001, 0.01, 0.03, 0.1}$ and $\beta$ from $\mathset{0.0, 0.01, 0.1, 1.0, 10.0}$.
The objective of LIMIT instances has two terms: $\lambda H_{p,q}$ and $\beta \lVert \mu_t \rVert_2^2$. 
Consequently, we need only one hyperparameter instead of two. We choose to set $\lambda = 1$ and select $\beta$ from $\mathset{0.0, 0.1, 0.3, 1.0, 3.0, 10.0, 30.0, 100.0}$.
When sampling is enabled, we select $\sigma_q$ from $\mathset{0.01, 0.03, 0.1, 0.3}$.

For all baselines, model selection is done by choosing the model with highest accuracy on a validation set that follows the noise model of the corresponding train set.
All scores are reported on a clean test set.
The implementation of the proposed method and the code for replicating the experiments is available at \url{https://github.com/hrayrhar/limit-label-memorization}.

\subsection{MNIST with uniform label noise}
\begin{table}[t]
    \small
    \centering
    \caption{Architecture of a convolutional neural network with 5 hidden layers.}
    \begin{tabular}{ll}
    \toprule
    Layer type & Parameters\\
    \midrule
    Conv & 32 filters, $4 \times 4$ kernels, stride 2, padding 1, batch normalization, ReLU\\
    Conv & 32 filters, $4 \times 4$ kernels, stride 2, padding 1, batch normalization, ReLU\\
    Conv & 64 filters, $3 \times 3$ kernels, stride 2, padding 0, batch normalization, ReLU\\
    Conv & 256 filters, $3 \times 3$ kernels, stride 1, padding 0, batch normalization, ReLU\\
    FC & 128 units, ReLU\\
    FC & 10 units, linear activation\\
    \bottomrule
    \end{tabular}
\label{tab:4-layer-cnn-arch}
\end{table}

To compare the variants of our approach discussed earlier and see which ones work well, we do experiments on the MNIST dataset with corrupted labels.
In this experiment, we use a simple uniform label noise model, where each label is set to an incorrect value uniformly at random with probability $p$.
In our experiments we try 4 values of $p$ -- 0\%, 50\%, 80\%, 89\%.
We split the 60K images of MNIST into training and validation sets, containing 48K and 12K samples respectively.
For each noise amount we try 3 different training set sizes -- $10^3$, $10^4$, and $4.8\cdot 10^4$.
All classifiers and auxiliary networks are 5-layer convolutional neural networks (CNN) described in \cref{tab:4-layer-cnn-arch}.
We train all models for 400 epochs and terminate the training early when the best validation accuracy is not improved in the previous 100 epochs.

For this experiment we include two additional baselines where additive noise (Gaussian or Laplace) is added to the gradients with respect to logits.
We denote these baselines with names ``CE + GN'' and ``CE + LN''.
The comparison with these two baselines demonstrates that the proposed method does more than simply reduce information in gradients via noise.
We also consider a variant of LIMIT where instead of sampling $G_t$ from $q$ we use the predicted mean $\mu_t$.

\input{label-noise/sections/tables/mnist-with-error-bars}

\begin{figure}[t]
    \centering
    \begin{subfigure}{0.49\textwidth}
    \includegraphics[width=\textwidth]{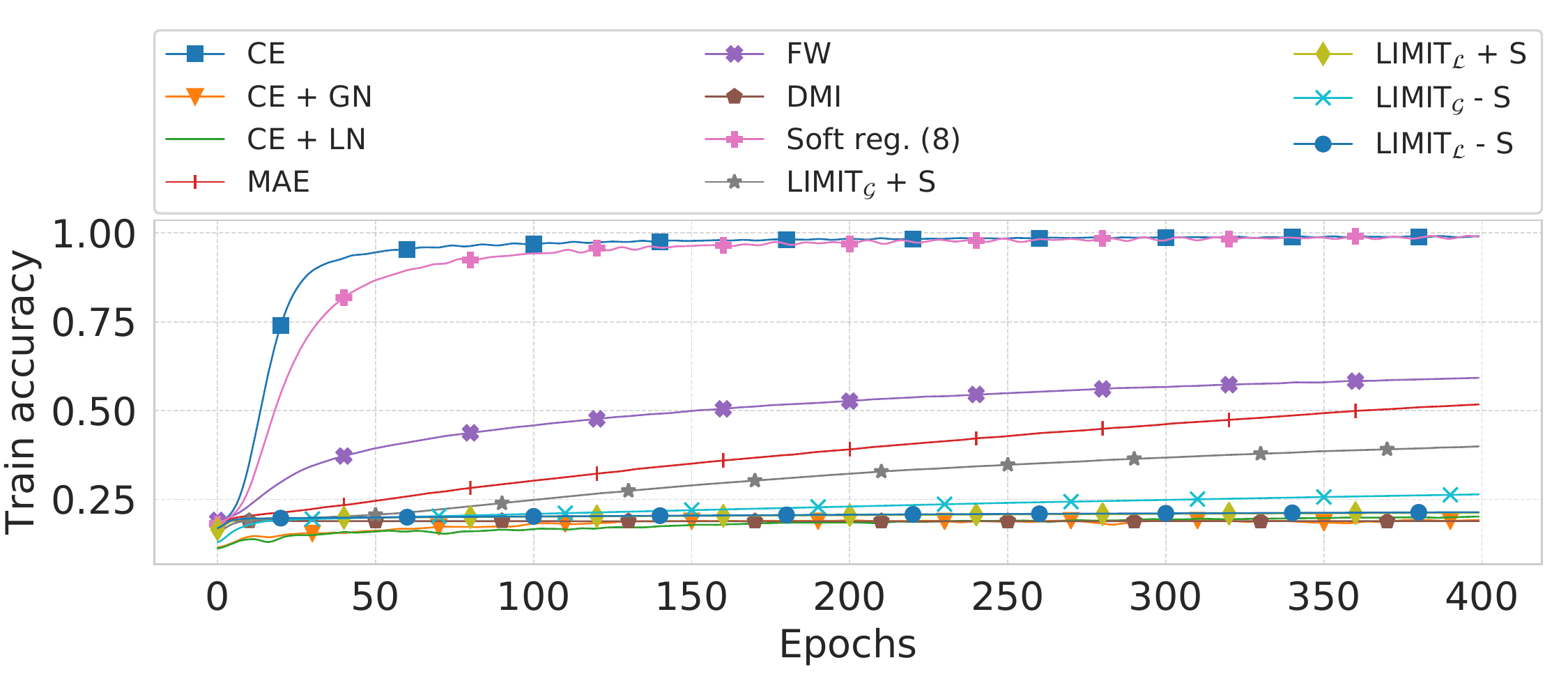}
    \caption{Training curves}
    \end{subfigure}%
    \hspace{0.5em}
    \begin{subfigure}{0.49\textwidth}
    \includegraphics[width=\textwidth]{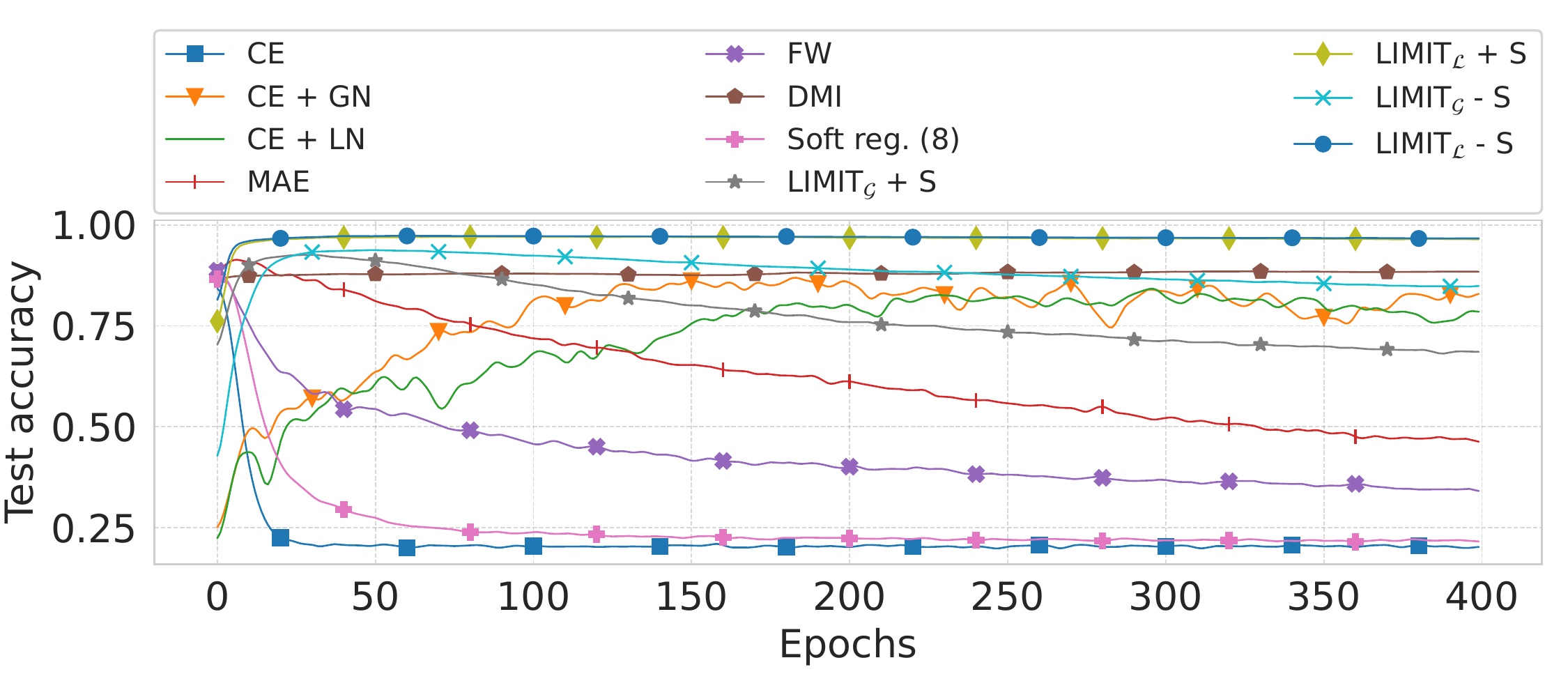}
    \caption{Testing curves}
    \end{subfigure}
    \caption{Smoothed training and testing accuracy plots of various approaches on MNIST with 80\% uniform noise.}
    \label{fig:mnist-error-noise80-curves}
\end{figure}

\cref{tab:mnist-with-error-bars} shows the average test performances and standard deviations of different approaches over 5 training/validation splits.
Additionally, \cref{fig:mnist-error-noise80-curves} shows the training and testing performances of the best methods during the training when $p=0.8$ and all training samples are used.
Overall, variants of LIMIT produce the best results and improve significantly over standard approaches.
The variants with a Laplace distribution perform better than those with a Gaussian distribution.
This is likely due to the robustness of MAE.
Interestingly, LIMIT works well and trains faster when the sampling of $G_t$ in $q$ is disabled (rows with ``-S'').
Thus, hereafter we consider this as our primary approach.
As expected, the soft regularization approach of \cref{eq:penalize-cross-entropy} and cross-entropy variants with noisy gradients perform significantly worse than LIMIT.
We exclude these baselines in our future experiments.

\paragraph{Effectiveness of gradient norm penalty.}
\begin{figure*}[t]
    \centering
    \begin{subfigure}{0.49\textwidth}
    \includegraphics[width=\textwidth]{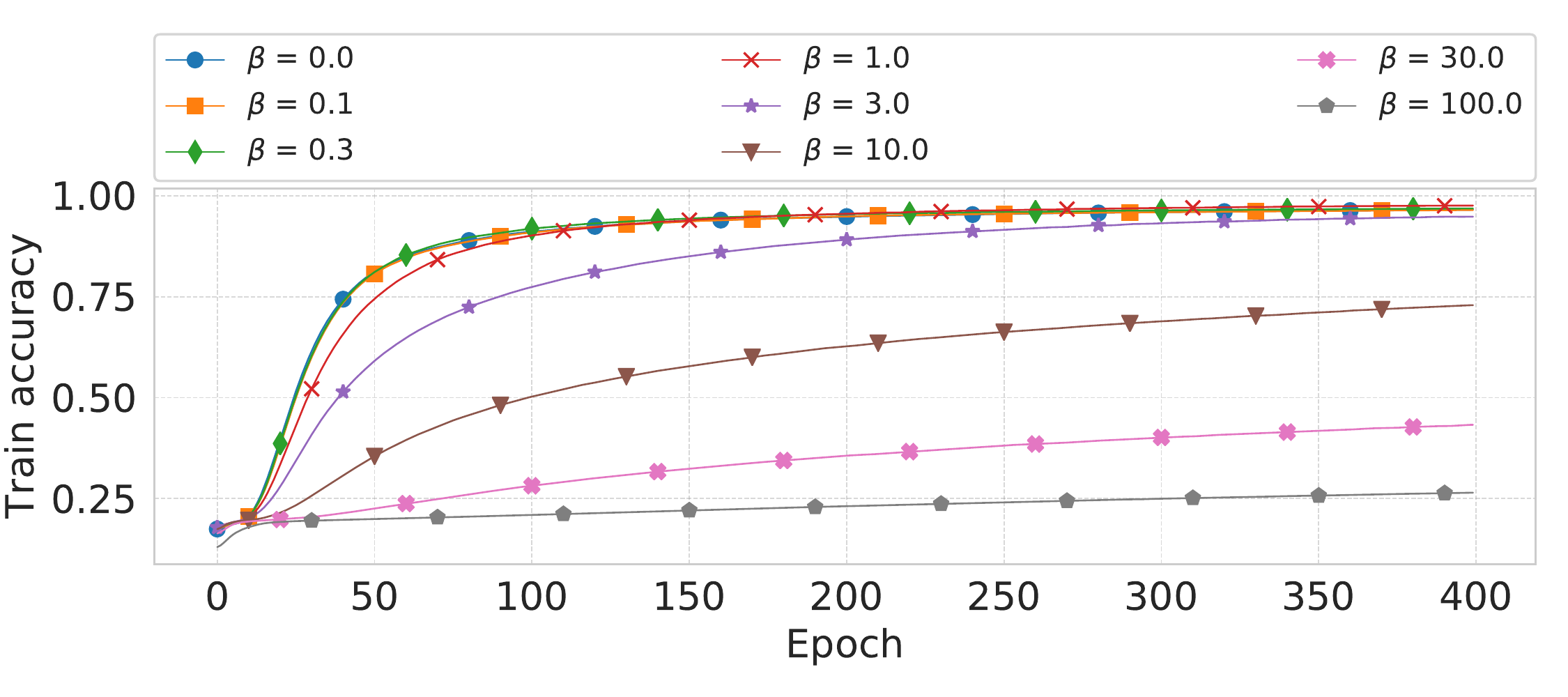}
    \caption{Training performance of LIMIT$_\mathcal{G}$ - S}
    \end{subfigure}%
    \hspace{0.5em}
    \begin{subfigure}{0.49\textwidth}
    \includegraphics[width=\textwidth]{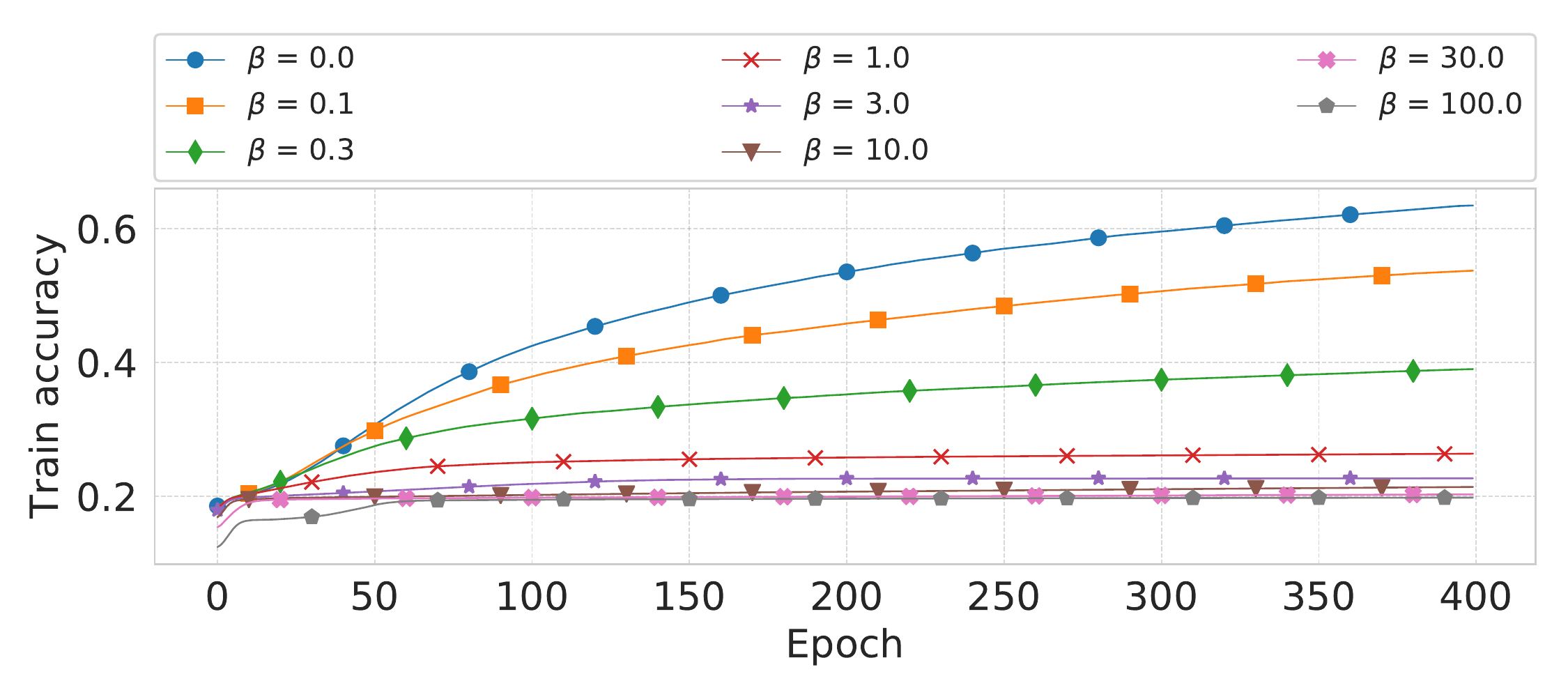}
    \caption{Training performance of LIMIT$_\mathcal{L}$ - S}
    \end{subfigure}
    
    \begin{subfigure}{0.49\textwidth}
    \includegraphics[width=\textwidth]{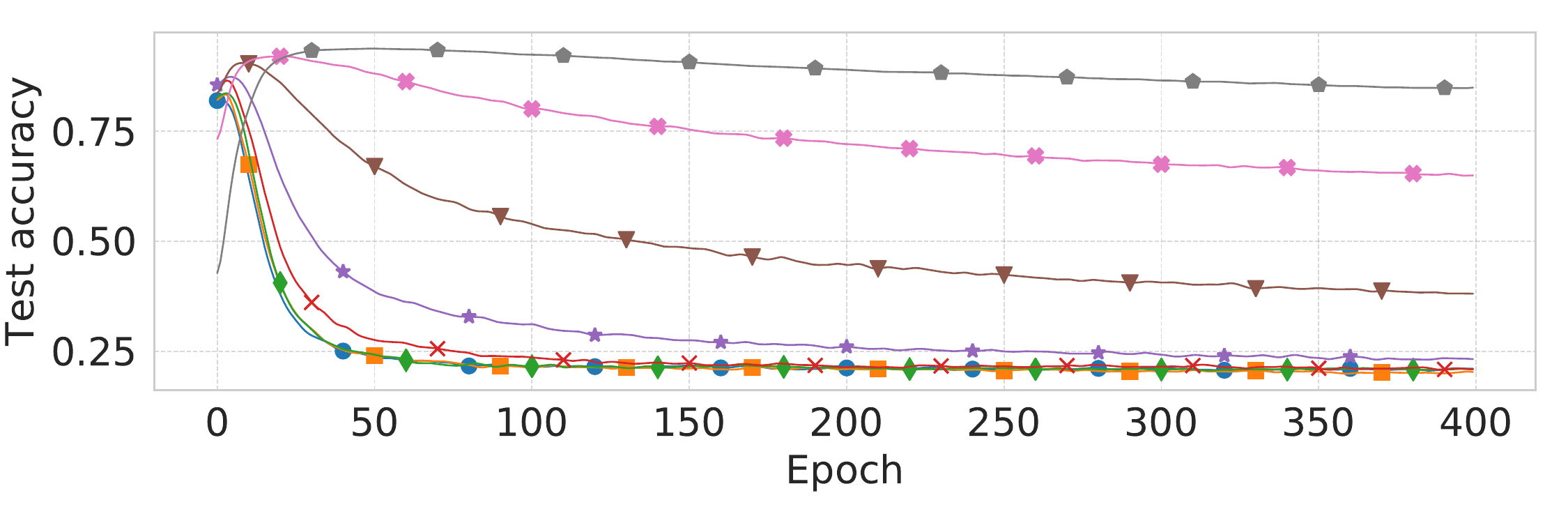}
    \caption{Testing performance of LIMIT$_\mathcal{G}$ - S}
    \end{subfigure}%
    \hspace{0.5em}
    \begin{subfigure}{0.49\textwidth}
    \includegraphics[width=\textwidth]{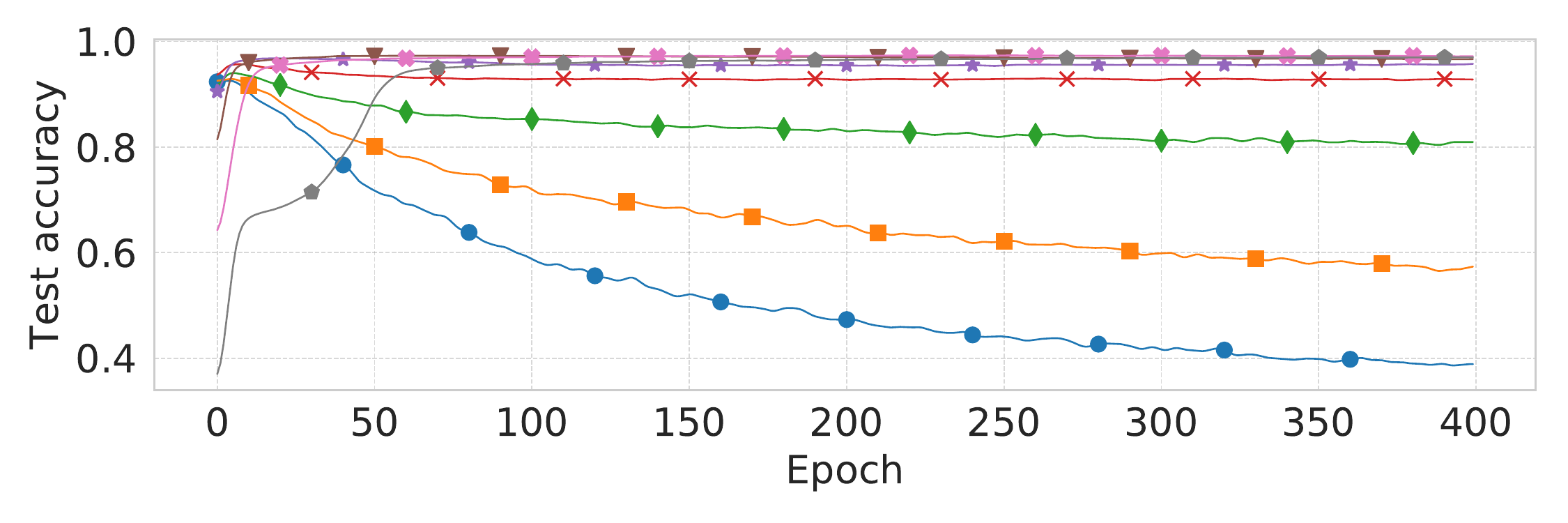}
    \caption{Testing performance LIMIT$_\mathcal{L}$ - S}
    \end{subfigure}
    \caption{Training and testing accuracies of ``LIMIT$_\mathcal{G}$ - S'' and ``LIMIT$_\mathcal{L}$ - S'' instances with varying values of $\beta$ on MNIST with 80\% uniform label noise.
    The curves are smoothed for better presentation.}
    \label{fig:mnist-error-noise80-beta}
\end{figure*}
Additionally, we test the importance of penalizing norm of predicted gradients by comparing training and testing performances of LIMIT with varying regularization strength $\beta$.
\cref{fig:mnist-error-noise80-beta} presents the training and testing accuracy curves of LIMIT with varying values of $\beta$.
We see that increasing $\beta$ decreases overfitting on the training set and usually results in better generalization.

\subsection{CIFAR-10 with uniform and pair noise}\label{subsec:cifar-pair}
\input{label-noise/sections/tables/cifar10-with-error-bars}
Next we consider a harder dataset, CIFAR-10~\citep{krizhevsky2009learning},
with two label noise models: uniform noise and pair noise.
For pair noise, certain classes are confused with some other similar class.
Following the setup of \citet{dmi} we use the following four pairs: truck $\rightarrow$ automobile, bird $\rightarrow$ airplane, deer $\rightarrow$ horse, cat $\rightarrow$ dog.
Note in this type of noise $H(\bs{Y} \mid \bs{X})$ is much smaller than in the case of uniform noise.
We split the 50K images of CIFAR-10 into training and validation sets, containing 40K and 10K samples respectively.
For the CIFAR experiments we use ResNet-34 networks~\citep{he2016deep} that differ from the standard ResNet-34 architecture (which is used for $224 \times 224$ images) in two ways: (a) the first convolutional layer has 3x3 kernels and stride 1, and (b) the max pooling layer after it is skipped.
We use standard data augmentation consisting of random horizontal flips and random 28x28 crops padded back to 32x32.
We train all models for 400 epochs and terminate the training early when the best validation accuracy is not improved in the previous 100 epochs.

For our proposed methods, the auxiliary network $q$ is ResNet-34 as well.
We noticed that for more difficult datasets, it may happen that while $q$ still learns to produce good gradients, the updates with these less informative gradients may corrupt the initialization of the classifier.
For this reason, we add an additional variant of LIMIT, which initializes the $q$ network with the best CE baseline, similar to the DMI and FW baselines.

\cref{tab:cifar10_with_error_bars} presents the results on CIFAR-10.
Again, variants of LIMIT improve significantly over standard baselines, especially in the case of uniform label noise.
As expected, when $q$ is initialized with the best CE model (similar to FW and DMI), the results are better.
As in the case of MNIST, our approach helps even when the dataset is noiseless.

\subsection{CIFAR-100 with uniform label noise}
\input{label-noise/sections/tables/joint}
To test proposed methods on a classification task with many classes, we apply them on CIFAR-100 with 40\% uniform noise.
We use the same networks and training strategies as in the case of CIFAR-10.
Results presented in \cref{tab:joint-table} indicate several interesting phenomena.
First, training with the MAE loss fails, which was observed by other works as well~\citep{gce}.
The gradient of MAE with respect to logits is $f(x)_y (\widehat{y}-y)$.
% Assuming small additive noise in gradients, this can be viewed as reducing label information in gradients.
% This is why MAE is robust to label noise and has poor performance on challenging datasets.
When $f(x)_y$ is small, there is small signal to fix the mistake.
In fact, in the case of CIFAR-100, $f(x)_y$ is approximately 0.01 in the beginning, slowing down the training.
The performance of FW degrades as the approximation error of noise transition matrix become large.
The DMI does not give significant improvement over CE due to numerical issues with computing a determinant of a 100x100 confusion matrix.
LIMIT$_\mathcal{L}$ performs worse than other variants, as training $q$ with MAE becomes challenging.
However, performance improves when $q$ is initialized with the CE model.
LIMIT$_\mathcal{G}$ does not suffer from the mentioned problem and works with or without initialization.

\subsection{Clothing1M}
Finally, as in our last experiment, we consider the Clothing1M dataset~\citep{xiao2015learning}, which has 1M images labeled with one of 14 possible clothing labels.
The dataset has very noisy training labels, with roughly 40\% of examples incorrectly labeled.
More importantly, the label noise in this dataset is realistic and instance dependent.
For this dataset we use ResNet-50 networks~\citep{he2016deep} and employ standard data augmentation, consisting of random horizontal flips and random crops of size 224x224 after resizing images to size 256x256.
We train all models for 30 epochs.
The results shown in the last column of \cref{tab:joint-table} demonstrate that DMI and \predict{} with initialization perform the best, producing similar results.

\subsection{Detecting mislabeled examples}
\begin{figure}[t]
    \captionsetup[subfigure]{justification=centering}
    \centering
    \begin{subfigure}{0.242\textwidth}
    \centering
    \includegraphics[width=\textwidth]{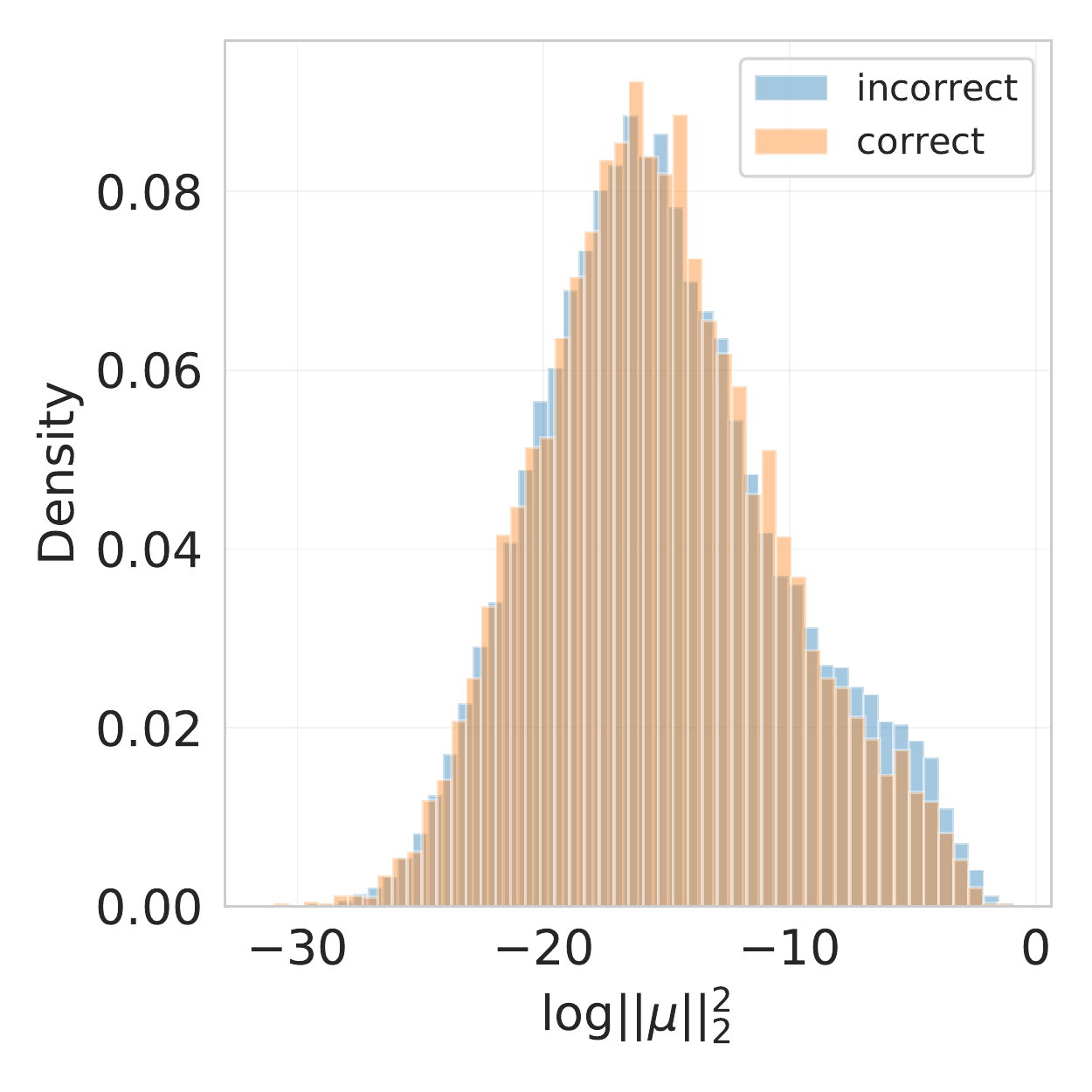}
    \caption{MNIST\\80\% uniform noise}
    \end{subfigure}% 
    ~
    \begin{subfigure}{0.242\textwidth}
    \includegraphics[width=\textwidth]{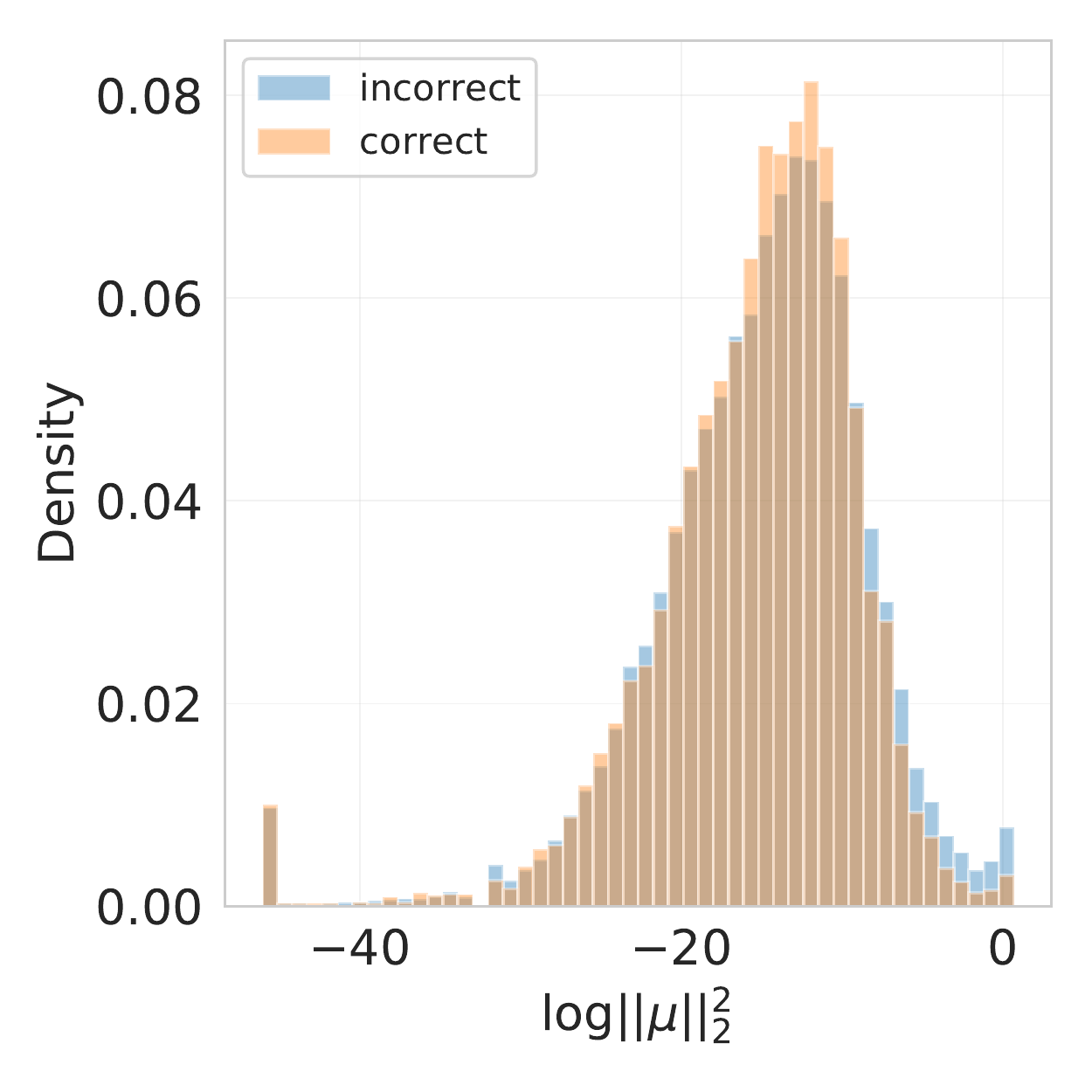}
    \caption{CIFAR-10\\40\% uniform noise}
    \end{subfigure}%
    ~
    \begin{subfigure}{0.242\textwidth}
    \includegraphics[width=\textwidth]{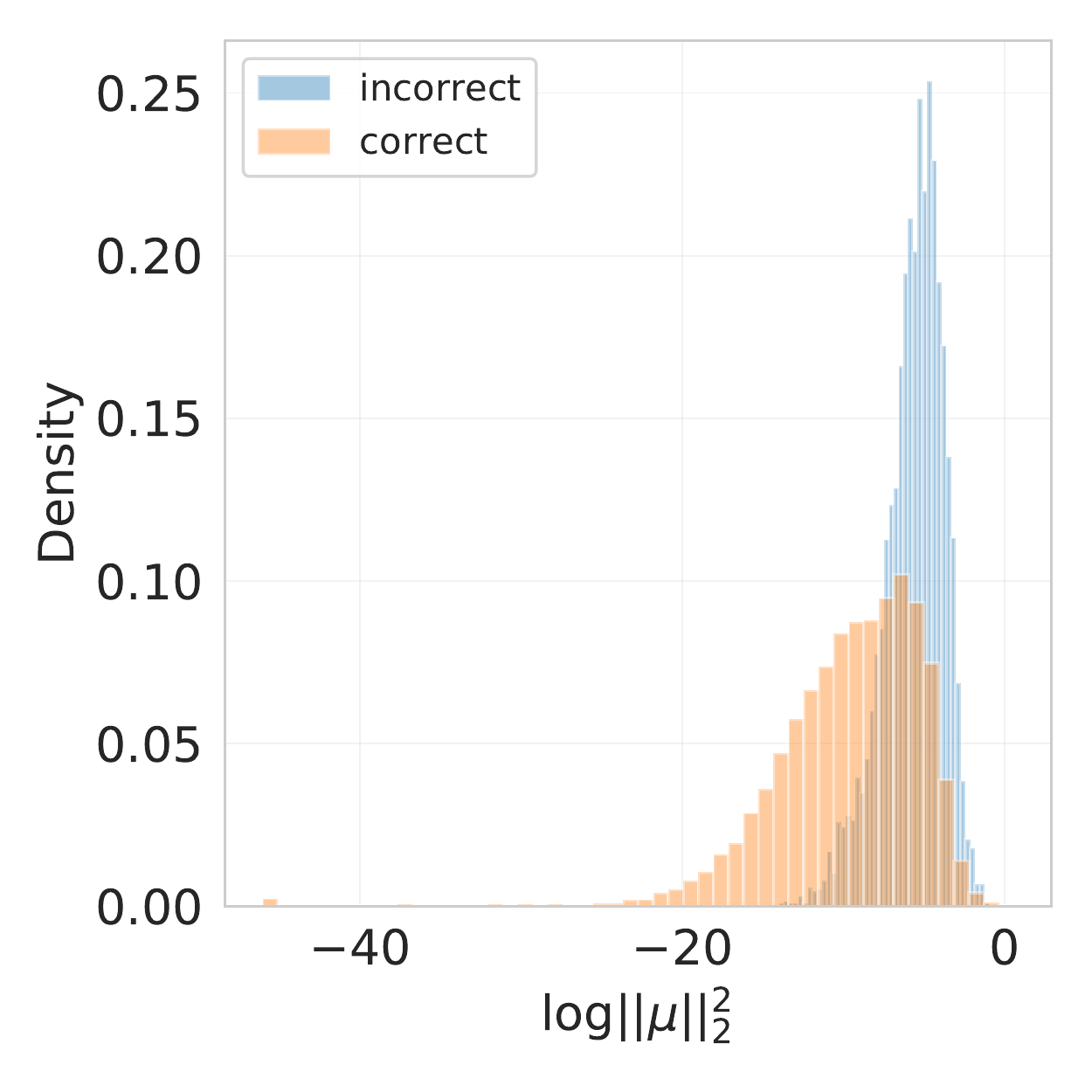}
    \caption{CIFAR-10\\40\% pair noise}
    \end{subfigure}%
    ~
    \begin{subfigure}{0.242\textwidth}
    \includegraphics[width=\textwidth]{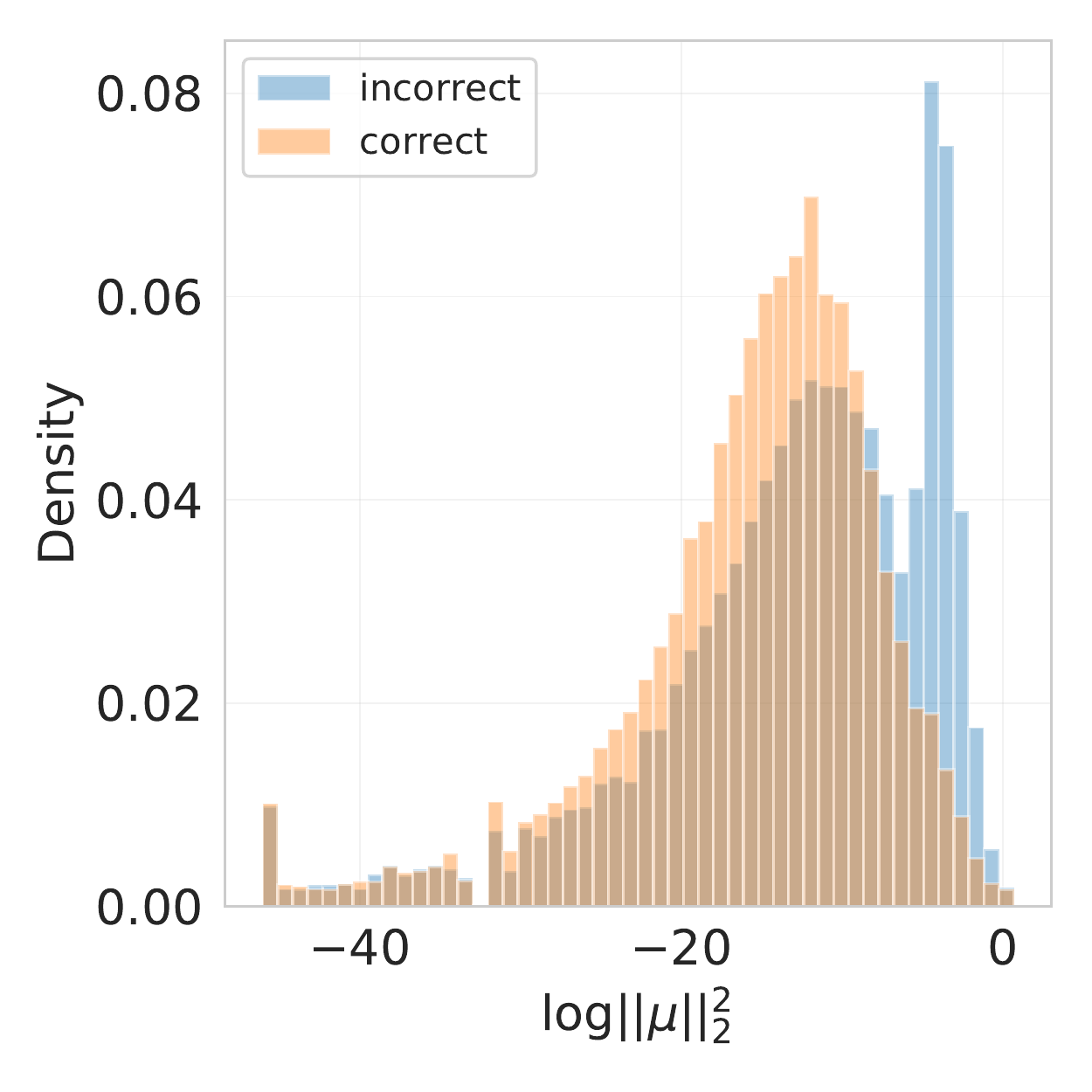}
    \caption{CIFAR-100\\40\% uniform noise}
    \end{subfigure}
    \caption{Histograms of the norm of predicted gradients for examples with correct and incorrect labels. The gradient predictions are done using the best instances of LIMIT.}
    \label{fig:grad-norm-hist}
\end{figure}

\begin{figure}[t]
    \captionsetup[subfigure]{justification=centering}
    \centering
    \begin{subfigure}{0.242\textwidth}
    \centering
    \includegraphics[width=\textwidth]{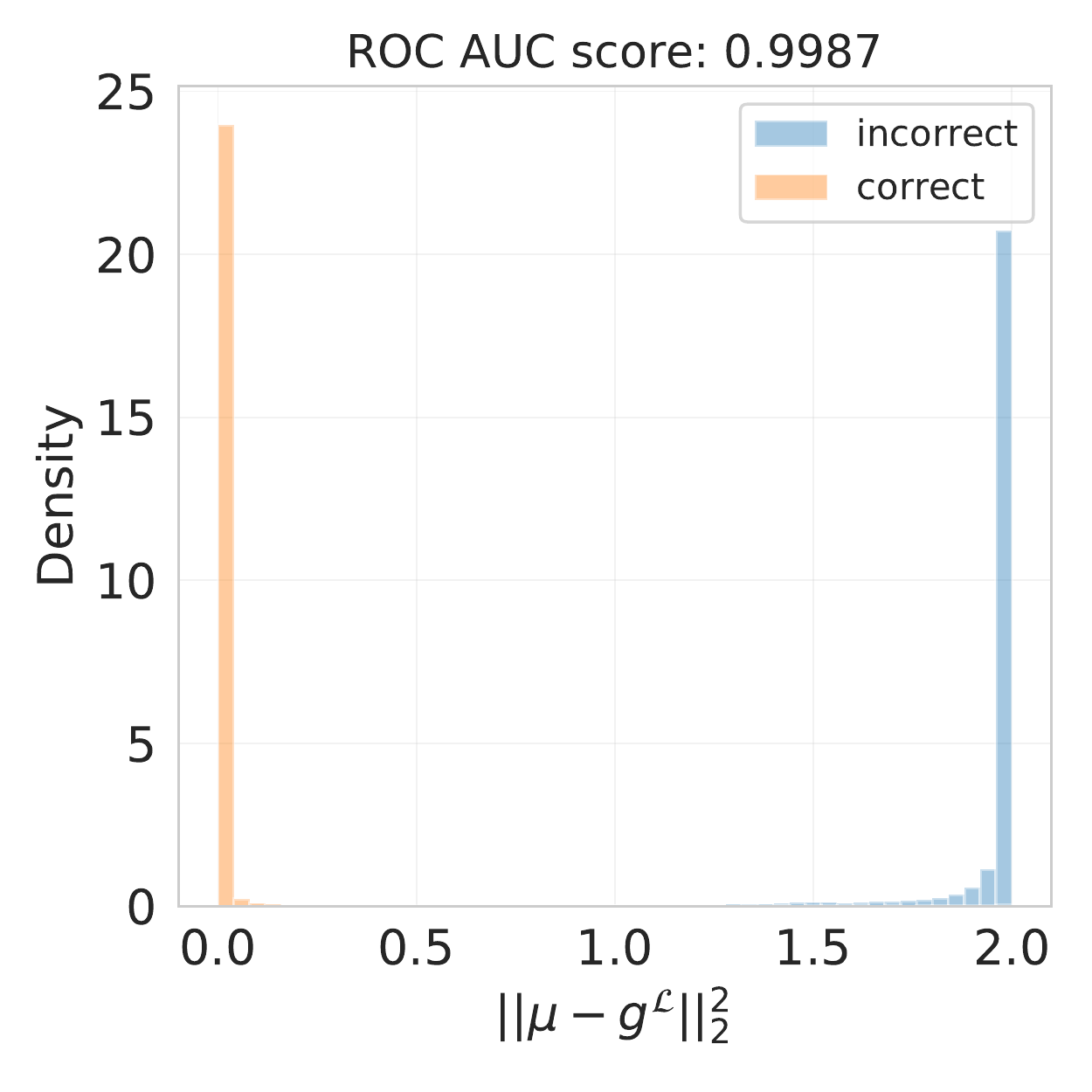}
    \caption{MNIST\\80\% uniform noise}
    \end{subfigure}% 
    ~
    \begin{subfigure}{0.242\textwidth}
    \includegraphics[width=\textwidth]{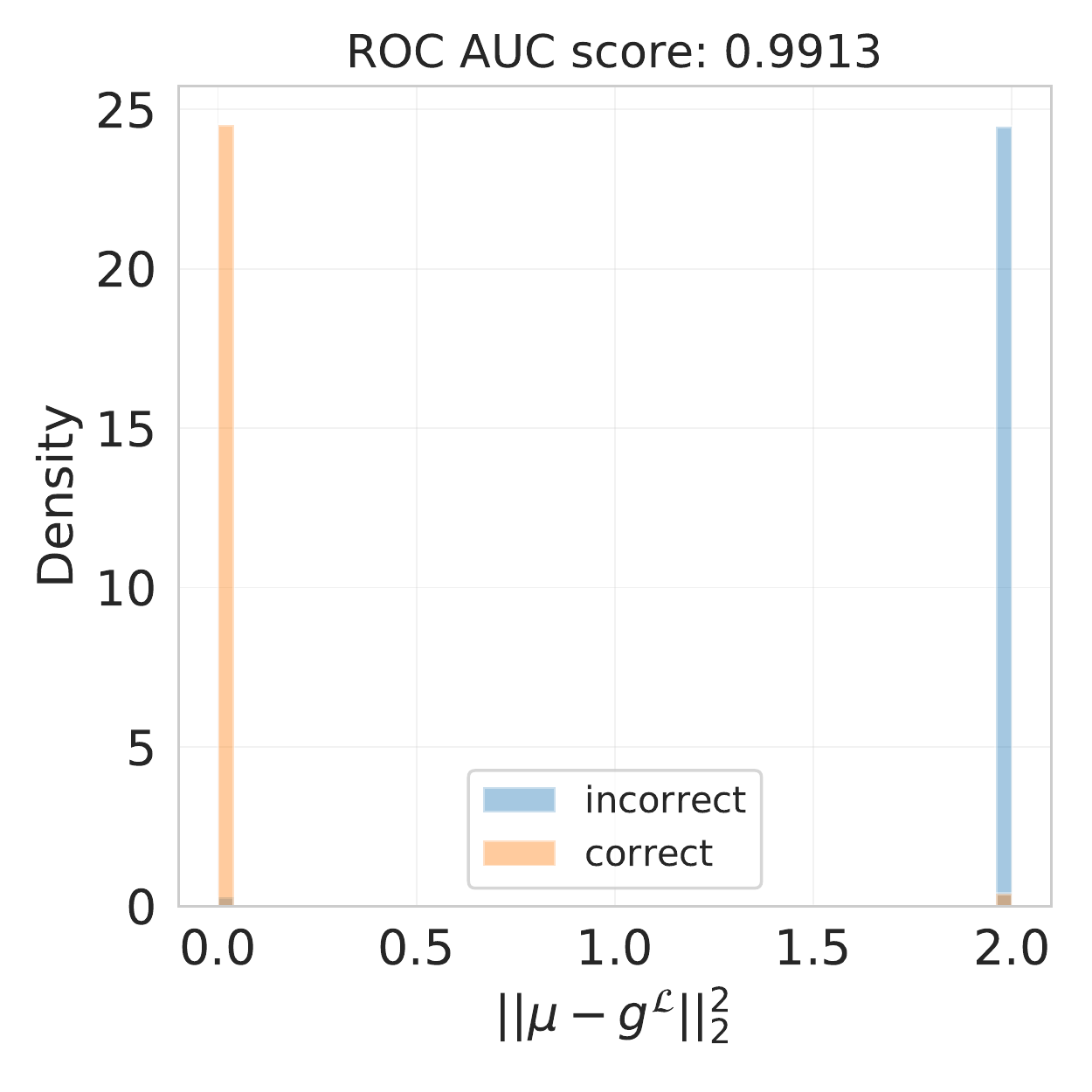}
    \caption{CIFAR-10\\40\% uniform noise}
    \end{subfigure}%
    ~
    \begin{subfigure}{0.242\textwidth}
    \includegraphics[width=\textwidth]{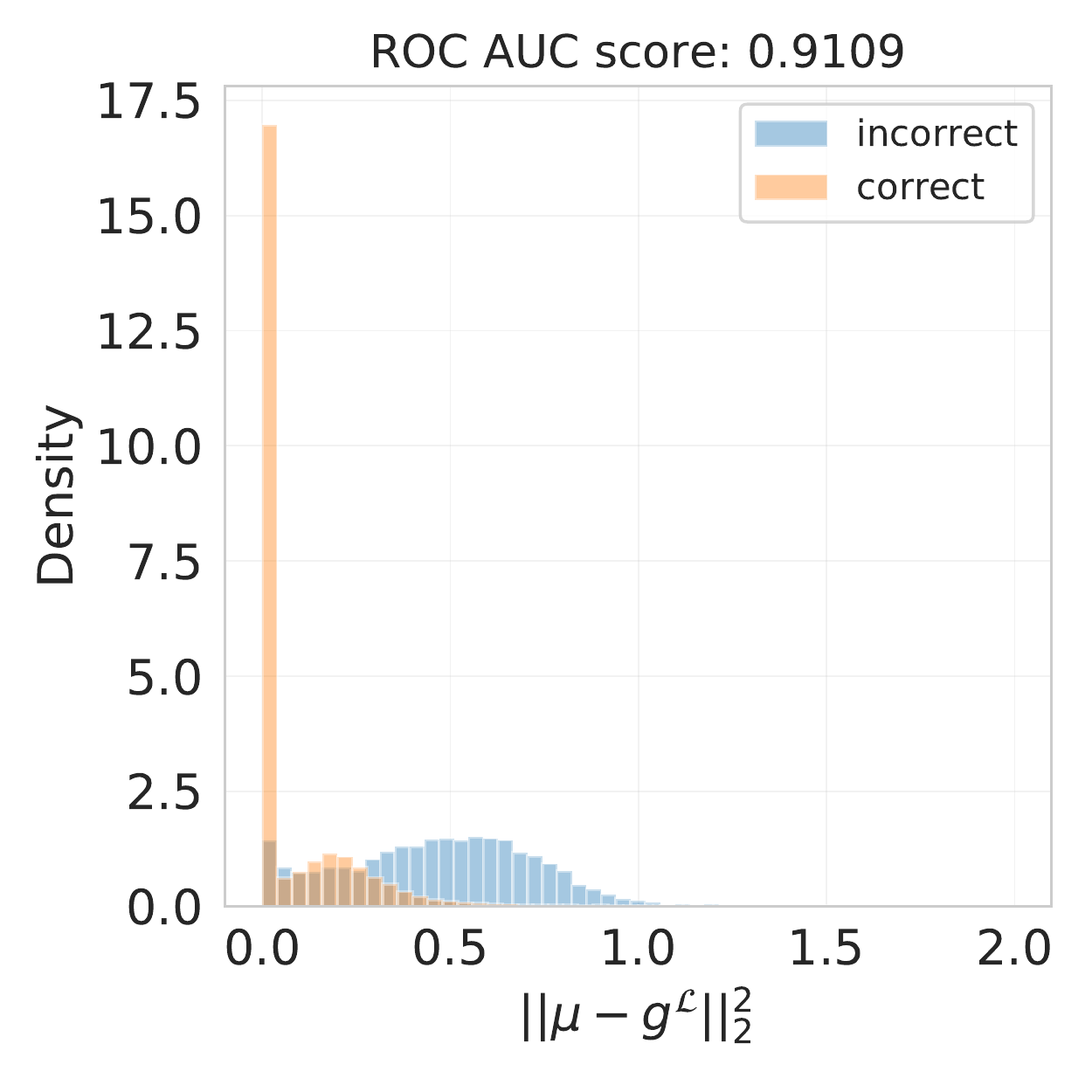}
    \caption{CIFAR-10\\40\% pair noise}
    \end{subfigure}%
    ~
    \begin{subfigure}{0.242\textwidth}
    \includegraphics[width=\textwidth]{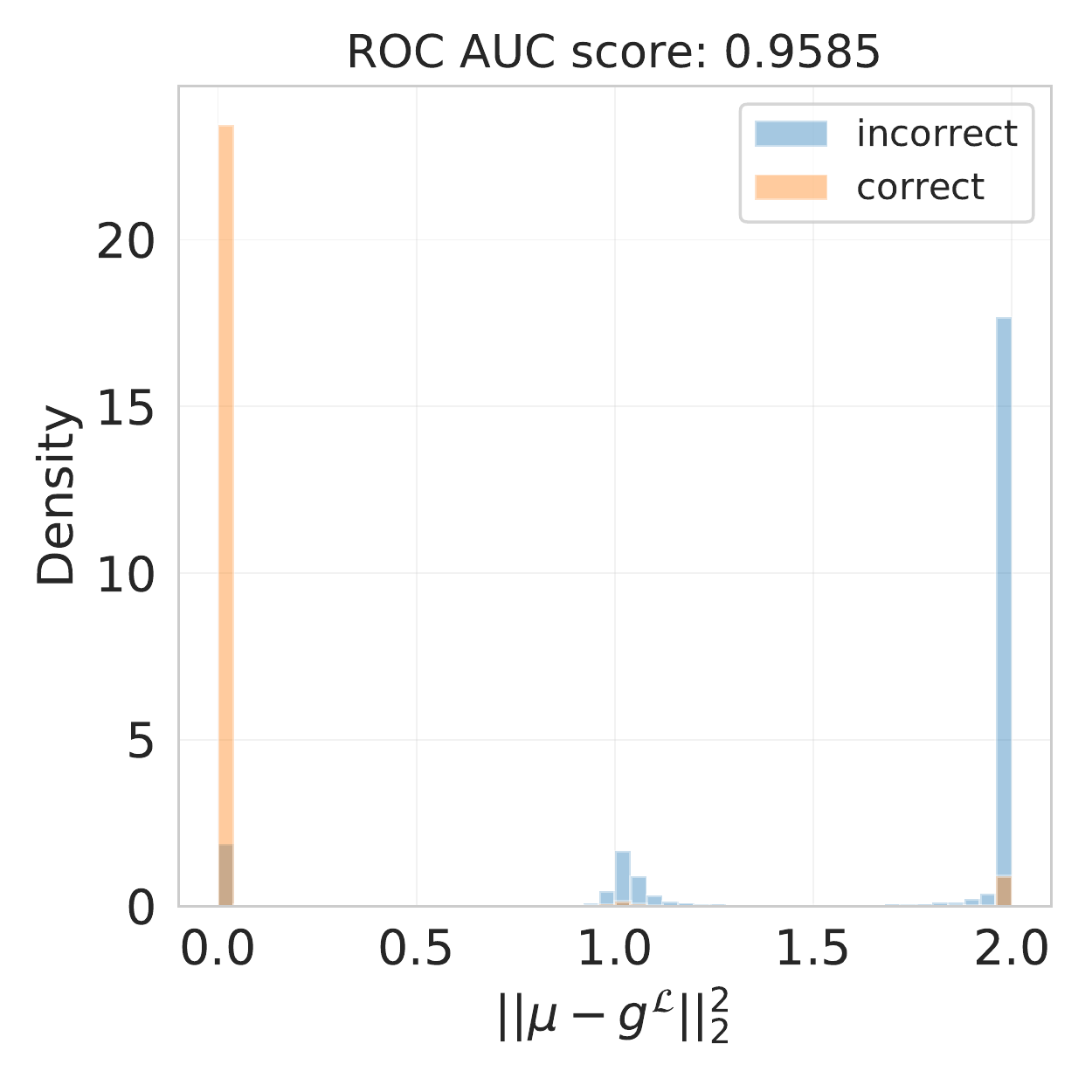}
    \caption{CIFAR-100\\40\% uniform noise}
    \end{subfigure}
    \caption{Histograms of the distance between predicted and actual gradient for examples with correct and incorrect labels. The gradient predictions are done using the best instances of LIMIT.}
    \label{fig:grad-diff-hist}
\end{figure}

In the proposed approach, the auxiliary network $q$ should not be able to distinguish correct and incorrect samples, unless it overfits.
In fact, \cref{fig:grad-norm-hist} shows that if we look at the norm of predicted gradients, examples with correct and incorrect labels are indistinguishable in easy cases (MNIST with 80\% uniform noise and CIFAR-10 with 40\% uniform noise) and have large overlap in harder cases (CIFAR-10 with 40\% pair noise and CIFAR-100 with 40\% uniform noise).
Therefore, we hypothesize that the auxiliary network learns to utilize incorrect samples effectively by predicting ``correct'' gradients.
\cref{fig:grad-diff-hist} confirms this intuition, demonstrating that this distance separates correct and incorrect samples perfectly in easy cases (MNIST with 80\% uniform noise and CIFAR-10 with 40\% uniform noise) and separates them well in harder cases (CIFAR-10 with 40\% pair noise and CIFAR-100 with 40\% uniform noise).
If we interpret this distance as a score for classifying correctness of a label, we get 91.1\% ROC AUC score in the hardest case: CIFAR-10 with 40\% pair noise, and more than 99\% score in the easier cases: MNIST with 80\% uniform noise and CIFAR-10 with 40\% uniform noise.

Motivated by these results, we use the same technique to detect samples with incorrect or confusing labels in the original MNIST, CIFAR-10, and Clothing1M datasets.
\cref{fig:confusing-samples} presents one incorrectly labeled or confusing example per class.
More examples for each class are presented in \cref{fig:mnist-cifar-confusing-more-examples,fig:clothgin1m-confusing-more-examples} of the \cref{app:add-results}.

\begin{figure*}[t]
    \centering
    \begin{subfigure}{0.49\textwidth}
    \includegraphics[width=\textwidth]{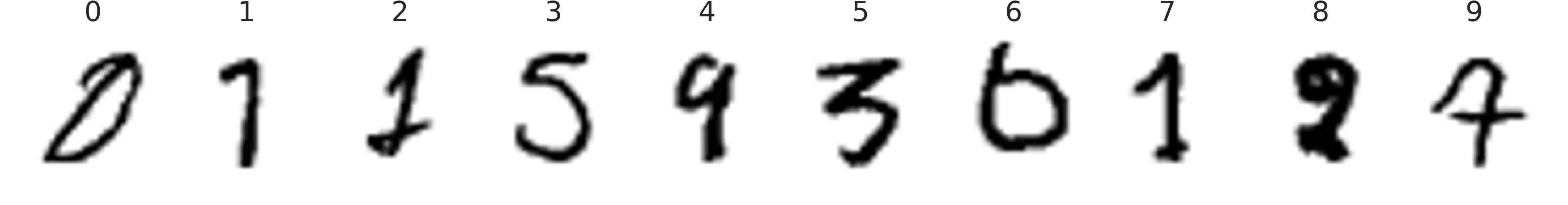}
    \caption{MNIST}
    \end{subfigure}%
    ~
    \begin{subfigure}{0.48\textwidth}
    \includegraphics[width=\textwidth]{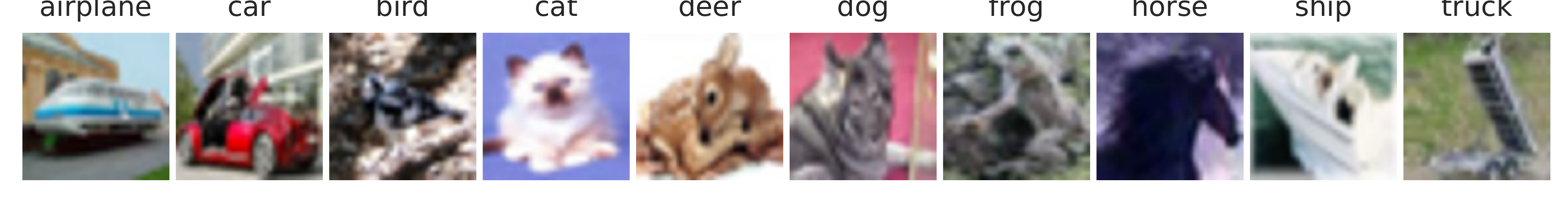}
    \caption{CIFAR-10}
    \end{subfigure}

    \vspace{1em}

    \begin{subfigure}{0.99\textwidth}
    \includegraphics[width=0.99\textwidth]{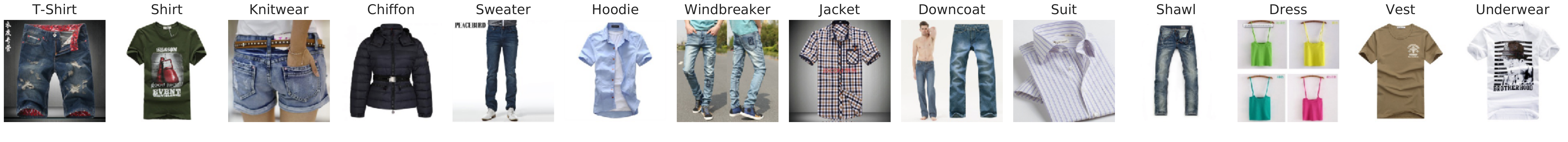}
    \caption{Clothing1M}
    \end{subfigure}%
    % The predicted label is given by $\hat y$.
    \caption{Most mislabeled examples in MNIST, CIFAR-10,  and Clothing1M datasets, according to the distance between predicted and cross-entropy gradients.}
    \label{fig:confusing-samples}
\end{figure*}

%% file: label-noise/sections/tables/mnist-with-error-bars.tex
\begin{table}[t]
\small
\centering
\caption{Test accuracy comparison on multiple versions of MNIST corrupted with uniform label noise. The error bars are standard deviations computed over 5 random training/validation splits.}
\begin{tabular}{lcccccc}
\toprule
\multirow{2}{*}{Method} & \multicolumn{3}{c}{$p=0.0$} & \multicolumn{3}{c}{$p=0.5$}\\
\cmidrule(lr){2-4}
\cmidrule(lr){5-7}
& $n=10^3$ & $n=10^4$ & All & $n=10^3$ & $n=10^4$ & All\\
\midrule
CE                             & 94.3 $\pm$  0.5 & 98.4 $\pm$  0.2 & \best{99.2 $\pm$  0.0} & 71.8 $\pm$  4.3 & 93.1 $\pm$  0.6 & 97.2 $\pm$  0.2\\
CE + GN         & 89.5 $\pm$  0.8 & 95.4 $\pm$  0.5 & 97.1 $\pm$  0.5 & 70.5 $\pm$  3.5 & 92.3 $\pm$  0.7 & 97.4 $\pm$  0.5\\
CE + LN          & 90.0 $\pm$  0.5 & 95.3 $\pm$  0.6 & 96.7 $\pm$  0.7 & 66.8 $\pm$  1.3 & 92.0 $\pm$  1.5 & 97.6 $\pm$  0.1\\
MAE                            & 94.6 $\pm$  0.5 & 98.3 $\pm$  0.2 & \best{99.1 $\pm$  0.1} & 75.6 $\pm$  5.0 & 95.7 $\pm$  0.5 & 98.1 $\pm$  0.1\\
FW                             & 93.6 $\pm$  0.6 & 98.4 $\pm$  0.1 & \best{99.2 $\pm$  0.1} & 64.3 $\pm$  9.1 & 91.6 $\pm$  2.0 & 97.3 $\pm$  0.3\\
DMI                            & 94.5 $\pm$  0.5 & 98.5 $\pm$  0.1 & \best{99.2 $\pm$  0.0} & 79.8 $\pm$  2.9 & 95.7 $\pm$  0.3 & 98.3 $\pm$  0.1\\
Soft reg. (\ref{eq:penalize-final})                       & \best{95.7 $\pm$  0.2} & 98.4 $\pm$  0.1 & \best{99.2 $\pm$  0.0} & 76.4 $\pm$  2.4 & 95.7 $\pm$  0.0 & 98.2 $\pm$  0.1\\
LIMIT$_\mathcal{G}$ + S        & \best{95.6 $\pm$  0.3} & \best{98.6 $\pm$  0.1} & \best{99.3 $\pm$  0.0} & 82.8 $\pm$  4.6 & 97.0 $\pm$  0.1 & 98.7 $\pm$  0.1\\
LIMIT$_\mathcal{L}$ + S         & 94.8 $\pm$  0.3 & \best{98.6 $\pm$  0.2} & \best{99.3 $\pm$  0.0} & \best{88.7 $\pm$  3.8} & \best{97.6 $\pm$  0.1} & \best{98.9 $\pm$  0.0}\\
LIMIT$_\mathcal{G}$ - S               & \best{95.7 $\pm$  0.2} & \best{98.7 $\pm$  0.1} & \best{99.3 $\pm$  0.1} & 83.3 $\pm$  2.3 & 97.1 $\pm$  0.2 & 98.6 $\pm$  0.1\\
LIMIT$_\mathcal{L}$ - S                & 95.0 $\pm$  0.2 & \best{98.7 $\pm$  0.1} & \best{99.3 $\pm$  0.1} & \best{88.2 $\pm$  2.9} & \best{97.7 $\pm$  0.1} & \best{99.0 $\pm$  0.1}\\
\midrule
\midrule
\multirow{2}{*}{Method} & \multicolumn{3}{c}{$p=0.8$} & \multicolumn{3}{c}{$p=0.89$}\\
\cmidrule(lr){2-4}
\cmidrule(lr){5-7}
& $n=10^3$ & $n=10^4$ & All & $n=10^3$ & $n=10^4$ & All\\
\midrule
CE                             & 27.0 $\pm$  3.8 & 69.9 $\pm$  2.6 & 87.2 $\pm$  1.0 & 10.3 $\pm$  1.6 & 13.4 $\pm$  3.3 & 13.2 $\pm$  1.8\\
CE + GN         & 25.9 $\pm$  4.6 & 51.9 $\pm$ 10.5 & 85.3 $\pm$  8.3 & 10.4 $\pm$  4.5 & 10.2 $\pm$  3.3 & 11.1 $\pm$  0.4\\
CE + LN          & 30.2 $\pm$  4.8 & 53.1 $\pm$  6.4 & 74.5 $\pm$ 19.1 & 11.9 $\pm$  3.9 &  8.8 $\pm$  5.4 & 14.1 $\pm$  4.3\\
MAE                            & 25.1 $\pm$  3.3 & 74.6 $\pm$  2.7 & 93.2 $\pm$  1.1 & 10.9 $\pm$  1.4 & 12.1 $\pm$  3.9 & 17.6 $\pm$  8.1\\
FW                             & 19.0 $\pm$  4.1 & 61.2 $\pm$  5.0 & 89.1 $\pm$  2.1 &  8.7 $\pm$  2.8 & 11.4 $\pm$  1.4 & 12.3 $\pm$  1.8\\
DMI                            & 30.3 $\pm$  5.1 & 79.0 $\pm$  1.5 & 88.8 $\pm$  0.9 & 10.5 $\pm$  1.2 & 14.1 $\pm$  5.1 & 12.5 $\pm$  1.5\\
Soft reg. (\ref{eq:penalize-final})                       & 28.8 $\pm$  2.2 & 67.0 $\pm$  1.9 & 89.3 $\pm$  0.6 & 10.3 $\pm$  1.6 & 10.5 $\pm$  0.8 & 12.7 $\pm$  2.6\\
LIMIT$_\mathcal{G}$ + S        & \best{35.9 $\pm$  6.3} & 80.6 $\pm$  2.8 & 93.4 $\pm$  0.5 & 10.0 $\pm$  1.0 & 14.3 $\pm$  5.4 & 13.1 $\pm$  4.3\\
LIMIT$_\mathcal{L}$ + S         & \best{35.6 $\pm$  3.2} & \best{93.3 $\pm$  0.3} & \best{97.6 $\pm$  0.3} & 10.1 $\pm$  0.7 & 12.5 $\pm$  2.1 & \best{28.3 $\pm$  8.1}\\
LIMIT$_\mathcal{G}$ - S               & \best{37.1 $\pm$  5.4} & 82.0 $\pm$  1.5 & 94.7 $\pm$  0.6 &  9.9 $\pm$  1.0 & 12.6 $\pm$  0.3 & 16.0 $\pm$  5.9\\
LIMIT$_\mathcal{L}$ - S                & \best{35.9 $\pm$  4.3} & \best{93.9 $\pm$  0.8} & \best{97.7 $\pm$  0.2} & 11.1 $\pm$  0.7 & 11.8 $\pm$  1.0 & \best{28.6 $\pm$  4.0}\\
\bottomrule
\end{tabular}
\label{tab:mnist-with-error-bars}
\end{table}

%% file: label-noise/sections/tables/cifar10-with-error-bars.tex
\begin{table}[t]
    \small
    \centering
    \caption{Test accuracy comparison on CIFAR-10, corrupted with uniform label noise (top) and pair label noise (bottom). The error bars are standard deviations computed by bootstrapping the test set 1000 times.}
    \begin{tabular}{lccccc}
    \toprule
    \multirow{2}{*}{Method} & \multicolumn{5}{c}{uniform label noise}\\
    \cmidrule{2-6}
    & $p=0.0$ & $p=0.2$ & $p=0.4$  & $p=0.6$  & $p=0.8$\\
    \midrule
    CE                             & 92.7 $\pm$  0.3 & 85.2 $\pm$  0.4 & 81.0 $\pm$  0.4 & 69.0 $\pm$  0.5 & 38.8 $\pm$  0.5\\
    MAE                            & 84.4 $\pm$  0.4 & 85.4 $\pm$  0.4 & 64.6 $\pm$  0.5 & 15.4 $\pm$  0.4 & 12.0 $\pm$  0.3\\
    FW                             & 92.9 $\pm$  0.3 & 86.2 $\pm$  0.3 & 81.4 $\pm$  0.4 & 69.7 $\pm$  0.5 & 34.4 $\pm$  0.5\\
    DMI                            & 93.0 $\pm$  0.3 & 88.3 $\pm$  0.3 & 85.0 $\pm$  0.3 & 72.5 $\pm$  0.4 & 38.9 $\pm$  0.5\\
    LIMIT$_\mathcal{G}$               & \best{93.5 $\pm$  0.2} & 90.7 $\pm$  0.3 & 86.6 $\pm$  0.3 & 73.7 $\pm$  0.4 & 38.7 $\pm$  0.5\\
    LIMIT$_\mathcal{L}$                & 93.1 $\pm$  0.3 & 91.5 $\pm$  0.3 & 88.2 $\pm$  0.3 & 75.7 $\pm$  0.4 & 35.8 $\pm$  0.5\\
    LIMIT$_\mathcal{G}$ + init.        & \best{93.3 $\pm$  0.3} & \best{92.4 $\pm$  0.3} & \best{90.3 $\pm$  0.3} & 81.9 $\pm$  0.4 & \best{44.1 $\pm$  0.5}\\
    LIMIT$_\mathcal{L}$ + init.         & \best{93.3 $\pm$  0.2} & \best{92.2 $\pm$  0.3} & \best{90.2 $\pm$  0.3} & \best{82.9 $\pm$  0.4} & \best{44.3 $\pm$  0.5}\\
    \midrule
    \midrule
    \multirow{2}{*}{Method} & \multicolumn{5}{c}{pair label noise}\\
    \cmidrule{2-6}
    & $p=0.0$ & $p=0.2$ & $p=0.4$  & $p=0.6$  & $p=0.8$\\
    \midrule
    CE                             & 92.7 $\pm$  0.3 & 90.0 $\pm$  0.3 & 88.1 $\pm$  0.3 & 87.2 $\pm$  0.3 & 81.8 $\pm$  0.4\\
    MAE                            & 84.4 $\pm$  0.4 & 88.6 $\pm$  0.3 & 83.2 $\pm$  0.4 & 72.1 $\pm$  0.4 & 61.1 $\pm$  0.5\\
    FW                             & 92.9 $\pm$  0.3 & 90.1 $\pm$  0.3 & 88.0 $\pm$  0.3 & 86.8 $\pm$  0.3 & 84.6 $\pm$  0.3\\
    DMI                            & 93.0 $\pm$  0.3 & 91.4 $\pm$  0.3 & 90.6 $\pm$  0.3 & 90.4 $\pm$  0.3 & \best{89.6 $\pm$  0.3}\\
    LIMIT$_\mathcal{G}$               & \best{93.5 $\pm$  0.2} & 92.8 $\pm$  0.3 & 91.3 $\pm$  0.3 & 89.2 $\pm$  0.3 & 86.0 $\pm$  0.3\\
    LIMIT$_\mathcal{L}$                & 93.1 $\pm$  0.3 & 91.9 $\pm$  0.3 & 91.1 $\pm$  0.3 & 88.8 $\pm$  0.3 & 84.2 $\pm$  0.4\\
    LIMIT$_\mathcal{G}$ + init.        & \best{93.3 $\pm$  0.3} & \best{93.3 $\pm$  0.3} & \best{92.9 $\pm$  0.3} & \best{90.8 $\pm$  0.3} & 88.3 $\pm$  0.3\\
    LIMIT$_\mathcal{L}$ + init.         & \best{93.3 $\pm$  0.2} & \best{93.0 $\pm$  0.2} & 92.3 $\pm$  0.3 & \best{91.1 $\pm$  0.3} & \best{90.0 $\pm$  0.3}\\
    \bottomrule
    \end{tabular}
    \label{tab:cifar10_with_error_bars}
\end{table}

%% file: label-noise/sections/tables/joint.tex
\begin{table}[!t]
    \small
    \centering
    \caption{Test accuracy comparison on CIFAR-100 with 40\% uniform label noise and on Clothing1M dataset. The error bars are standard deviations computed by bootstrapping the test set 1000 times.}
    % \resizebox{\linewidth}{!}{%
    \begin{tabular}{lcc}
    \toprule
    \multirow{2}{*}{Method} & CIFAR-100 & Clothing1M\\
    & \multicolumn{1}{c}{40\% uniform label noise} & \multicolumn{1}{c}{original noisy train set}\\
    \midrule
    CE                             & 44.9 $\pm$ 0.5 &  68.91 $\pm$ 0.46\\
    MAE                            & 1.8 $\pm$ 0.1 & 6.52 $\pm$ 0.23\\
    FW                             & 23.3 $\pm$ 0.4 & 68.70 $\pm$ 0.45\\
    DMI                            & 46.1 $\pm$ 0.5 & \best{71.19} $\pm$ \best{0.43}\\
  \predict$_\mathcal{G}$               & 58.4 $\pm$ 0.4 & 70.32 $\pm$ 0.42\\
    \predict$_\mathcal{L}$                & 49.5 $\pm$ 0.5 & 70.35 $\pm$ 0.45\\
    \predict$_\mathcal{G}$ + init.        & 59.2 $\pm$ 0.5 & \best{71.39} $\pm$ \best{0.44}\\
    \predict$_\mathcal{L}$ + init.         & \best{60.8 $\pm$ 0.5} & 70.53 $\pm$ 0.44\\
    \bottomrule
    \end{tabular}
    % }
    \label{tab:joint-table}
\end{table}

%% file: label-noise/sections/related.tex
\section{Related work}
Our approach is related to many works that study memorization and learning with noisy labels.
Our work also builds on theoretical results studying how generalization relates to information in neural network weights.
In this section we present the related work and discuss the connections.

\paragraph{Learning with noisy labels.}
Learning with noisy labels is a longstanding problem and has been studied extensively~\citep{survey}.
Many works studied and proposed loss functions that are robust to label noise.
\citet{natarajan2013learning} propose robust loss functions for binary classification with label-dependent noise.
\citet{mae} generalize this result for multiclass classification problem and show that the mean absolute error (MAE) loss function is tolerant to label-dependent noise.
However, as seen in our experiments, training with MAE progresses slowly and performs poorly on challenging datasets.
\citet{gce} propose a new loss function, called generalized cross-entropy (GCE), that interpolates between MAE and CE with a single parameter $q \in [0,1]$.
\citet{dmi} propose a new loss function (DMI), which is equal to the log-determinant of the confusion matrix between predicted and given labels, and show that it is robust to label-dependent noise.
These loss functions are robust in the sense that the best performing hypothesis on clean data and noisy data are the same in the regime of infinite data.
When training on finite datasets, training with these loss functions may result in memorization of training labels.
% The gradient of the GCE loss function w.r.t. logits is equal to $f(x)_y^q (f(x) - y)$.
% At $q=0$ (which reduces to CE) gradients contain full information about labels.
% When $q>0$ (at the extreme of $q=1$ reducing to MAE), the gradients are scaled down depending on $f(x)^q_y$ -- the predicted probability for the class $y$.
% Assuming small additive noise in gradients, this can be viewed as reducing label information in gradients.
% Therefore, when $f(x)^q_y$ is small, there is small signal to fix the mistake.
% This is why MAE is robust to label noise and has poor performance on challenging datasets.
% For example, on CIFAR-100, $f(x)_y$ will be close to $0.01$ in the beginning, slowing down the learning.

Another line of research seeks to estimate label noise and correct the loss function \citep{Sukhbaatar2014TrainingCN,xiao2015learning,goldberger2016training, patrini2017making,hendrycks2018using,safeguard2019}.
Some works use meta-learning to treat the problem of noisy/incomplete labels as a decision problem in which one determines the reliability of a sample~\citep{jiang2018mentornet, ren2018learning, shu2019meta}.
Others seek to detect incorrect examples and relabel them~\citep{Reed2014TrainingDN, tanaka2018joint, ma2018dimensionality, han2019deep, arazo2019unsupervised}.
\citet{han2018co,yu2019does} employ an approach where two networks select training examples for each other using the small-loss trick. While our approach also has a teaching component, the network uses all samples instead of filtering.
\citet{li2019learning} propose a meta-learning approach that optimizes a classification loss along with a consistency loss between predictions of a mean teacher and predictions of the model after a single gradient descent step on a synthetically labeled mini-batch.

Some approaches assume particular label noise models, while our approach assumes that $H(\bs{Y} \mid \bs{X}) > 0$, which may happen because of any type of label noise or attribute noise (e.g., corrupted images or partially observed inputs).
Additionally, the techniques used to derive our approach can be adopted for regression or multilabel classification tasks.
Furthermore, some methods require access to small clean validation data, which is not required in our approach.

\paragraph{Information in weights and generalization.}
Defining and quantifying information in neural network weights is an open challenge and has been studied by multiple authors.
One approach is to relate information in weights to their description length.
A simple way of measuring description length was proposed by \citet{hinton1993keeping} and reduces to the L2 norm of weights.
Another way to measure it is through the intrinsic dimension of an objective landscape~\citep{li2018intrinsic, blier2018description}.
\citet{li2018intrinsic} observed that the description length of neural network weights grows when they are trained with noisy labels~\citep{li2018intrinsic}, indicating memorization of labels.

\citet{achille2018emergence} define information in weights as the KL divergence from the posterior of weights to the prior.
In a subsequent study they provide generalization bounds involving the KL divergence term~\citep{achille2019information}.
Similar bounds were derived in the PAC-Bayesian setup and have been shown to be nonvacuous~\citep{DBLP:conf/uai/DziugaiteR17}.
With an appropriate selection of prior on weights, the above KL divergence becomes the Shannon mutual information between the weights and training dataset, $I(W; S)$.
\citet{xu2017information} derive generalization bounds that involve this latter quantity.
\citet{pensia2018generalization} upper bound $I(W; S)$ when the training algorithm consists of iterative noisy updates.
They use the chain-rule of mutual information as we did in \cref{eq:chain-rule} and bound information in updates by adding independent noise.
It has been observed that adding noise to gradients can help to improve generalization in certain cases~\citep{neelakantan2015adding}.
Another approach restricts information in gradients by clipping them~\citep{menon2020can} .

\citet{achille2018emergence} also introduce the term $I(W; \bs{Y} \mid \bs{X})$ and show the decomposition of the  cross-entropy described in \cref{eq:ce-decomp}.
\citet{yin2020metalearning} consider a similar term in the context of meta-learning and use it as a regularization to prevent memorization of meta-testing labels.
Given a meta-learning dataset $\mathscr{M}$, they consider the information in the meta-weights $\theta$ about the labels of meta-testing tasks given the inputs of meta-testing tasks, $I(\theta ; \bs{\mathscr{Y}}\mid \bs{\mathscr{X}})$.
They bound this information with a variational upper bound $\KL{q(\theta \mid \mathscr{M})}{r(\theta)}$ and use multivariate Gaussian distributions for both.
For isotropic Gaussians with equal covariances, the KL divergence reduces to $\lVert \theta - \theta_0 \rVert_2^2$, which was studied by~\citet{hu2020simple} as a regularization to achieve robustness to label noise.
Note that this bounds not only $I(\theta ; \bs{\mathscr{Y}} \mid \bs{\mathscr{X}})$ but also $I(\theta ; \bs{\mathscr{X}}, \bs{\mathscr{Y}})$. 
In contrast, we bound only $I(W; \bs{Y} \mid \bs{X})$ and work with information in gradients.

%% file: unique-info/main.tex
\input{unique-info/sections/introduction}

\paragraph{Notation.}
\input{unique-info/sections/notation.tex}

\section{Unique information of a sample in the weights}
\input{unique-info/sections/method}

\section{Experiments}\label{sec:unique-info-experiments}
\input{unique-info/sections/experiments}

\section{Discussion and future work}
\input{unique-info/sections/discussion}

\input{unique-info/sections/related}

\section{Conclusion}
There are many notions of information that are relevant to understanding the inner workings of neural networks.
Recent efforts have focused on defining information in the weights or activations that do not degenerate for deterministic training.
We look at the information in the training data, which ultimately affects both the weights and the activations.
In particular, we focus on the most elementary case, which is the unique information contained in an example, because it can be the foundation for understanding more complex notions.
However, our approach can be readily generalized to unique information of a group of samples.
Unlike most previously introduced information measures, ours is tractable even for real datasets used to train standard network architectures, and does not require restriction to limiting cases.
In particular, we can approximate our quantities without requiring the limit of small learning rate (continuous training time), or the limit of infinite network width.

%% file: unique-info/sections/introduction.tex
\section{Introduction}
Training a deep neural network (DNN) entails extracting information from samples in a dataset and storing it in the weights of the network, so that it may be used in future inference or prediction. But how much information does a \textit{particular} sample contribute to the trained model? The answer can be used to provide strong generalization bounds (if no information is used, the network is not memorizing the sample), privacy bounds (how much information the network can leak about a particular sample), and enable better interpretation of the training process and its outcome.
To determine the information content of samples, we need to define and compute information. In the classical sense, information is a property of random variables, which may be degenerate for the deterministic process of computing the output of a trained DNN in response to a given input (inference). So, even posing the problem presents some technical challenges. But beyond technicalities, how can we know whether a given sample is memorized by the network and, if it is, whether it is used for inference?

We propose a notion of {\em unique sample information} that, while rooted in information theory, captures some aspects of stability theory and influence functions.
Unlike most information-theoretic measures, our notion of information can be approximated efficiently for large networks, especially in the case of transfer learning, which encompasses many real-world applications of deep learning. Our definition can be applied to either ``weight space'' or ``function space.'' This allows us to study the non-trivial difference between information the weights possess (weight space) and the information the network actually {\em uses} to make predictions on new samples (function space).
 
Our method yields a valid notion of information without relying on the randomness of the training algorithm ({\em e.g.}, stochastic gradient descent, SGD), and works even for deterministic training algorithms. Our main work-horse is a first-order approximation of the network.
This approximation is accurate when the network is pretrained \citep{mu2020gradients} --- as is common in practical applications --- or is randomly initialized but very wide \citep{lee2019wide}, and can be used to obtain a closed-form  expression of the per-sample information.
%The latter has been studied extensively in the Neural Tangent Kernel (NTK) literature, albeit in the limit of very wide networks that are not representative of those in use today.
% Working with realistic network models and datasets, we show that (a) linearization around a random initialization does not, in fact, yield a good approximation, whereas (b) linearization around a pretrained model does, to the point where we can use the linearized model to compute the per-sample information in closed form.
In addition, our method has better scaling with respect to the number of parameters than most other information measures, which makes it applicable to massively overparametrized models such as DNNs.
Our information measure can be computed without actually training the network, making it amenable to use in problems like dataset summarization.

We apply our method to remove a large portion of uninformative examples from a training set with minimum impact on the accuracy of the resulting model (dataset summarization).
We also apply our method to detect mislabeled samples, which we show carry more unique information.

To summarize, our contributions are
    (1) We introduce a notion of unique information that a sample contributes to the training of a DNN, both in weight space and in function space, and relate it with the stability of the training algorithm;
    (2) We provide an efficient method to compute unique information even for large networks using a linear approximation of the DNN, and without having to train a network;
    (3) We show applications to dataset summarization and analysis.
The implementation of the proposed method and the code for reproducing the experiments is available at \url{https://github.com/awslabs/aws-cv-unique-information}.

%% file: unique-info/sections/notation.tex
In this chapter we consider a particular instance of a labeled training dataset $s = (z_1,\ldots,z_n)$, where $z_i = (x_i, y_i)$, $x_i \in \mathcal{X},\ y_i \in \mathbb{R}^k$.
We consider a neural network $f_w : \mathcal{X} \rightarrow \mathbb{R}^k$ with parameters $w \in \mathbb{R}^d$.
Throughout this chapter $s_{-i}=\{z_1, \ldots, z_{i-1}, z_{i+1}, \ldots, z_n\}$ denotes the training set without the $i$-th sample; $f_{w_t}$ is often shortened to $f_t$; the concatenation of all training examples is denoted by $\bs{x}$; the concatenation of all training labels by $\bs{y} \in \mathbb{R}^{nk}$; and the concatenation of all outputs by $f_w(\bs{x}) \in \mathbb{R}^{nk}$.
The loss on the $i$-th example is denoted by $\ell_i(w)$ and is equal to $\frac{1}{2}\lVert f_w(x_i) - y_i\rVert_2^2$, unless specified otherwise. This choice is useful when dealing with linearized models and is justified by \citet{hui2021evaluation}, who show that the mean-squared error (MSE) loss is as effective as cross-entropy for classification tasks.
The training loss is $\LL(w) = \sum_{i=1}^n \ell_i(w) + \frac{\lambda}{2} \lVert w - w_0\rVert_2^2$, where $\lambda \ge 0$ is a weight decay regularization coefficient and $w_0$ is the weight initialization point.
Note that the regularization term differs from standard weight decay $\lVert w\rVert_2^2$ and is more appropriate for linearized neural networks, as it allows us to derive the dynamics analytically (see \cref{sec:wd-linearized}).

%% file: unique-info/sections/method.tex
\label{sec:smoothed-sample-information}
We start with defining a notion of unique information in the weight space.
Consider a (possibly stochastic) training algorithm $Q_{W|S}$ that, given a training dataset $s$, returns weights $W \sim Q_{W|S=s}$ for the classifier $f_w$.
From an information-theoretic point of view, the amount of unique information a sample $z_i = (x_i, y_i)$ provides about the weights is given by the conditional point-wise mutual information:
\begin{equation}
I(W; Z_i=z_i \mid S_{-i}=s_{-i}) = \KL{Q_{W|S=s}}{P_{W|S_{-i}=s_{-i}}},
\label{eq:unique-info}
\end{equation}
where $P_{W|S_i}$ is derived from the joint distribution $P_{S,W}$ induced by $P_S$ and $Q_{W|S}$ (i.e., denotes the average distribution of the weights over all possible samplings of $Z_i$ and fixed $S_{-i}=s_{-i}$).

\subsection{Approximating unique information with leave-one-out KL divergence}\label{sec:appendix-approximating-sample-info-with-kl}

\begin{figure}[t]
    \centering
    \begin{subfigure}{0.25\textwidth}
    \includegraphics[width=\textwidth]{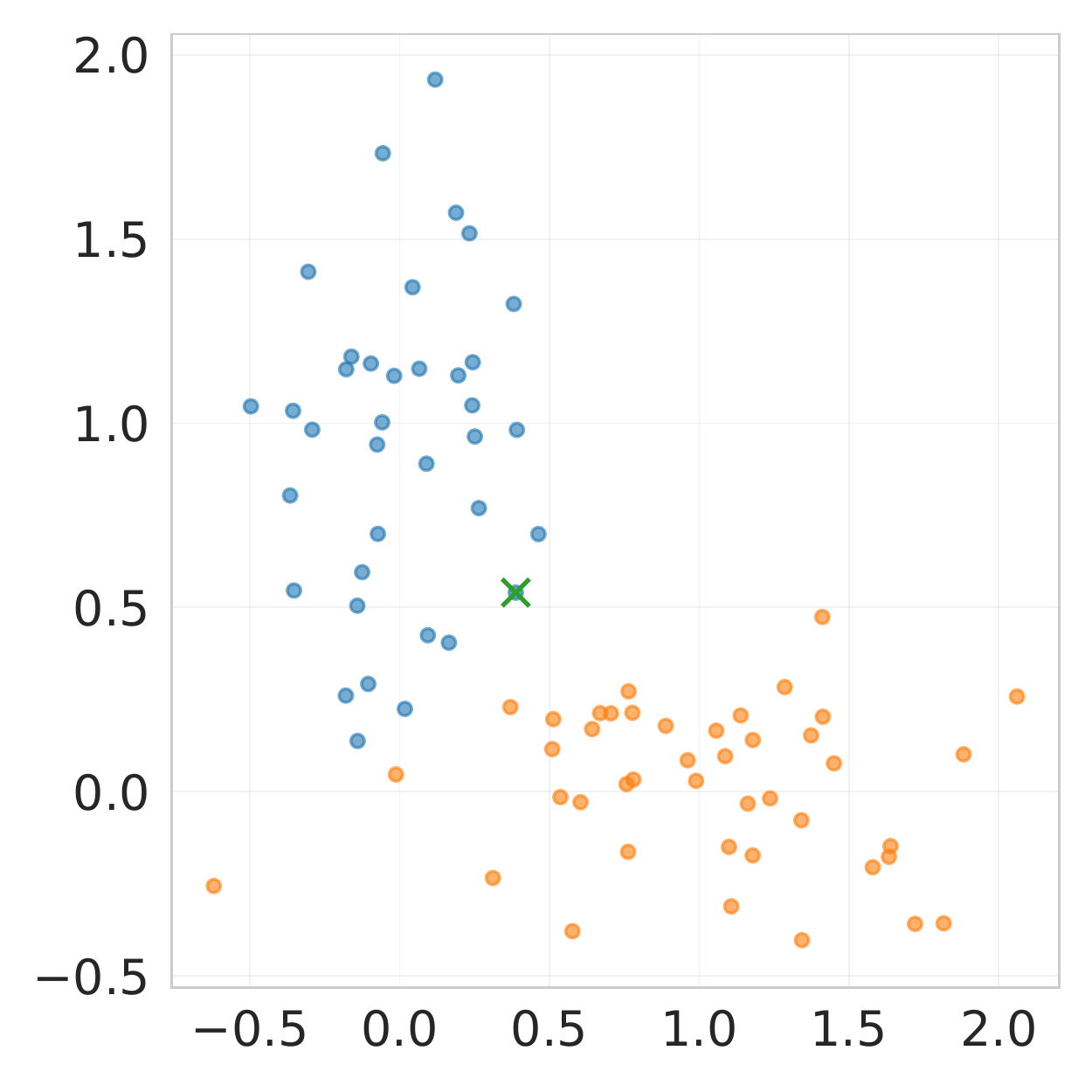}
    \caption{Data}
    \label{fig:2d-dataset}
    \end{subfigure}%
    \begin{subfigure}{0.25\textwidth}
    \includegraphics[width=\textwidth]{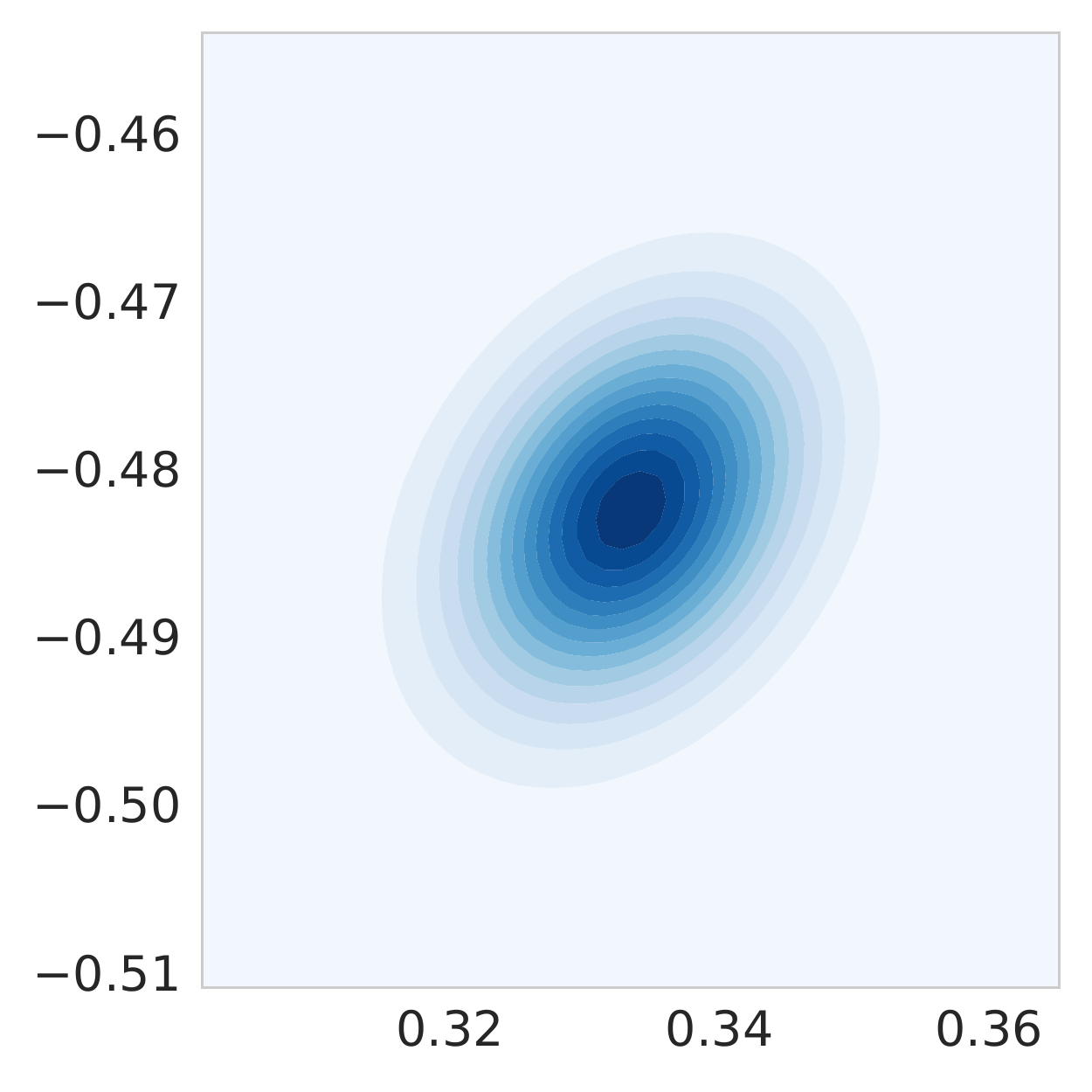}
    \caption{$Q_{W|S=s}$}
    \label{fig:kl-approx-full}
    \end{subfigure}%
    \begin{subfigure}{0.25\textwidth}
    \includegraphics[width=\textwidth]{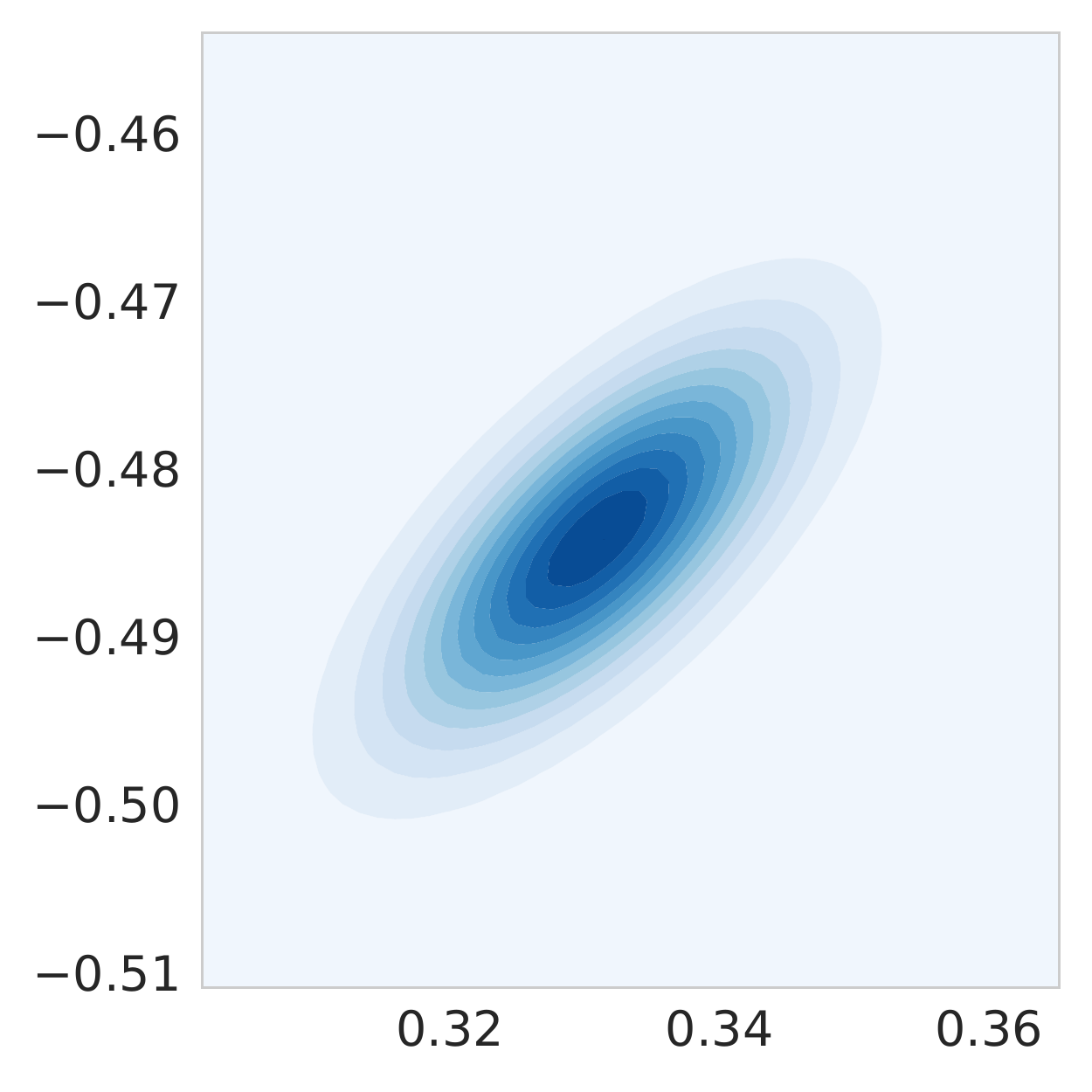}
    \caption{$Q_{W|S=s_{-i}}$}
    \label{fig:kl-approx-remove}
    \end{subfigure}%
    \begin{subfigure}{0.25\textwidth}
    \includegraphics[width=\textwidth]{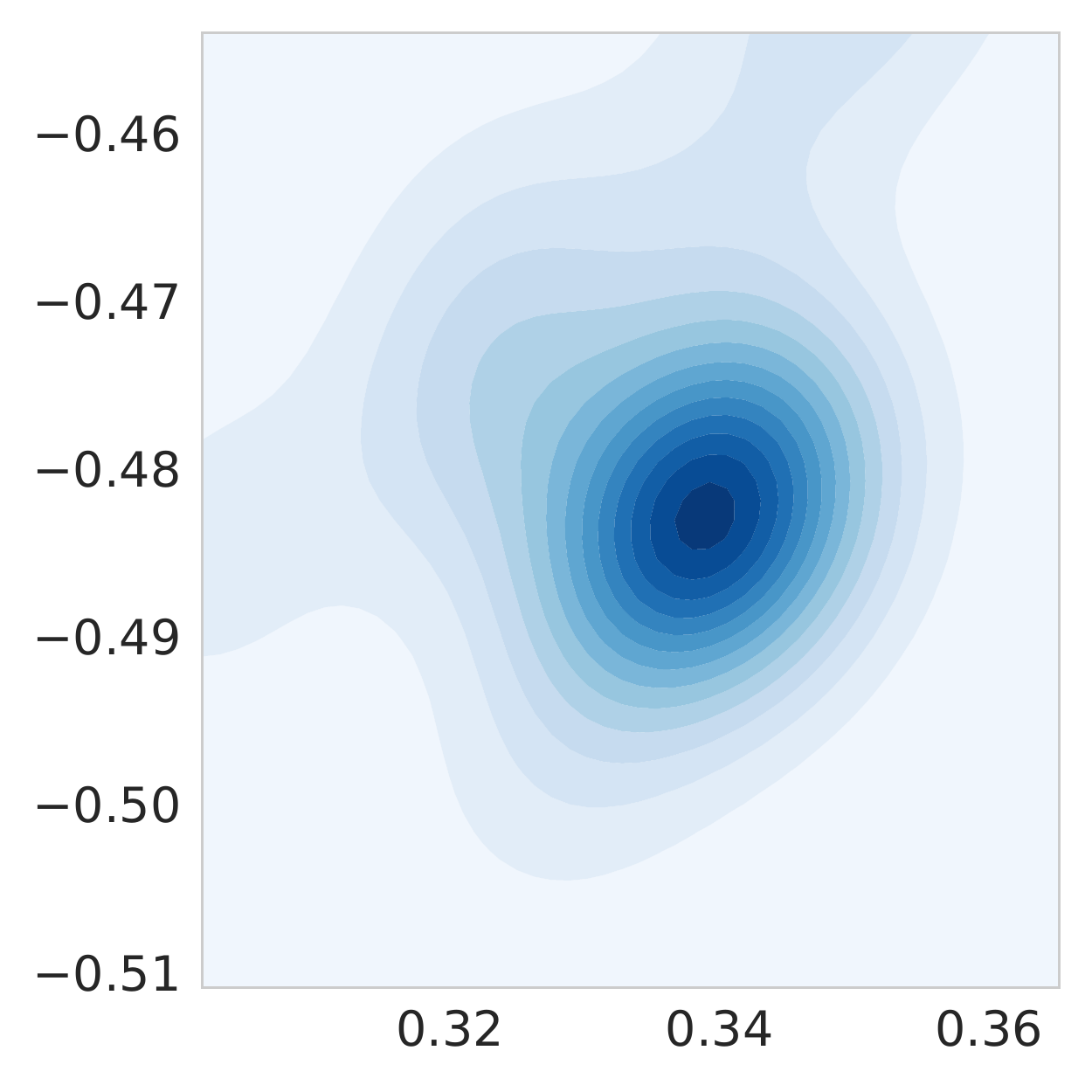}
    \caption{$P_{W|S_{-i}=s_{-i}}$}
    \label{fig:kl-approx-marginal}
    \end{subfigure}
    \caption{A toy dataset and key distributions involved in upper bounding the unique sample information with leave-one-out KL divergence.}
    \label{fig:apprimxating-sample-info-with-kl}
\end{figure}

Computing the conditional distribution $P_{W|S_{-i}=s_{-i}}$ is challenging because of the high-dimensionality and the cost of training for multiple replacements of $z_i$.
One can address this problem by using the following upper bound.
\begin{proposition}
Let $P_{W,S}$ be the joint distribution induced by $P_S$ and $Q_{W|S}$. Assume that $\forall S=s, i\in[n],\ P_{W|S_{-i}=s_{-i}} \ll Q_{W|S=s_{-i}}$. Then $\forall i \in [n]$
\begin{equation}
    \KL{Q_{W|Z_i, S_{-i}=s_{-i}}}{P_{W|S_{-i}=s_{-i}}} \le \KL{Q_{W|Z_i,S_{-i}=s_{-i}}}{Q_{W|S=s_{-i}}}.
\end{equation}
\label{prop:unique-info-upper-bound}
\end{proposition}
This proposition shows that the expectation (over $Z_i$) of the unique information of \cref{eq:unique-info} can be upper bounded by the expectation of the following the following quantity:
\begin{equation}
    \USI(z_i, Q) \triangleq \KL{Q_{W|S=s}}{Q_{W|S=s_{-i}}},
    \label{eq:si-truly-continuous-algorithm}
\end{equation}
which we call \emph{sample information} of $z_i$ w.r.t. algorithm $Q$.

\cref{fig:apprimxating-sample-info-with-kl} illustrates this approximation step for a toy 2D dataset.
The dataset has two classes, each with 40 examples, generated from a Gaussian distribution (see \cref{fig:2d-dataset}).
We consider training a linear regression on this dataset using stochastic gradient descent for 200 epochs, with batch size equal to 5 and 0.1 learning rate.
We are interested in approximating the unique information of the $i$-th example (denoted with a green cross in \cref{fig:2d-dataset}).
\cref{fig:kl-approx-full} plots the distribution of regression weights after training on the entire dataset: $Q_{W|S=s}$.
\cref{fig:kl-approx-marginal} plots the distribution of regression weights averaging out the effect of the $i$-th example: $P_{W \mid S_{-i}=s_{-i}}=\E_{Z'\sim P_Z}\sbr{Q_{W | S_{-i}=s_{-i},Z_i=Z'}}$.
We see that compared to $Q_{W|S=s}$ this distribution is more complex.
Precisely for this reason we replace it with the distribution of regression weights after training on the dataset that excludes the $i$-th example: $Q_{W | S=s_{-i}}$, shown in \cref{fig:kl-approx-remove}.
In this case we have that $I(W; Z_i=z_i \mid S_{-i}=s_{-i}) \approx 1.3$, while $\mathrm{SI}(z_i, Q) = \KL{Q_{W|S=s}}{Q_{W|S=s_{-i}}} \approx 3.0$.

\subsection{Smoothed sample information}
The formulation above is valid in theory but, in practice, even SGD is used in a deterministic fashion by fixing the random seed and, in the end, we obtain just one set of weights rather than a distribution of them.
Under these circumstances, all the above KL divergences are degenerate, as they evaluate to infinity. 
It is common to address the problem by assuming that $Q_{W|S}$ is a continuous stochastic optimization algorithm, such as stochastic gradient Langevin dynamics (SGLD) or a continuous approximation of SGD which adds Gaussian noise to the gradients.
However, this creates a disconnect with the practice, where such approaches do not perform at the state-of-the-art.
Our definitions below aim to overcome this disconnect.

\begin{definition}[Smooth sample information]
Let $Q_{W|S}$ be a possibly stochastic algorithm.
Following \cref{eq:si-truly-continuous-algorithm}, we define the \emph{smooth sample information} with smoothing $\Sigma$ as:
\begin{equation}
\boxed{\USI_\Sigma(z_i, Q) = \KL{Q^\Sigma_{W|S=s}}{Q^\Sigma_{W|S=s_{-i}}},}
\label{eq:si-definition}
\end{equation}
where $Q^\Sigma_{W|S}$ denotes the distribution of $W + \xi$ with $W \sim Q_{W|S}$ and $\xi \sim \mathcal{N}(0, \Sigma)$.
\end{definition}
Note that if the algorithm $Q_{W|S}$ is continuous, we can pick $\Sigma \rightarrow 0$, which will make $\USI_\Sigma(z_i, Q) \rightarrow \USI(z_i, Q)$.
The following proposition shows how to compute the value of $\USI_\Sigma$ when $Q_{W|S}$ is deterministic (the most common case in practice).
\begin{proposition}\label{prop:smooth-SI-deterministic}
Let $Q_{W|S}$ be a deterministic training algorithm (i.e., a distribution that puts all the mass on a point $W = A(S)$). We have that
\begin{align}
\USI_\Sigma(z_i, Q) &= \frac{1}{2} (w - w_{-i})^T \Sigma^{-1} (w - w_{-i}),
\label{eq:si-stability-approximation}
\end{align}
where $w=A(s)$ and $w_{-i} = A(s_{-i})$ are the weights obtained by training respectively with and without the training sample $z_i$.
\end{proposition}
That is, the value of $\USI_\Sigma(z_i, Q)$ depends on the distance between the solutions obtained training with and without the sample $z_i$, rescaled by $\Sigma$.
\cref{prop:smooth-SI-deterministic} follows from the fact that the KL divergence between two Gaussian distributions with with means $\mu_1$ and $\mu_2$ and equal covariance matrices $\Sigma_1=\Sigma_2=\Sigma$ is equal to $\frac{1}{2} (\mu_1 - \mu_2)^T \Sigma^{-1} (\mu_1 - \mu_2)$.

The smoothing of the weights by a matrix $\Sigma$ can be seen as a form of soft-discretization.
Rather than simply using an isotropic discretization $\Sigma=\sigma^2I$ -- since different filters have different norms and/or importance for the final output of the network -- it makes sense to discretize them differently.
In \cref{sec:functional-sample-information,sec:linear-networks} we show two canonical choices for $\Sigma$.
One is the inverse of the Fisher information matrix, which discounts weights not used for classification, and the other is the covariance of the steady-state distribution of SGD, which respects the level of SGD noise and flatness of the loss.

\section{Unique information in the predictions}
\label{sec:functional-sample-information}

$\USI_\Sigma(z_i, Q)$ measures how much information an example $z_i$ provides to the weights.
Alternatively, instead of working in weight-space, we can approach the problem in function-space, and measure the informativeness of a training example for the network predictions.
Let $X \sim P_X$ be an independent test example and let $Q_{\widehat{Y} | S, X}$ denote the distribution of the prediction on $X$ after training on $S$.
Following the reasoning in \cref{sec:appendix-approximating-sample-info-with-kl}, we define functional sample information as
\begin{equation}
    \FSI(z_i, Q) = \E_{P_X}\KL{Q_{\widehat{Y}|S=s, X}}{Q_{\widehat{Y} | S=s_{-i}, X}}.
    \label{eq:fsi-truly-continuous}
\end{equation}
Again, when training with a discrete algorithm and/or when the output of the network is deterministic, the above quantity may be infinite.
Similar to smooth sample information, we define:
\begin{definition}[Smooth functional sample information]
Let $Q_{W|S}$ be a possibly stochastic training algorithm and let $Q_{\widehat{Y} | S, X}$ denote the distribution of the prediction on $X$ after training on $S$. We define the \emph{smooth functional sample information} ($\FSI$) as:
\begin{equation}
    \boxed{\FSI_\sigma(z_i, Q) = \E_{P_X}\KL{Q^\sigma_{\widehat{Y}|S=s, X}}{Q^\sigma_{\widehat{Y} | S=s_{-i}, X}}},
    \label{eq:fsi-smoothed}
\end{equation}
where $Q^\sigma_{\widehat{Y}|S, X}$ is the distribution of $\widehat{Y} + \xi$ with $\widehat{Y} \sim Q_{\widehat{Y} | S, X}$ and $\xi \sim \mathcal{N}(0, \sigma^2 I)$.
\end{definition}

The following proposition shows how to compute the value of $\FSI_\Sigma$ when $Q_{W|S}$ is deterministic.

\begin{proposition}\label{prop:smooth-functional-SI-deterministic}
Let $Q_{W|S}$ be a deterministic training algorithm (i.e., a distribution that puts all the mass on a point $W = A(S)$). Let $w = A(s)$ and $w_{-i} = A(s_{-i})$ be the weights obtained training respectively with and without sample $z_i$. Then,
\begin{align}
\FSI_\sigma(z_i, Q) &= \frac{1}{2\sigma^2}  \E_{P_X}\lVert f_w (X) - f_{w_{-i}}(X) \rVert_2^2\label{eq:fsi-difference-of-predictions}.
\end{align}
\end{proposition}
As \cref{prop:smooth-SI-deterministic}, this proposition also follows directly from formula of KL divergence between two Gaussians.

By using first-order Taylor approximation of $f_w(x)$ with respect to $w$ and assuming that $\nabla_w f_w(x) \approx \nabla_{w_{-i}}f_{w_{-i}}(x)$, we can approximate smooth functional sample information as follows:
\begin{align}
\FSI_\sigma(z_i, A) &= \frac{1}{2\sigma^2}  \E_{P_X}\lVert f_w (X) - f_{w_{-i}}(X) \rVert_2^2\\
&\approx \frac{1}{2} (w-w_{-i})^T \E_{P_X} \left[ \nabla_w f_w(X) \nabla_w f_w(X)^T \right](w-w_{-i})\label{eq:fsi-fisher-approx-intermediate-step} \\
&=\frac{1}{2\sigma^2} (w-w_{-i})^T F(w) (w-w_{-i}),
\label{eq:fsi-fisher-approx}
\end{align}
with $F(w) = \E_{P_X} \left[ \nabla_w f_w(X) \nabla_w f_w(X)^T\right]$ being the Fisher information matrix of $Q^{\sigma=1}_{\widehat{Y} | S, X}$.
By comparing \cref{eq:si-stability-approximation} and \cref{eq:fsi-fisher-approx}, we see that the functional sample information is approximated by using the inverse of the Fisher information matrix to smooth the weight space. However, this smoothing is not isotropic as it depends on the point $w$.

\section{Exact solution for linearized networks}
\label{sec:linear-networks}
In this section, we derive a close-form expression for $\USI_\Sigma$ and $\FSI_\sigma$ using a linear approximation of the network around the initial weights. We show that this approximation can be computed efficiently and, as we validate empirically in \cref{sec:unique-info-experiments}, correlates well with the actual informativeness values. We also show that the covariance matrix of SGD's steady-state distribution is a canonical choice for the smoothing matrix $\Sigma$ of $\USI_\Sigma$.

\paragraph{Linearized Network.}
Linearized neural networks are a class of neural networks obtained by taking the first-order Taylor expansion of a DNN around the initial weights~\citep{jacot2018ntk, lee2019wide}:
\begin{equation}
f^\mathrm{lin}_w(x) \triangleq f_{w_0}(x) + \nabla_w f_w(x)^T|_{w=w_0} (w - w_0).
\end{equation}
These networks are linear with respect to their parameters $w$, but can be highly non-linear with respect to their input $x$.
Continuous-time gradient descent dynamics of linearized neural networks is
\begin{equation}
\dot w_t = -\eta \left(\nabla_{f_t^\mathrm{lin}(\bs{x})} \LL\right)^T \nabla_w f_0(\bs{x}),
\label{eq:cont-time-gd-linearized}
\end{equation}
where $\eta > 0$ is the learning rate.
One of the advantages of linearized neural networks is that the dynamics of continuous-time or discrete-time gradient descent can be written analytically if the loss function is the mean squared error (MSE).
In particular, for the continuous-time gradient descent of \cref{eq:cont-time-gd-linearized}, we have that \citep{lee2019wide}:
\begin{align}
w_t &= \nabla_w f_0(\bs{x}) \Theta_0^{-1} \left(I - e^{-\eta \Theta_0 t}\right) (f_0(\bs{x}) - \bs{y}),\label{eq:lin-weights-at-t} \\
f^\mathrm{lin}_t(x) &= f_0(x) + \Theta_0(x, \bs{x}) \Theta_0^{-1} \left(I - e^{-\eta \Theta_0 t}\right) (\bs{y} - f_0(\bs{x})), \label{eq:lin-preds-at-t}
\end{align}
where $\Theta_0 = \nabla_w f_0(\bs{x})^T\nabla_w f_0(\bs{x}) \in \mathbb{R}^{nk \times nk}$ is the Neural Tangent Kernel (NTK) \citep{jacot2018ntk, lee2019wide} and $\Theta_0(x, \bs{x}) = \nabla_w f_0(x)^T \nabla_w f_0(\bs{x})$.
Note that the NTK matrix will generally be invertible for overparametrized neural networks ($d \gg nk$).
Analogs of equations \cref{eq:lin-weights-at-t,eq:lin-preds-at-t} for discrete time can be derived by replacing $e^{-\eta t \Theta_0}$ with $(I - \eta \Theta_0)^t$.
The expressions for networks trained with weight decay is essentially the same (see \cref{sec:wd-linearized}).
Namely, when weight decay of form $\frac{\lambda}{2} \lVert w - w_0\rVert_2^2$ is added, the only change that happens to \cref{eq:lin-weights-at-t,eq:lin-preds-at-t} is that $\Theta_0$ gets replaced by $(\Theta_0 + \lambda I)$.
To keep the notation simple, we will use  $f_w(x)$ to indicate  $f^\mathrm{lin}_w(x)$ from now on.

\paragraph{Stochastic Gradient Descent.}
As mentioned in \cref{sec:smoothed-sample-information}, a popular alternative approach to make information quantities well-defined is to use continuous-time SGD~\citep{li2017stochastic, mandt2017stochastic}, which is defined as
\begingroup
\setlength{\abovedisplayskip}{3pt}
\setlength{\belowdisplayskip}{3pt}
\begin{align}
d w_t = -\eta \nabla_w \LL_w(w_t) dt + \eta \sqrt{\frac{1}{b} \Lambda(w_t)} dn(t),
\label{eq:SGD-SDE-plain}
\end{align}
\endgroup
where $\eta$ is the learning rate, $b$ is the batch size, $n(t)$ is a Brownian motion, and $\Lambda(w_t)$ is the covariance matrix of the per-sample gradients (see \cref{sec:sgd-noise-cov} for details).
Let $Q^\mathrm{SGD}$ be the algorithm that returns a random sample from the steady-state distribution of \cref{eq:SGD-SDE-plain}, and let $Q^\mathrm{ERM}$ be the deterministic algorithm that returns the global minimum $w^*$ of the loss $\LL(w)$ (for a regularized linearized network $\LL(w)$ is strictly convex). 
We now show that the non-smooth sample information $\USI(z_i, Q^\mathrm{SGD})$ is the same as the smooth sample information using SGD's steady-state covariance as the smoothing matrix and $Q^\mathrm{ERM}$ as the training algorithm.

\begin{proposition}\label{prop:equality-of-non-smooth-and-smooth-SIs}
\begingroup
\setlength{\abovedisplayskip}{3pt}
\setlength{\belowdisplayskip}{3pt}
Let the loss function be regularized MSE, $w^*$ be the global minimum of it, and algorithms $Q^\mathrm{SGD}$ and $Q^\mathrm{ERM}$ be defined as above. Assuming $\Lambda(w)$ is approximately constant around $w^*$ and SGD's steady-state covariance remains constant after removing an example, we have
\begin{align}
\USI(z_i, Q^\mathrm{SGD}) = \USI_\Sigma(z_i, Q^\mathrm{ERM}) = \frac{1}{2} (w^* - w^*_{-i})^T \Sigma^{-1} (w^* - w^*_{-i}),
\label{eq:equality-of-non-smooth-and-smooth-SIs}
\end{align}
where $\Sigma$ is the solution of
\begin{align}
    H \Sigma + \Sigma H^T &= \frac{\eta}{b}\Lambda(w^*), \label{eq:lyapunov}
\end{align}
with $H = (\nabla_w f_0(\bs{x}) \nabla_w f_0(\bs{x})^T + \lambda I)$ being the Hessian of the loss function.
\endgroup
\end{proposition}

This proposition motivates the use of SGD's steady-state covariance as a smoothing matrix.
From \cref{eq:equality-of-non-smooth-and-smooth-SIs,eq:lyapunov} we see that SGD's steady-state covariance is proportional to the flatness of the loss at the minimum, the learning rate, and to SGD's noise, while inversely proportional to the batch size.
When $H$ is positive definite, as in our case when using weight decay, the continuous Lyapunov equation (\ref{eq:lyapunov}) has a unique solution,
which can be found in $O(d^3)$ time using the Bartels-Stewart algorithm~\citep{bartles1972}.
One particular case when the solution can be found analytically is when $\Lambda(w^*)$ and $H$ commute, in which case $\Sigma = \frac{\eta}{2b} \Lambda H^{-1}$.
For example, this is the case for Langevin dynamics, for which $\Lambda(w) = \sigma^2 I$ in \cref{eq:SGD-SDE-plain}.
In this case, we have
\begin{equation}
    \USI(z_i, Q^\mathrm{SGD}) = \USI_\Sigma(z_i, Q^\mathrm{ERM}) = \frac{b}{\eta \sigma^2} (w^* - w^*_{-i})^T H (w^* - w^*_{-i}),
\label{eq:weights-stability-diagonal-noise}
\end{equation}
which was already suggested by \citet{cook1977detection} as a way to measure the importance of a datum in linear regression.

\paragraph{Functional Sample Information.}
The definition in \Cref{sec:functional-sample-information}  simplifies for linearized neural networks: The step from \cref{eq:fsi-difference-of-predictions} to \cref{eq:fsi-fisher-approx} becomes exact,  and the Fisher information matrix becomes independent of $w$ and equal to $F = \E_{P_X}\left[ \nabla_w f_0(X) \nabla_w f_0(X)^T\right]$.
This shows that functional sample information can be seen as weight sample information with smoothing covariance $\Sigma=F^{-1}$.
The functional sample information depends on the data distribution $P_X$, which is usually unknown.
We can estimate $\FSI$ using a test set:
\begin{align}
    \FSI_\sigma(z_i, Q) &\approx \frac{1}{2 \sigma^2  n_\mathrm{test}} \sum_{j=1}^{n_\mathrm{test}} \left\lVert f_w(x_j^\mathrm{test}) - f_{w_{-i}}(x_j^\mathrm{test})\right\rVert\label{eq:fsi-pred-diff-empirical}\\
    &= \frac{1}{2 \sigma^2 n_\mathrm{test}} (w-w_{-i})^T \nabla f_0(\bs{x}_\mathrm{test}) \nabla f_0(\bs{x}_\mathrm{test})^T (w-w_{-i})\\
    &= \frac{1}{2 \sigma^2 n_\mathrm{test}} (w-w_{-i})^T (H_\mathrm{test}-\lambda I) (w-w_{-i}).\label{eq:fsi-empirical}
\end{align}
It is instructive to compare the sample weight information of \cref{eq:weights-stability-diagonal-noise} and functional sample information of \cref{eq:fsi-empirical}.
Besides the constants, the former uses the Hessian of the training loss, while the latter uses the Hessian of the test loss (without the $\ell_2$ regularization term).
One advantage of the latter is computational cost: As demonstrated in the next section,
we can use \cref{eq:fsi-pred-diff-empirical} to compute the prediction information, entirely in the function space, without any costly operation on weights.
For this reason, we focus on the linearized F-SI approximation in our experiments.
Since $\sigma^{-2}$ is just a multiplicative factor in \cref{eq:fsi-empirical} we set $\sigma=1$.
We also focus on the case where the training algorithm $A$ is discrete gradient descent running for $t$ epochs (\cref{eq:lin-weights-at-t,eq:lin-preds-at-t}).

\paragraph{Efficient Implementation.}
To compute the proposed sample information measures for linearized neural networks, we need to compute the change in weights $w - w_{-i}$ (or change in predictions $f_w(x) - f_{w_i}(x)$) after excluding an example from the training set.
This can be done without retraining using the analytical expressions of weight and prediction dynamics of linearized neural networks \cref{eq:lin-weights-at-t,eq:lin-preds-at-t}, which also work when the algorithm has not yet converged ($t<\infty$).
We now describe a series of measures to make the problem tractable.  First, to compute the NTK matrix we would need to store the Jacobian $\nabla f_0(x_i)$ of all training points and compute $\nabla_w f_0(\bs{x})^T \nabla_w f_0(\bs{x})$.
This is prohibitively slow and memory consuming for large DNNs.
Instead, similarly to \citet{zancato2020predicting}, we use low-dimensional random projections of per-example Jacobians to obtain provably good approximations of dot products~\citep{achlioptas2003database, li2006very}.
We found that just taking $2000$ random weights coordinates per layer provides a good enough approximation of the NTK matrix.
Importantly, we consider each layer separately, as different layers may have different gradient magnitudes.
With this method, computing the NTK matrix takes $O(nkd + n^2k^2 d_0)$ time, where $d_0 \approx 10^4$ is the number of subsampled weight indices ($d_0 \ll d$).
% The Jacobian $\nabla_w f_0(X)$ is also used in the weight prediction \cref{eq:lin-weights-at-t}.
% This can be done efficiently by noticing that $\frac{\partial f_w(x)}{\partial w} v = \frac{\partial(f_w(x)^T v)}{\partial w}$.
We also need to recompute $\Theta_0^{-1}$ after removing an example from the training set. This can be done in quadratic time by using rank-one updates of the inverse (see \cref{sec:inverse-update}).
Finally, when $t \neq \infty$ we need to recompute $e^{-\eta \Theta_0 t}$ after removing an example. This can be done in $O(n^2 k^2)$ time by downdating the eigendecomposition of $\Theta_0$~\citep{gu1995downdating}.
Overall, the complexity of computing $w-w_i$ for all training examples is $O(n^2 k^2 d_0 + n (n^2 k^2 + C))$, $C$ is the complexity of a single pass over the training dataset.
The complexity of computing functional sample information for $m$ test samples is $O(C + n m k^2 d_0 + n (mnk^2 + n^2 k^2))$.
This depends on the network size lightly, only through $C$.

%% file: unique-info/sections/experiments.tex
In this section, we test the validity of linearized network approximation in terms of estimating the effects of removing an example and show several applications of the proposed information measures.

\subsection{Accuracy of the linearized network approximation}
\input{unique-info/tables/ground-truth-correlations}

\input{unique-info/tables/ground-truth-details}

We measure $\lVert w - w_{-i} \rVert_2^2$ and $\E_\rVert f_w(X_\mathrm{test}) - f_{w_{-i}}(X_\mathrm{test})\rVert_2^2$ for each sample $z_i$ by training with and without that example. Then, instead of retraining, we use the efficient linearized approximation to estimate the same quantities and measure their correlation with the ground-truth values (\cref{tab:ground-truth-correlations}). For comparison, we also estimate these quantities using influence functions~\citep{koh2017understanding}. We consider two classification tasks: (a) a toy MNIST 4 vs 9 classification task and (b) Kaggle cats vs dogs classification task \citep{dogsvscats}, both with 1000 training examples.
For MNIST we consider a fully connected network with a single hidden layer of 1024 ReLU units (MLP) and a small 5-layer convolutional neural network (the same as in \cref{tab:4-layer-cnn-arch} but with one output unit), either trained from scratch or pretrained on EMNIST letters~\citep{cohen2017emnist}.
For cats vs dogs classification, we consider  ResNet-18 and ResNet-50 networks~\citep{he2016deep} pretrained on ImageNet. In both tasks, we train both with and without weight decay  ($\ell_2$ regularization).
In all our experiments, when using pretrained ResNets, we disable the exponential averaging of batch statistics in batch norm layers.
For computing functional sample information we use tests sets consisting of 1000 examples for both MNIST and cats vs dogs.
The exact details of running influence functions and linearized neural network predictions are presented in \cref{tab:ground-truth-details}.

The results in \cref{tab:ground-truth-correlations} shows that linearized approximation correlates well with ground-truth when the network is wide enough (MLP) and/or pretraining is used (CNN with pretraining and pretrained ResNets).
This is expected, as wider networks can be approximated with linearized ones better~\citep{lee2019wide}, and pretraining decreases the distance from initialization, making the Taylor approximation more accurate.
Adding regularization also keeps the solution close to initialization, and generally increases the accuracy of the approximation.
Furthermore, in most cases linearization gives better results compared to influence functions, while also being around 30 times faster in our settings.

\subsection{Which examples are informative?}
\begin{figure}
    \centering
    \includegraphics[width=\textwidth]{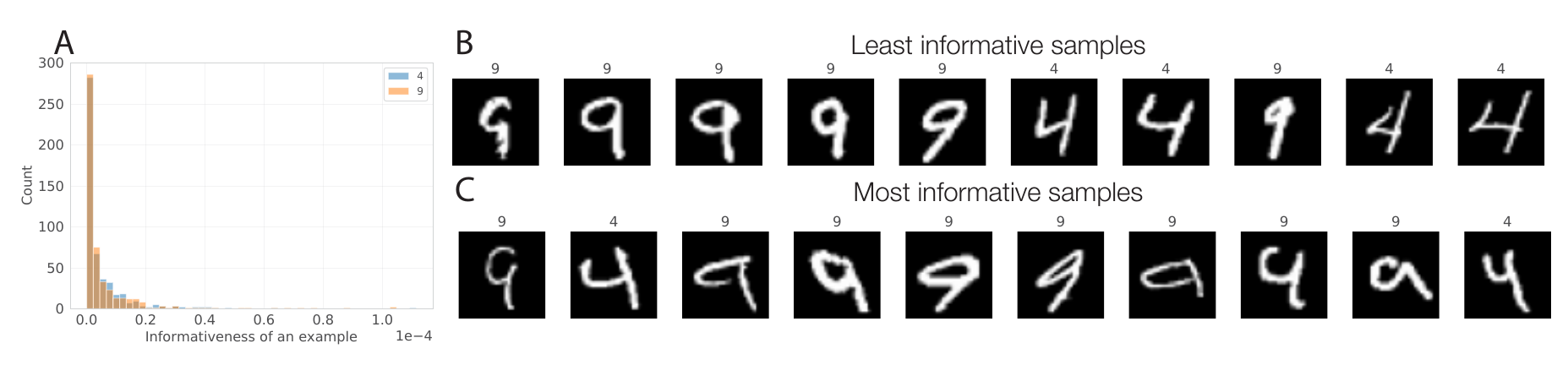}
    \includegraphics[width=\textwidth]{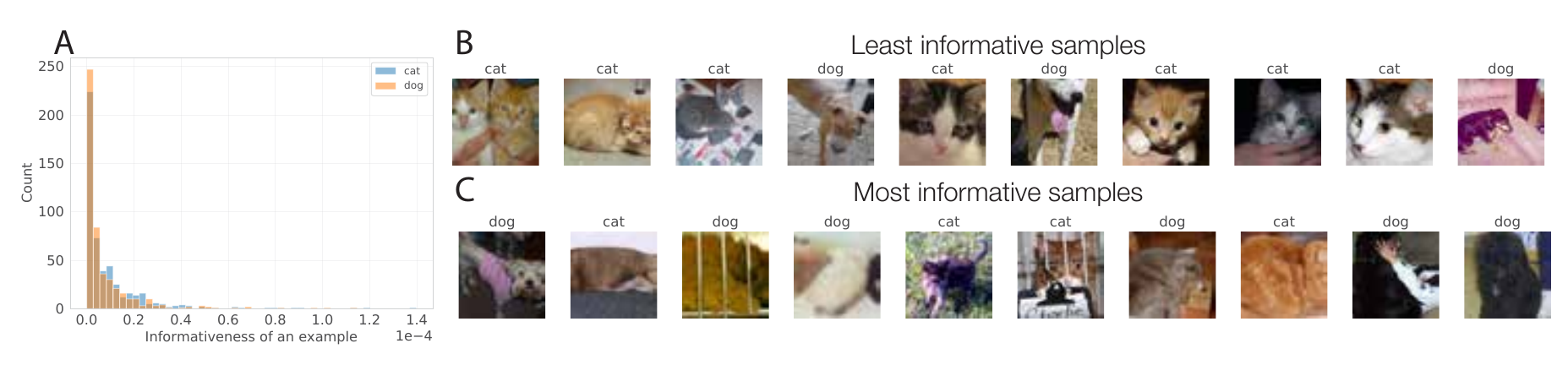}
    \includegraphics[width=\textwidth]{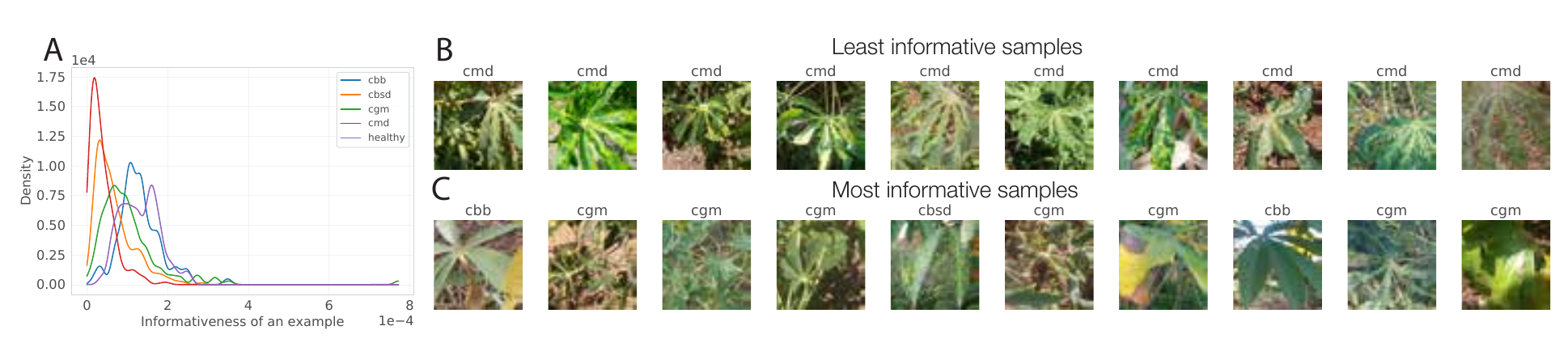}
    \caption{Functional sample information of samples in MNIST 4 vs 9 classification task (top), Kaggle cats vs dogs (middle), and iCassava (bottom) classification tasks.
    \textbf{A}: histogram of sample informations, \textbf{B}: 10 least informative samples, \textbf{C}: 10 most informative samples.}
    \label{fig:informative-examples-summary}
\end{figure}

To understand which examples are informative, we start with analyzing the top 10 least and most informative examples in MNIST 4 vs 9, Kaggle cats vs dogs, and iCassava plant disease classification~\citep{mwebaze2019icassava} tasks.
For MNIST 4 vs 9, we use the MLP of the previous subsection, setting $t=2000$ and $\eta=0.001$.
For cats vs dogs classification we use ResNet-18 and set $t=1000$, $\eta=0.001$.
Finally, for iCassava with subsample a training set of 1000 examples, use an ImageNet-pretrained ResNet-18 and set $t=5000, \eta=0.0003$.

The results presented in \cref{fig:informative-examples-summary} indicate that in all datasets majority of the examples are not highly informative.
Furthermore, especially in the cases of MNIST 4 vs 9 and Kaggle cats vs dogs, we see that the least informative samples look typical and easy, while the most informative ones look more challenging and atypical.
In the case of iCassava, the informative samples are more zoomed on features that are important for classification (e.g., the plant disease spots).
We observe that most samples have small unique information, possibly because they are easier or because the dataset may have many similar-looking examples.
While in the case of MNIST 4 vs 9 and cats vs dogs, the two classes have on average similar information scores, in Fig.~1a we see that in iCassava examples from rare classes (such as `healthy' and `cbb') are on average more informative.
For all tasks, it is true that most of the examples are relatively not informative.

\paragraph{Mislabeled examples.}
\begin{figure}[t]
    \centering
    \begin{subfigure}{0.45\textwidth}
    \includegraphics[width=1.0\textwidth]{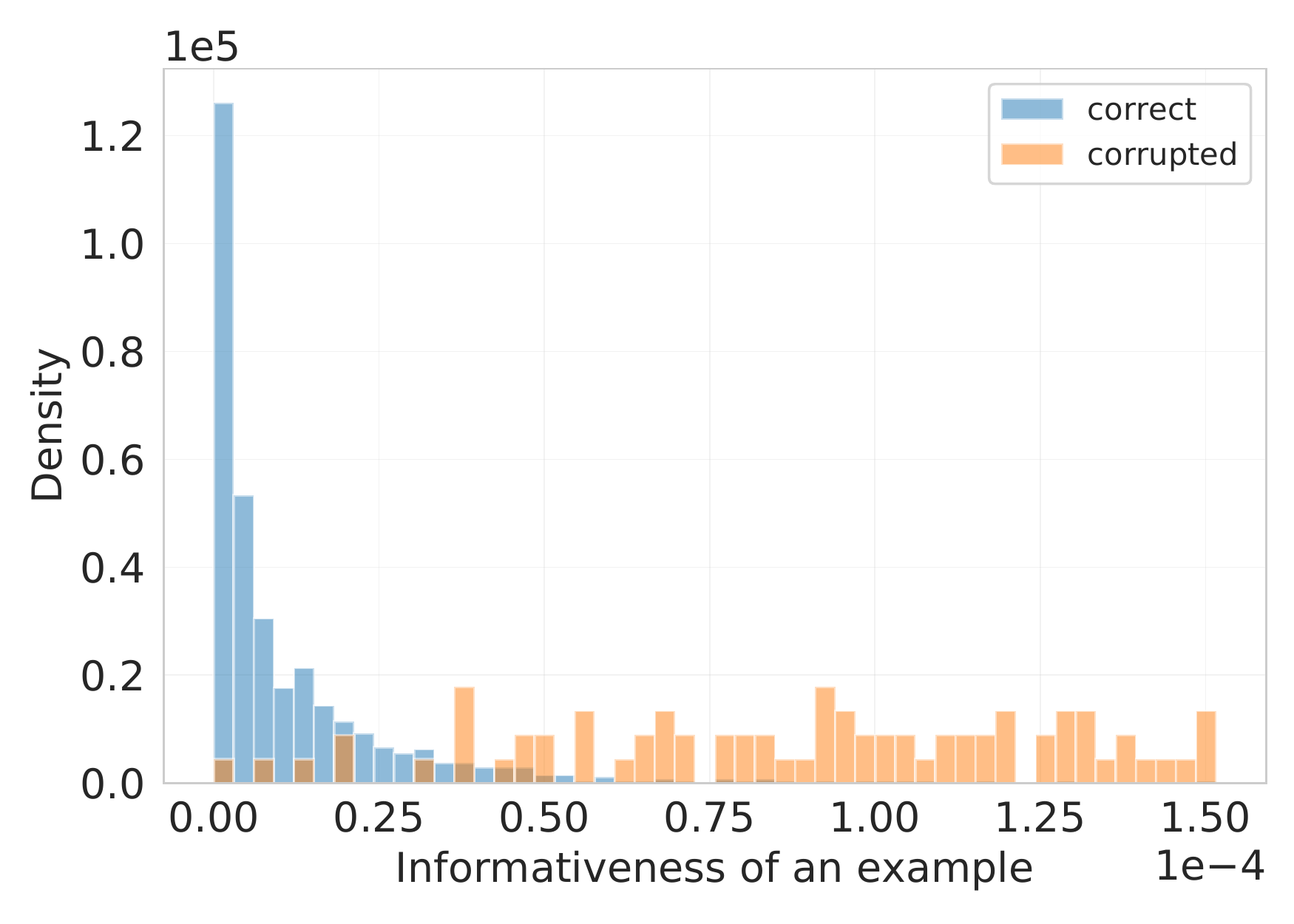}
    \caption{Kaggle cats vs dogs}
    \end{subfigure}
    \hspace{1em}
    \begin{subfigure}{0.45\textwidth}
    \includegraphics[width=1.0\textwidth]{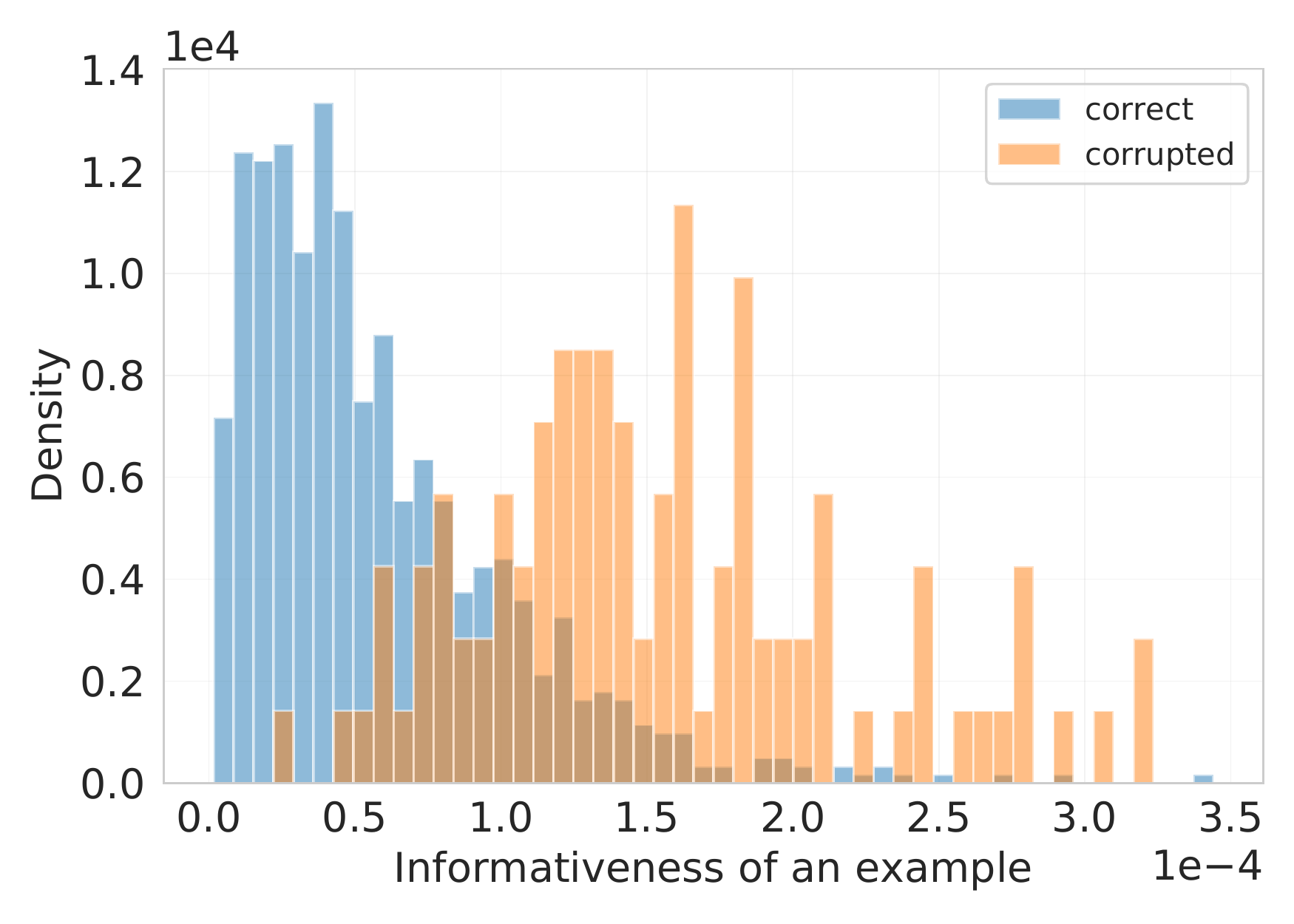}
    \caption{iCassava}
    \end{subfigure}%
    \caption{Comparison of functional sample information of examples with correct and incorrect labels.}
    \label{fig:unique-info-detecting-mislabeled}
\end{figure}
We expect a mislabeled example to carry more unique information, since the network needs to memorize unique features of that particular example to classify it.
To test this, we add 10\% uniform label noise to Kaggle cats vs dogs and iCassava classification tasks (both with 1000 examples in total), while keeping the test sets clean.
For both datasets with use a ResNet-19 pretrained on ImageNet.
We set $t=1000, \eta=0.001$ for cats vs dogs classification and 
$t=10000 \eta = 0.0001$ for the iCassava task.
\cref{fig:unique-info-detecting-mislabeled} presents the histogram of functional sample information for both correct and mislabeled examples for cats vs dogs and iCassava, respectively.
The results indicate that mislabeled examples are indeed much more informative on average.
This suggests that the proposed informativeness measure can be used to detect outliers or corrupted examples.

\paragraph{Harder examples.}
\begin{figure}[t]
  \centering
  \begin{minipage}[b]{0.45\linewidth}
    \centering
    \includegraphics[width=\textwidth]{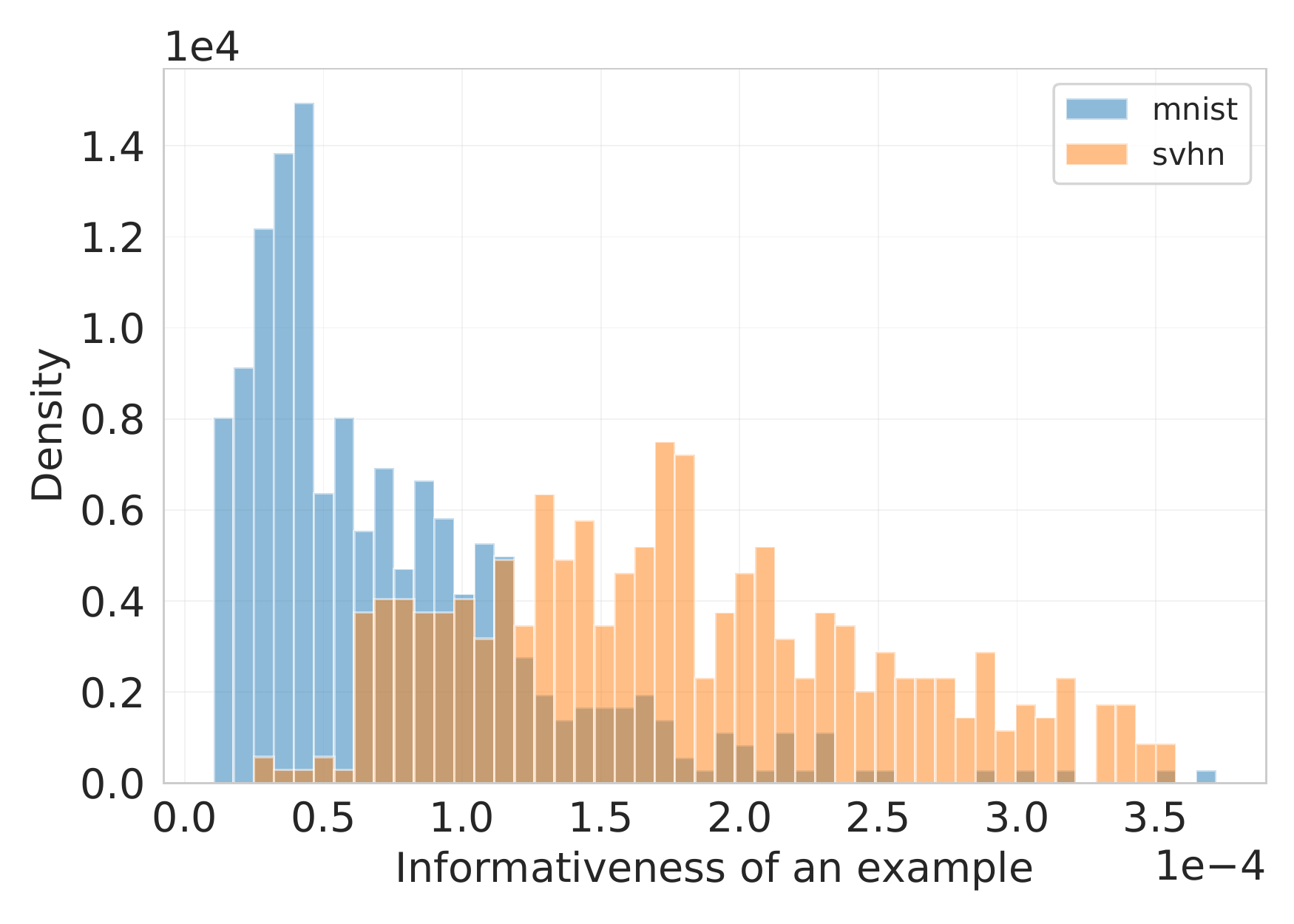}
    \caption{Comparison of functional sample information of MNIST and SVHN examples in the context of a joint digit classification task with equal number of examples per dataset.}
    \label{fig:mnist-svhn}
  \end{minipage}
  \hspace{1em}
  \begin{minipage}[b]{0.45\linewidth}
    \centering
    \includegraphics[width=\textwidth]{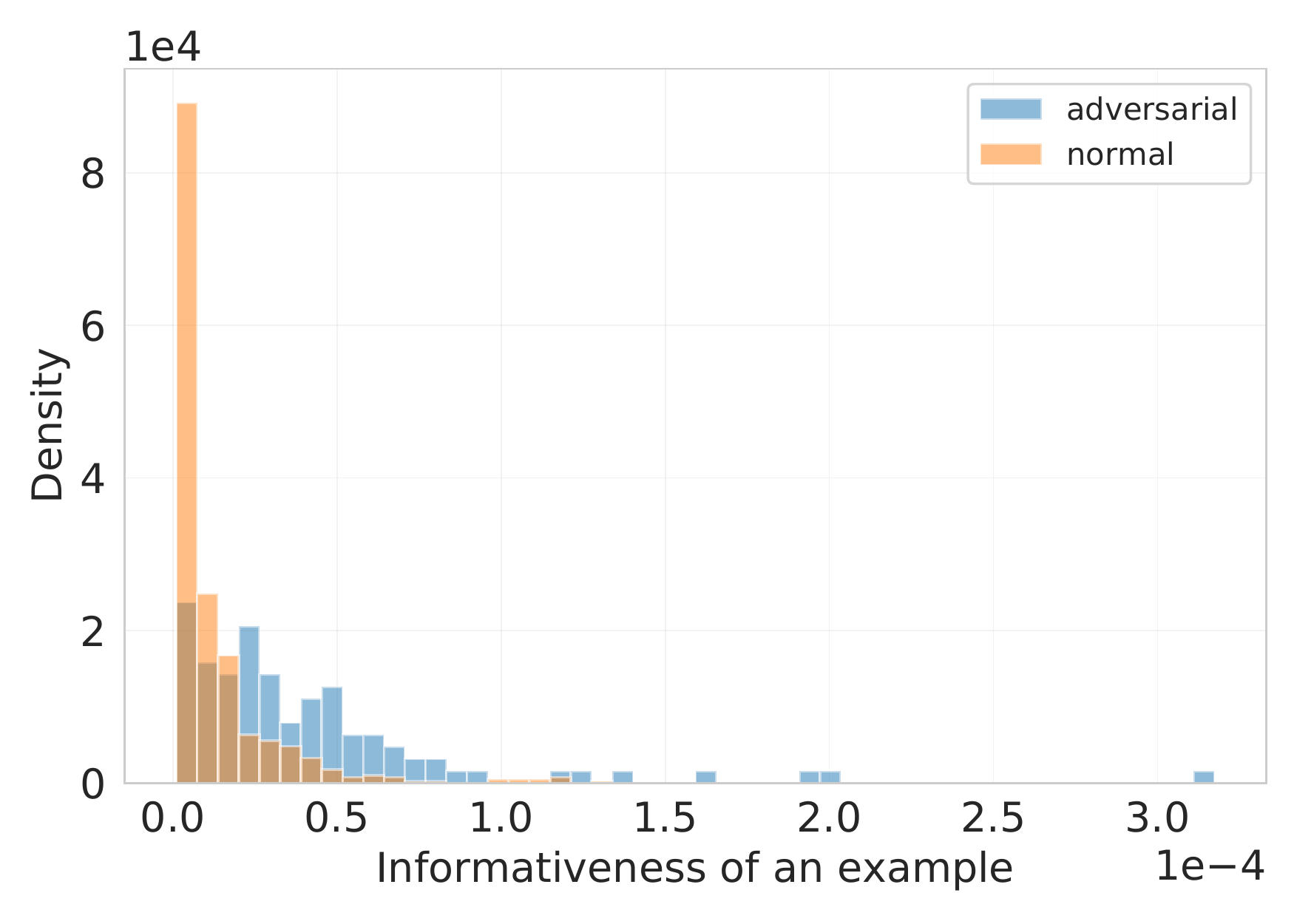}
    \caption{Histogram of the functional sample information of samples of the Kaggle cats vs dogs classification task, where 10\% of examples are adversarial.}
    \label{fig:adversarial}
  \end{minipage}
\end{figure}
After examining the top 10 most and least informative examples in \cref{fig:informative-examples-summary}, we hypothesize that ``more challenging'' examples are more informative.
To test this hypothesis, we train a 10-way digit classification on a dataset consisting of randomly selected 500 examples from MNIST and randomly selected 500 examples from SVHN.
Examples from SVHN are colored and have more variations.
Therefore, for the purpose of this experiment we consider them harder examples.
We use an ImageNet-pretrained ResNet-18 and set $t=20000,\eta=0.0001$.
We compute functional sample information of training examples with respect to a test set consisting of 500 MNIST and 500 SVHN unseen examples.
The results presented in \cref{fig:mnist-svhn} demonstrate that, as expected, SVHN examples are more informative than MNIST examples.

\paragraph{Adversarial examples.}
It is also natural to expect that training examples closer to the decision boundary of a neural network are have probably been informative in the training of that same neural network.
Perhaps such examples can be also informative in training of other neural networks.
To test this hypothesis, we consider creating adversarial examples~\citep{adv-examples} for one neural network and measuring their informativeness with respect to another neural network.
In particular, we fine-tune a pretrained ResNet-18 on 1000 examples from the cats vs dogs dataset.
We then use the FGSM method \citep{adv-training} with $\epsilon = 0.01$ to create successful adversarial examples.
Next, for the 10\% of examples, we replace the original images with the corresponding adversarial ones (keeping the original correct label), and fine-tune a new pretrained ResNet-18.
Finally, we compute functional sample information of all examples of the modified training set, setting $t=1000$ and $\eta=0.001$.
The results reported in \cref{fig:adversarial} confirm that adversarial examples are on average more informative.
This partly explains the findings of \citet{adv-training} who show that adversarial training (i.e., adding adversarial examples to the training dataset) improves adversarial robustness and generalization.

\begin{figure}[t]
  \centering
  \begin{minipage}[b]{0.45\linewidth}
    \centering
    \includegraphics[width=\textwidth]{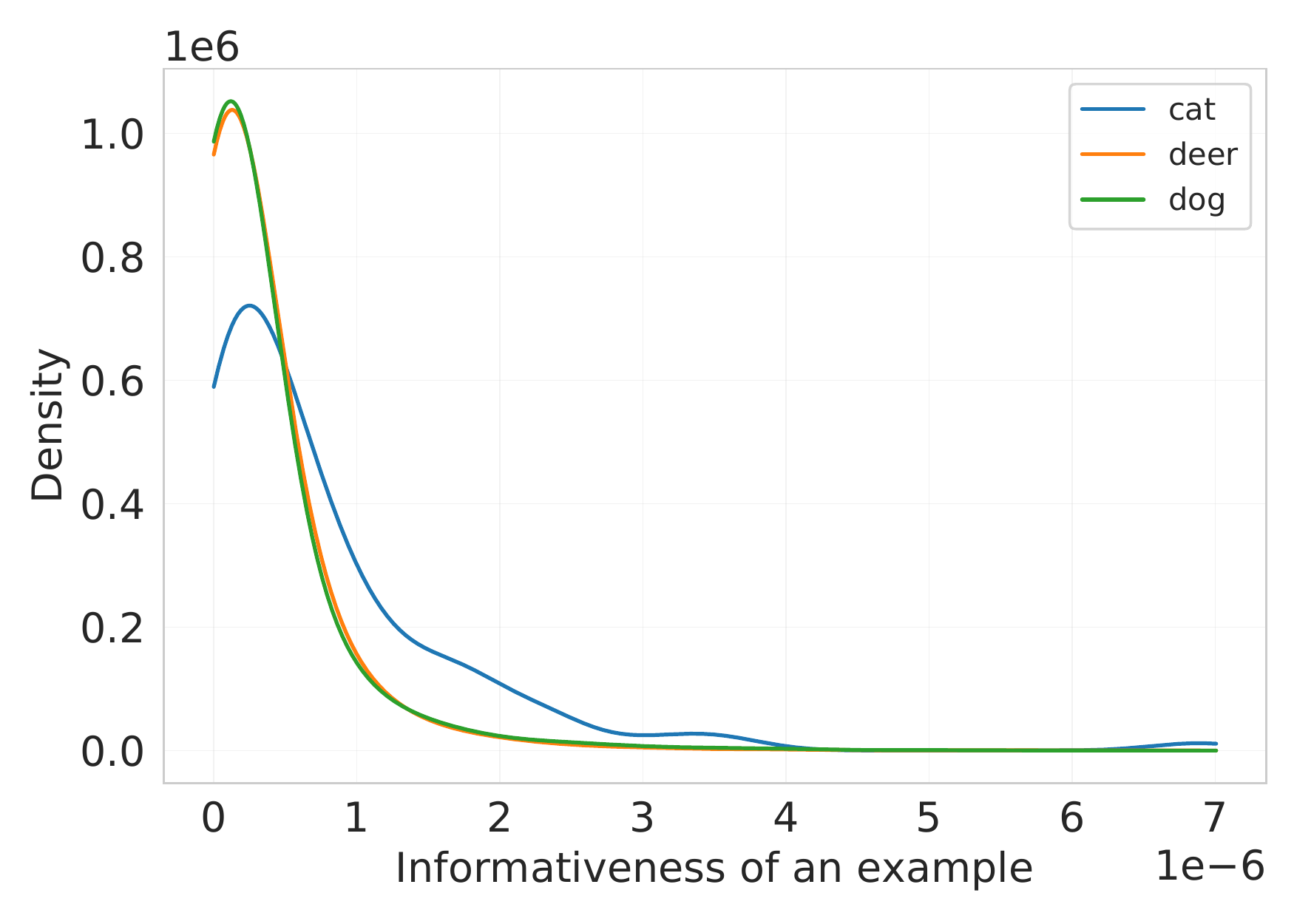}
    \caption{Histogram of the functional sample information of examples of the 3 subpopulations of the pets vs deer dataset. Since cats are under-represented, cat images tend to be more informative on average compared to dog images.}
    \label{fig:cat-dog-deer}
  \end{minipage}
  \hspace{1em}
  \begin{minipage}[b]{0.45\linewidth}
    \centering
    \includegraphics[width=\textwidth]{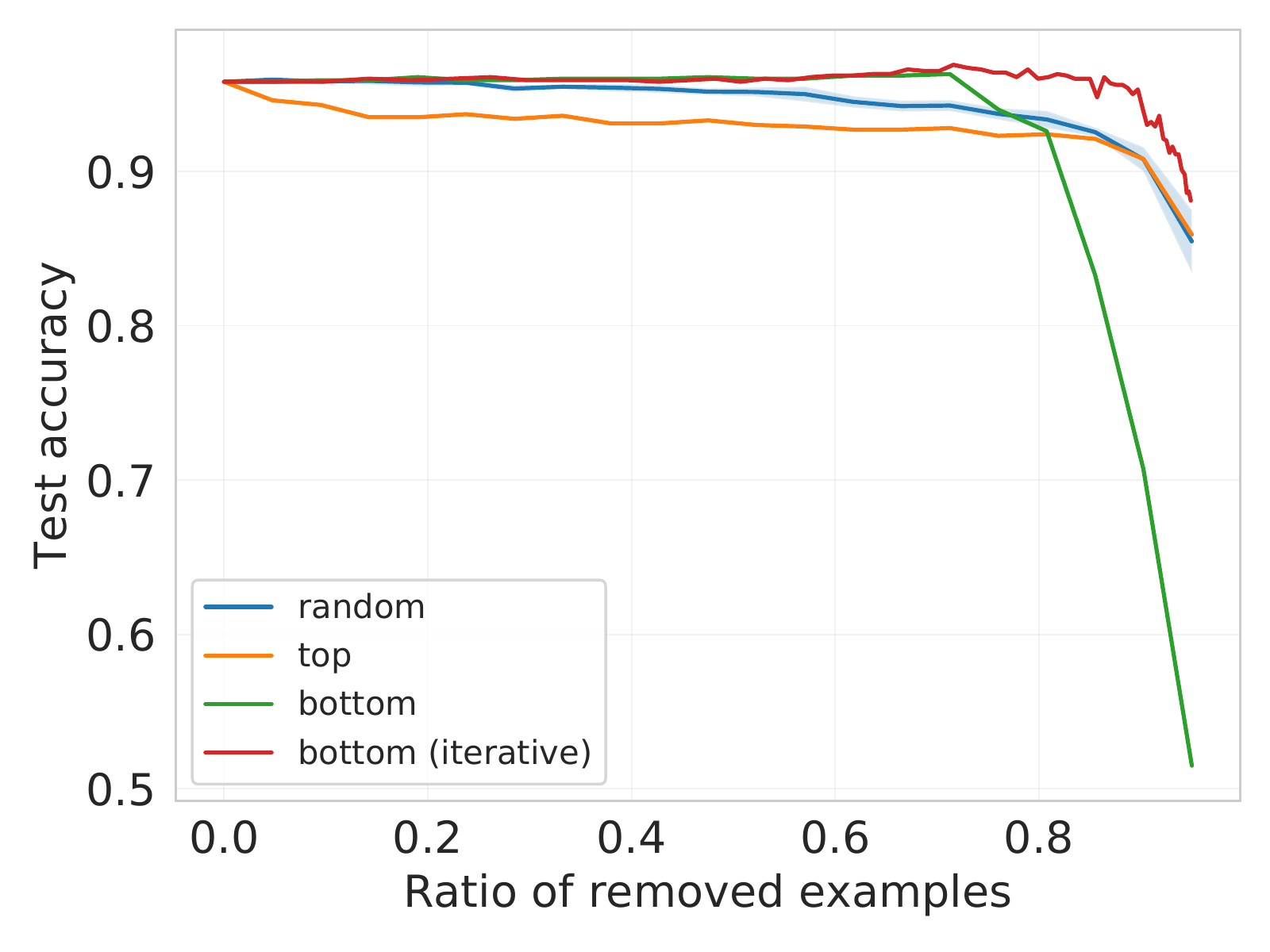}
    \caption{Dataset summarization without training. Test accuracy as a function of the ratio of removed training examples for different strategies.}
    \label{fig:mnist-data-sum}
  \end{minipage}
\end{figure}
\paragraph{Underrepresented subpopulations.}
The histogram of functional sample information in the case of iCassava (\cref{fig:informative-examples-summary} bottom) suggests that examples of underrepresented classes or subpopulations might be more informative on average.
However, we cannot conclude this based on the results of \cref{fig:informative-examples-summary} as there can be confounding factors affecting both representation and informativeness.
Therefore, using CIFAR-10 images, we create a dataset for ``pets vs deer'' classification, where the class ``pets'' consists of two subpopulations: cats (200 examples) and dogs (4000 examples), while the class 'deer' consist of 4000 deer examples.
In this dataset, cats are underrepresented, but there is no other significant difference between the cat and dog subpopulations.
Since there are relatively few cat images, we expect each to carry more unique information. This is confirmed when we compute $\FSI$ for a pretrained ResNet-18 with $t=10000$ and $\eta=0.0001$ (see \cref{fig:cat-dog-deer}).
This result also suggests that analyzing $\FSI$ can help to detect underrepresented subpopulations.

\subsection{Data summarization}
As we saw in \cref{fig:informative-examples-summary}, majority of examples in considered datasets are have low sample information.
This hints that it might be possible to remove a significant fraction of examples from these datasets, while not degrading performance of networks trained on them.
To test whether such data summarization is possible, given a training dataset, we remove a fraction of its least informative examples and measure the test performance of a network trained on the remaining examples.
We expect that removing the least informative training samples should not affect the performance of the model.
Note however that, since we are considering the \textit{unique} information, removing one example can increase the informativeness of another.
For this reason, we consider two strategies: In one we compute the informativeness scores once, and remove a given percentage of the least informative samples. In the other we remove 5\% of the least informative samples, recompute the scores, and iterate until we remove the desired number of samples. For comparison, we also consider removing the most informative examples (we call this ``top'' baseline) and randomly selected examples (we call this ``random'' baseline).

The results on MNIST 4 vs 9 classification task with the one-hidden-layer network described earlier, $t=2000$, and $\eta=0.001$, are shown in \cref{fig:mnist-data-sum}.
Indeed, removing the least informative training samples has little effect on the test error, while removing the top examples has the most impact.
Moreover, recomputing the information scores after each removal steps (``bottom iterative'') greatly improves the performance when many samples are removed, confirming that $\USI$ and $\FSI$ are good practical measures of unique information of an example, but also that the total information of a group of examples is not simply the sum of the unique information of its members.
Interestingly, removing more than 80\% of the least informative examples degrades the performance more than removing the same number of the most informative examples.
In the former case, we are left with a small number of hard, atypical, or mislabeled examples, while in the latter case we are left with the same number of easy and typical examples.
Consequently, the performance is better in the latter case.
In other words, when learning a classifier with just 200 examples, it is better to have these examples be typical.

\subsection{How much does sample information depend on algorithm?}\label{sec:dependence-on-training-algo}
\begin{figure}[t]
\centering
\begin{subfigure}{0.49\textwidth}
\includegraphics[width=0.8\textwidth]{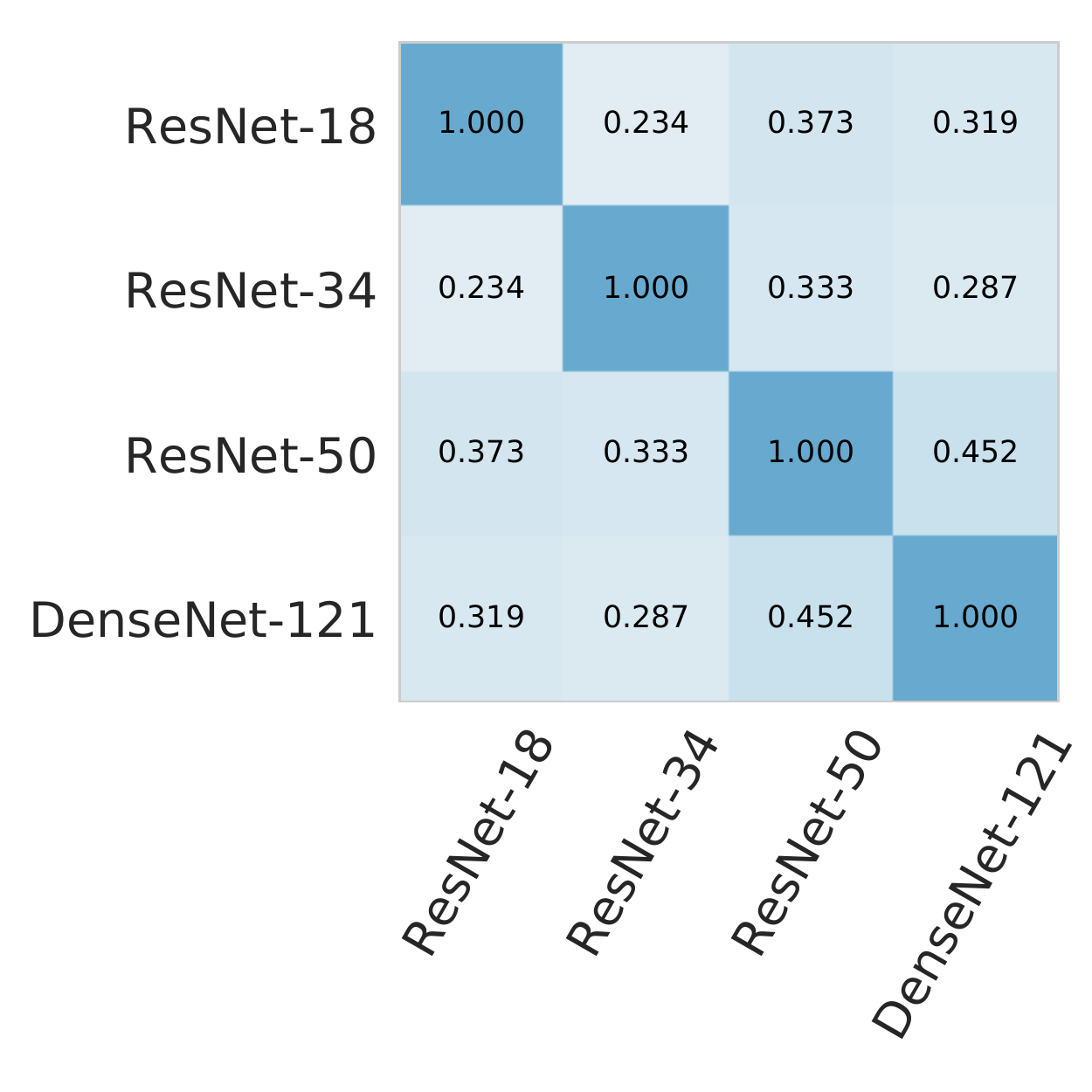}
\caption{Varying architecture}
\label{fig:inf-scores-for-different-networks}
\end{subfigure}%
\begin{subfigure}{0.49\textwidth}
\includegraphics[width=0.8\textwidth]{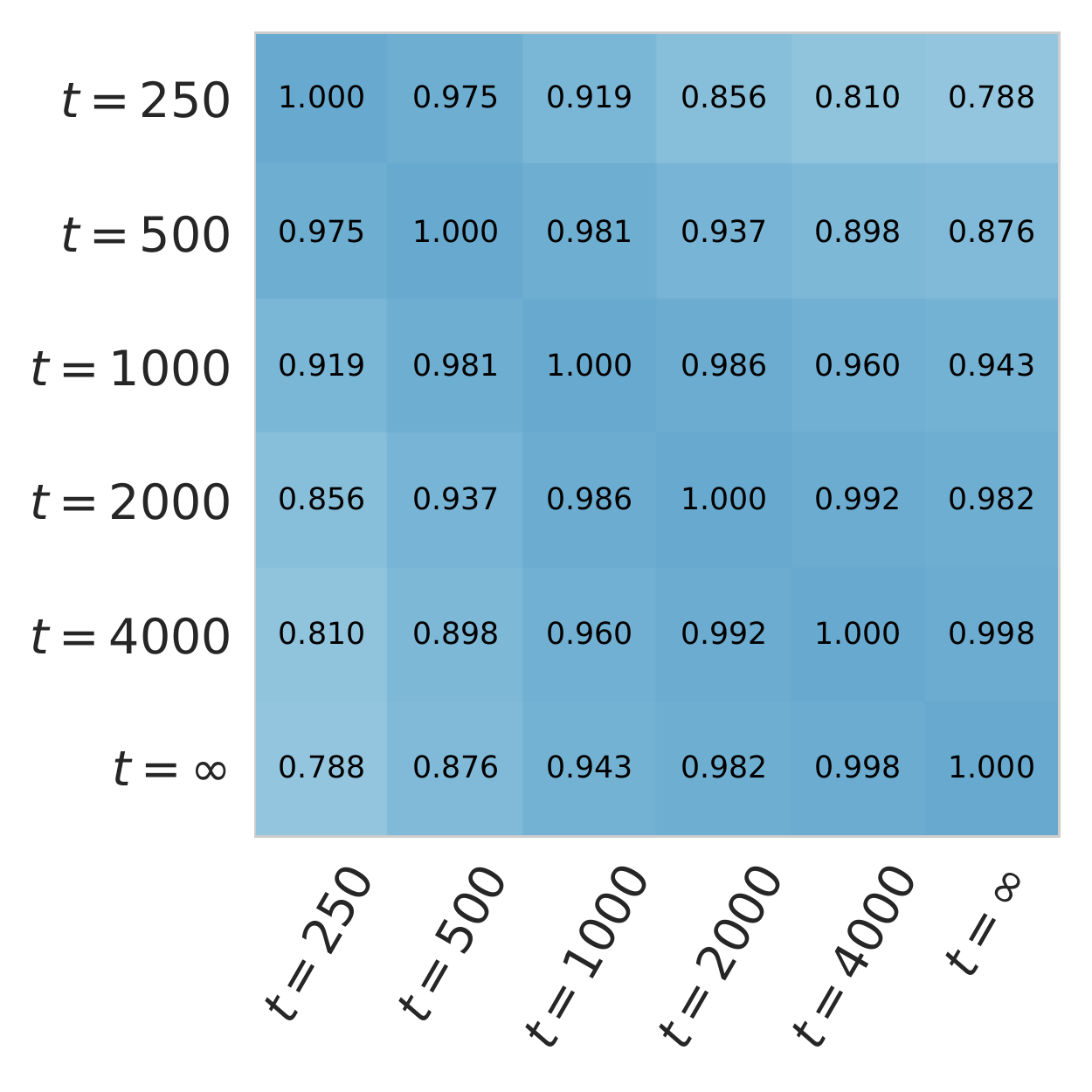}
\caption{Varying training length}
\label{fig:inf-scores-at-different-times}
\end{subfigure}
\caption{Correlations between functional sample information scores computed for different architectures and training lengths. \textbf{On the left}: correlations between F-SI scores of the 4 pretrained networks, all computed with setting $t=1000$ and $\eta=0.001$.
\textbf{On the right:} correlations between F-SI scores computed for pretrained ResNet-18s, with learning rate $\eta=0.001$, but varying training lengths $t$. All reported correlations are averages over 10 different runs. The training dataset consists of 1000 examples from the Kaggle cats vs dogs classification task.
}
\end{figure}
\begin{figure}
    \centering
    \begin{subfigure}{\textwidth}
    \includegraphics[width=\textwidth]{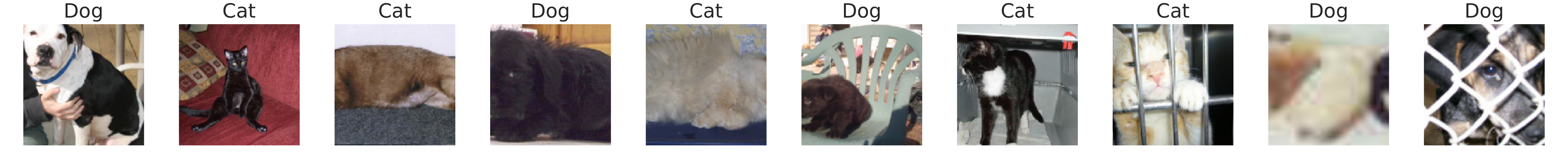}
    \caption{ResNet-18}
    \end{subfigure}
    
    \begin{subfigure}{\textwidth}
    \includegraphics[width=\textwidth]{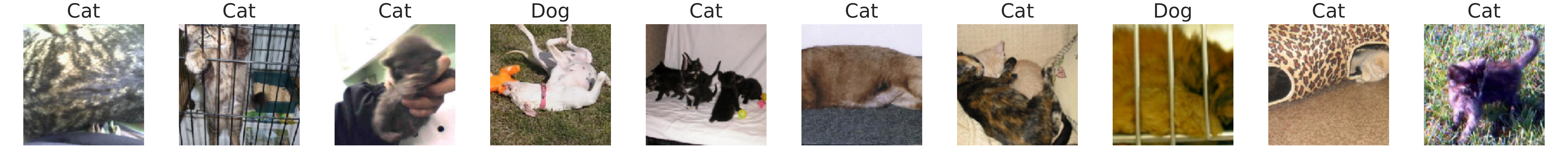}
    \caption{ResNet-50}
    \end{subfigure}
    
    \begin{subfigure}{\textwidth}
    \includegraphics[width=\textwidth]{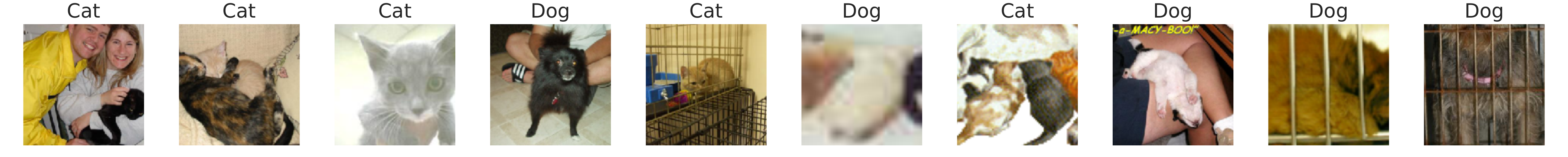}
    \caption{DenseNet-121}
    \end{subfigure}
    \caption{Top 10 most informative examples from Kaggle cats vs dogs classification task for 3 pretrained networks: ResNet-18, ResNet-50, and DenseNet-121.}
    \label{fig:most-informative-examples-for-different-networks}
\end{figure}
The proposed information measures depend on the training algorithm, which includes the architecture, random seed, initialization, and training length.
This is unavoidable as one example can be more informative for one algorithm and less informative for another.
Nevertheless, in this subsection, we test how much does informativeness depend on the network, initialization, and training time.
We consider the Kaggle cats vs dogs classification task with 1000 training examples.
First, fixing the training time $t=1000$, we consider four pretrained architectures: ResNet-18, ResNet-34, ResNet-50, and DenseNet-121.
The correlations between F-SI scores computed for these four architectures are presented in \cref{fig:inf-scores-for-different-networks}.
We see that F-SI scores computed for two completely different architectures, such as ResNet-50 and DenseNet-121 have significant correlation, around 45\%.
Furthermore, there is a significant overlap in top 10 most informative examples for these networks (see \cref{fig:most-informative-examples-for-different-networks}).
Next, fixing the network to be a pretrained ResNet-18 and fixing the training length $t=1000$, we consider changing initialization of the classification head (which is not pretrained).
In this case the correlation between F-SI scores is $0.364 \pm 0.066$.
Finally, fixing the network to be a ResNet-18 and fixing the initialization of the classification head, we consider changing the number of iterations in the training.
We find strong correlations between F-SI scores of different training lengths (see \cref{fig:inf-scores-at-different-times}).

\begin{figure}[t]
    \centering
    \begin{subfigure}{0.49\textwidth}
    \includegraphics[width=0.8\textwidth]{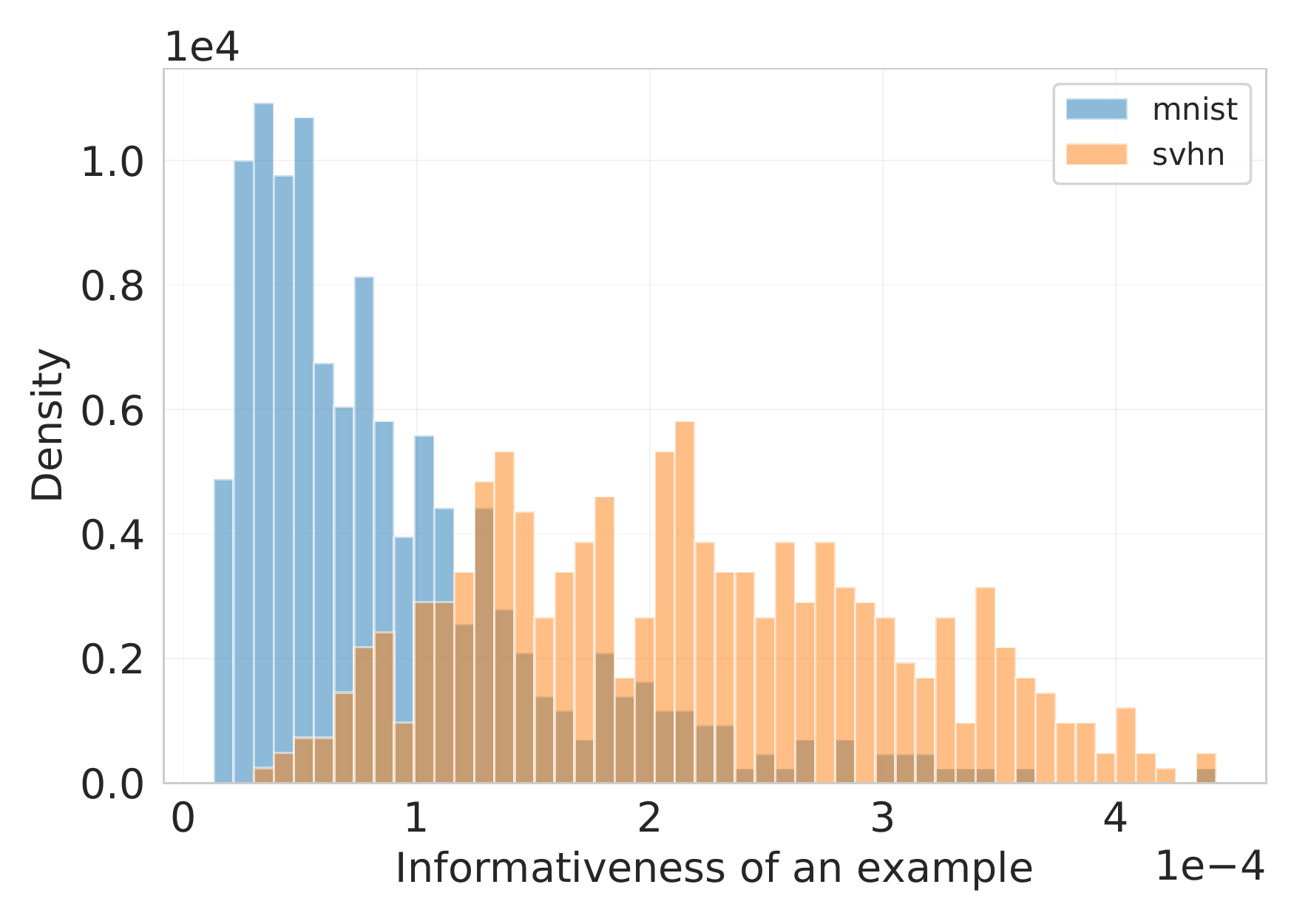}
    \caption{MNVT vs SVHN experiment for DenseNet-121}
    \label{fig:mnist-svhn-densenet121}
    \end{subfigure}%
    \begin{subfigure}{0.49\textwidth}
    \includegraphics[width=0.8\textwidth]{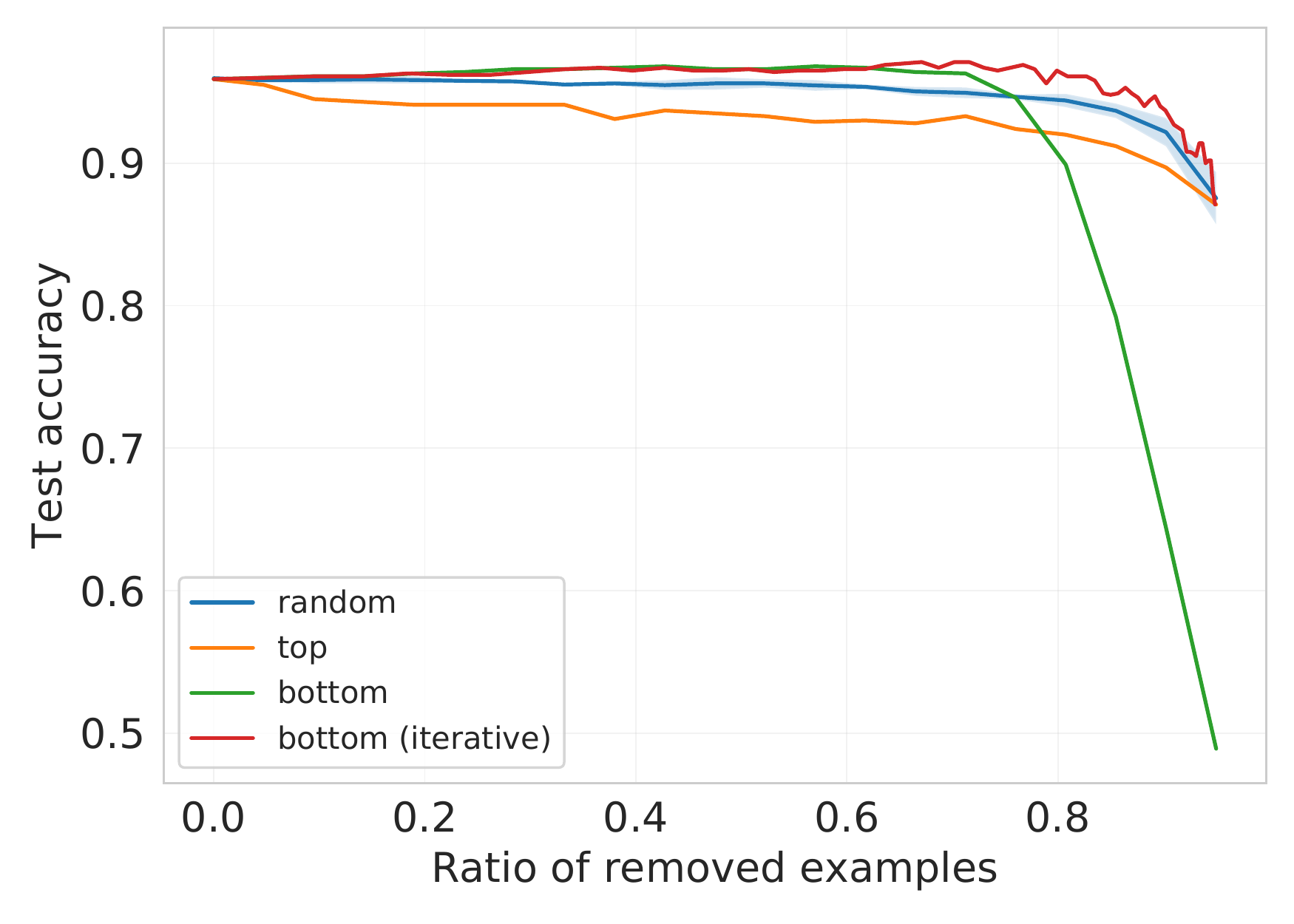}
    \caption{MNIST 4 vs 9 data summarization}
    \label{fig:mnist-data-sum-chaning-network}
    \end{subfigure}%
    \label{fig:dependence-of-results-on-network-choice}
    \caption{Testing how much F-SI scores computed for different networks are qualitatively different. \textbf{On the left:} the MNIST vs SVHN experiment with a pretrained DesneNet-121 instead of a pretrained ResNet-18.
    \textbf{On the right:} Data summarization for the MNIST 4 vs 9 classification task, where the F-SI scores are computed for a one-hidden-layer network, but a two-hidden-layer network is trained to produce the test accuracies.}
\end{figure}

\paragraph{MNIST vs SVHN experiment for DenseNet-121.} We also redo the experiment with the joint MNIST and SVHN classification experiment (\cref{fig:mnist-svhn}) but for a different network, a pretrained DenseNet-121, to test the dependence of the results on the architecture choice. The results are presented in \cref{fig:mnist-svhn-densenet121} and are qualitatively identical to those with a pretrained ResNet18 (\cref{fig:mnist-svhn}).

\paragraph{Data summarization with a change of architecture.} To test how much sample information scores computed for one network are useful for another network, we reconsider the MNIST 4 vs 9 data summarization experiment (\cref{fig:mnist-data-sum}). This time we compute F-SI scores for the original network with one hidden layer, but train a two-hidden-layer neural network (both layers having 1024 ReLU units). The data summarization results presented in \cref{fig:mnist-data-sum-chaning-network} are qualitatively and quantitively almost identical to the original results presented in \cref{fig:mnist-data-sum}. This confirms that F-Si scores computed for one network can be useful for another network.

%% file: unique-info/tables/ground-truth-correlations.tex
\begin{table}[t]
    \centering
    % \resizebox{.85\columnwidth}{!}{%
    \small
    \caption{Pearson correlations of weight change $\lVert w - w_{-i}\rVert_2^2$ and test prediction change $\lVert f_w(\bs{x}_\text{test}) - f_{w_{-i}}(\bs{x}_\text{test}) \lVert_2^2$ norms computed with influence functions and linearized neural networks with their corresponding measures computed for standard neural networks with retraining.}
    \begin{tabular}{lllccccc}
    \toprule
     & Reg. & Method & MNIST MLP & \multicolumn{2}{c}{MNIST CNN} & \multicolumn{2}{c}{Cats and Dogs}\\
    \cmidrule{4-6}\cmidrule{7-8}
     &  &  & scratch & scratch & pretrained & pr. ResNet-18 & pr. ResNet-50 \\
    \midrule
    \multirow{4}{*}{\rotatebox[origin=c]{90}{weights}} & \multirow{2}{*}{$\lambda=0$} & Linearization & \thesame{0.987} & 0.193 & \thesame{0.870} & \thesame{0.895} & \thesame{0.968}\\
    & & Infl. functions & 0.935 & \thesame{0.319} & 0.736 & 0.675 & 0.897\\
    \cmidrule{2-8}
    & \multirow{2}{*}{$\lambda=10^3$} & Linearization & \thesame{0.977} & -0.012 & 0.964 & \thesame{0.940} & 0.816\\
    & & Infl. functions & \thesame{0.978} & \thesame{0.069} & \thesame{0.979} & 0.858 & \thesame{0.912}\\
    \midrule
    \multirow{4}{*}{\rotatebox[origin=c]{90}{predictions}} & \multirow{2}{*}{$\lambda=0$} & Linearization & \thesame{0.993} & 0.033 & \thesame{0.875} & \thesame{0.877} & \thesame{0.895}\\
    & & Infl. functions & 0.920 & \thesame{0.647} & 0.770 & 0.530 & 0.715\\
    \cmidrule{2-8}
    & \multirow{2}{*}{$\lambda=10^3$} & Linearization & \thesame{0.993} & 0.070 & \thesame{0.974} & \thesame{0.931} &  \thesame{0.519}\\
    & & Infl. functions & \thesame{0.990} & \thesame{0.407} & 0.954 & 0.753 & 0.506 \\
    \bottomrule
    \end{tabular}
    % }
    \label{tab:ground-truth-correlations}
\end{table}

%% file: unique-info/tables/ground-truth-details.tex
\begin{table}[th]
    \centering
    % \resizebox{.8\columnwidth}{!}{%
    \small
    \caption{Details of experiments presented in \cref{tab:ground-truth-correlations}. For influence functions, we add a dumping term with magnitude 0.01 whenever $\ell_2$ regularization is not used ($\lambda = 0$).}
    \begin{tabular}{lll}
    \toprule
    Experiment & Method & Details\\
    \midrule
    \multirow{3}{*}{MNIST MLP (scratch)} & Brute force & 2000 epochs, learning rate $=$ 0.001, batch size $=$ 1000\\
    & Infl. functions & LiSSA algorithm, 1000 recursion steps, scale $=$ 1000 \\
    & Linearization & $t=2000$, learning rate $=$ 0.001\\
    \midrule
    \multirow{3}{*}{MNIST CNN (scratch)} & Brute force & 1000 epochs, learning rate $=$ 0.01, batch size $=$ 1000\\
    & Infl. functions & LiSSA algorithm, 1000 recursion steps, scale $=$ 1000 \\
    & Linearization & $t=1000$, learning rate $=$ 0.01\\
    \midrule
    \multirow{3}{*}{MNIST CNN (pretrained)} & Brute force & 1000 epochs, learning rate $=$ 0.002, batch size $=$ 1000\\
    & Infl. functions & LiSSA algorithm, 1000 recursion steps, scale $=$ 1000 \\
    & Linearization & $t=1000$, learning rate $=$ 0.002\\
    \midrule
    \multirow{3}{*}{Cats and dogs} & Brute force & 500 epochs, learning rate $=$ 0.001, batch size $=$ 500\\
    & Infl. functions & LiSSA algorithm, 50 recursion steps, scale $=$ 1000 \\
    & Linearization & $t=1000$, learning rate $=$ 0.001\\
    \bottomrule
    \end{tabular}
    % }
    \label{tab:ground-truth-details}
\end{table}

%% file: unique-info/sections/discussion.tex
The smooth (functional) sample information depends not only on the example itself, but on the network architecture, initialization, and the training procedure (i.e. the training algorithm).
This has to be the case, since an example can be informative with respect to one algorithm or architecture, but not informative for another one. Similarly, some examples may be more informative at the beginning of the training (e.g., simpler examples) rather than at the end.
Nevertheless, the results presented in \cref{sec:dependence-on-training-algo} indicate that $\FSI$ still captures something inherent in the example.
% In particular, as shown \cref{sec:dependence-on-training-algo}, $\FSI$ scores computed with respect to different architectures are significantly correlated to each other (e.g. around 45\% correlation in case of ResNet-50 and DenseNet-121). 
% Furthermore, $\FSI$ scores computed for different initializations are significantly correlated (e.g. around 36\%).
% Similarly, $\FSI$ scores computed for different training lengths are strongly correlated (see \cref{fig:inf-scores-at-different-times}).
This suggests that $\FSI$ computed with respect to one network can reveal useful information for another one.
% Indeed, this is verified in \cref{sec:dependence-on-training-algo}, where we redo the data summarization experiment, but with a slight change that $\FSI$ scores are computed for one network, but another network is trained to check the accuracy.
% As shown in the \cref{fig:mnist-data-sum-chaning-network} the results look qualitatively similar to those of the original experiment.

The proposed sample information measures only the unique information provided by an example.
For this reason, it is not surprising that typical examples are usually the least informative, while atypical and rare ones are more informative.
While the typical examples are usually less informative according to the proposed measures, they still provide information about the decision functions, which is evident in the data summarization experiment -- removing lots of typical examples was worse than removing the same number of random examples.
Generalizing sample information to capture this kind of contributions is an interesting direction for future work.
Similar to Data Shapley~\citep{pmlr-v97-ghorbani19c}), one can look at the average unique information an example provides when considering along with a random subset of data.
One can also consider common information, high-order information, synergistic information, and other notions of information \textit{between} samples. The relation among these quantities is complex in general, even for 3 variables and is an open challenge~\citep{williams-beer}.

%% file: unique-info/sections/related.tex
\section{Related work}
Our work is related to information-theoretic stability notions~\citep{bassily2016algorithmic, raginsky2016information, feldman2018calibrating} that seek to measure the influence of a sample on the output, and to measure generalization.
\citet{raginsky2016information} define information stability as $\E_S\left[ \frac{1}{n}\sum_{i=1}^n I(W; Z_i \mid S_{-i})\right]$,
the expected average amount of unique (Shannon) information that weights have about an example.
This, without the expectation over $S$, is also our starting point (\cref{eq:unique-info}).
\citet{bassily2016algorithmic} define KL-stability $\sup_{s,s'}\KL{Q_{W|S=s}}{Q_{W|S=s'}}$, where $s$ and $s'$ are datasets that differ by one example, while \citet{feldman2018calibrating} define average leave-one-out KL stability as $\sup_s \frac{1}{n}\sum_{i=1}^n \KL{Q_{W|S=s}}{Q_{W|S=s_{-i}}}$.
The latter closely resembles our definition (\cref{eq:si-definition}).
Unfortunately,
while the weights are continuous, the optimization algorithm (such as SGD) is usually discrete. This generally makes the resulting quantities degenerate (infinite).
Most works address this issue by replacing the discrete optimization algorithm with a continuous one, such as stochastic gradient Langevin dynamics \citep{welling2011bayesian} or continuous stochastic differential equations that approximate SGD \citep{li2017stochastic} in the limit.
We aim to avoid such assumptions and give a definition that is directly applicable to real networks trained with standard algorithms. To do this, we apply a smoothing procedure to a standard discrete algorithm. The final result can still be interpreted as a valid bound on Shannon mutual information, but for a slightly modified optimization algorithm. 
Our definitions relate informativeness of a sample to the notion of algorithmic stability~\citep{bousquet2002stability,hardt2016train}, where a training algorithm is called stable if its outputs on datasets differing by only one example are close to each other.

To ensure our quantities are well-defined, we apply a smoothing technique which is reminiscent of a soft discretization of weight space. 
In \Cref{sec:functional-sample-information}, we show that a canonical discretization is obtained using the Fisher information matrix, which relates to classical results of \cite{rissanen1996fisher} on optimal coding length.  It also relates to the use of a post-distribution by \cite{achille2019information}, who however use it to estimate the total amount of information in the weights of a network.

We use a first-order approximation (linearization) inspired by the Neural Tangent Kernel (NTK)~\citep{jacot2018ntk,lee2019wide} to efficiently estimate informativeness of a sample.
While NTK predicts that, in the limit of an infinitely wide network, the linearized model is an accurate approximation, we do not observe this on more realistic architectures and datasets. 
However, we show that, when using pretrained networks as common in practice, linearization yields an accurate approximation, similarly to what is observed by \cite{mu2020gradients}.
\cite{ziv2020information} study the total information contained by an ensemble of randomly initialized linearized networks. They notice that, while considering ensembles makes the mutual information finite, it still diverges to infinity as training time goes to infinity. On the other hand, we consider the unique information about a single example, without the need for ensembles, by considering smoothed information, which remains bounded for any time.
Other complementary works study how information about an input sample propagates through the network~\citep{shwartzziv2017opening, achille2018emergence, saxe2019information} or total amount of information (complexity) of a classification dataset~\citep{lorena2019complex}, rather than how much information the sample itself contains. 

In terms of applications, our work is related to works that estimate influence of an example~\citep{koh2017understanding, toneva2018an, katharopoulos-fleuret-2018, pmlr-v97-ghorbani19c, yoon2020data}.
This can be done by estimating the change in weights if a sample is removed from the training set,
which is addressed by several works~\citep{koh2017understanding, golatkar2020forgetting, wu2020deltagrad}.
Influence functions \citep{cook1977detection, koh2017understanding} model removal of a sample as reducing its weight infinitesimally in the loss function, and show an efficient first-order approximation of its effect on other measures (such as test time predictions).
We found influence functions to be prohibitively slow for the networks and data regimes we consider.
\citet{basu2021influence} found that influence functions are not accurate for large DNNs.
Additionally, influence functions assume that the training has converged, which does not hold when early stopping is used.
We instead use linearization of neural networks to estimate the effect of removing an example efficiently.
We find that this approximation is accurate in realistic settings, and that the computational cost scales better with network size, making it applicable to very large neural networks.

Our work is orthogonal to that of feature selection: while we aim to evaluate the informativeness for the final weights of a subset of training samples, feature selection aims to quantify the informativeness for the task variable of a subset of features. However, they share some high-level similarities. In particular, \cite{kohavi1997wrappers} propose the notion of strongly-relevant feature as one that changes the discriminative distribution when it is excluded. This notion is similar to the notion of unique sample information in \cref{eq:unique-info}.

%% file: sample-info/main.tex
\section{Introduction}
\input{sample-info/sections/intro}

\section{Weight-based generalization bounds}\label{sec:w-gen-bounds}
\input{sample-info/sections/w-gen-bounds}

\section{Functional conditional mutual information}\label{sec:fcmi}
\input{sample-info/sections/fcmi}

\section{Applications}\label{sec:applications}
\input{sample-info/sections/applications}

\section{Experiments}
\input{sample-info/sections/experiments}

\section{Related work}
\input{sample-info/sections/related}

\section{Conclusion}
We derived information-theoretic generalization bounds for supervised learning algorithms based on the information contained in predictions rather than in the output of the training algorithm.
These bounds improve over the existing information-theoretic bounds, are applicable to a wider range of algorithms, and solve two key challenges: (a) they give meaningful results for deterministic algorithms and (b) they are significantly easier to estimate.
We showed experimentally that the proposed bounds closely follow the generalization gap in practical scenarios for deep learning.

%% file: sample-info/sections/intro.tex
Large neural networks trained with variants of stochastic gradient descent have excellent generalization capabilities, even in regimes where the number of parameters is much larger than the number of training examples.
\citet{zhang2016understanding} showed that classical generalization bounds based on various notions of complexity of hypothesis set fail to explain this phenomenon, as the same neural network can generalize well for one choice of training data and memorize completely for another one.
This observation has spurred a tenacious search for algorithm-dependent and data-dependent generalization bounds that give meaningful results in practical settings for deep learning~\citep{Jiang*2020Fantastic,dziugaite2020search}.

One line of attack bounds generalization error based on the information about training dataset stored in the weights~\citep{xu2017information, bassily2018learners, negrea2019information, bu2020tightening, steinke2020reasoning, haghifam2020sharpened, neu2021information, raginsky202110}.
The main idea is that when the training and testing performance of a neural network are different, the network weights necessarily capture some information about the training dataset.
However, the opposite might not be true: A neural network can store significant portions of training set in its weights and still generalize well~\citep{shokri2017membership,yeom2018privacy,nasr2019comprehensive}.
Furthermore, because of their information-theoretic nature, these generalization bounds become infinite or produce trivial bounds for deterministic algorithms.
When such bounds are not infinite, they are notoriously hard to estimate, due to the challenges arising in estimation of Shannon mutual information between two high-dimensional variables (e.g., weights of a ResNet and a training dataset).

This work addresses the aforementioned challenges.
We first improve some of the existing information-theoretic generalization bounds, providing a unified view and derivation of them (\cref{sec:w-gen-bounds}).
We then derive novel generalization bounds that measure information with predictions, rather than with the output of the training algorithm  (\cref{sec:fcmi}).
These bounds are applicable to a wide range of methods, including neural networks, Bayesian algorithms, ensembling algorithms, and non-parametric approaches.
In the case of neural networks, the proposed bounds improve over the existing weight-based bounds, partly because they avoid a counter-productive property of weight-based bounds that information stored in \emph{unused} weights affects generalization bounds, even though it has no effect on generalization.
The proposed bounds produce meaningful results for deterministic algorithms and are significantly easier to estimate.
For example, in case of classification, computing our most efficient bound involves estimating mutual information between a pair of predictions and a binary variable.

We apply the proposed bounds to ensembling algorithms, binary classification algorithms with finite VC dimension hypothesis classes, and to stable learning algorithms (\cref{sec:applications}).
We compute our most efficient bound on realistic classification problems involving neural networks, and show that the bound closely follows the generalization error, even in situations when a neural network with 3M parameters is trained deterministically on 4000 examples, achieving 1\% generalization error.

%% file: sample-info/sections/w-gen-bounds.tex
We start by describing the necessary notation and definitions, after which we present some of the existing weigh-based information-theoretic generalization bounds, slightly improve some of them, and prove relations between them.
The purpose of this section is to introduce the relevant existing bounds and prepare grounds for the functional conditional mutual information bounds introduced in \cref{sec:fcmi}, which we consider our main contribution. 
Theorems proved in the subsequent sections will be relying on the following lemma.
\begin{lemma} Let $(\Phi, \Psi)$ be a pair of random variables with joint distribution $P_{\Phi, \Psi}$. 
If $g(\phi, \psi)$ is a measurable function such that $\E_{P_{\Phi, \Psi}}\sbr{g(\Phi, \Psi)}$ exists and $g(\Phi, \Psi)$ is $\sigma$-subgaussian under $P_\Phi \otimes P_\Psi$, then
\begin{align}
\abs{\E_{P_{\Phi, \Psi}}\sbr{g(\Phi, \Psi)} - \E_{P_{\Phi}\otimes P_{\Psi}}\sbr{g(\Phi, \Psi})} \le \sqrt{2 \sigma^2 I(\Phi; \Psi)}.
\end{align}
Furthermore, if $g(\phi,\Psi)$ is $\sigma$-subgaussian for each $\phi$, then
\begin{equation}
\E_{P_{\Phi,\Psi}}\sbr{\rbr{g(\Phi, \Psi) - \E_{\Psi' \sim P_\Psi} g(\Phi, \Psi')}^2} \le 4 \sigma^2 (I(\Phi; \Psi) + \log 3),
\label{eq:mutual-inf-lemma-squared-gap}
\end{equation}
and
\begin{equation}
P_{\Phi,\Psi}\rbr{\abs{g(\Phi, \Psi) - \E_{\Psi' \sim P_{\Psi}}g(\Phi,\Psi')} \ge \epsilon} \le \frac{4 \sigma^2 (I(\Phi; \Psi) + \log 3)}{\epsilon^2}, \quad\forall \epsilon > 0,
\end{equation}
provided that the expectation in \cref{eq:mutual-inf-lemma-squared-gap} exists.
\label{lemma:mutual-info-lemma}
\end{lemma}
The first part of this lemma is equivalent to Lemma 1 of \citet{xu2017information}, which in turn has its roots in \citep{russo2019much}.
The second part generalizes Lemma 2 of \citet{hafez2020conditioning} by also providing bounds on the expected squared difference.

\subsection{Generalization bounds with input-output mutual information}
Let $S = (Z_1, Z_2, \ldots, Z_n)$ be a dataset of $n$ i.i.d. examples sampled from $P_Z$ and $Q_{W|S}$ be a training algorithm with hypothesis set $\WW$.
Given a loss function $\ell : \WW \times \ZZ \rightarrow \mathbb{R}$, the empirical risk of a hypothesis $w$ is $r_S(w) = \frac{1}{n}\sum_{i=1}^n \ell(w,Z_i)$ and the population risk is $R(w) = \E_{Z'\sim P_Z}\ell(w,Z')$.
Let $W \sim Q_{W|S}$ be a random hypothesis drawn after training on $S$.
We are interested in bounding the generalization gap $R(W) - r_S(W)$, also referred as generalization error sometimes.
\citet{xu2017information} establish the following information-theoretic bound on the absolute value of the expected generalization gap.
\begin{theorem}[Thm. 1 of \citet{xu2017information}]
Let $W \sim Q_{W|S}$. If $\ell(w,Z')$, where $Z' \sim P_Z$, is $\sigma$-subgaussian for all $w\in\WW$, then
\begin{equation}
    \abs{\E_{P_{S,W}}\sbr{R(W) - r_S(W)}} \le \sqrt{\frac{2\sigma^2 I(W; S)}{n}}.
    \label{eq:xu-ragisnky}
\end{equation}
\label{thm:xu-raginsky}
\end{theorem}

We generalize this result by showing that instead of measuring information with the entire dataset, one can measure information with a subset of size $m$ chosen uniformly at random.
For brevity, hereafter we call subsets chosen uniformly at random just ``random subsets''.

\begin{theorem}
Let $W \sim Q_{W|S}$. Let also $U$ be a random subset of $[n]$ with size $m$, independent of $S$ and $W$. If $\ell(w,Z')$, where $Z'\sim P_Z$, is $\sigma$-subgaussian for all $w \in \WW$, then
\begin{equation}
    \abs{\E_{P_{S,W}}\sbr{R(W) - r_S(W)}} \le \E_{P_U}\sqrt{\frac{2\sigma^2}{m} I^U(W; S_U)},
    \label{eq:generalized-xu-raginsky-bound}
\end{equation}
and
\begin{equation}
    \E_{P_{S,W}}\sbr{\rbr{R(W) - r_S(W)}^2} \le \frac{4\sigma^2}{n} \rbr{I(W; S) + \log 3}.\label{eq:generalized-xu-raginsky-variance}
\end{equation}
\label{thm:generalized-xu-raginsky}
\end{theorem}
With a simple application of Markov's inequality one can get tail bounds from the second part of the theorem.
Furthermore, by taking square root of both sides of \cref{eq:generalized-xu-raginsky-variance} and using Jensen's inequality on the left side, one can also construct an upper bound for the expected absolute value of generalization gap, $\E_{P_{S,W}}\abs{R(W) - r_S(W)}$.
These observations apply also to the other generalization gap bounds presented later in this work.

Note the bound on the squared generalization gap is written only for the case of $m=n$.
It is possible to derive squared generalization gap bounds of form $\frac{4\sigma^2} {m}(I(W; S_U \mid U) + \log 3)$.
Unfortunately, for small $m$ the $\log 3$ constant starts to dominate, resulting in vacuous bounds.

Picking a small $m$ decreases the mutual information term in \cref{eq:generalized-xu-raginsky-bound}, however, it also decreases the denominator.
When setting $m=n$, we get the bound of \citet{xu2017information} (\cref{thm:xu-raginsky}).
When $m=1$, the bound of \cref{eq:generalized-xu-raginsky-bound} becomes
\begin{equation}
    \abs{\E_{P_{S,W}}\sbr{R(W) - r_S(W)}} \le \frac{1}{n}\sum_{i=1}^n \sqrt{2\sigma^2 I(W; Z_i)},
    \label{eq:generalized-xu-raginsky-bound-m=1}
\end{equation}
matching the result of \citet{bu2020tightening} (Proposition 1).
A similar bound, but for a different notion of information, was derived by \citet{alabdulmohsin2015algorithmic}.
\citet{bu2020tightening} prove that the bound with $m=1$ is tighter than the bound with $m=n$.
We generalize this result by proving that the bound of \cref{eq:generalized-xu-raginsky-bound} is non-descreasing in $m$.
\begin{proposition}
Let $m \in [n-1]$, $U$ be a random subset of $[n]$ of size $m$, $U'$ be a random subset of size $m+1$, and $\phi : \mathbb{R} \rightarrow \mathbb{R}$ be any non-decreasing concave function. Then
\begin{equation}
\E_{P_U}\sbr{\phi\rbr{\frac{1}{m} I^U(W; S_U)}} \le \E_{P_{U'}}\sbr{\phi\rbr{\frac{1}{m+1} I^{U'}(W; S_{U'})}}.
\end{equation}
\label{prop:xu-raginsky-generalization-choice-of-m}
\end{proposition}
When $\phi(x) = \sqrt{x}$, this result proves that the optimal value for $m$ in \cref{eq:generalized-xu-raginsky-bound} is 1.
Furthermore, when we use Jensen's inequality to move expectation over $U$ inside the square root in \cref{eq:generalized-xu-raginsky-bound}, then the resulting bound becomes $\sqrt{\frac{2\sigma^2}{m} I(W; S_U \mid U)}$
and matches the result of \citet{negrea2019information} (Thm. 2.3). These bounds are also non-decreasing with respect to $m$ (using \cref{prop:xu-raginsky-generalization-choice-of-m} with $\phi(x) = x$).

\cref{thm:xu-raginsky} can be used to derive generalization bounds that depend on the information between $W$ and a single example $Z_i$ conditioned on the remaining examples $Z_{-i}=(Z_1,\ldots,Z_{i-1},Z_{i+1},\ldots,Z_n)$.
\begin{theorem}
Let $W \sim Q_{W|S}$. If $\ell(w,Z')$, where $Z'\sim P_Z$, is $\sigma$-subgaussian for all $w \in \WW$, then
\begin{equation}
    \abs{\E_{P_{S,W}}\sbr{R(W) - r_S(W)}} \le \frac{1}{n}\sum_{i=1}^n\sqrt{2\sigma^2 I(W; Z_i \mid Z_{-i})},
    \label{eq:generalized-xu-raginsky-bound-stability}
\end{equation}
and
\begin{equation}
    \E_{P_{S,W}}\sbr{\rbr{R(W) - r_S(W)}^2} \le \frac{4\sigma^2}{n} \rbr{\sum_{i=1}^n I(W; Z_i \mid Z_{-i}) + \log 3}.
\end{equation}
\label{thm:generalized-xu-raginsky-stability}
\end{theorem}
This theorem is a simple corollary of \cref{thm:generalized-xu-raginsky}, using the facts that $I(W; Z_i) \le I(W; Z_i \mid Z_{-i})$ and that $I(W; S)$ is upper bounded by $\sum_{i=1}^n I(W; Z_i \mid Z_{-i})$, which is also known as erasure information~\citep{verdu2008information}.
The first part of it improves the result of \citet{raginsky2016information} (Thm. 2), as the averaging over $i$ is outside of the square root.
While these bounds are worse that the corresponding bounds of \cref{thm:generalized-xu-raginsky}, it is sometimes easier to manipulate them analytically.

The bounds described above measure information with the output $W$ of the training algorithm.
In the case of prediction tasks with parametric methods, the parameters $W$ might contain information about the training dataset, but not use it to make predictions.
Partly for this reason, our the main goal is to derive generalization bounds that measure information with the prediction function, rather than with the weights.
In general, there is no straightforward way of encoding the prediction function into a random variable.

When the domain $\ZZ$ is finite, we can encode the prediction function as the collection of predictions on all examples of $\ZZ$.
While this approach is ineffective when $|\ZZ|$ is large or infinite, it suggests to pick a random finite subset of examples $\ZZ'\subset \ZZ$ to evaluate the learned function on.
Since the learned function might behave very differently on seen and unseen examples, $\ZZ'$ should include examples from both types.
This naturally leads us to the random subsamling setting introduced by \citet{steinke2020reasoning} (albeit with a different motivation), where one first fixes a set of $2n$ examples, and then randomly selects $n$ of them to form the training set.
Evaluations of the learned function on the $2n$ examples makes a good representation of the learned function and allows us to derive functional generalization bounds (presented in \cref{sec:fcmi}).
Before describing these bounds we present the setting of \citet{steinke2020reasoning} in detail and generalize some of the existing weight-based bounds in that setting.

\subsection{Generalization bounds with conditional mutual information}
Let $\tilde{Z} \in \ZZ^{n\times 2}$ be a collection of $2n$ i.i.d samples from $P_Z$, grouped in $n$ pairs.
The random variable $J \sim \mathrm{Uniform}(\mathset{0,1}^n)$ specifies which example to select from each pair to form the training set $S=\tilde{Z}_J \triangleq (\tilde{Z}_{i,J_i})_{i=1}^n$.
Let $W \sim Q_{W|S}$. 
In this setting \citet{steinke2020reasoning} defined condition mutual information (CMI) of algorithm $Q$ with respect to the data distribution $P_Z$ as
\begin{equation}
    \mathrm{CMI}_P(A) = I(W; J \mid \tilde{Z}) = \E_{P_{\tilde{Z}}}\sbr{ I^{\tilde{Z}}(W; J)},
\end{equation}
and proved the following upper bound on expected generalization gap.
\begin{theorem}[Thm. 2, \citet{steinke2020reasoning}]
Let $W\sim Q_{W|S}$. If the loss function $\ell(w,z) \in [0,1], \forall w \in \WW, z \in \ZZ$, then the expected generalization gap can be bounded as follows:
\begin{equation}
    \abs{\E_{P_{S,W}}\sbr{R(W) - r_S(W)}} \le \sqrt{\frac{2}{n} \mathrm{CMI}_P(Q)}.
\end{equation}
\label{thm:cmi}
\end{theorem}
\citet{haghifam2020sharpened} improved this bound in two aspects. First, they provided bounds where expectation over $\tilde{Z}$ is outside of the square root. Second, they considered measuring information with subsets of $J$, as we did in the previous section.
\begin{theorem}[Thm. 3.1 of \citet{haghifam2020sharpened}]
Let $W\sim Q_{W|S}$. Let also $m \in [n]$ and $U \subseteq [n]$ be a random subset of size $m$, independent from $\tilde{Z}, J$, and $W$. If the loss function $\ell(w,z) \in [0,1], \forall w\in\WW, z\in\ZZ$, then
\begin{equation}
    \abs{\E_{P_{S,W}}\sbr{R(W) - r_S(W)}} \le \E_{P_{\tilde{Z}}}\sqrt{\frac{2}{m} I^{\tilde{Z}}(W; J_U \mid U)}.
    \label{eq:cmi-sharpened}
\end{equation}
\label{thm:cmi-sharpened}
\end{theorem}
Furthermore, for $m=1$ they tighten the bound by showing that one can move the expectation over $U$ outside of the squared root (\citet{haghifam2020sharpened}, Thm 3.4).
We generalize these results by showing that for all $m$ expectation over $U$ can be done outside of the square root.
Furthermore, our proof closely follows the proof of \cref{thm:generalized-xu-raginsky}.
\begin{theorem}
Let $W\sim Q_{W|S}$. Let also $m \in [n]$ and $U \subseteq [n]$ be a random subset of size $m$, independent from $\tilde{Z}, J$, and $W$. If $\ell(w,z) \in [0,1], \forall w \in \WW, z \in \ZZ$, then
\begin{equation}
    \abs{\E_{P_{S,W}}\sbr{R(W) - r_S(W)}} \le \E_{P_{\tilde{Z}, U}}\sqrt{\frac{2}{m} I^{\tilde{Z}, U}(W; J_U)},
    \label{eq:cmi-sharpened-generalized}
\end{equation}
and
\begin{equation}
   \E_{P_{S,W}}\sbr{\rbr{R(W) - r_S(W)}^2} \le \frac{8}{n} \rbr{I(W; J \mid \tilde{Z}) + 2}.
\end{equation}
\label{thm:cmi-sharpened-generalized}
\end{theorem}
The bound of \cref{eq:cmi-sharpened-generalized} improves over the bound of \cref{thm:cmi-sharpened} and matches the special result for $m=1$.
\citet{rodriguez2021random} proved even tighter expected generalization gap bound by replacing $I^{\tilde{Z},U}(W; J_U)$ with $I^{\tilde{Z}_U, U}(A(W; J_U)$.
\citet{haghifam2020sharpened} showed that if one takes the expectations over $\tilde{Z}$ inside the square root in~\cref{eq:cmi-sharpened}, then the resulting looser upper bounds become non-decreasing over $m$.
Using this result they showed that their special case bound for $m=1$ is the tightest.
We generalize their results by showing that even without taking the expectations inside the squared root, the bounds of ~\cref{thm:cmi-sharpened} are non-decreasing over $m$.
We also show that the same holds for our tighter bounds of \cref{eq:cmi-sharpened-generalized}.
\begin{proposition}
Let $W \sim Q_{W|S}$. Let also $m \in [n-1]$, $U$ be a random subset of $[n]$ of size $m$, $U'$ be a random subset of size $m+1$, and $\phi : \mathbb{R} \rightarrow \mathbb{R}$ be any non-decreasing concave function. Then
\begin{equation}
\E_{P_U}\sbr{\phi\rbr{\frac{1}{m} I^{\tilde{Z}, U}(W; J_U)}} \le \E_{P_{U'}}\sbr{\phi\rbr{\frac{1}{m+1} I^{\tilde{Z}, U'}(W; J_{U'})}}.
\label{eq:cmi-sharpened-generalized-choice-of-m}
\end{equation}
\label{prop:cmi-sharpened-generalized-choice-of-m}
\end{proposition}
By setting $\phi(x)=x$, taking square root of both sides of \cref{eq:cmi-sharpened-generalized-choice-of-m}, and then taking expectation over $\tilde{Z}$, we prove that bounds of \cref{eq:cmi-sharpened} are non-decreasing over $m$.
By setting $\phi(x)=\sqrt{x}$ and then taking expectation over $\tilde{Z}$, we prove that bounds of \cref{eq:cmi-sharpened-generalized} are non-decreasing with $m$.

Similarly to the \cref{thm:generalized-xu-raginsky-stability} of the previous section,  the following result establishes generalization bounds with information-theoretic stability quantities.
\begin{theorem}
If $\ell(w,z) \in [0,1], \forall w \in \WW, z \in \ZZ$, then
\begin{equation}
    \abs{\E_{P_{W,S}}\sbr{R(W) - r_S(W)}} \le \E_{P_{\tilde{Z}}}\sbr{\frac{1}{n}\sum_{i=1}^n\sqrt{2 I^{\tilde{Z}}(W; J_i \mid J_{-i})}},
    \label{eq:cmi-sharpened-generalized-stability}
\end{equation}
and
\begin{equation}
  \E_{P_{W,S}}\rbr{R(W) - r_S(W)}^2 \le \frac{8}{n} \rbr{\E_{P_{\tilde{Z}}}\sbr{\sum_{i=1}^n I^{\tilde{Z}}(W; J_i \mid J_{-i})} + 2}.
\end{equation}
\label{thm:cmi-sharpened-generalized-stability}
\end{theorem}

%% file: sample-info/sections/fcmi.tex
The bounds in \cref{sec:w-gen-bounds} leverage information in the output of the algorithm, $W$.
In this section we focus on supervised learning problems: $\ZZ = \XX \times \YY$.
To encompass many types of approaches, we do not assume that the training algorithm has an output $W$, which is then used to make predictions.
Instead, we assume that the learning method implements a function $f: \ZZ^n \times \XX \times \mathcal{R} \rightarrow \mathcal{K}$ that takes a training set $s$, a test input $x'$, an auxiliary argument $r$ capturing the stochasticity of training and predictions, and outputs a prediction $f(s,x',r)$ on the test example. Note that the prediction domain $\mathcal{K}$ can be different from $\YY$.
This setting includes non-parametric methods (for which $W$ is the training dataset itself), parametric methods, Bayesian algorithms, and more.
For example, in parametric methods, where a hypothesis set $\HH = \mathset{h_w : \XX \rightarrow \mathcal{K} \mid w \in \WW}$ is defined, $f(s,x,r) = h_{A(s,r)}(x)$, with $A(s,r)$ denoting the weights after training on dataset $s$ with randomness $r$.

In this supervised setting, the loss function $\ell : \mathcal{K} \times \mathcal{Y} \rightarrow \mathbb{R}$ measures the discrepancy between a prediction and a label.
As in the previous subsection, we assume that a collection of $2n$ i.i.d examples $\tilde{Z} \sim P_Z^{n\times 2}$ is given, grouped in $n$ pairs, and the random variable $J \sim \mathrm{Uniform}(\mathset{0,1}^n)$ specifies which example to select from each pair to form the training set $S=\tilde{Z}_J \triangleq (\tilde{Z}_{i,J_i})_{i=1}^n$.
Let $R$ be an auxiliary random variable, independent of $\tilde{Z}$ and $J$, that provides stochasticity for predictions (e.g., in neural networks $R$ can be used to make the training stochastic).
The empirical risk of learning method $f$ trained on dataset $S$ with randomness $R$ is defined as $r_S(f) = \frac{1}{n}\sum_{i=1}^n \ell(f(S,X_i,R),Y_i)$.
The population risk is defined as $R(f) = \E_{Z'\sim P_Z} \ell(f(S,X',R),Y')$.

Before moving forward we adopt two conventions.
First, if $\bs{z}$ is a collection of examples, then $\bs{x}$ and $\bs{y}$ denote the collection of its inputs and labels respectively.
If $s$ is a training set and $\bs{x}$ is a collection of inputs, then $f(s, \bs{x}, r)$ denotes the collection of predictions on $\bs{x}$ after training on $s$ with randomness $r$.
Let $\tilde{F}$ denote the predictions on $\tilde{Z}$ after training on $S$ with randomness $R$, i.e., $\tilde{F} = f(S,\tilde{X},R)$.
Given a subset $u \subset [n]$, we use the $\tilde{F}_u \in \mathcal{K}^{|u| \times 2}$ notation to denote the rows of $\tilde{F}$ corresponding to $u$.

% We define functional conditional mutual information ($f$-CMI).
% and give generalization gap bounds that are based on it.
\begin{definition}
Let $f$, $\tilde{Z}$, $J$, $R$, and $\tilde{F}$ be defined as above and let $u \subseteq [n]$ be a subset of size $m$. Then \emph{pointwise functional conditional mutual information $\ofcmi(f,\tilde{z},u)$} is defined as
\begin{equation}
\ofcmi(f,\tilde{z},u) = I^{\tilde{Z}=\tilde{z}}(\tilde{F}_u; J_u),
\end{equation}
while \emph{functional conditional mutual information $\fcmi(f,u)$} is defined as
\begin{equation}
\fcmi(f,u) = \E_{P_{\tilde{Z}}} \ofcmi(f,\tilde{z},u).
\end{equation}
\end{definition}
When $u=[n]$ we simply use the notations $\ofcmi(f,\tilde{z})$ and $\fcmi(f)$, instead of $\ofcmi(f,\tilde{z},[n])$ and $\fcmi(f,[n])$, respectively.

\begin{theorem}
Let $U$ be a random subset of size $m$, independent of $\tilde{Z}$, $J$, and $R$.
If $\ell(\widehat{y},y) \in [0,1], \forall \widehat{y} \in \mathcal{K}, z \in \ZZ$, then
\begin{equation}
\abs{\E_{P_{S,R}}}\sbr{R(f) - r_S(f)} \le \E_{P_{\tilde{Z},U}}\sqrt{\frac{2}{m} \ofcmi(f,\tilde{Z},U)},
\label{eq:f-cmi-bound}
\end{equation}
and
\begin{equation}
\E_{P_{S,R}}\rbr{R(f) - r_S(f)}^2 \le \frac{8}{n} \rbr{\E_{P_{\tilde{Z}}} \ofcmi(f,\tilde{Z}) + 2}.
\end{equation}
\label{thm:f-cmi-bound}
\end{theorem}

For parametric methods $w=A(s,r)$, the bound of \cref{eq:f-cmi-bound} improves over the bound of \cref{eq:cmi-sharpened-generalized}, as $J_u \textrm{\ ---\ } W \textrm{\ ---\ } \tilde{F}_u$ is a Markov chain under $P_{\tilde{F}, W, J \mid\tilde{Z}}$, implying the following data processing inequality:
\begin{equation}
    I^{\tilde{Z}}(F_u; J_u) \le I^{\tilde{Z}}(W; J_u).
\end{equation}
For deterministic algorithms $I^{\tilde{Z}}(W; J_u)$ is often equal to $H(J_u)= m \log 2$, as most likely each choice of $J$ produces a different $W$.
In such cases the bound with $I^{\tilde{Z}}(W; J_u)$ is vacuous.
In contrast, the proposed bounds with $f$-CMI (especially when $m=1$) do not have this problem.
Even when the algorithm is stochastic, information between $W$ and $J_u$ can be much larger than information between predictions and $J_u$, as having access to weights makes it easier to determine $J_u$ (e.g., by using gradients).
A similar phenomenon has been observed in the context of membership attacks, where having access to weights of a neural network allows constructing more successful membership attacks compared to having access to predictions only~\citep{nasr2019comprehensive, gupta2021membership}.

\begin{corollary}
When $m=n$, the bound of \cref{eq:f-cmi-bound} becomes
\begin{equation}
    \abs{\E_{P_{S,R}}\sbr{R(f) - r_S(f)}} \le \E_{P_{\tilde{Z}}}\sqrt{\frac{2}{n}\ofcmi(f,\tilde{Z})} \le \sqrt{\frac{2}{n}\fcmi(f)}.
    \label{eq:f-cmi-bound-m=n}
\end{equation}
\label{corr:f-cmi-bound-m=n}
\end{corollary}
For parametric models, this improves over the CMI bound (\cref{thm:cmi}), as by data processing inequality, $\fcmi(f)=I(\tilde{F}; J \mid \tilde{Z}) \le I(W); J \mid \tilde{Z}) = \mathrm{CMI}_P(A)$.

\begin{remark}
Note that the collection of training and testing predictions $\tilde{F}=f(\tilde{Z}_J,\tilde{X},R)$ cannot be replaced with only testing predictions $f(\tilde{Z}_J, \tilde{X}_{\bar{J}},R)$. As an example, consider an algorithm that memorizes the training examples and outputs a constant prediction on any other example. This algorithm will have non-zero generalization gap, but $f(\tilde{Z}_J, \tilde{X}_{\bar{J}},R)$ will be constant and will have zero information with $J$ conditioned on any random variable.
Moreover, if we replace $f(\tilde{Z}_J,\tilde{X},R)$ with only training predictions $f(\tilde{Z}_J,\tilde{X}_J,R)$, the resulting bound can become too loose, as one can deduce $J$ by comparing training set predictions with the labels $\tilde{Y}$.
\end{remark}

\begin{corollary}
When $m=1$, the bound of \cref{eq:f-cmi-bound} becomes
\begin{align}
\abs{\E_{P_{S,R}}\sbr{R(f) - r_S(f)}} &\le \frac{1}{n}\sum_{i=1}^n\E_{P_{\tilde{Z}}} \sqrt{2 I^{\tilde{Z}}(\tilde{F}_i; J_i)}\label{eq:f-cmi-bound-m=1}.
\end{align}
\label{corr:f-cmi-bound-m=1}
\end{corollary}
A great advantage of this bound compared to all other bounds described so far is that the mutual information term is computed between a relatively low-dimensional random variable $\tilde{F}_i$ and a binary random variable $J_i$.
For example, in the case of binary classification with $\mathcal{K}=\mathset{0,1}$, $\tilde{F}_i$ will be a pair of 2 binary variables.
This allows us to estimate the bound efficiently and accurately.
Note that estimating other information-theoretic bounds is significantly harder.
The bounds of \citet{xu2017information}, \citet{negrea2019information}, and \citet{bu2020tightening} are hard to estimate as they involve estimation of mutual information between a high-dimensional non-discrete variable $W$ and at least one example $Z_i$.
Furthermore, this mutual information can be infinite in case of deterministic algorithms or when $H(Z_i)$ is infinite.
The bounds of \citet{haghifam2020sharpened} and \citet{steinke2020reasoning} are also hard to estimate as they involve estimation of mutual information between $W$ and at least one train-test split variable $J_i$.

As in the case of bounds presented in the previous section (\cref{thm:generalized-xu-raginsky} and \cref{thm:cmi-sharpened-generalized}), we prove that the bound of \cref{thm:f-cmi-bound} is non-decreasing in $m$. This stays true even when we increase the upper bounds by moving the expectation over $U$ or the expectation over $\tilde{Z}$ or both under the square root. The following proposition allows us to prove all these statements.
\begin{proposition}
Let $m \in [n-1]$, $U$ be a random subset of $[n]$ of size $m$, $U'$ be a random subset of size $m+1$, and $\phi : \mathbb{R} \rightarrow \mathbb{R}$ be any non-decreasing concave function. Then
\begin{equation}
\E_{P_U}\phi\rbr{\frac{1}{m} I^{\tilde{Z},U}(\tilde{F}_U; J_U)} \le \E_{P_{U'}}\phi\rbr{\frac{1}{m+1} I^{\tilde{Z},U'}(\tilde{F}_{U'}; J_{U'})}.
\label{eq:fcmi-choice-of-m}
\end{equation}
\label{prop:fcmi-choice-of-m}
\end{proposition}
By setting $\phi(x)=\sqrt{x}$ and then taking expectation over $\tilde{Z}$ and $u$, we prove that bounds of \cref{thm:f-cmi-bound} are non-decreasing over $m$.
By setting $\phi(x)=x$, taking expectation over $\tilde{Z}$, and then taking square root of both sides of \cref{eq:fcmi-choice-of-m}, we prove that bounds are non-decreasing in $m$ when both expectations are under the square root.
\cref{prop:fcmi-choice-of-m} proves that $m=1$ is the optimal choice in \cref{thm:f-cmi-bound}.
Notably, the bound that is the easiest to compute is also the tightest!

Analogously to \cref{thm:cmi-sharpened-generalized-stability}, we provide the following stability-based bounds.
\begin{theorem}
If $\ell(\widehat{y},y) \in [0,1], \forall \widehat{y} \in \mathcal{K}, z \in \ZZ$, then
\begin{equation}
\abs{\E_{P_{S,R}}\sbr{R(f) - r_S(f)}} \le \E_{P_{\tilde{Z}}}\sbr{\frac{1}{n}\sum_{i=1}^n\sqrt{2 I^{\tilde{Z}}(\tilde{F}_i; J_i \mid J_{-i})}},
\label{eq:f-cmi-bound-stability}
\end{equation}
and
\begin{equation}
\E_{P_{S,R}}\rbr{R(f) - r_S(f)}^2 \le \frac{8}{n} \rbr{\E_{P_{\tilde{Z}}}\sbr{\sum_{i=1}^n I^{\tilde{Z}}(\tilde{F}; J_i \mid J_{-i})} + 2}.
\label{eq:f-cmi-squared-bound-stability}
\end{equation}
\label{thm:f-cmi-bound-stability}
\end{theorem}
Note that unlike \cref{eq:f-cmi-bound-stability}, in the second part of \cref{thm:f-cmi-bound-stability} we measure information with predictions on all $2n$ pairs and $J_i$ conditioned on $J_{-i}$.
Whether $\tilde{F}$ can be replaced with $\tilde{F}_i$, predictions only on the $i$-th pair, is addressed in \cref{ch:limitations}.

%% file: sample-info/sections/applications.tex
In this section we describe 3 applications of the $f$-CMI-based generalization bounds. 

\subsection{Ensembling algorithms}
Ensembling algorithms combine predictions of multiple learning algorithms to obtain better performance.
Let us consider $t$ learning algorithms, $f_1,f_2,\ldots,f_t$, each with its own independent randomness $R_i,\ i\in[t]$.
Some ensembling algorithms can be viewed as a possibly stochastic function $g : \mathcal{K}^t \rightarrow \mathcal{K}$ that takes predictions of the $t$ algorithms and combines them into a single prediction.
Relating the generalization gap of the resulting ensembling algorithm to that of individual $f_i$s can be challenging for complicated choices of $g$.
However, it is easy to bound the generalization gap of $g(f_1,\ldots,f_t)$ in terms of $f$-CMIs of individual predictors.
Let $\bs{x'}$ be an arbitrary collection of inputs.
Denoting $F_i=f_i(\tilde{Z}_J,\bs{x'},R_i),\ i\in[t]$, we have that
\begin{align}
I^{\tilde{Z}}(g(F_1, \ldots, F_t); J) &\stackrel{(a)}{\le} I^{\tilde{Z}}(F_1,\ldots,F_t ; J)\\
&\hspace{-4em}\stackrel{(b)}{=} I^{\tilde{Z}}(F_1; J) + I^{\tilde{Z}}(F_2,\ldots,F_t; J) - I^{\tilde{Z}}(F_1; F_2,\ldots,F_t) + I^{\tilde{Z}}(F_1; F_2,\ldots,F_t \mid J)\\
&\hspace{-4em}\stackrel{(c)}{\le} I^{\tilde{Z}}(F_1; J) + I^{\tilde{Z}}(F_2,\ldots,F_t; J)\\
&\hspace{-4em}\stackrel{(d)}{\le} \ldots\le I^{\tilde{Z}}(F_1; J)+\cdots+I^{\tilde{Z}}(F_t; J),
\end{align}
where (a) follows from the data processing inequality, (b) follows from the chain rule, (c) follows from non-negativity of mutual information and the fact that $F_1 \indep F_2,\ldots,F_t \mid J$, and (d) is derived by repeating the arguments above to separate all $F_i$.
Unfortunately, the same derivation above does not work if we replace $J$ with $J_u$, where $u$ is a proper subset of $[n]$, as $I^{\tilde{Z}}(F_1; F_2,\ldots,F_t \mid J_u) \neq 0$ in general.

\subsection{Binary classification with finite VC dimension}
Let us consider the case of binary classification: $\YY = \mathset{0,1}$, where the learning method $f : \ZZ^n \times \XX \times \mathcal{R} \rightarrow \mathset{0,1}$ is implemented using a learning algorithm $A: \ZZ^n \times \mathcal{R} \rightarrow \mathcal{W}$ that selects a classifier from a hypothesis set $\mathcal{H}=\mathset{h_w: \XX \rightarrow \YY}$.
If $\mathcal{H}$ has finite VC dimension $d$~\citep{vapnik1998statistical}, then for any algorithm $f$, the quantity $\ofcmi(f,\tilde{z})$ can be bounded the following way.

\begin{theorem}
Let $\mathcal{Z}$, $\mathcal{H}$, $f$ be defined as above, and let $d<\infty$ be the VC dimension of $\mathcal{H}$. Then for any algorithm $f$ and $\tilde{z} \in \ZZ^{n\times 2}$,
\begin{equation}
\ofcmi(f, \tilde{z}) \le \max\mathset{(d+1)\log{2},\ d\log\rbr{2en/d}}.
\end{equation}
\label{thm:vc}
\end{theorem}

Considering the 0-1 loss function and using this result in \cref{corr:f-cmi-bound-m=n}, we get an expect generalization gap bound   that is $\rbr{\sqrt{\frac{d}{n} \log\rbr{\frac{n}{d}}}}$, matching the classical uniform convergence bound~\citep{vapnik1998statistical}.
% However, the upper bound of \cref{thm:vc} is a worst-case result. It is expected that for $\tilde{Z} \sim P_Z^{n\times 2}$ the number of possible labelings will be much smaller than the upper bound used in the proof.
The $\sqrt{\log{n}}$ factor can be removed in some cases~\citep{hafez2020conditioning}.

Both \citet{xu2017information} and \citet{steinke2020reasoning} prove similar information-theoretic bounds in the case of finite VC dimension classes, but their results holds for specific algorithms only.
Even in the simple case of threshold functions: $\XX = [0,1]$ and $\mathcal{H} = \mathset{h_w: x \mapsto \ind{x > w} \mid w \in [0, 1]}$,
all weight-based bounds described in \cref{sec:w-gen-bounds} are vacuous if one uses a training algorithm that encodes the training set in insignificant bits of $W$, while still getting zero error on the training set and hence achieving low test error.

\subsection{Stable deterministic or stochastic algorithms}
\label{sec:stable-algorithms}
\cref{thm:generalized-xu-raginsky-stability,thm:cmi-sharpened-generalized-stability,thm:f-cmi-bound-stability} provide generalization bounds involving information-theoretic stability measures, such as $I(W; Z_i \mid Z_{-i})$, $I^{\tilde{Z}}(W; J \mid J_{-i})$ and $I^{\tilde{Z}}(\tilde{F}; J_i \mid J_{-i})$.
In this section we build upon the predication-based stability bounds of \cref{thm:f-cmi-bound-stability}.
First, we show that for any collection of examples $\bs{x'}$, the mutual information $I^{\tilde{Z}}(f(\tilde{Z}_J,\bs{x'}); J_i \mid J_{-i})$ can be bounded as follows.
\begin{proposition}
Let $J^{i \leftarrow c}$ denote $J$ with $J_i$ set to $c \in \mathset{0,1}$. Then for any $\tilde{z} \in \ZZ^{n\times 2}$ and $\bs{x'} \in \XX^k$, the mutual information $I(f(\tilde{z}_J,\bs{x'},R); J_i \mid J_{-i})$ is upper bounded by 
\begin{align*}
\frac{1}{4} \KL{f(\tilde{z}_{J^{i\leftarrow 1}},\bs{x'},R) | J_{-i}}{f(\tilde{z}_{J^{i\leftarrow 0}},\bs{x'},R) | J_{-i}} + \frac{1}{4}\KL{f(\tilde{z}_{J^{i\leftarrow 0}},\bs{x'},R) | J_{-i}}{f(\tilde{z}_{J^{i\leftarrow 1}},\bs{x'},R)| {J_{-i}}}.
\end{align*}
\label{prop:stability-to-kl}
\end{proposition}
To compute the right-hand side of \cref{prop:stability-to-kl} one needs to know how much on-average the distribution of predictions on $x$ changes after replacing the $i$-th example in the training dataset.
The problem arises when we consider deterministic algorithms.
In such cases, the right-hand side is infinite, while the left-hand side $I(f(\tilde{z}_J,x,R); J_i \mid J_{-i})$ is always finite and could be small.
Therefore, for deterministic algorithms, directly applying the result of \cref{prop:stability-to-kl} will not give meaningful generalization bounds.
Nevertheless, we show that we can add an optimal amount of noise to predictions, upper bound the generalization gap of the resulting noisy algorithm, and relate that to the generalization gap of the original deterministic algorithm.

Let us consider a  deterministic algorithm $f: \ZZ^n \times \XX \rightarrow \bR^k$. 
We define the following notions of functional stability.

\begin{definition}[Functional stability] Let $S=(Z_1,\ldots,Z_n)$ be a collection of $n$ i.i.d. samples from $P_Z$, and $Z'$ and $Z_{\mathrm{test}}$ be two additional independent samples from $P_Z$. Let $S^{(i)} \triangleq (Z_1,\ldots,Z_{i-i},Z',Z_{i+1},\ldots,Z_n)$ be the collection constructed from $S$ by replacing the $i$-th example with $Z'$. A deterministic algorithm $f: \ZZ^n \times \XX \rightarrow \bR^k$ is
\begin{align}
&\text{a) $\beta$ self-stable if } \ \forall i \in [n],\  \E_{S,Z'} \nbr{f(S,Z_i)-f(S^{(i)},Z_i)}^2 \le \beta^2,\\
&\text{b) $\beta_1$ test-stable if }\ \forall i \in [n],\  \E_{S,Z',Z_{\mathrm{test}}} \nbr{f(S,Z_{\mathrm{test}})-f(S^{(i)},Z_{\mathrm{test}})}^2 \le \beta_1^2,\\
&\text{c) $\beta_2$ train-stable if } \ \forall i,j\in[n],i\neq j,\  \E_{S,Z'} \nbr{f(S,Z_j)-f(S^{(i)},Z_j)}^2 \le \beta_2^2.
\end{align}
\end{definition}

\begin{theorem}
Let $\YY = \bR^k$, $f: \ZZ^n \times \XX \rightarrow \bR^k$ be a deterministic algorithm that is $\beta$ self-stable, and $\ell(\widehat{y},y) \in [0,1]$ be a loss function that is $\gamma$-Lipschitz in the first coordinate. Then
\begin{equation}
\abs{\E_{S,R}\sbr{R(f) - r_S(f)}} \le 2^{\frac{3}{2}} d^{\frac{1}{4}}\sqrt{\gamma  \beta}.
\end{equation}
Furthermore, if $f$ is also $\beta_1$ train-stable and $\beta_2$ test-stable, then
\begin{equation}
\E_{S,R}\rbr{R(f) - r_S(f)}^2 \le \frac{32}{n} + 12^{\frac{3}{2}}\sqrt{d}\gamma\sqrt{2\beta^2 + n\beta_1^2 + n \beta_2^2}.
\label{eq:f-cmi-stability-squared}
\end{equation}
\label{thm:f-cmi-deterministic-stability}
\end{theorem}
It is expected that $\beta_2$ is smaller than $\beta$ and $\beta_1$.
For example, in the case of neural networks interpolating the training data or in the case of empirical risk minimization in the realizable setting, $\beta_2$ will be zero.
It is also expected that $\beta$ is larger than $\beta_1$.
However, the relation of $\beta^2$ and $n\beta_1^2$ is not trivial.

The notion of pointwise hypothesis stability $\beta'_2$ defined by \citet{bousquet2002stability} (definition 4) is comparable to our notion of self-stability $\beta$.
The first part of Theorem 11 in \cite{bousquet2002stability} describes a generalization bound where the difference between empirical and population losses is of order $1/\sqrt{n}+\sqrt{\beta'_2}$, which is comparable with our result of \cref{thm:f-cmi-deterministic-stability} ($\Theta(\sqrt{\beta})$).
The proof there also contains a bound on the expected squared difference of empirical and population losses. That bound is of order $1/n + \beta'_2$.
In contrast, our result of \cref{eq:f-cmi-stability-squared} contains two extra terms related to test-stability and train-stability (the terms 
$n \beta_1^2$ and $n \beta_2^2$).
If $\beta$ dominates $n\beta_1^2 + n\beta_2^2$, then the bound of \cref{eq:f-cmi-stability-squared} will match the result of \citet{bousquet2002stability}.
% It is worth to mention that uniform stability, which is a much stronger condition, allows to derive quadratically better bounds~\cite{bousquet2002stability. Whether such results can be derived from information-theoretic bounds is an open question.

%% file: sample-info/sections/experiments.tex
As mentioned earlier, the expected generalization gap bound of \cref{corr:f-cmi-bound-m=1} is significantly easier to compute compared to existing information-theoretic bounds, and does not give trivial results for deterministic algorithms.
To understand how well the bound does in challenging situations, we consider cases when the algorithm generalizes well despite the high complexity of the hypothesis class and relatively small number of training examples.
The code for reproducing the experiments can be found at \url{github.com/hrayrhar/f-CMI}.
The exact experimental details of the experiments presented in \cref{tab:mnist4vs9-standard,tab:mnist4vs9-langevin,tab:cifar10-resnet50,tab:cifar10-langevin} of \cref{app:add-exp-details}.

\subsection{Experimental setup}

\paragraph{Estimation of generalization gap.}
In all experiments below we draw $k_1$ samples of $\tilde{Z}$, each time by randomly drawing $2n$ examples from the corresponding dataset and grouping then into $n$ pairs.
For each sample $\tilde{z}$, we draw $k_2$ samples of the training/test split variable $J$ and randomness $R$.
We then run the training algorithm on these $k_2$ splits (in total $k_1 k_2$ runs).
For each $\tilde{z}$, $j$ and $r$, we estimate the population risk with the average error on the test examples $\tilde{z}_{\bar{j}}$.
For each $\tilde{z}$, we average over the $k_2$ samples of $J$ and $R$ to get an estimate $\widehat{g}(\tilde{z})$ of $\E_{J,R|\tilde{Z}=\tilde{z}}\sbr{\ell(\tilde{F}_{\bar{J}}, \tilde{Y}_{\bar{J}}) - \ell(\tilde{F}_J, \tilde{Y}_J)}$.
Note that this latter quantity is not the expected generalization gap yet, as it still misses an expectation over $\tilde{Z}$.
In the figures of this section, we report the mean and standard deviation of $\widehat{g}(\tilde{Z})$ estimated using the $k_1$ samples of $\tilde{Z}$, unless $k_1 = 1$ in which case we only report the single estimate.
Note that this mean will be an unbiased estimate of the true expected generalization gap $\E_{S,R}\sbr{R(f) - r_S(f)}$.

\paragraph{Estimation of $f$-CMI bound.}
Similarly, for each $\tilde{z}$ we use the $k_2$ samples of $J$ and $R$ to estimate $\ofcmi(f, \tilde{z}, \mathset{i}) = I^{\tilde{Z}=\tilde{z}}(\tilde{F}_i; J_i),\ i\in[n]$.
As in all considered cases we deal with classification problems (i.e., having discrete output variables), this is done straightforwardly by estimating all the states of the joint distribution of $\tilde{F}_i$ and $J_i$ given $\tilde{Z}=\tilde{z}$, and then using a plug-in estimator of mutual information.
The bias of this plug-in estimator is $\rbr{\frac{1}{k_2}}$, while the variance is $\rbr{\frac{(\log k_2)^2}{k_2}}$~\citep{paninski2003estimation}.
To estimate $\fcmi(f,\mathset{i})=\E_{\tilde{Z}}\sbr{\ofcmi(f,\tilde{Z},\mathset{i})}$ we use $k_1$ samples of $\tilde{Z}$.
After this step the estimation bias stays the same, while the variance increases by $\rbr{\frac{1}{k_1}}$.
In figures of this section, we report the mean and standard deviation of our estimate of $\fcmi(f,\mathset{i})$ computed using $k_1$ samples of $\tilde{Z}$.

\begin{figure}
    \centering
    \begin{subfigure}{0.33\textwidth}
        \includegraphics[width=\textwidth]{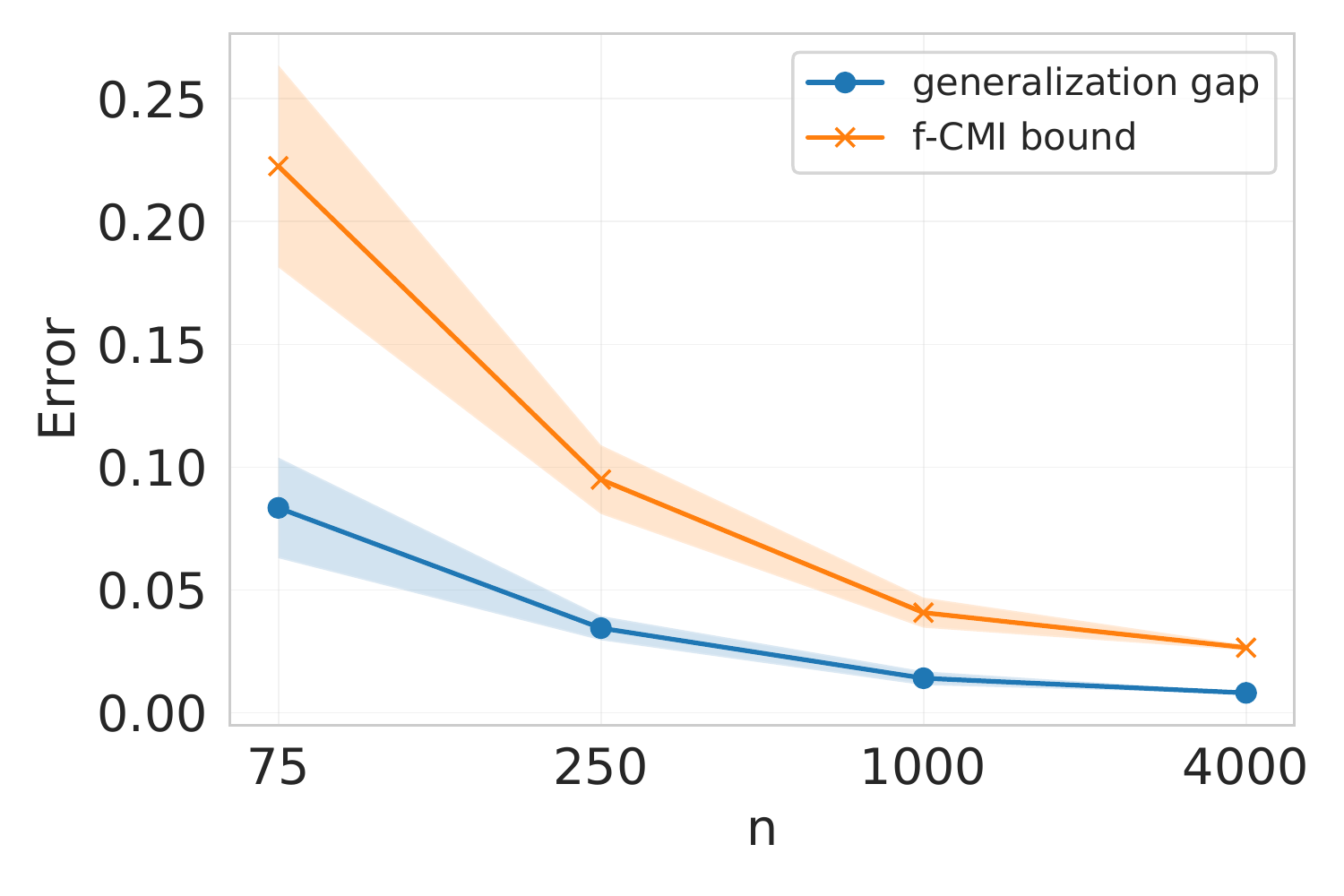}
        \caption{using the default parameters}
        \label{fig:mnist4vs9-cnn-x=n-deterministic}
    \end{subfigure}%
    \begin{subfigure}{0.33\textwidth}
        \includegraphics[width=\textwidth]{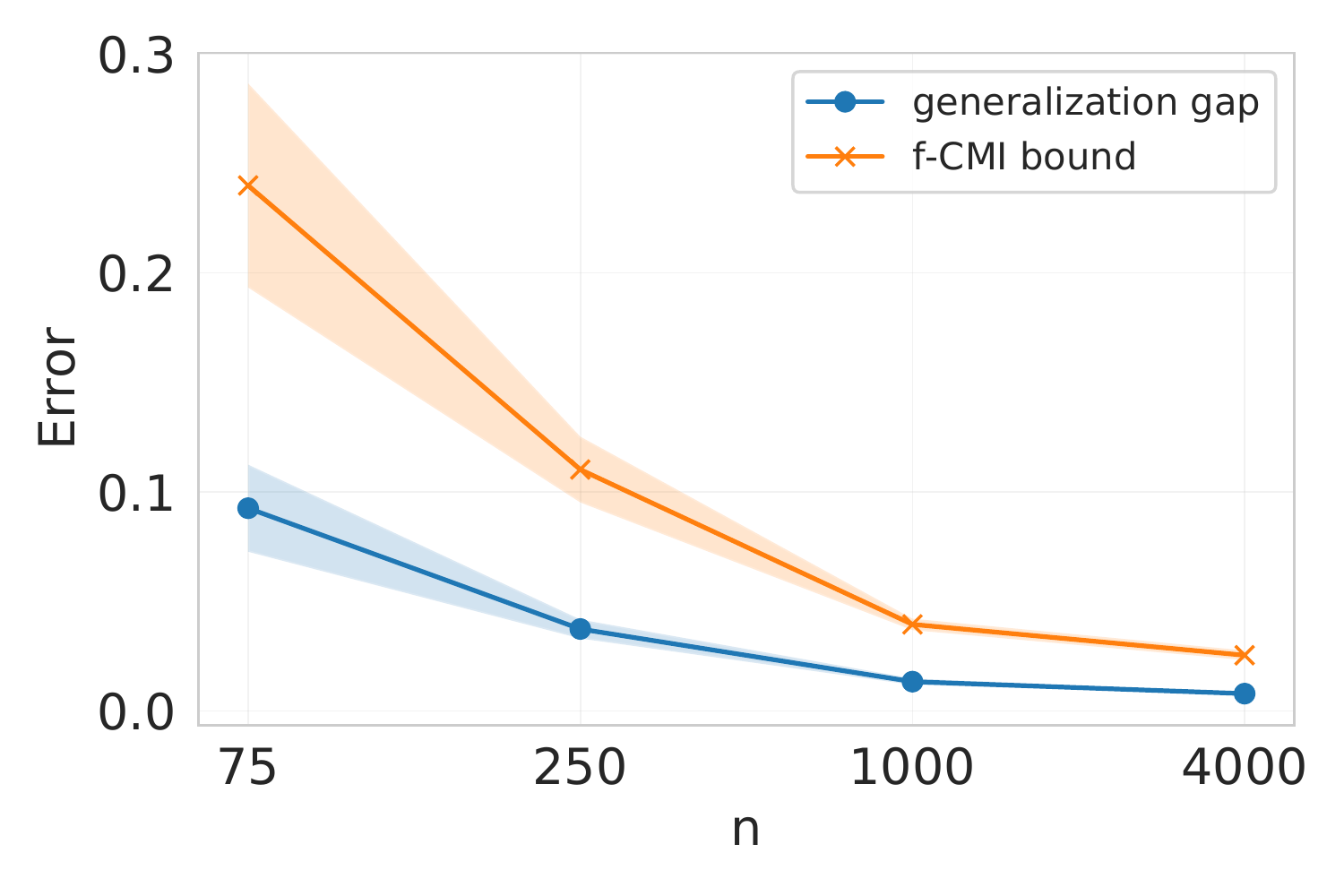}
        \caption{using 4 times wider network}
        \label{fig:mnist-4vs-9-wide}
    \end{subfigure}%
    \begin{subfigure}{0.33\textwidth}
        \includegraphics[width=\textwidth]{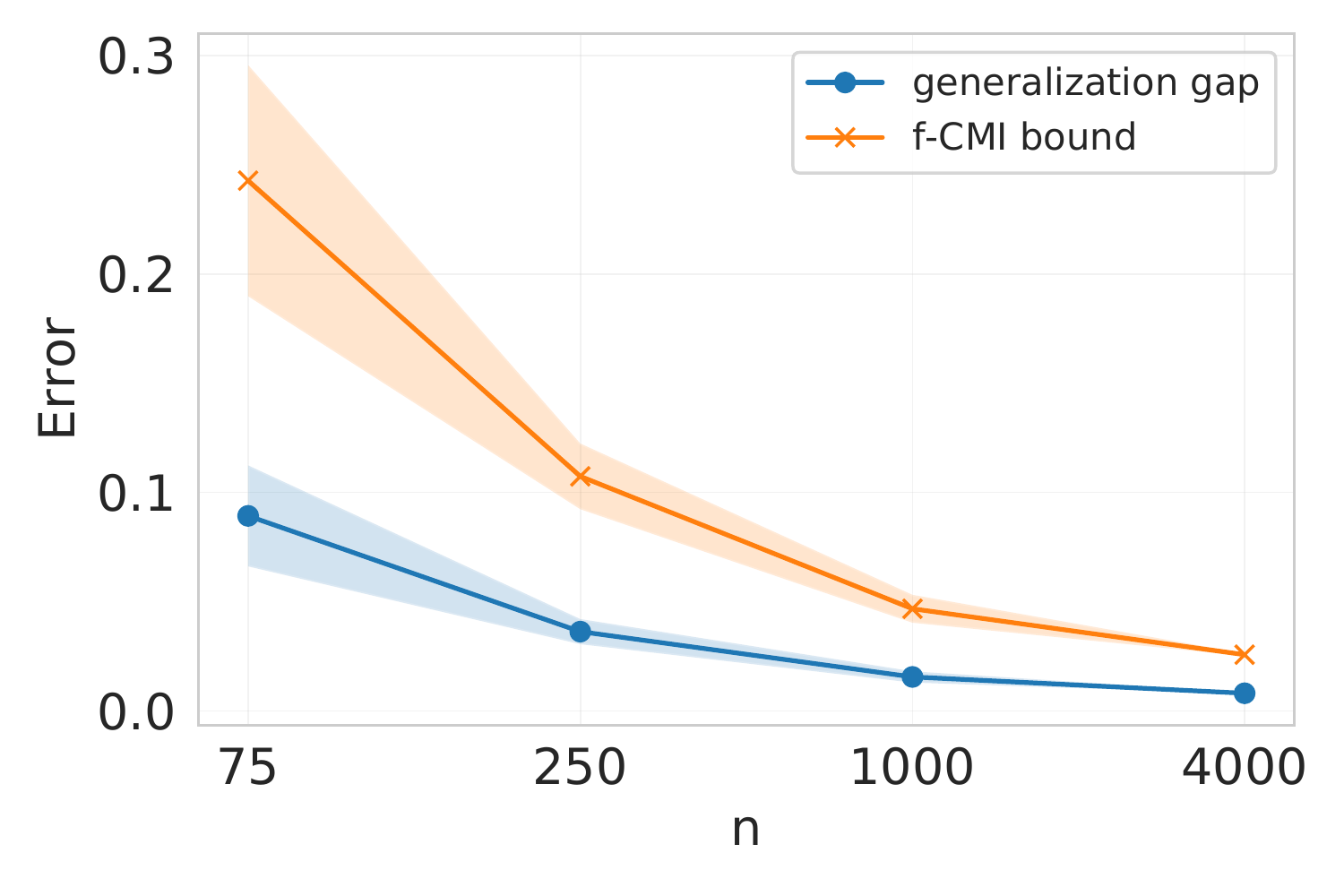}
        \caption{training with a random seed}
        \label{fig:mnist-4vs-9-stochastic}
    \end{subfigure}
    \caption{Comparison of expected generalization gap and $f$-CMI bound for MNIST 4 vs 9 classification with a 5-layer CNN. Panel (a) shows the results for the fixed-seed deterministic algorithm. Panel (b) repeats the experiment of panel (a) while modifying the network to have 4 times more neurons at each hidden layer. Panel (c) repeats the experiment of panel (a) while making the training algorithm stochastic by randomizing the seed.
    }
    \label{fig:mnist-4vs9-fcmi}
\end{figure}

\subsection{Results and discussion}
First, we consider the MNIST 4 vs 9 digit classification task~\citep{lecun2010mnist} using a 5-layer convolutional neural network (CNN) described in \cref{tab:4-layer-cnn-arch} (but with 2 output units) that has approximately 200K parameters.
We train the network with the cross-entropy loss for 200 epochs using the ADAM algorithm~\citep{kingma2014adam} with 0.001 learning rate, $\beta_1=0.9$, and mini-batches of 128 examples.
Importantly, we fix the random seed that controls the initialization of weights and the shuffling of training data, making the training algorithm deterministic.
To estimate generalization gap and the bound, we set $k_1 = 5$ and $k_2 = 30$.
\cref{fig:mnist4vs9-cnn-x=n-deterministic} plots the expected generalization gap and the $f$-CMI bound of \cref{eq:f-cmi-bound-m=1}.
We see that the bound is not vacuous and is not too far from the expected generalization gap even when considering only 75 training examples -- multiple orders of magnitude smaller than the number of parameters.
As shown in the \cref{fig:mnist-4vs-9-wide}, if we increase the width of all layers 4 times, making the number of parameters approximately 3M, the results remain largely unchanged.

Next, we move away from binary classification and consider the CIFAR-10 classication task~\citep{krizhevsky2009learning}.
To construct a well-generalizing algorithm, we use the ResNet-50~\citep{he2016deep} network pretrained on the ImageNet~\citep{deng2009imagenet}, and fine-tune it with the cross-entropy loss for 40 epochs using SGD with mini-batches of size 64, 0.01 learning rate, 0.9 momentum, and standard data augmentations.
The results presented in \cref{fig:cifar-10-pretrained-resnet50} indicate that the $f$-CMI bound is always approximately 3 times larger than the expected generalization gap.
In particular, when $n=20000$, the expected generalization gap is 5\%, while the bound predicts 16\%.

\begin{figure}
    \centering
    \includegraphics[width=0.35\textwidth]{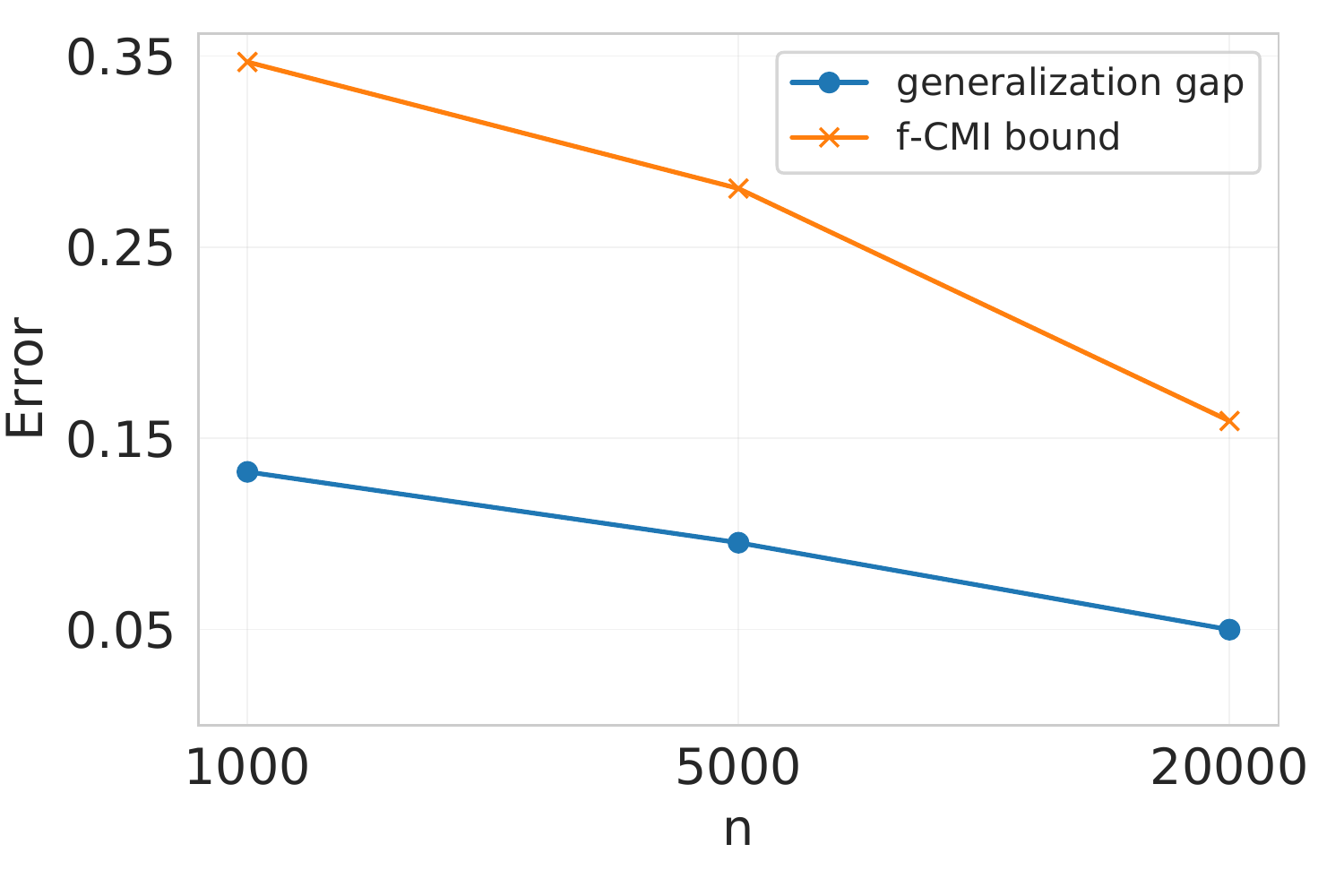}
    \caption{Comparison of expected generalization gap and $f$-CMI bound for a pretrained ResNet-50 fine-tuned on CIFAR-10 in a standard fashion.}
    \label{fig:cifar-10-pretrained-resnet50}
\end{figure}

\begin{figure}
    \centering
    \begin{subfigure}{0.33\textwidth}
        \includegraphics[width=\textwidth]{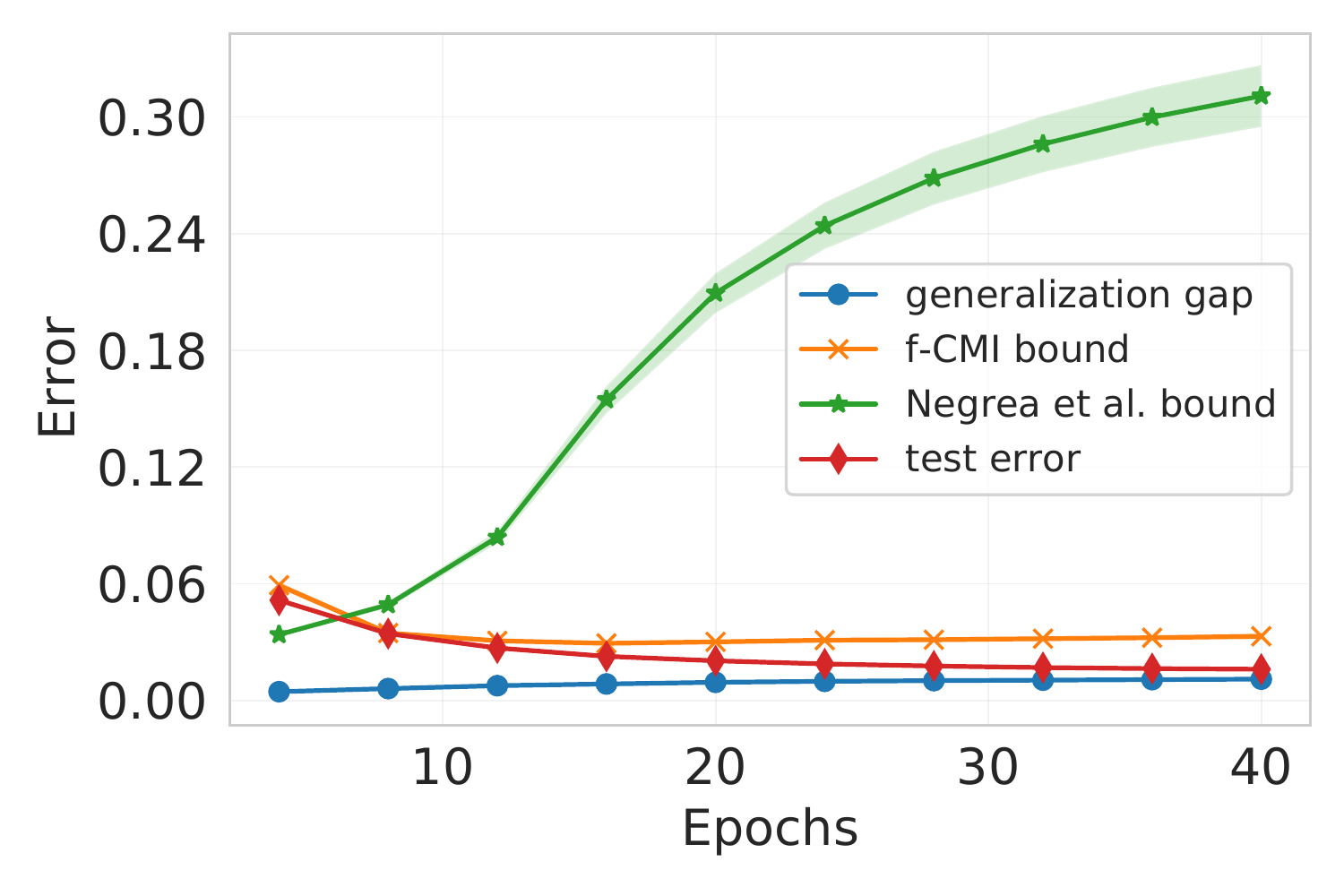}
        \caption{MNIST 4 vs 9}
        \label{fig:mnist-4vs9-langevin-dynamics}
    \end{subfigure}
    \begin{subfigure}{0.33\textwidth}
        \includegraphics[width=\textwidth]{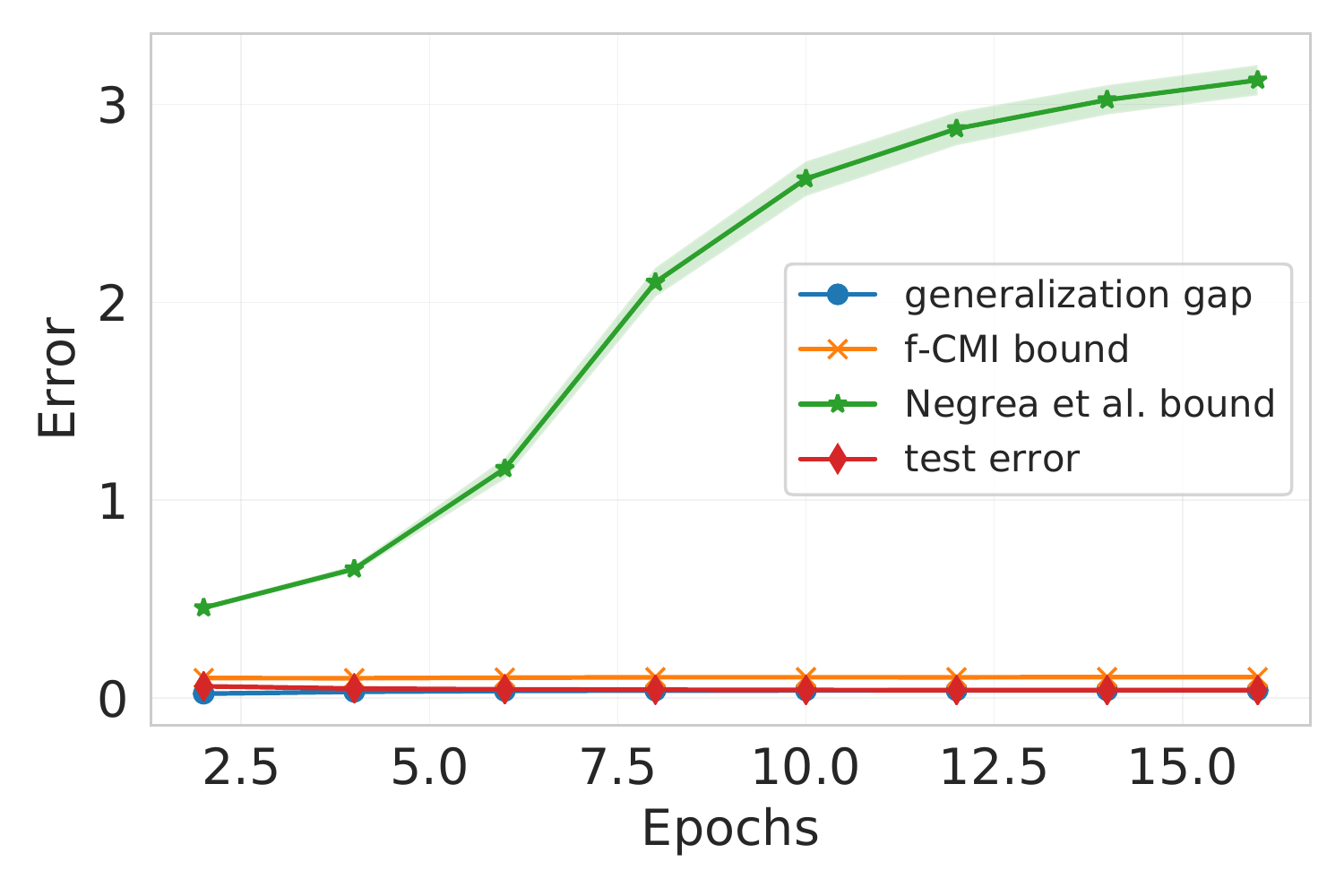}
        \caption{CIFAR-10}
        \label{fig:cifar-10-pretrained-resnet50-LD-w-SGLD}
    \end{subfigure}%
    \begin{subfigure}{0.33\textwidth}
        \includegraphics[width=\textwidth]{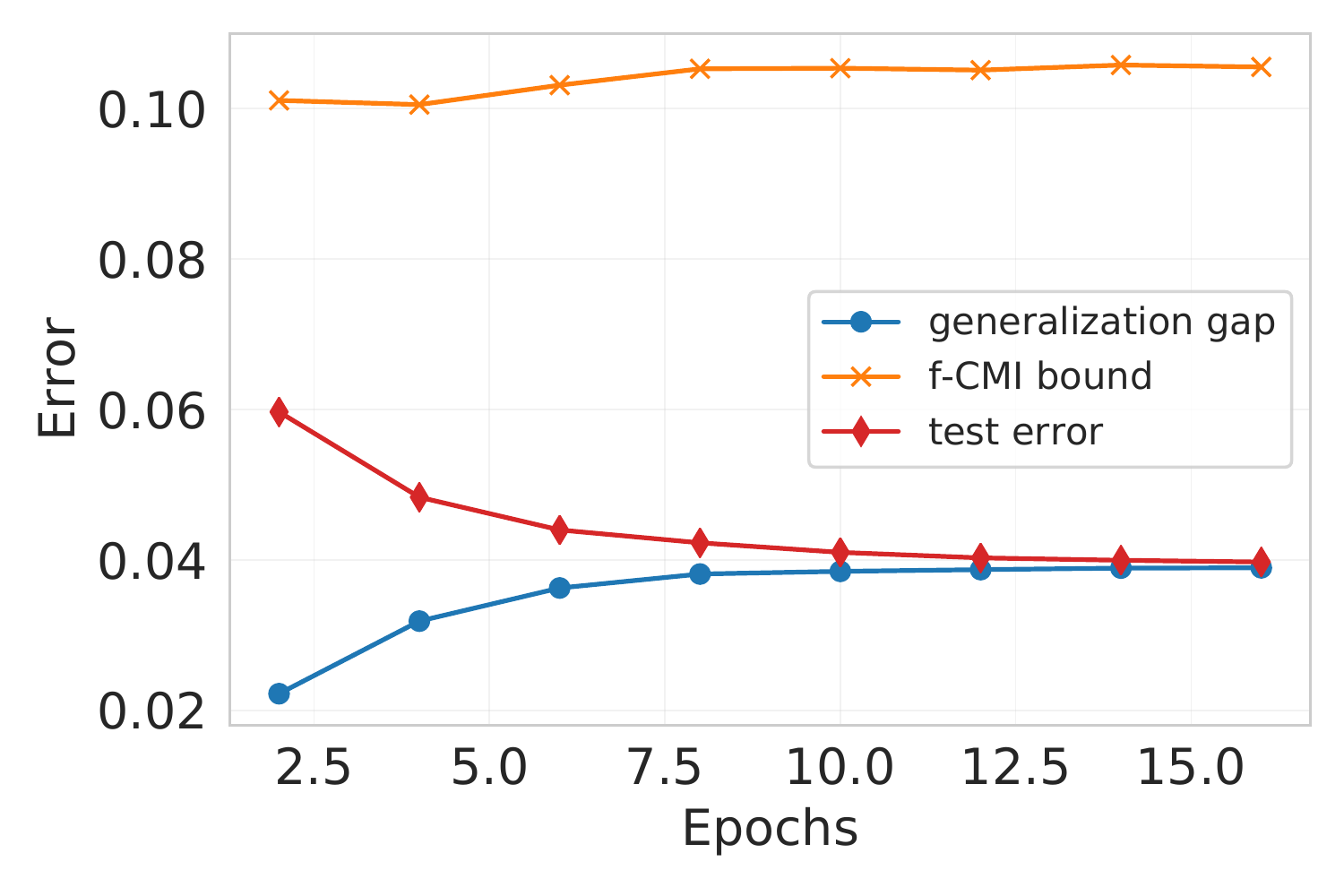}
        \caption{CIFAR-10, zoomed-in}
        \label{fig:cifar-10-pretrained-resnet50-LD-wno-SGLD}
    \end{subfigure}
    \caption{Comparison of expected generalization gap, \citet{negrea2019information} SGLD bound and $f$-CMI bound in case of a pretrained ResNet-50 fine-tuned with SGLD on a subset of CIFAR-10 of size $n=20000$. The figure on the right is the zoomed-in version of the figure in the middle.
    }
    \label{fig:cifar-10-pretrained-resnet50-LD}
\end{figure}

Note that the weight-based information-theoretic bounds discussed in \cref{sec:w-gen-bounds} would give either infinite or trivial bounds for the deterministic algorithm described above.
Even when we make the training algorithm stochastic by randomizing the seed, the quantities like $I(W; S)$ still remain infinite, while both the generalization gap and the $f$-CMI bound do not change significantly (see \cref{fig:mnist-4vs-9-stochastic}).
For this reason, we change the training algorithm to Stochastic Gradient Langevin Dynamics (SGLD)~\citep{gelfand1991recursive, welling2011bayesian} and compare the $f$-CMI-based bound against the specialized bound of \citet{negrea2019information} (see eq. (6) of \citep{negrea2019information}).
This bound (referred as SGLD bound here) is derived from a weight-based information-theoretic generalization bound, and depends on the the hyperparameters of SGLD and on the variance of per-example gradients along the training trajectory.
The SGLD algorithm is trained for 40 epochs, with learning rate and inverse temperature schedules described in \cref{tab:cifar10-langevin} of \cref{app:add-exp-details}.
\cref{fig:mnist-4vs9-langevin-dynamics} plots the expected generalization gap, the expected test error, the $f$-CMI bound and the SGLD bound.
We see that the test accuracy plateaus after 16 epochs.
Starting at epoch 8, the estimated $f$-CMI bound is always smaller than the SGLD bound.
Furthermore, as one increases the number of epochs, the former stays small, while the latter increases to very high values.

The difference between the $f$-CMI bound and the SGLD bound becomes more striking when we change the dataset to be a subset of CIFAR-10 consisting of 20000 examples, and fine-tune a pretrained ResNet-50 with SGLD. As shown in \cref{fig:cifar-10-pretrained-resnet50-LD}, even after a single epoch the SGLD bound is approximately 0.45, while the generalization gap is around 0.02.
For comparison, the $f$-CMI is approximately 0.1 after one epoch of training.

% Interestingly, \cref{fig:mnist-4vs9-langevin-dynamics} shows that the f-CMI bound is large in the early epochs, despite of the extremely small generalization gap.
% This indicates a possible area of improvement for the $f$-CMI bound.

%% file: sample-info/sections/related.tex
This work is closely related to a rich literature of information-theoretic generalization bounds, some of which were discussed earlier~\citep{xu2017information, bassily2018learners, pensia2018generalization, negrea2019information, bu2020tightening, steinke2020reasoning, haghifam2020sharpened, hafez2020conditioning, alabdulmohsin2020towards, neu2021information, raginsky202110, esposito2021generalization}.
Most of these work derive generalization bounds that depend on a mutual information quantity measured between the output of the training algorithm and some quantity related to training data.
Different from this major idea, \citet{xu2017information} and \citet{russo2019much} discussed the idea of bounding generalization gap with the information between the input and the vector of loss functions computed on training examples.
This idea was later extended to the setting of conditional mutual information by \citet{steinke2020reasoning}.
These works are similar to ours in the sense that they move away from measuring information with weights, but they did not develop this line of reasoning enough to arrive to efficient bounds similar to \cref{corr:f-cmi-bound-m=1}.
Additionally, we believe that measuring information with the prediction function allows better interpretation and is easier to work with analytically.

Another related line of research are the stability-based bounds~\citep{bousquet2002stability, alabdulmohsin2015algorithmic, raginsky2016information,  bassily2016algorithmic, feldman2018calibrating, wang2016average, raginsky202110}.
In \cref{sec:w-gen-bounds} and \cref{sec:fcmi} we improve existing generalization bounds that use information stability.
In \cref{sec:stable-algorithms} we describe a technique of applying information stability bounds to deterministic algorithms.
The main idea is to add noise to predictions, but only for analysis purposes.
We employed the same idea in \cref{ch:unique-info} when defining smooth unique information of an individual example.
In fact, our notion of test-stability defined in \cref{sec:stable-algorithms} comes very close to the definition of functional sample information (\cref{eq:fsi-smoothed}).
A similar idea was used by \citet{neu2021information} in analyzing generalization performance of SGD.
More broadly our work is related to PAC-Bayes bounds and to classical generalization bounds. Please refer to the the survey by \citet{Jiang*2020Fantastic} for more information on these bounds.

Finally, our work has connections with attribute and membership inference attacks~\citep{shokri2017membership, yeom2018privacy, nasr2019comprehensive, gupta2021membership}.
Some of these works show that having a white-box access to models allows constructing better membership inference attacks, compared to having a black-box access.
This is analogous to our observation that prediction-based bounds are better than weight-based bounds. \citet{shokri2017membership} and \citet{yeom2018privacy} demonstrate that even in the case of \emph{black-box access to a well-generalizing} model, sometimes it is still possible to construct successful membership attacks.
This is in line with our observation that the $f$-CMI bound can be significantly large, despite of small generalization gap (see epoch 4 of \cref{fig:mnist-4vs9-langevin-dynamics}).
This suggests a possible direction of improving the $f$-CMI-based bounds.

%% file: limitations/main.tex
\section{Introduction}\label{sec:limitations-intro}
In the previous chapter we discussed various information-theoretic bounds based on different notions of training set information captured by the training algorithm~\citep{xu2017information, bassily2018learners, negrea2019information, bu2020tightening, steinke2020reasoning, haghifam2020sharpened, neu2021information,raginsky202110,hellstrom2020generalization, esposito2021generalization}.
The data and algorithm dependent nature of these bounds make them applicable to typical settings of deep learning, where powerful and overparameterized neural networks are employed.
Related to the considered information-theoretic generalization bounds are PAC-Bayes bounds, which are usually based on the Kullback-Leilber divergence from the distribution of hypotheses after training (the ``posterior'' distribution) to a fixed ``prior'' distribution~\citep{ShaweTaylor1997APA,mcallester1999some, mcallester1999pac, catonibook2007, alquier2021user}.
For both types of bounds the main conclusion is that when a training algorithm captures little information about the training data then the generalization gap should be small.
A few works demonstrated that some of the best information-theoretic and PAC-Bayes generalization bounds produce nonvacuous results in practical settings of deep learning~\citep{DBLP:conf/uai/DziugaiteR17,perez2021tighter,harutyunyan2021informationtheoretic}.

A key ingredient in recent improvements of information-theoretic generalization bounds was the introduction of sample-wise information bounds by \citet{bu2020tightening}, where one measures how much information on average the learned hypothesis has about a single training example.
While PAC-Bayes and information-theoretic bounds are intimately connected to each other, the technique of measuring information with single examples 
has not appeared in PAC-Bayes bounds.  
In this chapter we explain the curious omission of single example PAC-Bayes bounds by proving the non-existence of such bounds, revealing a striking difference between information-theoretic and PAC-Bayesian perspectives. 
The reason for this difference is that PAC-Bayes upper bounds the probability of average population and empirical risks being far from each other, while information-theoretic generalization methods upper bound expected difference of population and empirical risks, which is an easier task.

\subsection{Preliminaries}
Let us consider again the abstract learning setting (not necessarily supervised), in which
the learner observes a collection of $n$ i.i.d examples $S=(Z_1,\ldots,Z_n)$ sampled from $P_Z$ and outputs a hypothesis $W \sim Q_{W|S}$ belonging to a hypothesis space $\mathcal{W}$.
As before, $P_S$ and $Q_{W|S}$ induce a joint probability $P_{W,S}$ on $\mathcal{W} \times \mathcal{Z}^n$.
In this setting one can consider different types of generalization bounds.

\paragraph{Expected generalization gap bounds.}
The simplest quantity to consider is the expected generalization gap: $\abs{\E_{P_{W,S}}\sbr{R(W) - r_S(W)}}$, which is the left-hand side of the majority of the generalization bounds presented in the previous chapter.
We saw that when $\ell(w, Z')$ with $Z'\sim P_Z$ is $\sigma$-subgaussian for all $w\in\mathcal{W}$, the following bounds hold~\citep[\cref{thm:generalized-xu-raginsky}]{xu2017information,bu2020tightening}:
\begin{align}
    \abs{\E_{P_{W,S}}\sbr{R(W) - r_S(W)}} &\le \frac{1}{n}\sum_{i=1}^n\sqrt{2\sigma^2 I(W;Z_i)}\label{eq:bu-et-al-bound}\\
    &\le \E_{U}\sbr{\sqrt{\frac{2\sigma^2 I^{U}(W;S_{U})}{m}}}\label{eq:random-subset-bound}\\
    &\le \sqrt{\frac{2\sigma^2 I(W;S)}{n}}\label{eq:full-info-bound},
\end{align}
where $U$ is a uniformly random subset of $[n]$, independent of $S$ and $W$.
We call bounds like \cref{eq:bu-et-al-bound} that depend on information quantities related to individual examples \emph{sample-wise} bounds.
The difference between sample-wise bounds and bounds that depend on $I(W;S)$ can be significant.
In fact, there are cases when the sample-wise bound of \cref{eq:bu-et-al-bound} is finite, while the bound of \cref{eq:full-info-bound} is infinite~\citep{bu2020tightening}.

\paragraph{PAC-Bayes bounds.}
A practically more useful quantity is the average difference between population and empirical risks for a fixed training set $S$: $\E_{P_{W|S}}\sbr{R(W) - r_S(W)}$.
Typical PAC-Bayes bounds are of the following form: with probability at least $1-\delta$ over $P_S$,
\begin{equation}
    \E_{P_{W | S}}\sbr{R(W) - r_S(W)} \le B\rbr{\KL{P_{W|S}}{\pi}, n, \delta},
    \label{eq:typical-pac-bayes}
\end{equation}
where $\pi$ is a prior distribution over $\mathcal{W}$ that does not depend on $S$.
If we choose the prior distribution to be the marginal distribution of $W$ (i.e., $\pi = P_W=\E_{S}\sbr{Q_{W|S}}$), then the KL term in \cref{eq:typical-pac-bayes} will be the integrand of the mutual information, as $I(W;S) = \E_{S}\sbr{\KL{P_{W|S}}{P_W}}$.
When the function $B$ depends on the KL term linearly, the expectation of the bound over $S$ will depend on the mutual information $I(W;S)$.
There are no known PAC-Bayes bounds where $B$ depends on KL divergences of form $\KL{P_{W|Z_i}}{P_W}$ or on sample-wise mutual information $I(W; Z_i)$.

\paragraph{Single-draw bounds.}
In practice, even when the learning algorithm is stochastic, one usually draws a single hypothesis $W \sim Q_{W|S}$ and is interested in bounding the population risk of $W$. Such bounds are called single-draw bounds and are usually of the following form: with probability at least $1-\delta$ over $P_{W,S}$:
\begin{equation}
    R(W) - r_S(W) \le B\rbr{\der{Q_{W|S}}{\pi}, n,\delta},
\end{equation}
where $\pi$ is a prior distribution as in PAC-Bayes bounds.
When $\pi = P_W$ the single-draw bounds depends on the information density $\iota(W,S) = \der{Q_{W|S}}{P_W}$.
Single-draw bounds are the hardest to obtain and no sample-wise versions are known for them either.

\paragraph{Expected squared generalization gap bounds.}
In terms of difficulty, expected generalization bounds are the easiest to obtain, and as indicated above, some of those bounds are sample-wise bounds.
If we consider a simple change of moving the absolute value inside: $\E_{P_{W,S}}\abs{R(W) - r_S(W)}$, then no sample-wise bounds are known. 
The same is true for the expected squared generalization gap $\E_{P_{W,S}}\rbr{R(W) - r_S(W)}^2$, for which the known bounds are of the following form~\citep[\cref{thm:generalized-xu-raginsky}]{steinke2020reasoning, aminian2021information}:
\begin{equation}
    \E_{P_{W,S}}\sbr{\rbr{R(W) - r_S(W)}^2} \le \frac{I(W;S) + c}{n},
\end{equation}
where c is some constant.
From this result one can upper bound the expected absolute value of generalization gap using Jensen's inequality.

\subsection{Our contributions}
In \cref{sec:counterexample} we show that even for the expected squared generalization gap, sample-wise information-theoretic bounds are impossible.
The same holds for PAC-Bayes and single-draw bounds as well.
In \cref{sec:implications} we discuss the consequences for other information-theoretic generalization bounds.
Finally, in \cref{sec:m=2} we show that starting at subsets of size 2, there are expected squared generalization gap bounds that measure information between $W$ and a subset of examples.
This in turn implies that such PAC-Bayes and single-draw bounds are also possible, albeit they are not tight and high-probability bounds.

\section{A useful lemma}
We first prove a lemma that will be used in the proof of the main result of the next section.
\begin{lemma}
Consider a collection of bits $a_1,\ldots,a_{N_0 + N_1}$, such that $N_0$ of them are zero and the remaining $N_1$ of them are one.
We want to partition these numbers into $k=(N_0 + N_1)/n$ groups of size $n$ (assuming that $n$ divides $N_0 + N_1$).
Consider a uniformly random ordered partition $(A_1,\ldots,A_k)$.
Let $Y_i = \oplus_{a \in A_i} a$ be the parity of numbers of subset $A_i$.
For any given $\delta > 0$, there exists $N'$ such that when $\min\mathset{N_0,N_1} \ge N'$ then $\cov\sbr{Y_i, Y_j} \le \delta\E\sbr{Y_1}^2$, for all $i, j \in [k], i \neq j$.
\label{lemma:cov}
\end{lemma}
\begin{proof}
By symmetry all random variables $Y_i$ are identically distributed.
Without loss of generality let's prove the result for the covariance of $Y_1$ and $Y_2$, which can be written as follows:
\begin{align*}
    \cov\sbr{Y_1,Y_2} &= \E\sbr{Y_1 Y_2} - \E\sbr{Y_1}\E\sbr{Y_2}\\
    &\hspace{-4em}= P_{Y_1,Y_2}\rbr{Y_1 = 1, Y_2 = 1} - P(Y_1=1)P(Y_2=1).\numberthis\label{eq:cov_y1_y2}
\end{align*}
Consider the process of generating a uniformly random ordered partition by first picking $n$ elements for the first subset, then $n$ elements for the second subset, and so on.
In this scheme, the probability that both $Y_1=1$ and $Y_2=1$ equals to
\begin{align}
\frac{1}{M} &\sum_{u=0}^{\floor{(n-1)/2}}\sum_{v=0}^{\floor{(n-1)/2}}\Choose{N_1}{2u+1}\Choose{N_0}{n-2u-1}\Choose{N_1-2u-1}{2v+1}\Choose{N_0-n+2u+1}{n-2v-1},\label{eq:joint-prob}
\end{align}
where
\begin{equation}
M = \Choose{N_0+N1}{n}\Choose{N_0+N_1-n}{n}.
\end{equation}
Let $q_{u,v}$ be the $(u,v)$-th summand of \cref{eq:joint-prob} divided by $M$.
On the other hand, the product of marginals $P(Y_1=1) P(Y_2=1)$ is equal to
\begin{align*}
\frac{1}{M'} &\sum_{u=0}^{\floor{(n-1)/2}}\sum_{v=0}^{\floor{(n-1)/2}}\left[ \Choose{N_1}{2u+1}\Choose{N_0}{n-2u-1}\Choose{N_1}{2v+1}\Choose{N_0}{n-2v-1}\right],\numberthis\label{eq:product-of-marginals-prob}
\end{align*}
where
\begin{equation}
M' = \Choose{N_0 + N_1}{n}\Choose{N_0 + N_1}{n}.
\end{equation}
Let $q'_{u,v}$ be the $(u,v)$-th summand of \cref{eq:product-of-marginals-prob} divided by $M'$.
Consider the ratio $q_{u,v}/q'_{u,v}$:
\begin{align}
\frac{q_{u,v}}{q'_{u,v}} = \frac{\Choose{N_0 + N_1}{n}}{\Choose{N_0+N_1-n}{n}}\cdot\frac{\Choose{N_1-2u-1}{2v+1}}{\Choose{N_1}{2v+1}}\cdot\frac{\Choose{N_0-n+2u+1}{n-2v-1}}{\Choose{N_0}{n-2v-1}}.
\end{align}
By picking $N_0$ and $N_1$ large enough, we will make $q_{u,v}$ close to $q'_{u,v}$.
This is possible as for any fixed $n$, these 3 fractions converge to 1 as $\min\mathset{N_0, N_1}\rightarrow \infty$.
Therefore, for any $\delta > 0$ there exists $N'$ such that when $\min\mathset{N_0, N_1}\ge N'$ then $q_{u,v} \le (1+\delta) q'_{u,v}$.
This implies that  $P_{Y_1,Y_2}(Y_1=1,Y_2=1) \le (1+\delta) P(Y_1=1) P(Y_2=1)$. Combining this result with \cref{eq:cov_y1_y2} proves the desired result.
\end{proof}

\section{A counterexample}\label{sec:counterexample}
\begin{theorem}
For any training set size $n=2^r$ and $\delta > 0$, there exists a finite input space $\ZZ$, a data distribution $P_Z$, a learning algorithm $Q_{W|S}$ with a finite hypothesis space $\mathcal{W}$, and a binary loss function $\ell : \mathcal{W} \times \ZZ \rightarrow \mathset{0, 1}$ such that 
\begin{itemize}
    \item[(a)] $\KL{Q_{W|S}}{P_{W}} \ge n-1$ with probability at least $1-\delta$,
    \item[(b)] $W$ and $Z_i$ are independent for each $i\in[n]$, 
    \item[(c)] $\E_{P_{W,S}}\sbr{R(W) - r_S(W)} = 0$,
    \item[(d)] $P_{W,S}\rbr{R(W) - r_S(W) \ge \frac{1}{4}} \ge \frac{1}{2}$.
\end{itemize}
\label{thm:main}
\end{theorem}
This result shows that there can be no meaningful sample-wise expected squared generalization bounds, as $I(W; Z_i) = 0$ while $\E_{P_{W,S}}\rbr{R(W) - r_S(W)}^2 \ge 1/16$.
Similarly, there can be no sample-wise PAC-Bayes or single-draw generalization bounds that depend on quantities such as $\KL{P_{W|Z_i}}{P_W}$, $I(W;Z_i)$ or $\iota(W, Z_i)$, as all of them are zero while with probability at least $1/2$ the generalization gap will be at least $1/4$.
The first property verifies that in order to make the generalization gap large $W$ needs to capture a significant amount of information about the training set.

The main idea behind the counterexample construction is to ensure that $W$ contains sufficient information about the whole training set but no information about individual examples.
This captured information is then used to make the losses on all training examples equal to each other, but possibly different for different training sets.
This way we induce significant variance of empirical risk.
Along with making the population risk to be roughly the same for all hypotheses, we ensure that the generalization gap will be large on average.
Satisfying the information and risk constraints separately is trivial, the challenge is in satisfying both at the same time.

To better understand this result and its implications, it is instructive to go into the details of the construction.

\begin{proof}[Proof of \cref{thm:main}]
Let $n=2^r$ and $\ZZ$ be the set of all binary vectors of size $d$: $\ZZ=\mathset{0,1}^d$ with $d > r$. 
Let $N=2^d$ denote the cardinality of the input space.
We choose the data distribution to be the uniform distribution on $\ZZ$: $P_Z = \mathrm{Uniform}(\ZZ)$.
Let the hypothesis set $\mathcal{W}$ be the set of all partitions of $\ZZ$ into subsets of size $n$:
\begin{equation}
\mathcal{W} = \mathset{\{A_1,\ldots,A_{N/n}\} \mid |A_i|=n, \cup_i A_i=\ZZ, A_i\cap A_j = \emptyset}.
\end{equation}
This hypothesis space is finite, namely, $\abs{\mathcal{W}}=\frac{N!}{(n!)^{N/n}(N/n)!}$.
When the training algorithm receives a training set $S$ that contains duplicate examples, it outputs a hypothesis from $\mathcal{W}$ uniformly at random.
When $S$ contains no duplicates, then the learning algorithm outputs a uniformly random hypothesis from the set $\mathcal{W}_S \subset \mathcal{W}$ of hypotheses/partitions that contain $S$ as a partition subset.
Formally,
\begin{equation}
    Q_{W|S} = \left\{\begin{array}{ll}
    \mathrm{Uniform}(\mathcal{W}), & \text{if $S$ has duplicates},\\
    \mathrm{Uniform}\rbr{\mathcal{W}_S}, & \text{otherwise},
    \end{array}\right.
\end{equation}
where $\mathcal{W}_S \triangleq \mathset{\mathset{A_1,\ldots,A_{N/n}} \in \mathcal{W} \mid \exists i \text{ s.t. } A_i = S}$.
Let $\rho(n,d)$ be the probability of $S$ containing duplicate examples.
By picking $d$ large enough we can make $\rho$ as small as needed.

Given a partition $w=\{A_1,\ldots,A_{2^d/n}\}\in\mathcal{W}$ and an example $z\in\ZZ$, we define $[z]_w$ to be the subset $A_i \in w$ that contains $z$.
Given a set of examples $A \subset \ZZ$, we define $\oplus^{(2)}(A)$ to be xor of all bits of all examples of $A$:
\begin{equation}
    \oplus^{(2)}(A) = \oplus_{(a_1,\ldots,a_d)\in A}\rbr{\oplus_{i=1}^d a_i}.
\end{equation}
Finally, we define the loss function as follows:
\begin{equation}
    \ell(w,z) = \oplus^{(2)}\rbr{[z]_w} \in \mathset{0,1}.
\end{equation}

Let $W\sim Q_{W|S}$. Let us verify now that properties (a)-(d) listed in the statement of \cref{thm:main} hold.

\paragraph{(a).} By symmetry, the marginal distribution $P_W = \E_S\sbr{Q_{W|S}}$ will be the uniform distribution over $\mathcal{W}$. With probability $1-\rho(n,d)$, the training set $S$ has no duplicates and $Q_{W|S} = \mathrm{Uniform}(\mathcal{W}_S)$. In such cases, the support size of $Q_{W|S}$ is equal to $ \frac{(N-n)!}{(n!)^{N/n-1}(N/n-1)!}$, while the support size of $P_W$ is always equal to $\frac{N!}{(n!)^{N/n}(N/n)!}$.
Therefore, 
\begin{align}
    \KL{Q_{W|S}}{P_W} &= \log\rbr{\frac{\abs{\mathrm{supp}(P_W)}}{\abs{\mathrm{supp}(Q_{W|S})}}}\\
    % &=\log\rbr{\frac{N!}{(n!)^{N/n}(N/n)!} \cdot \frac{(n!)^{N/n-1}(N/n-1)!}{(N-n)!}} \\
    &\hspace{-3em}=\log\rbr{\frac{(N-n+1)(N-n+2)\cdots N}{n!(N/n)}}\\
    &\hspace{-3em}\ge\log\rbr{\frac{n(N-n+1)^n}{n^n N}}\\
    &\hspace{-3em}=n\log\rbr{N-n+1} - (n-1)\log n - \log N\\
    &\hspace{-3em}\approx (n-1) d\log 2 - n\log n. &&\text{\hspace{-7.5em}(for a suff. large $d$)}
\end{align}
For large $d$, the quantity above is approximately $(n-1)d\log 2$, which is expected 
as knowledge of any $Z_i$ ($d$ bits) along with $W$ is enough to reconstruct the training set $S$ that has $nd$ bits of entropy.
To satisfy the property (a), we need to pick $d$ large enough to make $\rho(n,d)<\delta$ and $(n-1)d\log 2 - n\log n \ge n-1$.

\paragraph{(b).}
Consider any $z\in\ZZ$ and any $i \in [n]$. Then
    \begin{align}
        P_{W|Z_i=z}(W=w) &= \frac{P_W(W=w)P_{Z_i|W=w}(Z_i=z)}{P_Z(Z=z)}\\
        &= P_W(W=w),
    \end{align}
    where the second equality follows from the fact that conditioned on a fixed partition $w$, because of the symmetry, there should be no difference between probabilities of different $Z_i$ values.

\paragraph{(c).}
With probability $1-\rho(n,d)$,
    \begin{align}
        r_S(W) &= \frac{1}{n}\sum_{i=1}^n \ell(W,Z_i)\\
        &= \frac{1}{n}\sum_{i=1}^n \oplus^{(2)}([Z_i]_W)\\
        &= \frac{1}{n}\sum_{i=1}^n \oplus^{(2)}(S)\\
        &= \oplus^{(2)}(S) \in \{0,1\}\numberthis\label{eq:emp_risk}.
    \end{align}
    Furthermore, 
    \begin{align}
        R(W) &= \E_{P_Z}\sbr{\ell(W,Z)} = \frac{1}{N/n} \sum_{A \in W} \oplus^{(2)}(A).
        \numberthis\label{eq:risk}
    \end{align}
    Given \cref{eq:emp_risk} and \cref{eq:risk}, due to symmetry  $\E_{P_{W,S}}\sbr{r_S(W)}=1/2$ and $\E_{P_{W,S}}\sbr{R(W)} = 1/2$. Hence, the expected generalization gap is zero: $\E_{P_{W,S}}\sbr{R(W) - r_S(W)} = 0$.

\paragraph{(d).}
Consider a training set $S$ that has no duplicates and let $W=\mathset{S, A_1,\ldots,A_{N/n-1}} \sim Q_{W|S}$.
The population risk can be written as follows:
\begin{align}
    R(W) = \underbrace{\frac{n}{N} \oplus^{(2)}(S)}_{\le n/N} + \frac{n}{N} \sum_{i=1}^{N/n-1} \underbrace{\oplus^{(2)}(A_i)}_{\triangleq Y_i}.
\end{align}
Consider the set $\ZZ \setminus S$. Let $N_0$ be the number of examples in this set with parity 0 and $N_1$ be the number of examples with parity 1. We use \cref{lemma:cov} to show that $Y_i$ are almost pairwise independent. Formally, for any $\delta'>0$ there exists $N'$ such that when $\min\mathset{N_0, N_1} \ge N'$, we have that $\cov\sbr{Y_i, Y_j} \le \delta'\E\sbr{Y_1}^2$, for all $i, j \in [N/n-1], i\neq j$. Therefore,
\begin{align}
    \var\sbr{\frac{n}{N} \sum_{i=1}^{N/n-1}Y_i} &= \frac{n^2}{N^2} \rbr{\frac{N}{n}-1} \var\sbr{Y_1}\\
    &+\frac{n^2}{N^2}\rbr{\frac{N}{n}-1}\rbr{\frac{N}{n}-2}\cov\sbr{Y_1, Y_2}\\
    &\le \frac{n}{4N} + \delta'.
\end{align}
By Chebyshev's inequality
\begin{align}
P\rbr{\abs{\frac{n}{N}\sum_{i=1}^{N/n-1} Y_i - \frac{n}{N}\rbr{\frac{N}{n}-1}\E\sbr{Y_1}} \ge t} \le \frac{\frac{n}{4N} + \delta'}{t^2}.
\end{align}
Furthermore, $\E\sbr{Y_1}\rightarrow 1/2$ as $d \rightarrow \infty$. Therefore, we can pick a large enough $d$, appropriate $\delta'$ and $t$ to ensure that with at least 0.5 probability over $P_{W,S}$, $R(W) \in [1/4, 3/4]$ and $r_S(W) \in \mathset{0,1}$.
\end{proof}

\section{Implications for other bounds}\label{sec:implications}
Information-theoretic generalization bounds have been improved or generalized in many ways.
A few works have proposed to use other types of information measures and distances between distributions, instead of Shannon mutual information and Kullback-Leibler divergence respectively~\citep{esposito2021generalization,aminian2021jensen,rodriguez2021tighter}.
In particular, ~\citet{rodriguez2021tighter} derived expected generalization gap bounds that depend on the Wasserstein distance between $P_{W|Z_i}$ and $P_W$. \citet{aminian2021jensen} derived similar bounds but that depend on sample-wise Jensen-Shannon information $I_{JS}(W; Z_i)\triangleq\mathrm{JS}\rbr{P_{W,Z_i}\left\vert\right\vert P_W \otimes P_{Z_i}}$ or on lautum information $L(W;Z_i) \triangleq \KL{P_W \otimes P_{Z_i}}{P_{W,Z_i}}$.
\citet{esposito2021generalization} derive bounds on probability of an event in a joint distribution $P_{X,Y}$ in terms of the probability of the same event in the product of marginals distribution $P_X \otimes P_Y$ and an information measure between $X$ and $Y$ (Sibson’s $\alpha$-mutual information, maximal leakage, $f$-mutual information) or a divergence between $P_{X,Y}$ and $P_X \otimes P_Y$ (R\'{e}nyi $\alpha$-divergences, $f$-divergences, Hellinger divergences).
They note that one can derive in-expectation generalization bounds from these results. These results will be sample-wise if one starts with $X=W$ and $Y=Z_i$ and then takes an average over $i\in[n]$.
The property (b) of the counterexample implies that there are no PAC-Bayes, single-draw, or expected squared generalization bounds for the aforementioned information measures or divergences, as all of them will be zero when $\forall z\in\ZZ, P_W = P_{W|Z_i=z}$.

PAC-Bayes bounds have been improved by comparing  population and empirical risks differently (instead of just subtracting them)~\citep{langford2001bounds,germain2009pac,rivasplata2020pac}. The property (d) of the counterexamples implies that these improvements will not make sample-wise PAC-Bayes bounds possible, as the changed distance function will be at least constant when $r_S(W) \in\mathset{0,1}$ while $R(W) \in [1/4,3/4]$.

Another way of improving information-theoretic bounds is to use the random subsampling setting introduced by \citet{steinke2020reasoning}.
In this setting one considers $2n$ i.i.d. samples from $P_Z$ grouped into $n$ pairs: $\tilde{Z} \in \ZZ^{n\times 2}$.
A random variable $J \sim \mathrm{Uniform}(\mathset{0,1}^n)$, independent of $\tilde{Z}$, specifies which example to select from each pair to form the training set $S = (\tilde{Z}_{i,J_i})_{i=1}^n$.
\citet{steinke2020reasoning} proved that if $\ell(w,z) \in [0,1], \forall w \in \WW, z \in \ZZ$, then the expected generalization gap can be bounded as follows:
\begin{equation}
    \abs{\mathbb{E}_{S,W}\sbr{R(W) - r_S(W)}} \le \sqrt{\frac{2}{n} I(W; J \mid \tilde{Z})}.
\end{equation}
This result was improved in many works, leading to the following sample-wise bounds~\citep{haghifam2020sharpened, harutyunyan2021informationtheoretic,rodriguez2021random,Zhou2021IndividuallyCI}:
\begin{align}
    \abs{\mathbb{E}_{S,W}\sbr{R(W) - r_S(W)}} &\le \frac{1}{n}\sum_{i=1}^n \E_{\tilde{Z}_i}\sbr{\sqrt{2 I^{\tilde{Z}_i}(W; J_i)}}\label{eq:random-subsampling-sawple-wise-strong},\\
    \abs{\mathbb{E}_{S,W}\sbr{R(W) - r_S(W)}}&\le \frac{1}{n}\sum_{i=1}^n \E_{\tilde{Z}}\sbr{\sqrt{2 I^{\tilde{Z}}(W; J_i)}}\label{eq:random-subsampling-sawple-wise-weak},
\end{align}
where $\tilde{Z}_i=(\tilde{Z}_{i,1}, \tilde{Z}_{i,2})$ is the $i$-th row of $\tilde{Z}$.
Given a partition $W \sim Q_{W|S}$ and two examples $\tilde{Z}_i$, one cannot tell which of the examples was in the training set because of the symmetry.
Hence, the counterexample implies that expected squared, PAC-Bayes, and single draw generalization bounds that depend on quantities like $I^{\tilde{Z}_i}(W; J_i)$ cannot exist.
However, if we consider the weaker sample-wise bounds of \cref{eq:random-subsampling-sawple-wise-weak}, then the knowledge of $\tilde{Z}$ helps to reveal the entire $J$ at once with high probability.
This can be done by going over all possible choices of $J$ and checking whether $\tilde{Z}_J = (\tilde{Z}_{i, J_i})$ belongs to partition $W$.
This will be true for the true value of $J$, but will be increasingly unlikely for all other values of $J$ as $n$ and $d$ are increased.
In fact, we derive the following expected squared generalization gap that depends on terms $I^{\tilde{Z}}(W; J_i)$.
\begin{theorem}
In the random subsampling setting, let $W \sim Q_{W|S}$. If $\ell(w,z)\in[0,1]$, then
\begin{equation}
    \E_{P_{W,S}}\sbr{\rbr{R(W) - r_S(W)}^2} \le \frac{5}{2n} + \frac{2}{n}\sum_{i=1}^n \E_{\tilde{Z}}\sqrt{2 I^{\tilde{Z}}(W; J_i)}.
\end{equation}
\label{thm:g2-weak-sample-wise-bound}
\end{theorem}

Therefore, the counterexample is a discrete case where the bound of \cref{eq:random-subsampling-sawple-wise-strong} is much better than the weaker bound of \cref{eq:random-subsampling-sawple-wise-weak}.
It is also a case where $I(W; Z_i) \ll I(W; J_i \mid \tilde{Z})$ (i.e., CMI bounds are not always better).

Finally, another way of improving information-theoretic bounds is to use evaluated conditional mutual information (e-CMI)~\citep{steinke2020reasoning} or functional conditional mutual information (f-CMI, see \cref{ch:sample-info}).
Similarly to the functional CMI bounds, one can derive the following expected generalization gap bound.
\begin{theorem}
In the random subsampling setting, let $W \sim Q_{W|S}$. If $\ell(w,z)\in[0,1]$, then
\begin{align}
    \abs{\mathbb{E}\sbr{R(W) - r_S(W)}} &\le \frac{1}{n}\sum_{i=1}^n \sqrt{2 I(\ell(W, \tilde{Z}_i); J_i)},\label{eq:random-subsampling-sawple-wise-ECMI-strongest}\\
    \abs{\mathbb{E}\sbr{R(W) - r_S(W)}} &\le \frac{1}{n}\sum_{i=1}^n \E_{\tilde{Z}_i}\sbr{\sqrt{2 I^{\tilde{Z}_i}(\ell(W, \tilde{Z}_i); J_i)}},\label{eq:random-subsampling-sawple-wise-ECMI-strong}\\
    \abs{\mathbb{E}\sbr{R(W) - r_S(W)}} &\le \frac{1}{n}\sum_{i=1}^n \E_{\tilde{Z}}\sbr{\sqrt{2 I^{\tilde{Z}}(\ell(W, \tilde{Z}_i); J_i)}}\label{eq:random-subsampling-sawple-wise-ECMI-weak},
\end{align}
where $\ell(W, \tilde{Z}_i) \in \mathset{0,1}^2$ is the pair of losses on the two examples of $\tilde{Z}_i$.
\label{thm:ecmi-sample-wise-exp-gen-gap}
\end{theorem}

In the case of the counterexample bounds of \cref{eq:random-subsampling-sawple-wise-ECMI-strongest} and \cref{eq:random-subsampling-sawple-wise-ECMI-strong} will be zero as one cannot guess $J_i$ knowing losses $\ell(W,\tilde{Z}_i)$ and possibly also $\tilde{Z}_i$.
This rules out the possibility of such sample-wise expected squared, PAC-Bayes and single-draw generalization bounds.
Unlike the case of the weaker CMI bound of \cref{eq:random-subsampling-sawple-wise-weak}, the weaker e-CMI bound of \cref{eq:random-subsampling-sawple-wise-ECMI-weak} convergences to zero in the case of the counterexample as $d\rightarrow \infty$.
Therefore, the counterexample is a discrete case where the sample-wise e-CMI bound of \cref{eq:random-subsampling-sawple-wise-ECMI-weak} can be much stronger than the sample-wise CMI bound of \cref{eq:random-subsampling-sawple-wise-weak}.

\section{The case of \texorpdfstring{$m=2$}{m=2}}\label{sec:m=2}
In \cref{sec:limitations-intro} we mentioned that there are expected generalization bounds that are based on information contained in size $m$ subsets of examples.
In \cref{sec:counterexample} we showed that there can be no expected squared generalization bounds with $m=1$.
In this section we show that expected squared generalization bounds are possible for any $m \ge 2$.
For brevity, let $G^{(2)} \triangleq \E_{P_{W,S}}\sbr{\rbr{R(W) - r_S(W)}^2}$ denote 
expected squared generalization gap.

\begin{theorem} Assume $\ell(w,z) \in [0, 1]$. Let $W \sim Q_{W|S}$. Then
\begin{equation}
    G^{(2)} \le \frac{1}{n} + \frac{1}{n^2}\sum_{i \neq k} \sqrt{2 I\rbr{W; Z_i, Z_k}}.
\end{equation}
\label{thm:m=2}
\end{theorem}

\begin{proof}
We have that
\begin{align}
    G^{(2)} &= \E_{P_{W,S}}\sbr{\rbr{\frac{1}{n}\sum_{i=1}^n\rbr{\ell(W,Z_i) - R(W)}}^2}\\
    &= \underbrace{\frac{1}{n^2}\sum_{i=1}^n \E_{P_{W,Z_i}}\sbr{\rbr{\ell(W,Z_i) - R(W)}^2}}_{\le 1/n} \\
    &\quad+ \frac{1}{n^2}\sum_{i \neq k}\underbrace{\E\sbr{\rbr{\ell(W,Z_i) - R(W)}\rbr{\ell(W,Z_k) - R(W)}}}_{C_{i,k}}.
    % &\le \frac{1}{n} + \frac{1}{n^2}\sum_{i \neq k}C_{i,k}.
\end{align}
For bounding the $C_{i,k}$ terms we use \cref{lemma:mutual-info-lemma} with $\Phi = W$, $\Psi = (Z_i, Z_k)$, and $f(\Phi,(\Psi_1,\Psi_2)) = \rbr{\ell(\Phi, \Psi_1) - R(\Phi)}\rbr{\ell(\Phi,\Psi_2)-R(\Phi)}$.
As $f(\Phi,\Psi)$ is 1-subgaussian under $P_\Phi \otimes P_\Psi$, by the lemma
\begin{equation}
    \abs{\E_{P_{\Phi,\Psi}}\sbr{f(\Phi,\Psi)} - \E_{P_\Phi \otimes P_\Psi}\sbr{f(\Phi,\Psi)}} \le \sqrt{2 I(\Phi;\Psi)},
\end{equation}
which translates to
\begin{align}
    &\abs{C_{i,k} -  \underbrace{\E_{P_{W}P_{Z_i,Z_k}}\sbr{\rbr{\ell(W, Z_i) - R(W)}\rbr{\ell(W,Z_k)-R(W)}}}_{\bar{C}_{i,k}}} \le \sqrt{2 I(W;Z_i,Z_k)}.
\end{align}
It is left to notice that $\bar{C}_{i,k}=0$, as for any $w$ the factors $\rbr{\ell(w, Z_i) - R(w)}$ and $\rbr{\ell(w,Z_k)-R(w)}$ are independent and have zero mean. 
\end{proof}

\begin{corollary}
Let $U_m$ be a uniformly random subset of $[n]$ of size $m$, independent from $W$ and $S$. Then
\begin{equation}
    G^{(2)} \le \frac{1}{n} +  2\E_{U_m}\sbr{\sqrt{\frac{I^{U_m}(W; S_{U_m})}{m}}}.
\end{equation}
\label{corollary:m=k}
\end{corollary}
\begin{proof}
By \cref{prop:cmi-sharpened-generalized-choice-of-m} we have that for any $m\in[n-1]$,
\begin{equation}
\E_{U_m}\sbr{\sqrt{\frac{1}{m} I^{U_m}(W; S_{U_m})}} \le \E_{U_{m+1}}\sbr{\sqrt{{\frac{1}{m+1} I^{U_{m+1}}(W; S_{U_{m+1}})}}}.
\end{equation}
Therefore, for any $m=2,\ldots,n$, starting with \cref{thm:m=2},
\begin{align}
G^{(2)} &\le \frac{1}{n} + \frac{1}{n^2}\sum_{i \neq k} \sqrt{2 I\rbr{W; Z_i, Z_k}}\\
% &\le \frac{1}{n} + \frac{2}{n(n-1)}\sum_{i < k} \sqrt{2 I\rbr{W; Z_i, Z_k}}\\
&\le \frac{1}{n} + 2 \E_{U_2}\sbr{\sqrt{\frac{I^{U_2}(W; S_{U_2})}{2}}}\\
&\le \frac{1}{n} +  2 \E_{U_m}\sbr{\sqrt{\frac{I^{U_m}(W; S_{U_m})}{m}}}.
\end{align}
\end{proof}

At $m=n$ this bound is weaker that the bound of \cref{thm:generalized-xu-raginsky}, which depends on $I(W;S)/n$ rather than $\sqrt{I(W;S)/n}$.
We leave improving the bound of \cref{thm:m=2} for future work.
Nevertheless, \cref{corollary:m=k} shows that it is \emph{possible} to bound the expected \emph{squared} generalization gap with quantities that involve mutual information terms between $W$ and subsets of examples of size $m$, where $m \ge 2$ (unlike the case of $m=1$).
Possibility of bounding expected squared generalization gap with $m\ge 2$ information terms makes it possible for single-draw and PAC-Bayes bounds as well.
The simplest way is to use Markov's inequality, even though it will not give high probability bounds.

Bounds similar to \cref{thm:g2-weak-sample-wise-bound} but with CMI and e-CMI bounds are also possible, as shown by the following result.
\begin{theorem}
In the random subsampling setting, let $W \sim Q_{W|S}$. If $\ell(w,z)\in[0,1]$, then
\begin{align}
\E_{P_{W,S}}\sbr{\rbr{R(W) - r_S(W)}^2} &\le \frac{5}{2n} + \frac{2}{n^2}\sum_{i \neq k} \sqrt{2I(\ell(W,\tilde{Z}_i),\ell(W,\tilde{Z}_k); J_i, J_k)},\label{eq:g2-pairwise-ecmi-strongest}\\
\E_{P_{W,S}}\sbr{\rbr{R(W) - r_S(W)}^2} &\le \frac{5}{2n} + \frac{2}{n^2}\sum_{i \neq k} \E\sqrt{2I^{\tilde{Z}_i,\tilde{Z}_k}(\ell(W,\tilde{Z}_i),\ell(W,\tilde{Z}_k); J_i, J_k)},\label{eq:g2-pairwise-ecmi-strong}\\
\E_{P_{W,S}}\sbr{\rbr{R(W) - r_S(W)}^2} &\le \frac{5}{2n} + \frac{2}{n^2}\sum_{i \neq k} \E\sqrt{2I^{\tilde{Z}}(\ell(W,\tilde{Z}_i),\ell(W,\tilde{Z}_k); J_i, J_k)}\label{eq:g2-pairwise-ecmi-weak},
\end{align}
and
\begin{align}
\E_{P_{W,S}}\sbr{\rbr{R(W) - r_S(W)}^2} &\le \frac{5}{2n} + \frac{2}{n^2}\sum_{i \neq k} \E\sqrt{2I^{\tilde{Z}_i,\tilde{Z}_k}(W; J_i, J_k)},\label{eq:g2-pairwise-cmi-strong}\\
\E_{P_{W,S}}\sbr{\rbr{R(W) - r_S(W)}^2} &\le \frac{5}{2n} + \frac{2}{n^2}\sum_{i \neq k} \E\sqrt{2I^{\tilde{Z}}(W; J_i, J_k)}\label{eq:g2-pairwise-cmi-weak}.
\end{align}
\label{thm:g2-pairwise-CMI}
\end{theorem}

Finally, it is worth to mention that the bound of \cref{thm:m=2} holds for higher order moments of generalization gap too, as for $[0,1]$-bounded loss functions 
\begin{equation}
    \E_{P_{W,S}}\sbr{\rbr{R(W) - r_S(W)}^k} \le \E_{P_{W,S}}\sbr{\rbr{R(W) - r_S(W)}^2},
\end{equation}
for any $k \ge 2, k \in \mathbb{N}$.

\section{Conclusion}
In the counterexample presented in \cref{sec:counterexample} the empirical risk is sometimes larger than the population risk, which is  rare in practice.
In fact, if empirical risk is never larger than population risk, then $\E\sbr{\abs{R(W)-r_S(W)}}$ reduces to $\E\sbr{R(W)-r_S(W)}$, implying existence of sample-wise bounds.
Furthermore, the constructed learning algorithm intentionally captures only high-order information about samples.
This suggests, that sample-wise generalization bounds might be possible if we consider specific learning algorithms.

%% file: sup-complexity/main.tex
\input{sup-complexity/intro}

\input{sup-complexity/theory}

\input{sup-complexity/online-kd}

\input{sup-complexity/experiments}

\input{sup-complexity/related-work}

\section{Conclusion and future work}
\input{sup-complexity/conclusion}

%% file: sup-complexity/intro.tex
\section{Introduction}
\label{sec:sup-complexity-intro}

Knowledge distillation (KD)~\citep{Bucilla:2006, hinton2015distilling} is a popular method of  compressing a large ``teacher'' model into a more compact ``student'' model.
In its most basic form, this involves training the student to fit the teacher's predicted \emph{label distribution} or \emph{soft labels} for each sample.
There is strong empirical evidence that distilled students usually perform better than students trained on raw dataset labels~\citep{hinton2015distilling, furlanello2018born, stanton2021does, gou2021knowledge}.
Multiple works have devised novel knowledge distillation procedures that further improve the student model performance (see \citet{gou2021knowledge} and references therein).

The fact that knowledge distillation outperforms training on the raw dataset labels is surprising from a purely information-theoretic perspective.
This is because the teacher itself is usually trained on the same dataset.
By data-processing inequality, the teacher's predicted soft labels add no new information beyond what is present in the original dataset.
Clearly, the distillation dataset has some additional properties that enable better learning for the student.

Several works have aimed to rigorously formalize \emph{why} knowledge distillation can improve the student model performance. 
Some prominent observations from this line of work are that (self-)distillation induces certain favorable optimization biases in the training objective
~\citep{phuong19understand,ji2020knowledge}, 
lowers variance of the objective~\citep{menon2021statistical,Dao:2021,ren2022better}, 
increases regularization towards learning ``simpler'' functions~\citep{mobahi2020self}, 
transfers information from different data views~\citep{allen-zhu2023understanding}, and 
scales per-example gradients based on the teacher's confidence~\citep{furlanello2018born,tang2020understanding}.

Despite this remarkable progress, there are still many open problems and unexplained phenomena around knowledge distillation;
to name a few:
\begin{itemize}[label=---,itemsep=0pt,topsep=0pt,leftmargin=16pt]
    \item \emph{Why do soft labels (sometimes) help?}
    It is agreed that teacher's soft predictions carry information about class similarities~\citep{hinton2015distilling, furlanello2018born}, and that this softness of predictions has a regularization effect similar to label smoothing~\citep{yuan2020revisiting}.
    Nevertheless, knowledge distillation also works in binary classification settings with limited class similarity information~\citep{muller2020subclass}.
    How exactly the softness of teacher predictions (controlled by a temperature parameter) affects the student learning remains far from well understood.

    \item \emph{The role of capacity gap}. 
    There is evidence that when there is a significant capacity gap between the teacher and the student, the distilled model usually falls behind its teacher~\citep{mirzadeh2020improved, cho2019efficacy, stanton2021does}. It is unclear whether this is due to difficulties in optimization, or due to insufficient student capacity.

    \item \emph{What makes a good teacher?} 
    Sometimes less accurate models are better teachers~\citep{cho2019efficacy, mirzadeh2020improved}. Moreover, early stopped or exponentially averaged models are often better teachers~\citep{ren2022better}.
    A comprehensive explanation of this remains elusive.
\end{itemize}
The aforementioned wide range of phenomena suggest that there is a complex interplay between teacher accuracy, softness of teacher-provided targets, and complexity of the distillation objective.

This work provides a new theoretically grounded perspective on knowledge distillation through the lens of \emph{supervision complexity}.
In a nutshell, this quantifies why certain targets 
(e.g., temperature-scaled teacher probabilities) may be ``easier'' for a student model to learn compared to others (e.g., raw one-hot labels), 
owing to better alignment with the student's \emph{neural tangent kernel} (\emph{NTK})~\citep{jacot2018ntk,lee2019wide}.
In particular, we provide a novel theoretical analysis (\cref{sec:analysis}, \cref{thm:margin-bound-label-complexity,thm:margin-bound-label-complexity-vectorcase}) of the role of {supervision complexity} on kernel classifier generalization, and use this to derive a new generalization bound for distillation (\cref{prop:distillation-error-wrt-dataset-labels}).
The latter highlights how student generalization is controlled by a balance of 
the \emph{teacher generalization}, the student's \emph{margin} with respect to the teacher predictions, and the complexity of the teacher's predictions.

Based on the preceding analysis, we establish the conceptual and practical efficacy of a simple \emph{online distillation} approach (\cref{sec:sup-complexity-experiments}), wherein the student is fit to progressively more complex targets, in the form of teacher predictions at various checkpoints during its training.
This method can be seen as guiding the student in the function space (see \cref{fig:offline-vs-online}), and leads to better generalization compared to offline distillation.
We provide empirical results on a range of image classification benchmarks confirming the value of online distillation, particularly for students with weak inductive biases.

Beyond practical benefits, the supervision complexity view yields new insights into distillation:
\begin{itemize}[label=---,itemsep=0pt,topsep=0pt,leftmargin=16pt]
    \item \emph{The role of temperature scaling and early-stopping}.
    Temperature scaling and early-stopping of the teacher have proven effective for knowledge distillation.
    We show that 
    both of these techniques reduce the supervision complexity, at the expense of also lowering the classification margin.
    Online distillation manages to smoothly increase teacher complexity, without degrading the margin.
    
    \item \emph{Teaching a weak student}.
    We show that 
    for students with weak inductive biases, 
    and/or with much less capacity than the teacher,
    the final teacher predictions 
    \emph{are often as complex as dataset labels}, particularly during the early stages of training.
    In contrast, online distillation allows the supervision complexity to progressively increase, thus allowing even a weak student to learn.
    
    \item \emph{NTK and relational transfer}.
    We show that 
    online distillation is highly effective at matching the teacher and student NTK matrices.
    This transfers \emph{relational knowledge} in the form of example-pair similarity,
    as opposed to standard distillation which only transfers \emph{per-example knowledge}.
\end{itemize}

\begin{figure}[t]
    \centering
    \begin{subfigure}{0.4\textwidth}
       \includegraphics[width=\textwidth]{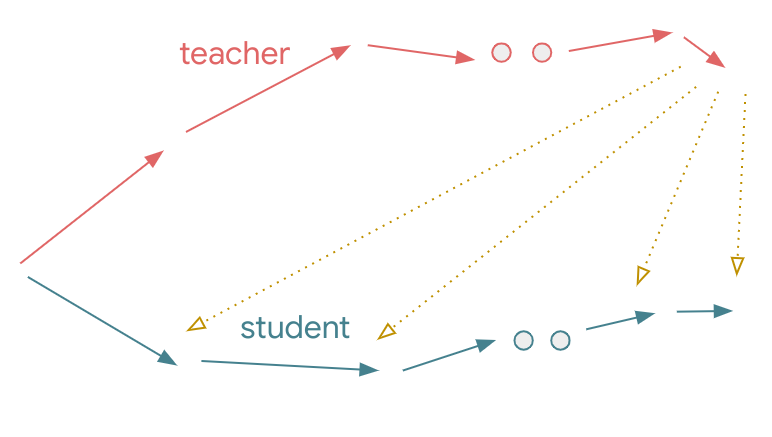}
       \caption{{\centering Offline distillation}}
    \end{subfigure}%
    \hspace{2em}%
    \begin{subfigure}{0.4\textwidth}
       \includegraphics[width=\textwidth]{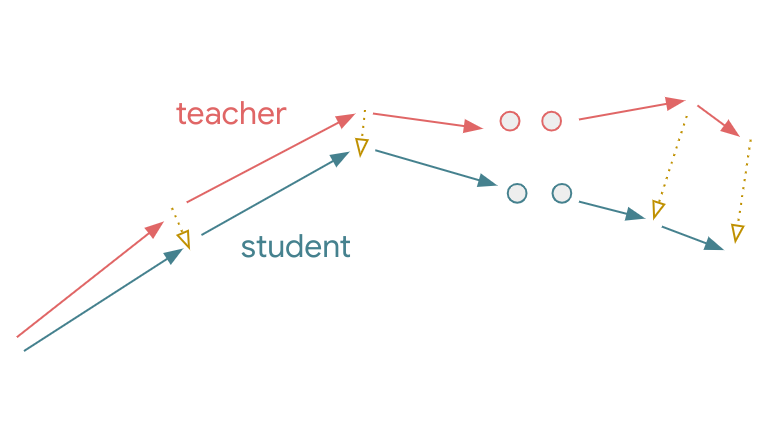}
       \caption{{\centering Online distillation}}
    \end{subfigure}%
    \caption{Online vs. online distillation. Figures (a) and (b) illustrate possible teacher and student function trajectories in offline and offline knowledge distillations respectively.
    The yellow dotted lines indicate knowledge distillation.
    }
    \label{fig:offline-vs-online}
\end{figure}

\paragraph{Problem setting.}
We focus on classification problems from input domain $\XX$ to $d$ classes.
We are given a training set of $n$ labeled examples $\{ (x_1, y_1), \ldots, (x_n, y_n) \}$, 
with one-hot encoded labels
$y_i \in \mathset{0,1}^d$.
Typically, a model $f_w : \XX \rightarrow \bR^d$ is trained with the \emph{softmax cross-entropy} loss:
\begin{equation}
    \label{eqn:softmax-ce}
    \LL_\mathrm{ce}( w ) = -\frac{1}{n}\sum\nolimits_{i=1}^n y_i^\top \log\rbr{ \mathrm{softmax}(f_w(x_i))}.
\end{equation}
In {standard} knowledge distillation, given a trained \emph{teacher} model $g : \XX \rightarrow \bR^d$ that outputs logits,
one trains a \emph{student} model $f_w : \XX \rightarrow \bR^d$ to fit the teacher predictions.
\citet{hinton2015distilling} propose the following knowledge distillation loss:
\begin{equation}
    \label{eqn:standard-kd}
    \LL_\mathrm{kd-ce}(w; g, \tau ) = -\frac{\tau^2}{n}\sum\nolimits_{i=1}^n \mathrm{softmax}(g(x_i) / \tau)^\top \log\rbr{\mathrm{softmax}(f_w(x_i) / \tau)},
\end{equation}
where temperature $\tau > 0$ controls the softness of teacher predictions.
To highlight the effect of knowledge distillation and simplify exposition, 
we assume that the student is not trained with the dataset labels.

%% file: sup-complexity/theory.tex
\section{Supervision complexity and generalization}
\label{sec:analysis}
One apparent difference between standard training and knowledge distillation (\cref{eqn:softmax-ce,eqn:standard-kd}) is that
the latter modifies the \emph{targets} that the student attempts to fit. The targets used during distillation ensure a better generalization for the student;
what is the reason for this?
Towards answering this question,
we present a new perspective on knowledge distillation in terms of \emph{supervision complexity}.
To begin, we show how the generalization of a \emph{kernel-based} classifier is controlled by a measure of alignment between the target labels and the kernel matrix. 
We first treat \emph{binary} kernel-based classifiers (\cref{thm:margin-bound-label-complexity}), and later extend our analysis to \emph{multiclass} kernel-based classifiers (\cref{thm:margin-bound-label-complexity-vectorcase}). 
Finally, 
{by leveraging the neural tangent kernel machinery,}
we discuss the implications of our analysis for neural classifiers in~\cref{sec:target-complexity-neural}.

\subsection{Supervision complexity controls kernel machine generalization}\label{subsec:target-complexity-scalar-output}
The notion of supervision complexity is easiest to introduce and study for kernel-based classifiers.
We briefly review some necessary background~\citep{Scholkopf:2001}.
Let $k : \XX \times \XX \rightarrow \bR$ be a positive semidefinite kernel defined over an input space $\XX$.
Any such kernel uniquely determines a reproducing kernel Hilbert space (RKHS) $\HH$ of functions from $\XX$ to $\bR$.
This RKHS is the completion of the set of functions of form $f(x) = \sum_{i=1}^m \alpha_i k(x_i, x)$, with 
$x_i \in \XX, \alpha_i \in \mathbb{R}$.
Any $f(x) = \sum_{i=1}^m \alpha_i k(x_i, x) \in \HH$ has (RKHS) norm
\begin{align}
\nbr{f}^2_\HH = \sum_{i=1}^m\sum_{j=1}^m \alpha_i \alpha_j k(x_i, x_j) = \bs{\alpha}^\top K \bs{\alpha},
\end{align}
where 
$\bs{\alpha} = (\alpha_1, \ldots, \alpha_n)^\top$ and $K_{i,j} = k(x_i, x_j)$.
Intuitively, $\nbr{f}_\HH^2$ measures the \emph{smoothness} of $f$, e.g.,
for a Gaussian kernel it measures the Fourier spectrum decay of $f$~\citep{Scholkopf:2001}.

For simplicity, we start with the case of binary classification.
Suppose $\{(X_i, Y_i)\}_{i \in [n]}$ are
$n$ i.i.d. examples sampled from some probability distribution
on $\XX \times \YY$, with $\YY \subset \mathbb{R}$, where positive and negative labels correspond to distinct classes.
Let $K_{i,j} = k(X_i, X_j)$ denote the kernel matrix, and
$\bs{Y} = (Y_1,\ldots,Y_n)^\top$ be the concatenation of all training labels.

\begin{definition}[Supervision complexity]
The \emph{supervision complexity} of targets $Y_1,\ldots,Y_n$ with respect to a kernel $k$ is defined to be $\bs{Y}^\top K^{-1}\bs{Y}$ in cases when $K$ is invertible, and $+\infty$ otherwise.    
\end{definition}

We now establish how supervision complexity controls the \emph{smoothness} of the optimal kernel classifier.
Consider a classifier obtained by solving a \textit{regularized} kernel classification problem:
\begin{equation}
    f^* \in \argmin_{f \in \HH} \frac{1}{n}\sum\nolimits_{i=1}^n \ell(f(X_i), Y_i) + \frac{\lambda}{2} \nbr{f}_\HH^2,
    \label{eq:kernel-method-general}
\end{equation}
where $\ell$ is a loss function and $\lambda > 0$.
The following proposition shows whenever the supervision complexity is small, the RKHS norm of any optimal solution $f^*$ will also be small.
This is an important learning bias that shall help us explain certain aspects of knowledge distillation.
\begin{proposition}
Assume that $K$ is full rank almost surely;
$\ell(y, y') \ge 0, \forall y, y' \in \YY$; and $\ell(y,y) = 0, \forall y\in\YY$. Then, with probability 1, for any solution $f^*$ of \cref{eq:kernel-method-general}, we have 
$
    \nbr{f^*}_\HH^2 \le \bs{Y}^\top K^{-1} \bs{Y}.
$
\label{prop:kernel-method-solution-norm}
\end{proposition}

Equipped with the above result, we now show how supervision complexity controls generalization. 
In the following, 
let $\phi_\gamma : \bR \rightarrow [0, 1]$ be the \emph{margin loss}~\citep{mohri2018foundations} with scale $\gamma > 0$:
\begin{equation}
    \phi_\gamma(\alpha) = \begin{cases}
        1 &\text{ if }  \alpha \le 0\\
        1 - \alpha / \gamma &\text{ if } 0 < \alpha \le \gamma \\
        0 &\text{ if } \alpha > \gamma.
        \end{cases}
% \phi_\gamma(\alpha) = \min\{1, \max\{1 - \alpha/\gamma, 0\}\}.
\end{equation}

\begin{theorem}
Assume that $\kappa=\sup_{x \in \XX} k(x,x) < \infty$ and $K$ is full rank almost surely.
Further, assume that $\ell(y, y') \ge 0, \forall y, y' \in \YY$ and $\ell(y,y) = 0, \forall y\in\YY$.
Let $M_0 = \left\lceil{\gamma\sqrt{n}}/{(2\sqrt{\kappa})}\right\rceil$.
Then, with probability at least $1-\delta$, for any solution $f^*$ of problem in \cref{eq:kernel-method-general}, we have
\begin{align*}
    P_{X,Y}(Y f^*(X) \le 0) &\le \frac{1}{n}\sum_{i=1}^n \phi_\gamma(\sign\rbr{Y_i} f^*(X_i)) + \frac{2\sqrt{\bs{Y}^\top K^{-1} \bs{Y}}+2}{\gamma n} \sqrt{\trace\rbr{K}}\\
    &\quad+3\sqrt{\frac{\ln\rbr{2M_0/\delta}}{2n}}.
    \numberthis
    \label{eqn:target-complexity-bound}
\end{align*}
\label{thm:margin-bound-label-complexity}
\end{theorem}
One can compare \cref{thm:margin-bound-label-complexity}
with the standard Rademacher bound for kernel classifiers~\citep{bartlett2002rademacher}.
The latter typically consider learning over functions with RKHS norm bounded by a \emph{constant} $M > 0$.
The corresponding complexity term then decays as $\mathscr{O}( \sqrt{M \cdot \trace\rbr{K} / n} )$,
which is \emph{data-independent}.
Consequently, 
such a bound cannot adapt to the intrinsic ``difficulty'' of the targets $\bs{Y}$.
In contrast, \cref{thm:margin-bound-label-complexity} considers functions with RKHS norm bounded by the \emph{data-dependent} supervision complexity term. This results in a more informative generalization bound, which captures the ``difficulty'' of the targets. 
Here, we note that \cite{arora2019fine} characterized the generalization of an overparameterized two-layer neural network via a term closely related to the supervision complexity (see \cref{sec:sup-complexity-related} for additional discussion).

The supervision complexity $\bs{Y}^\top K^{-1} \bs{Y}$ is small whenever $\bs{Y}$ is aligned with top eigenvectors of $K$ and/or $\bs{Y}$ has small scale. Furthermore, one cannot make the bound close to zero by just reducing the scale of targets, as one would need a small $\gamma$ to control the margin loss that would otherwise increase due to student predictions getting closer to zero (as the student aims to match $Y_i$).

To better understand the role of supervision complexity, it is instructive to consider two special cases that lead to a poor generalization bound:
(1) uninformative \emph{features}, and (2) uninformative \emph{labels}.

\paragraph{Complexity under uninformative features.} 
Suppose the kernel matrix $K$ is diagonal, so that the kernel provides \emph{no} information on example-pair similarity; i.e., the kernel is ``uninformative''.
An application of Cauchy-Schwarz reveals
the key expression in the second term in~\cref{eqn:target-complexity-bound} satisfies:
\begin{equation}
    \frac{1}{n}\sqrt{\bs{Y}^\top K^{-1} \bs{Y} \trace(K)} = \frac{1}{n}\sqrt{\bbr{\sum\nolimits_{i=1}^n Y^2_i \cdot k(X_i, X_i)^{-1}}\bbr{\sum\nolimits_{i=1}^n k(X_i,X_i)}} \ge \frac{1}{n}\sum_{i=1}^n \abs{Y_i}.
\end{equation}
Consequently, this term is least constant in order, and \emph{does not} vanish as $n \to \infty$.

\paragraph{Complexity under uninformative labels.} 
Suppose the labels $Y_i$ are 
purely random, and
independent from inputs $X_i$.
Conditioned on $\{X_i\}$, % $X_1,\ldots,X_n$, 
$\bs{Y}^\top K^{-1} \bs{Y}$ concentrates
around its mean by the Hanson-Wright inequality~\citep{vershynin2018high}.
Hence, $\exists \, \epsilon(K, \delta, n)$ such that with probability 
$\geq 1-\delta$, $\bs{Y}^\top K^{-1} \bs{Y} \ge \E_{\{Y_i\}}\big[{\bs{Y}^\top K^{-1} \bs{Y}}\big] - \epsilon = \E\sbr{Y_1^2}\trace(K^{-1}) - \epsilon$.
Thus, with the same probability, 
\begin{align}
    \frac{1}{n}\sqrt{\bs{Y}^\top K^{-1} \bs{Y} \trace(K)} &\ge \frac{1}{n}\sqrt{\rbr{\E\sbr{Y_1^2}\trace(K^{-1}) - \epsilon} \trace(K)}
    \ge \frac{1}{n}\sqrt{\E\sbr{Y_1^2}n^2 - \epsilon \trace(K)},
\end{align}
where the last inequality is by Cauchy-Schwarz.
For sufficiently large $n$, the quantity $\E\sbr{Y_1^2}n^2$ dominates $\epsilon \trace\rbr{K}$, rendering the bound of \cref{thm:margin-bound-label-complexity} close to a constant.

\subsection{Extensions: multiclass classification and neural networks}
\label{sec:target-complexity-neural}
We now show that a result similar to \cref{thm:margin-bound-label-complexity} holds for multiclass classification as well. In addition, we also discuss how our results are instructive about the behavior of neural networks.

\paragraph{Extension to multiclass classification.}
Let $\{(X_i, Y_i)\}_{i \in [n]}$ be drawn i.i.d.\, from a distribution over $\XX \times \YY$, where $\YY \subset \bR^d$.
Let $k : \XX \times \XX \rightarrow \bR^{d \times d}$ be a matrix-valued positive definite kernel and $\HH$ be the corresponding vector-valued RKHS.
As in the binary classification case, we consider a kernel problem in \cref{eq:kernel-method-general}.
Let $\bs{Y}^\top = (Y_1^\top, \ldots, Y_n^\top)$ and $K$ be the kernel matrix of training examples:
\begin{equation}
    K = \left[ \begin{matrix}
        k(X_1, X_1) & \cdots & k(X_1, X_n)\\
        \cdots & \cdots & \cdots \\
        k(X_n, X_1) & \cdots & k(X_n, X_n)\\
    \end{matrix} \right] \in \bR^{nd \times nd}.
\end{equation}
For %a hypothesis 
$f : \XX \rightarrow \bR^d$ and a labeled example $(x,y)$, let %be defined as 
$\rho_{f}(x,y) = f(x)_y - \max_{y'\neq y} f(x)_{y'}$ be the \emph{prediction margin}.
Then, the following analogue of \cref{thm:margin-bound-label-complexity} holds.

\begin{theorem}
Assume that $\kappa=\sup_{x \in \XX, y\in[d]} k(x,x)_{y,y} < \infty$,
and $K$ is full rank almost surely.
Further, assume that $\ell(y, y') \ge 0, \forall y, y' \in \YY$ and $\ell(y,y) = 0, \forall y\in\YY$.
Let $M_0 = \left\lceil{\gamma\sqrt{n}}/{(4d\sqrt{\kappa})}\right\rceil$.
Then, with probability at least $1-\delta$, for any solution $f^*$ of problem in \cref{eq:kernel-method-general},
\begin{align*}
    P_{X,Y}(\rho_{f^*}(X,Y) \le 0) &\le \frac{1}{n}\sum_{i=1}^n \ind{\rho_{f^*}(X_i, Y_i) \le \gamma} + \frac{4 d (\bs{Y}^\top K^{-1} \bs{Y}+1)}{\gamma n} \sqrt{\trace\rbr{K}}\\
    &\quad+3\sqrt{\frac{\log(2M_0/\delta)}{2n}}.
    \numberthis
\end{align*}
\label{thm:margin-bound-label-complexity-vectorcase}
\end{theorem}
\paragraph{Implications for neural classifiers.}
Our analysis has so far focused on kernel-based classifiers.
While neural networks are not exactly kernel methods, many aspects of their performance can be 
understood via a corresponding \emph{linearized neural network} (see \citet{ortiz2021can} and references therein).
We follow this approach, and given a neural network $f_w$ with current weights $w_0$, we consider the corresponding linearized neural network 
comprising
the linear terms of the Taylor expansion of $f_w(x)$ around $w_0$~\citep{jacot2018ntk, lee2019wide}:
\begin{equation}
f^\mathrm{lin}_w(x) \triangleq f_{w_0}(x) + \nabla_w f_{w_0}(x)^\top (w-w_0).
\end{equation}
Let $\omega \triangleq w - w_0$.
This network $f^\mathrm{lin}_\omega(x)$ is a linear function with respect to 
the parameters $\omega$, but is generally non-linear with respect to the input $x$.
Note that $\nabla_w f_{w_0}(x)$ acts as a feature representation, 
and induces the \emph{neural tangent kernel} (\emph{NTK}) $k_0(x, x') = \nabla_w f_{w_0}(x)^\top \nabla_w f_{w_0}(x') \in \bR^{d\times d}$.

Given a labeled dataset $S = \{(x_i, y_i)\}_{i\in[n]}$
and a loss function $\LL( f; S )$, the dynamics of gradient flow with learning rate $\eta > 0$ for $f^\mathrm{lin}_\omega$
can be fully characterized in the function space, and depends only on the predictions at $w_0$ and the NTK $k_0$:
\begin{align}
    \dot{f}_t^\mathrm{lin}(x') &= -\eta \cdot K_0(x', \bs{x}) \phantom{^\top} 
    \nabla_{f}\LL( f_{t}^\mathrm{lin}(\bs{x}); S ),\label{eq:lin-pred-dyn}
\end{align}
where $f(\bs{x}) \in \bR^{nd}$ denotes the concatenation of predictions on training examples and $K_0(x', \bs{x}) = \nabla_w f_{w_0}(x')^\top \nabla_w f_{w_0}(\bs{x})$.
\citet{lee2019wide} show that as one increases the width of the network or when $(w - w_0)$ does not change much during training, the dynamics of the linearized and original neural network become close.

When $f_w$ is sufficiently overparameterized and $\LL$ is convex with respect to $\omega$, then $\omega_t$ converges to an interpolating solution.
Furthermore, for the mean squared error objective, the solution has the minimum Euclidean norm~\citep{gunasekar2017implicit}.
As the Euclidean norm of $\omega$ corresponds to norm of $f^\mathrm{lin}_\omega(x) - f_{w_0}(x)$ in the vector-valued RKHS $\HH$ corresponding to %the NTK 
$k_0$, training a linearized network to interpolation is equivalent to solving the following with a small $\lambda > 0$:
\begin{equation}
    h^* = \argmin_{h \in \HH} \frac{1}{n}\sum\nolimits_{i=1}^n (f_{w_0}(x_i) + h(x_i) - y_i)^2 + \frac{\lambda}{2} \nbr{h}_\HH^2.
    \label{eq:kernel-method-NN-fn-space}
\end{equation}
Therefore, the generalization bounds of \cref{thm:margin-bound-label-complexity,thm:margin-bound-label-complexity-vectorcase} apply to $h^*$ with supervision complexity of residual targets $y_i - f_{w_0}(x_i)$.
However, we are interested in the performance of $f_{w_0} + h^*$.
As the proofs of these results rely on bounding the Rademacher complexity of hypothesis sets of form $\mathset{h \in \HH \colon \nbr{h} \le M}$, and shifting a hypothesis set by a constant function does not change the Rademacher complexity (see \cref{remark:shfiting-rademacher} of \cref{app:proofs}), these proofs can be easily modified to handle hypotheses shifted by the constant function $f_{w_0}$.

%% file: sup-complexity/online-kd.tex
\section{Knowledge distillation: a supervision complexity lens}
\label{sec:kd-analysis}
We now turn to knowledge distillation, 
and explore how supervision complexity affects student's generalization.
We show that student's generalization depends on three terms:
the \emph{teacher generalization}, 
the student's \emph{margin} with respect to the teacher predictions,
and the complexity of the teacher's predictions.

\subsection{Trade-off between teacher accuracy, margin, and complexity}
Consider the binary classification setting of \cref{sec:analysis}, and a fixed teacher $g : \XX \rightarrow \bR$ that outputs a logit.
Let $\{(X_i, Y^*_i)\}_{i \in [n]}$ be $n$ i.i.d. labeled examples, where $Y^*_i \in \mathset{-1, 1}$ denotes the ground truth labels.
For temperature $\tau > 0$,
let $Y_i \triangleq 2\ \mathrm{sigmoid}(g(X_i)/\tau) - 1 \in [-1, +1]$ denote the teacher's \emph{soft predictions}, 
for sigmoid function $z \mapsto (1 + \exp(-z))^{-1}$.
Our key observation is:
if the teacher predictions $Y_i$ are accurate enough
and have significantly lower complexity compared to ground truth labels 
$Y^*_i$, then a student kernel method (cf.~\cref{eq:kernel-method-general}) trained with $Y_i$  %teacher predictions
can generalize better than the one trained with $Y^*_i$. %ground truth labels.
The following result quantifies the trade-off between teacher accuracy, student prediction margin, and teacher prediction complexity.
\begin{proposition}
Assume that $\kappa=\sup_{x \in \XX} k(x,x) < \infty$ and $K$ is full rank almost surely, $\ell(y, y') \ge 0, \forall y, y' \in \YY$, and $\ell(y,y) = 0, \forall y\in\YY$.
Let $Y_i$ and $Y^*_i$ be defined as above.
Let $M_0 = \left\lceil{\gamma\sqrt{n}}/{(2\sqrt{\kappa})}\right\rceil$.
Then, with probability at least $1-\delta$, any solution $f^*$ of problem \cref{eq:kernel-method-general} satisfies
\begin{align*}
    \underbrace{P_{X,Y^*}(Y^* f^*(X) \le 0)}_{\text{student risk}} &\le \underbrace{P_{X,Y^*}(Y^* g(X) \le 0)}_{\text{teacher risk}}\ \ + \underbrace{\frac{1}{n}\sum\nolimits_{i=1}^n \phi_\gamma(\sign\rbr{Y_i} f^*(X_i))}_{\text{student's empirical margin loss w.r.t. teacher predictions}}\\
    &\quad+\underbrace{{\rbr{2\sqrt{\bs{Y}^\top K^{-1} \bs{Y}}+2}} \sqrt{\trace\rbr{K}}/{(\gamma n)}}_{\text{complexity of teacher's predictions}}
    +3\sqrt{{\ln\rbr{2M_0/\delta}}/{(2n)}}.\numberthis
\end{align*}
\label{prop:distillation-error-wrt-dataset-labels}
\end{proposition}

Note that a similar result is easy to establish for multiclass classification  using \cref{thm:margin-bound-label-complexity-vectorcase}.
% To unpack this result, we make a few remarks.
The first term in the above accounts for the misclassification rate of the teacher.
While this term is not irreducible (it is possible for a student to perform better than its teacher), generally a student performs worse that its teacher, especially when there is a significant teacher-student capacity gap.
The second term is student's empirical margin loss w.r.t. teacher predictions. This captures the price of making teacher predictions too soft.
Intuitively, the softer (i.e., closer to zero) teacher predictions are, the harder it is for the student to learn the classification rule.
The third term accounts for the supervision complexity and the margin parameter $\gamma$.
Thus, one has to choose $\gamma$ carefully to achieve a good balance between empirical margin loss and margin-normalized supervision complexity.

\paragraph{The effect of temperature.}
For a fixed margin parameter $\gamma > 0$, increasing the temperature $\tau$ makes teacher's predictions $Y_i$ softer.
On the one hand, the reduced scale decreases the supervision complexity $\bs{Y}^\top K^{-1} \bs{Y}$.
Moreover, we shall see that in the case of neural networks the complexity decreases even further due to $\bs{Y}$ becoming more aligned with top eigenvectors of $K$.
On the other hand, the scale of predictions of the (possibly interpolating) student $f^*$ will decrease too, increasing the empirical margin loss.
This suggests that setting the value of $\tau$ is not trivial: the optimal value can be different based on the kernel $k$ and teacher logits $g(X_i)$.

\subsection{From offline to online knowledge distillation}
\label{sec:online-kd}
We identified that supervision complexity plays a key role in determining the efficacy of a distillation procedure.
The supervision from a fully trained teacher model can prove to be very complex for a student model in an early stage of its training (\cref{fig:sup-res-complexity-comparison-test}).
This raises the question: 
\emph{is there value in providing progressively difficult supervision to the student}? 
In this section, we describe a simple online distillation method, where 
the the teacher is updated during the student training.

Over the course of their training, neural models learn functions of increasing complexity~\citep{Nakkiran:2019}. This provides a natural way to construct a set of teachers with varying prediction complexities.
Similar to \citet{jin2019rco}, for practical considerations of not training the teacher and the student simultaneously, we assume the availability of teacher checkpoints over the course of its training.
Given $m$ teacher checkpoints at times $\mathcal{T} = \{t_i\}_{i \in [m]}$,
during the $t$-th step of distillation, the student receives supervision from the teacher checkpoint at time $\min\{t' \in \mathcal{T}: t' > t\}$.
Note that the student is trained for the same number of epochs in total as in offline distillation.
We use the term ``online distillation'' for this approach (cf.~\cref{alg:online-distillation}).

\begin{algorithm}[t]
\small
    \caption{Online knowledge distillation.}
    \label{alg:online-distillation}
    \begin{algorithmic}[1] % The number tells where the line numbering should start
    \STATE \textbf{Require:} Training sample $S$;
    teacher checkpoints $\{ g^{(t_1)}, \ldots, g^{(t_m)} \}$;
    temperature $\tau > 0$;
    training steps $T$;
    minibatch size $b$
        % \Procedure{OnlineDistillation}{}
        \FOR{$t = 1, \ldots, T$}
            \STATE Draw random $b$-sized minibatch $S'$ from $S$
            \STATE Compute nearest teacher checkpoint $t^* = \min \{ i \in [m] \colon t_i > t \}$
            \STATE Update student $w \gets w - \eta_t \cdot \nabla_w \LL_{\mathrm{kd-ce}}( w; g^{(t^*)}, \tau )$ over $S'$
        \ENDFOR
        \STATE \textbf{return} $f_w$
        % \EndProcedure
    \end{algorithmic}
\end{algorithm}

Online distillation can be seen as guiding the student network to follow the teacher's trajectory in \emph{function space} (see \cref{fig:offline-vs-online}).
Given that NTK can be interpreted as a principled notion of example similarity and controls which examples affect each other during training~\citep{charpiat2019input}, it is desirable for the student to have an NTK similar to that of its teacher at each time step.
To test whether online distillation also transfers NTK, we propose to measure similarity between the final student and final teacher NTKs.
For computational efficiency we work with NTK matrices corresponding to a batch of $b$ examples ($bd \times bd$ matrices).
Explicit computation of even batch NTK matrices can be costly, especially when the number of classes $d$ is large.
We propose to view student and teacher batch NTK matrices (denoted by $K_f$ and $K_g$ respectively) as operators and measure their similarity by comparing their behavior on random vectors:
\begin{equation}
\mathrm{sim}(K_f, K_g)=\E_{V \sim \mathcal{N}(0, I_{bd})}\sbr{\frac{\langle K_f V, K_g V\rangle}{\nbr{K_f V}\nbr{K_g V}}}.
\end{equation}
Note that the cosine distance is used to account for scale differences of $K_g$ and $K_f$.
The kernel-vector products appearing in this similarity measure above can be computed efficiently without explicitly constructing the kernel matrices.
For example, with $\bs{x'}$ denoting the collection of inputs of the mini-batch, $K_f v =\nabla_{w} f_w(\bs{x'})^\top\rbr{\nabla_{w} f_w(\bs{x'}) v}$ can be computed with one vector-Jacobian product followed by a Jacobian-vector product.
The former can be computed efficiently using backpropagation, while the latter can be computed efficiently using forward-mode differentiation.

%% file: sup-complexity/experiments.tex
\section{Experimental results}
\label{sec:sup-complexity-experiments}
\begin{table}[!t]
    \centering
    \small
    \caption{Knowledge distillation results on CIFAR-10. Every second line is an MSE student.}
    \label{tbl:results-cifar10}
    % \resizebox{\linewidth}{!}{%
    \begin{tabular}{@{}lcccccc@{}}
    \toprule
    {\bf Setting} & {\bf No KD} & \multicolumn{2}{c}{\bf Offline KD} & \multicolumn{2}{c}{\bf Online KD} & {\bf Teacher} \\
    &  & $\tau = 1$ & $\tau = 4$ & $\tau = 1$ & $\tau = 4$ & \\
    \midrule
    ResNet-56 $\rightarrow$ LeNet-5x8 & 81.8 $\pm$ 0.5 & 82.4 $\pm$ 0.5 & 86.0 $\pm$ 0.2 & 86.8 $\pm$ 0.2 & \best{88.6 $\pm$ 0.1} & 93.2 \\
    ResNet-56 $\rightarrow$ LeNet-5x8 & 83.4 $\pm$ 0.3 & 83.1 $\pm$ 0.2 & 84.9 $\pm$ 0.1 & 85.6 $\pm$ 0.1 & \best{87.1 $\pm$ 0.1} & 93.2 \\
    \midrule
    ResNet-110 $\rightarrow$ LeNet-5x8 & 81.7 $\pm$ 0.3 & 81.9 $\pm$ 0.5 & 85.8 $\pm$ 0.1 & 86.5 $\pm$ 0.1 & \best{88.8 $\pm$ 0.1} & 93.9 \\
    ResNet-110 $\rightarrow$ LeNet-5x8 & 83.2 $\pm$ 0.4 & 83.2 $\pm$ 0.1 & 85.0 $\pm$ 0.3 & 85.6 $\pm$ 0.1 & \best{87.1 $\pm$ 0.2} & 93.9 \\
    \midrule
    ResNet-110 $\rightarrow$ ResNet-20 & 91.4 $\pm$ 0.2 & 91.4 $\pm$ 0.1 & 92.8 $\pm$ 0.0 & 92.2 $\pm$ 0.3 & \best{93.1 $\pm$ 0.1} & 93.9 \\
    ResNet-110 $\rightarrow$ ResNet-20 & 90.9 $\pm$ 0.1 & 90.9 $\pm$ 0.2 & 91.6 $\pm$ 0.2 & 91.2 $\pm$ 0.1 & \best{92.1 $\pm$ 0.2} & 93.9 \\
    \bottomrule
    \end{tabular}%
    % }
\end{table}

\begin{table}[!t]
    \centering
    \small
    \caption{Knowledge distillation results on CIFAR-100.}
    \label{tbl:results-cifar100}
    % \resizebox{\linewidth}{!}{%
    % \scalebox{0.95}{
    \begin{tabular}{@{}lcccccc@{}}
    \toprule
    {\bf Setting} & {\bf No KD} & \multicolumn{2}{c}{\bf Offline KD} & \multicolumn{2}{c}{\bf Online KD} & {\bf Teacher} \\
     &  & $\tau = 1$ & $\tau = 4$ & $\tau = 1$ & $\tau = 4$ & \\
    \midrule
    ResNet-56 $\rightarrow$ LeNet-5x8 & 47.3 $\pm$ 0.6 & 50.1 $\pm$ 0.4 & 59.9 $\pm$ 0.2 & 61.9 $\pm$ 0.2 & \best{66.1 $\pm$ 0.4} & 72.0\\
    ResNet-56 $\rightarrow$ ResNet-20 & 67.7 $\pm$ 0.5 & 68.2 $\pm$ 0.3 & \best{71.6 $\pm$ 0.2} & 69.6 $\pm$ 0.3 & 71.4 $\pm$ 0.3 & 72.0\\
    ResNet-110 $\rightarrow$ LeNet-5x8 & 47.2 $\pm$ 0.5 & 48.6 $\pm$ 0.8 & 59.0 $\pm$ 0.3 & 60.8 $\pm$ 0.2 & \best{65.8 $\pm$ 0.2} & 73.4\\
    ResNet-110 $\rightarrow$ ResNet-20 & 67.8 $\pm$ 0.3 & 67.8 $\pm$ 0.2 & 71.2 $\pm$ 0.0 & 69.0 $\pm$ 0.3 & \best{71.4 $\pm$ 0.0} & 73.4\\
    \bottomrule
    \end{tabular}%
    % }
\end{table}

\begin{table}[!t]
    \small
    \centering
    \caption{Knowledge distillation results on Tiny ImageNet.}
    \label{tbl:results-imagenet}
    % \resizebox{\linewidth}{!}{%
    \scalebox{0.97}{
    \begin{tabular}{@{}lcccccc@{}}
    \toprule
    {\bf Setting} & {\bf No KD} & \multicolumn{2}{c}{\bf Offline KD} & \multicolumn{2}{c}{\bf Online KD} & {\bf Teacher} \\
    &  & $\tau = 1$ & $\tau = 2$ & $\tau = 1$ & $\tau = 2$ & \\
    \midrule
    MobileNet-V3-125 $\rightarrow$ MobileNet-V3-35 & 58.5 $\pm$ 0.2 & 59.2 $\pm$ 0.1 & 60.2 $\pm$ 0.2 & 60.7 $\pm$ 0.2 & \best{62.3 $\pm$ 0.3} & 62.7 \\
    ResNet-101 $\rightarrow$ MobileNet-V3-35 & 58.5 $\pm$ 0.2& 59.4 $\pm$ 0.5 & 61.6 $\pm$ 0.2 & 61.1 $\pm$ 0.3 & \best{62.0 $\pm$ 0.3} & 66.0 \\
    MobileNet-V3-125 $\rightarrow$ VGG-16 & 48.9 $\pm$ 0.3 & 54.1 $\pm$ 0.4 & 59.4 $\pm$ 0.4 & 58.9 $\pm$ 0.7 & \best{62.3 $\pm$ 0.3} & 62.7 \\
    ResNet-101 $\rightarrow$ VGG-16 & 48.6 $\pm$ 0.4 & 53.1 $\pm$ 0.4 & 60.6 $\pm$ 0.2 & 60.4 $\pm$ 0.2 & \best{64.0 $\pm$ 0.1} & 66.0 \\
    \bottomrule
    \end{tabular}%
    }
\end{table}
We now present experimental results to showcase the importance of supervision complexity in distillation, and to establish efficacy of online distillation.

\subsection{Experimental setup}
We consider standard image classification benchmarks: CIFAR-10, CIFAR-100, and Tiny-ImageNet.
Additionally, we derive a binary classification task from CIFAR-100 by grouping the first and last 50 classes into two meta-classes.
We consider teacher and student architectures that are ResNets~\citep{he2016deep}, VGGs~\citep{vgg2014simonyan2014}, and MobileNets~\citep{howard2019searching} of various depths.
As a student architecture with relatively weaker inductive biases, we also consider the LeNet-5~\citep{lecun1998gradient} with 8 times wider hidden layers.

\begin{table}[t]
    \small
    \caption{Initial learning rates for different dataset and model pairs.}
    \centering
    \begin{tabular}{llc}
    \toprule
    Dataset & Model & Learning rate\\
    \midrule
    \multirow{4}*{CIFAR-10, CIFAR-100, binary CIFAR-100} & ResNet-56 (teacher) & 0.1 \\
     & ResNet-110 (teacher) & 0.1 \\
     & ResNet-20 (CE or MSE students) & 0.1 \\
     & LeNet-5x8 (CE or MSE students) & 0.04 \\
    \midrule
    \multirow{3}*{Tiny ImageNet} & MobileNet-V3-125 (teacher) & 0.04\\
     & ResNet-101 (teacher) & 0.1 \\
     & MobileNet-V3-35 (student) & 0.04 \\
     & VGG-16 (student) & 0.01 \\
    \bottomrule
    \end{tabular}
    \label{tab:learning-rate-table}
\end{table}
We use standard hyperparameters to train these models. In particular, in all experiments we use stochastic gradient descent optimizer with 128 batch size and 0.9 Nesterov momentum.
The starting learning rates are presented in \cref{tab:learning-rate-table}.
All models for CIFAR datasets are trained for 256 epochs, with a learning schedule that divides the learning rate by 10 at epochs 96, 192, and 224.
All models for Tiny ImageNet are trained for 200 epochs, with a learning rate schedule that divides the learning rate by 10 at epochs 75 and 135.
The learning rate is warmed-up linearly to its initial value in the first 10 and 5 epochs for CIFAR and Tiny ImageNet models respectively.
All VGG and ResNet models use 2e-4 weight decay, while MobileNet models use 1e-5 weight decay.

The LeNet-5 uses ReLU activations.
We use the CIFAR variants of ResNets in experiments with CIFAR-10 or (binary) CIFAR-100 datasets.
Tiny ImageNet examples are resized to 224x224 resolution to suit the original ResNet, VGG and MobileNet architectures.
In all experiments we use standard data augmentation -- random cropping and random horizontal flip.
In all online learning methods we consider one teacher checkpoint per epoch.

We compare
(1) regular one-hot training (without any distillation),
(2) regular offline distillation using the temperature-scaled softmax cross-entropy,
and (3) online distillation using the same loss.
For CIFAR-10 and binary CIFAR-100, we also consider training with mean-squared error (MSE) loss and its corresponding KD loss:
\begin{align}
    \LL_\mathrm{mse}( w) &= \frac{1}{2n}\sum_{i=1}^n \nbr{y_i - f_w(x_i)}_2^2,\\
    % \quad
    \LL_\mathrm{kd-mse}( w; g, \tau ) &= \frac{\tau}{2n}\sum_{i=1}^n \nbr{\mathrm{softmax}(g(x_i) / \tau) - f_w(x_i)}_2^2.\label{eq:kd-mse}
\end{align}
The MSE loss allows for interpolation in case of one-hot labels $y_i$, making it amenable to the analysis in \cref{sec:analysis,sec:kd-analysis}. 
Moreover,~\citet{hui2021evaluation} show that under standard training, the CE and MSE losses perform similarly; as we shall see, the same is true for distillation as well.

As mentioned in \cref{sec:sup-complexity-intro}, in all KD experiments student networks receive supervision only through a knowledge distillation loss (i.e., dataset labels are not used).
This choice help us decrease differences between the theory and experiments.
Furthermore, in our preliminary experiments we observed that this choice does not result in student performance degradation (see \cref{tbl:cifar100-kd-alpha}).
\subsection{Results and discussion}
\cref{tbl:results-cifar10,tbl:results-cifar100,tbl:results-imagenet,tbl:results-cifar100-binary} 
present the results 
(mean and standard deviation of test accuracy over 3 random trials).
First, we see that 
% in most cases 
online distillation with proper temperature scaling 
typically
yields the 
most accurate
student.
The gains over regular distillation are particularly pronounced when there is a large teacher-student gap.
For example, on CIFAR-100, 
ResNet to LeNet distillation
with temperature scaling
appears to hit a limit of $\approx 60\%$ accuracy.
Online distillation however manages to further increase accuracy  by $+6\%$, which is a
$\approx 20\%$ increase compared to standard training.
Second, the 
similar results on binary CIFAR-100 shows that ``dark knowledge'' in the form of 
membership information in multiple classes is not necessary for distillation to succeed.

The results also demonstrate that knowledge distillation with the MSE loss of \cref{eq:kd-mse} has a qualitatively similar behavior to KD with CE objective.
We use these MSE models to highlight the role of supervision complexity.
As an instructive case, we consider a LeNet-5x8 network trained on binary CIFAR-100 with the standard MSE loss function.
For a given checkpoint of this network and a given set of $m$ labeled (\emph{test}) examples $\mathset{(X'_i,Y'_i)}_{i \in [m]}$, we compute the \emph{adjusted supervision complexity} defined as 
\begin{equation}
1/n \sqrt{(\bs{Y'} - f(\bs{X'}))^\top (K')^{-1} (\bs{Y'} - f(\bs{X'})) \cdot \trace\rbr{K'}},
\label{eq:adjusted-res-complexity}
\end{equation}
where $f$ denotes the current prediction function,
and $K$ is derived from the current NTK.
Note that the subtraction of initial predictions is the appropriate way to measure complexity given the form of the optimization problem~\cref{eq:kernel-method-NN-fn-space}.
Nevertheless, it is meaningful to consider the following quantity as well:
\begin{equation}
    \frac{1}{n} \sqrt{\bs{Y'}^\top (K')^{-1} \bs{Y'} \cdot \trace\rbr{K'}},\label{eq:adjusted-non-res-complexity}
\end{equation}
in order to measure ``alignment'' of targets $\bs{Y'}$ with the NTK $k$. We call this quantity \emph{adjusted supervision complexity*}.
As the training NTK matrix becomes aligned with dataset labels during training (see~\citet{baratin2021implicit} and \cref{fig:ntk-alignment-with-labels}), we pick $\mathset{X'_i}_{i\in [m]}$ to be a set of $2^{12}$ 
\emph{test} examples.

\begin{table}[!t]
    \centering
    \small
    \caption{Knowledge distillation results on binary CIFAR-100. Every second line is an MSE student.}
    \label{tbl:results-cifar100-binary}
    
    %\resizebox{0.9\linewidth}{!}{%
    % \scalebox{0.85}{
    \begin{tabular}{@{}lcccccc@{}}
    \toprule
    {\bf Setting} & {\bf No KD} & \multicolumn{2}{c}{\bf Offline KD} & \multicolumn{2}{c}{\bf Online KD} & {\bf Teacher} \\
    &  & $\tau = 1$ & $\tau = 4$ & $\tau = 1$ & $\tau = 4$ & \\
    \midrule
    ResNet-56 $\rightarrow$ LeNet-5x8 & 71.5 $\pm$ 0.2 & 72.4 $\pm$ 0.1 & 73.6 $\pm$ 0.2 & 74.7 $\pm$ 0.2 & \best{76.1 $\pm$ 0.2} & 77.9 \\
    ResNet-56 $\rightarrow$ LeNet-5x8 & 71.5 $\pm$ 0.4 & 71.9 $\pm$ 0.3 & 73.0 $\pm$ 0.3 & \best{75.1 $\pm$ 0.3} & \best{75.1 $\pm$ 0.1} & 77.9 \\
    \midrule
    ResNet-56 $\rightarrow$ ResNet-20 & 75.8 $\pm$ 0.5 & 76.1 $\pm$ 0.2 & 77.1 $\pm$ 0.6 & 77.8 $\pm$ 0.3 & \best{78.1 $\pm$ 0.1} & 77.9 \\
    ResNet-56 $\rightarrow$ ResNet-20 & 76.1 $\pm$ 0.5 & 76.0 $\pm$ 0.2 & 77.4 $\pm$ 0.3 & 78.0 $\pm$ 0.2 & \best{78.4 $\pm$ 0.3} & 77.9 \\
    \midrule
    ResNet-110 $\rightarrow$ LeNet-5x8 & 71.4 $\pm$ 0.4 & 71.9 $\pm$ 0.1 & 72.9 $\pm$ 0.3 & 74.3 $\pm$ 0.3 & \best{75.4 $\pm$ 0.3} & 78.4 \\
    ResNet-110 $\rightarrow$ LeNet-5x8 & 71.6 $\pm$ 0.2 & 71.5 $\pm$ 0.4 & 72.6 $\pm$ 0.4 & \best{74.8 $\pm$ 0.4} & 74.6 $\pm$ 0.2 & 78.4 \\
    \midrule
    ResNet-110 $\rightarrow$ ResNet-20 & 76.0 $\pm$ 0.3 & 76.0 $\pm$ 0.2 & 77.0 $\pm$ 0.1 & 77.3 $\pm$ 0.2 & \best{78.0 $\pm$ 0.4} & 78.4 \\
    ResNet-110 $\rightarrow$ ResNet-20 & 76.1 $\pm$ 0.2 & 76.4 $\pm$ 0.3 & 77.6 $\pm$ 0.3 & 77.9 $\pm$ 0.2 & \best{78.1 $\pm$ 0.1} & 78.4 \\
    \bottomrule
    \end{tabular}%
    % }
\end{table}

\begin{figure}[!t]
    \centering
     \begin{subfigure}{0.45\textwidth}
       \includegraphics[width=\textwidth]{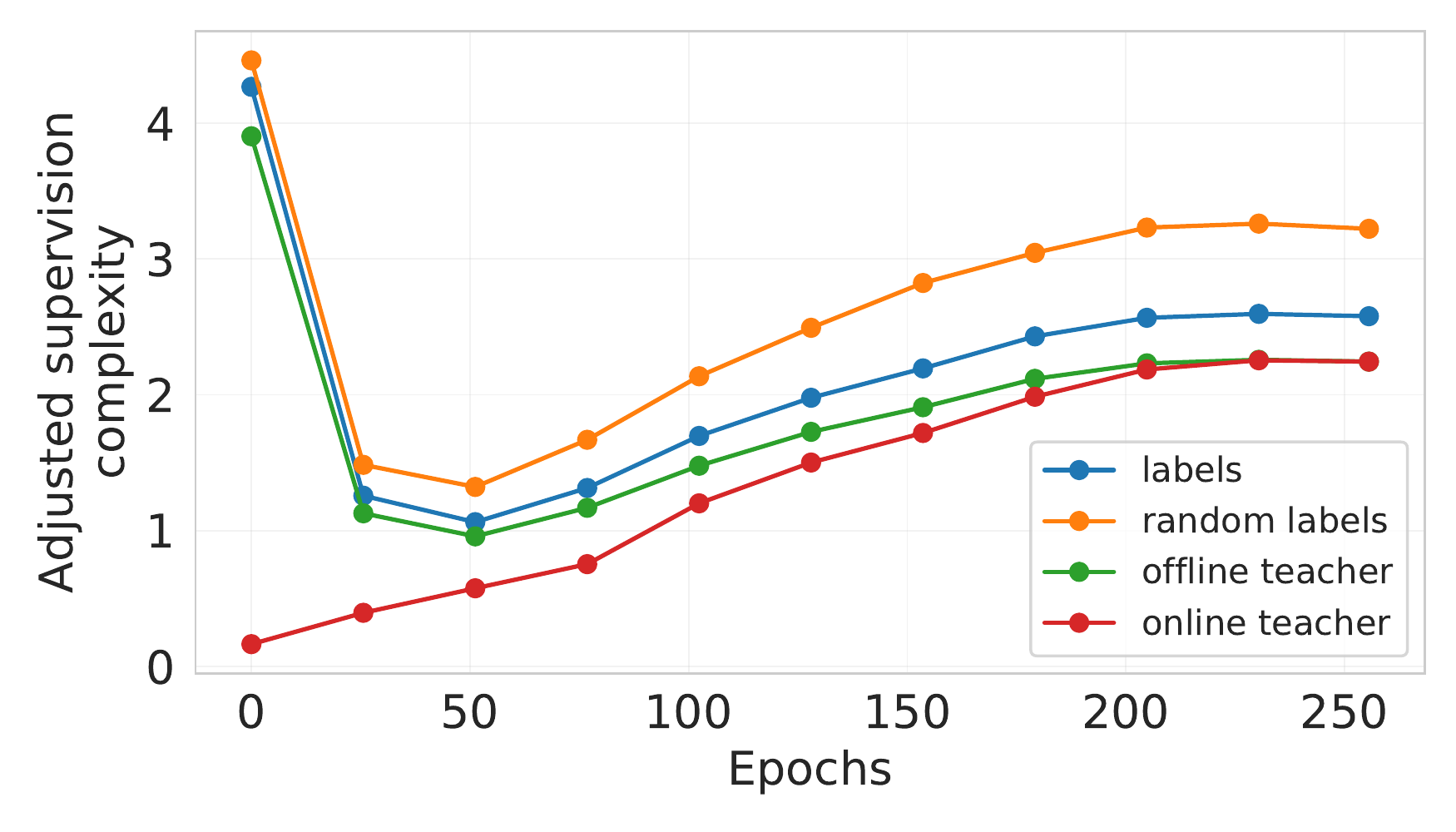}
       \caption{{\centering LeNet-5x8}}
       \label{fig:sup-res-complexity-lenet-test}
    \end{subfigure}%
    \hspace{2em}
    \begin{subfigure}{0.45\textwidth}
        \includegraphics[width=\textwidth]{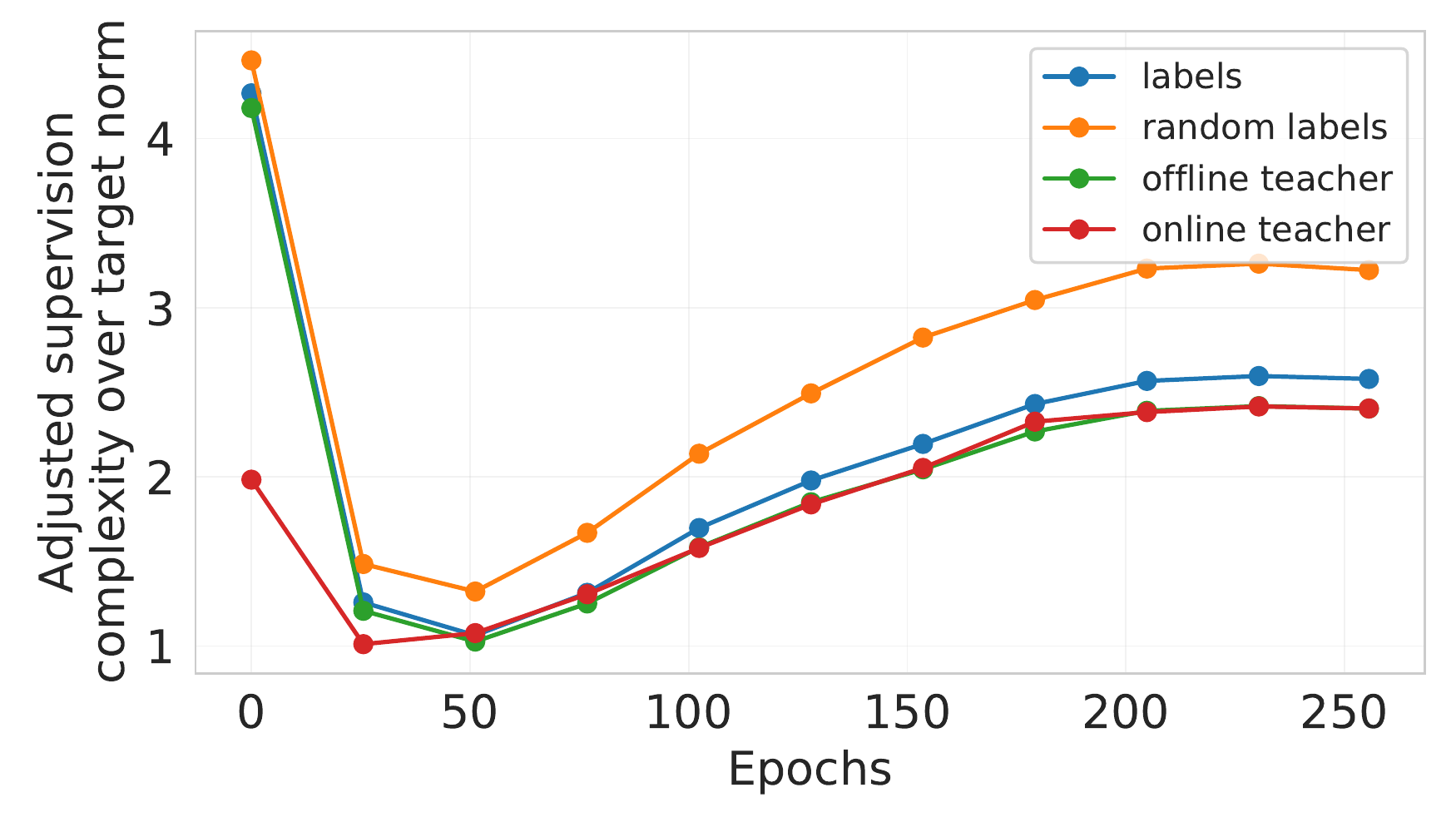}
        \caption{{\centering LeNet-5x8, normalized complexity}}
        \label{fig:normalzied-sup-res-complexity-lenet-test}
    \end{subfigure}

    \begin{subfigure}{0.45\textwidth}
        \includegraphics[width=\textwidth]{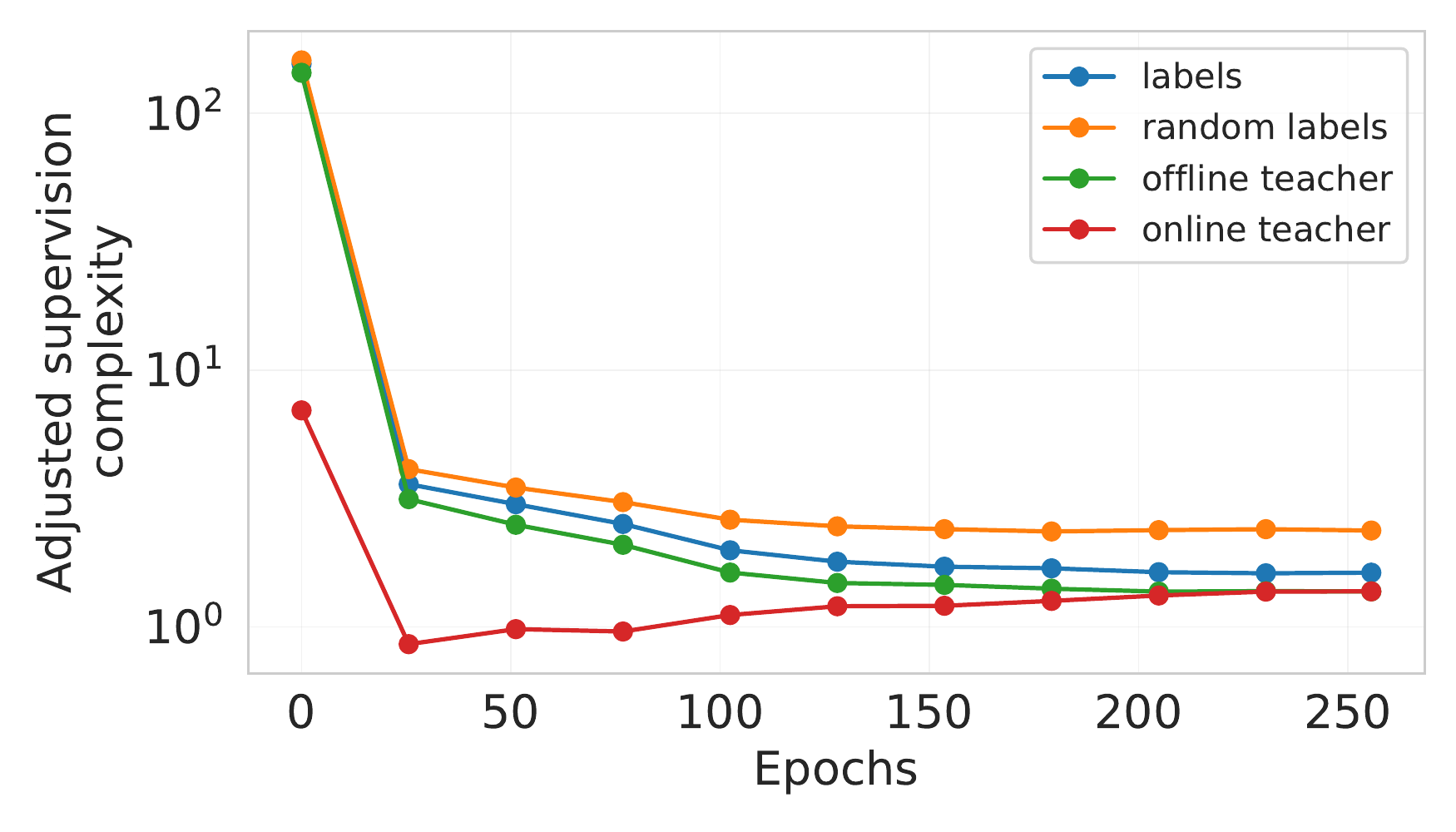}
        \caption{{\centering ResNet-20}}
        \label{fig:sup-res-complexity-resnet-test}
    \end{subfigure}
    \hspace{2em}
    \begin{subfigure}{0.45\textwidth}
        \includegraphics[width=\textwidth]{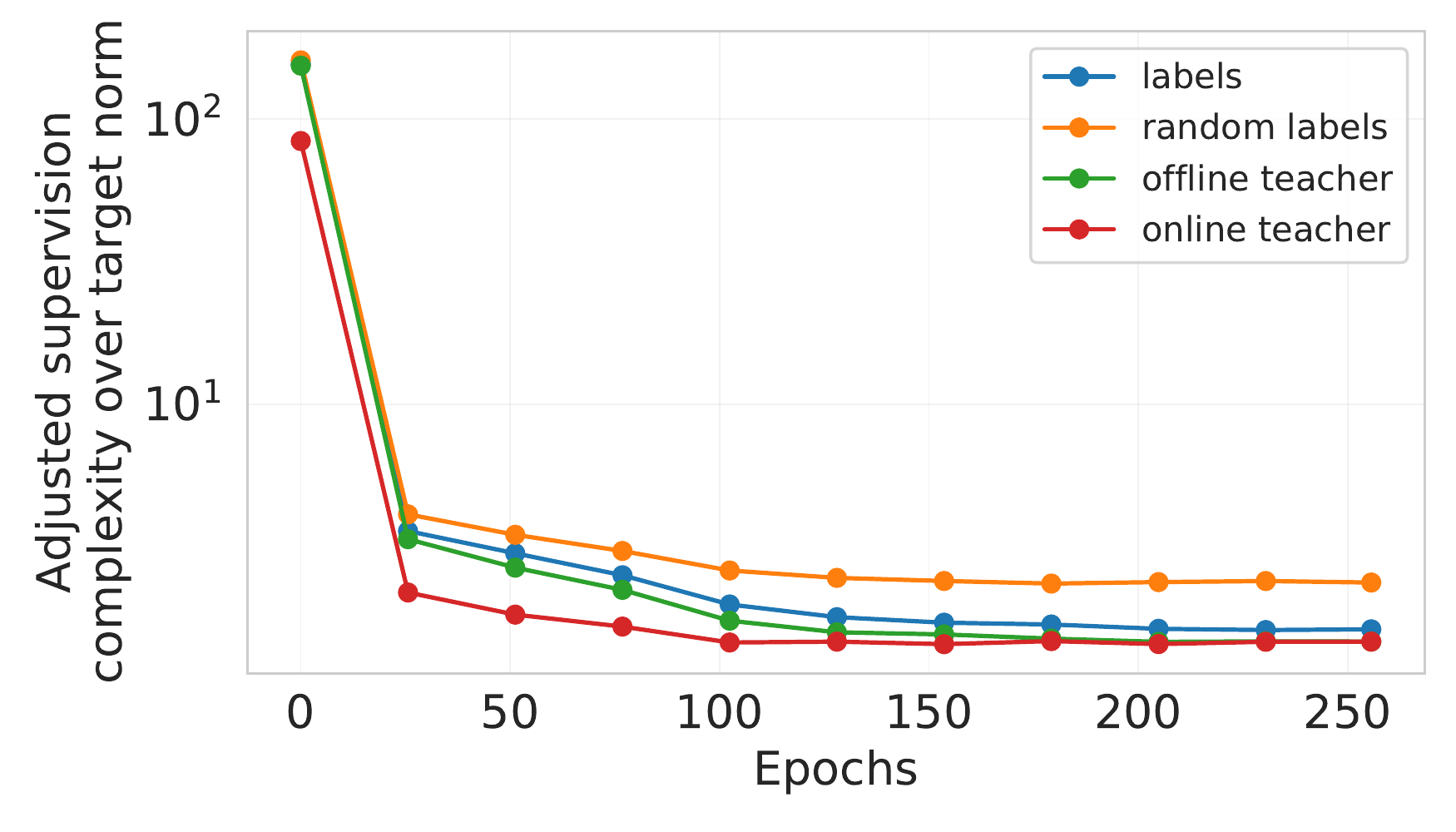}
        \caption{{\centering ResNet-20, normalized complexity}}
        \label{fig:normalzied-sup-res-complexity-resnet-test}
    \end{subfigure}
    \caption{Adjusted supervision complexity of various targets with respect to NTKs at different stages of training. The underlying dataset is binary CIFAR-100.
    Panels (b) and (d) plot adjusted supervision complexity normalized by norm of the targets.
    Note the y-axes of ResNet-20 plots are in logarithmic scale.
    }
    \label{fig:sup-res-complexity-comparison-test}
\end{figure}

\paragraph{Comparison of supervision complexities.}
We compare the adjusted supervision complexities of random labels, dataset labels, and predictions of an offline and online ResNet-56 teacher predictions with respect to various checkpoints of the LeNet-5x8 and ResNet-20 networks.
The results presented in \cref{fig:sup-res-complexity-lenet-test,fig:sup-res-complexity-resnet-test} indicate that 
the dataset labels and offline teacher predictions are as complex as random labels in the beginning.
After some initial decline, the complexity of these targets increases as the network starts to overfit.
Given the lower bound on the supervision complexity of random labels (see~\cref{sec:analysis}), this increase means that the NTK spectrum becomes less uniform.
This is confirmed in \cref{fig:ntk-condition-number}.
Unlike LeNet-5x8, for ResNet-20, random labels, dataset labels, and offline teacher predictions do not exhibit a U-shaped behavior.
In this case too, the shape of these curves is in agreement with the behavior of the condition number of the NTK (see \cref{fig:ntk-condition-number}).

\begin{figure}[!t]
    \centering
    \begin{subfigure}{0.45\textwidth}
        \includegraphics[width=\textwidth]{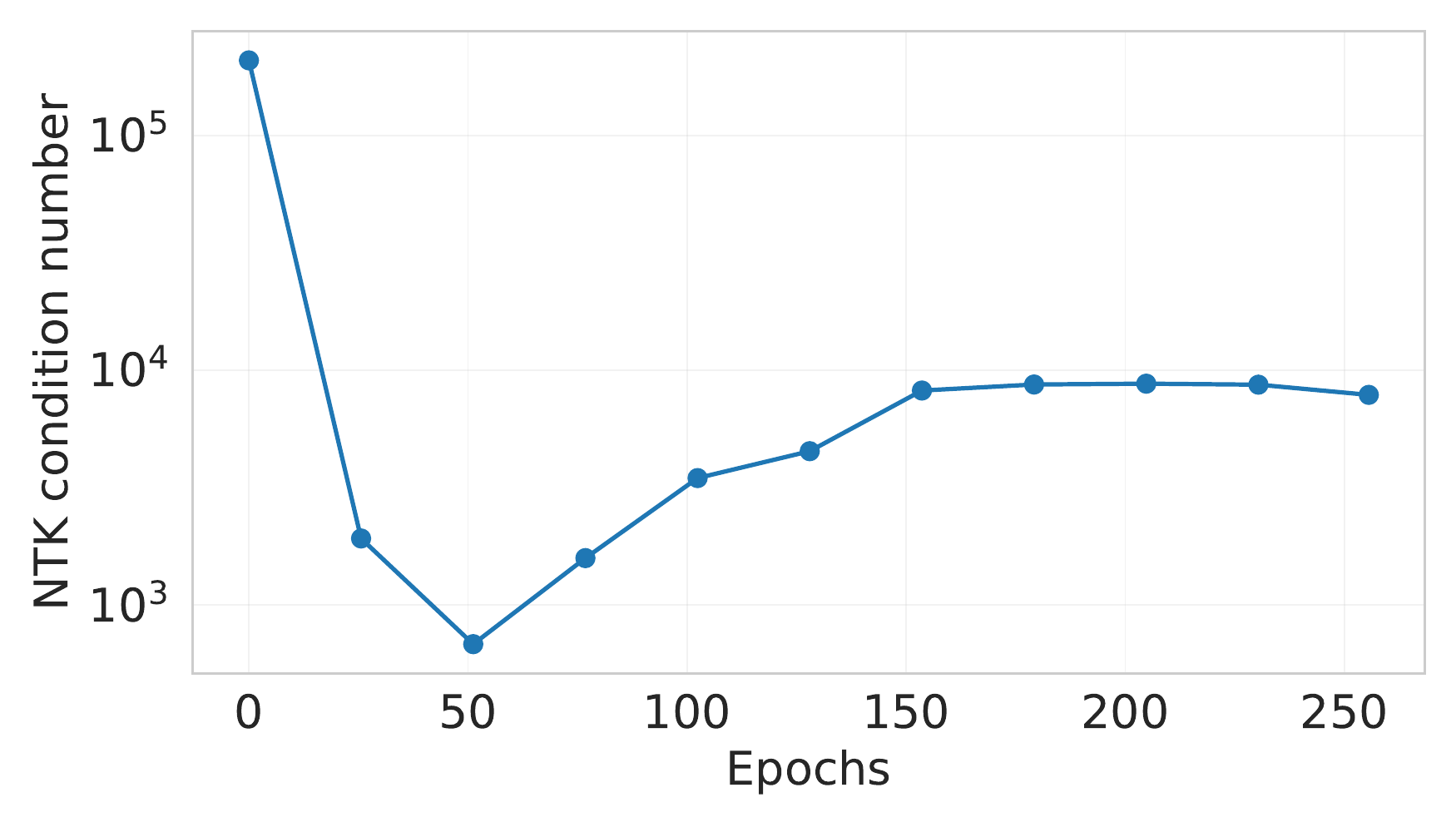}
        \caption{LeNet-5x8 student}
    \end{subfigure}
    \hspace{2em}
    \begin{subfigure}{0.45\textwidth}
        \includegraphics[width=\textwidth]{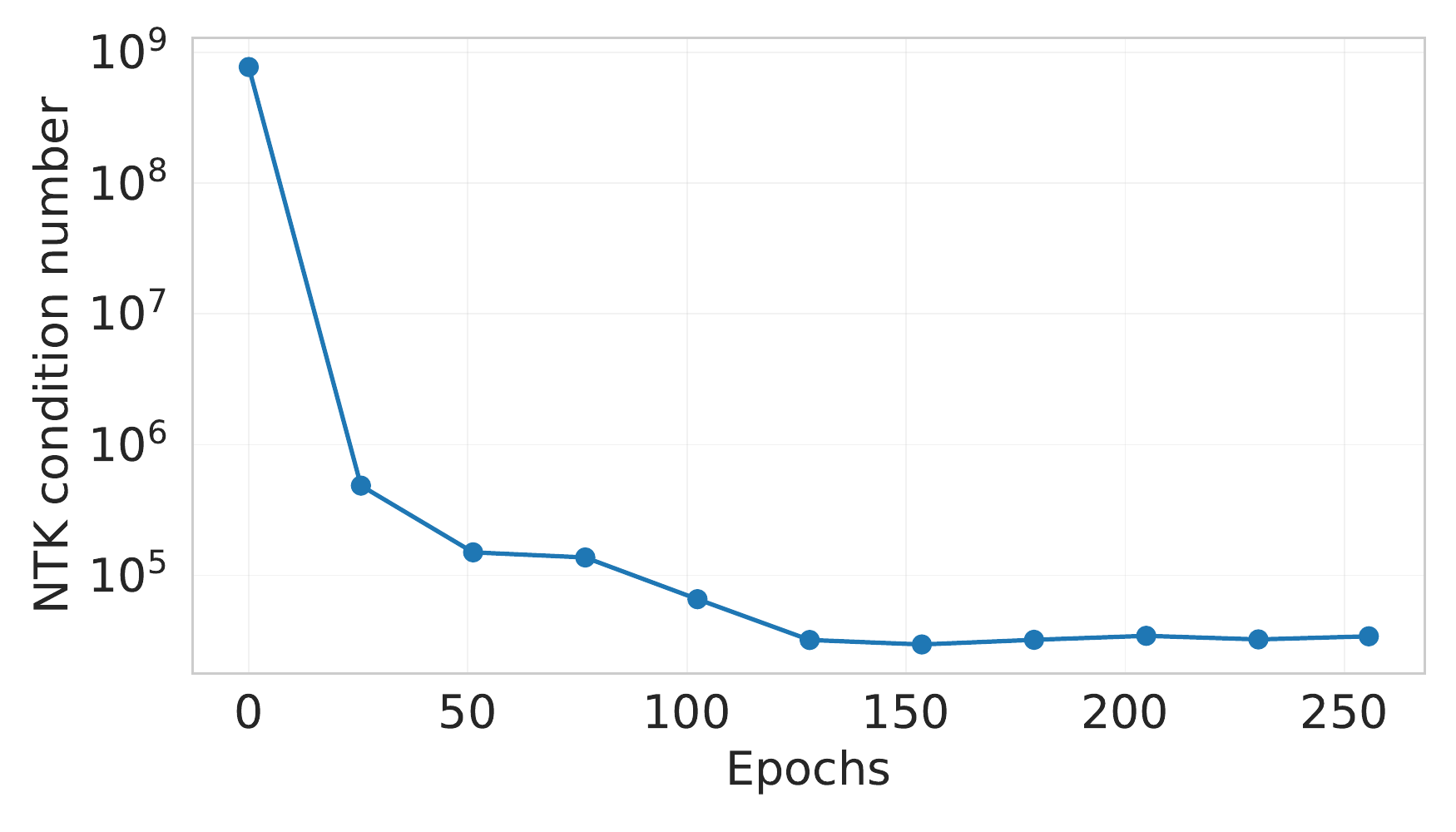}
        \caption{ResNet-20 student}
    \end{subfigure}
    \caption{Condition number of the NTK matrix of a LeNet5x8 (a) and ResNet-20 (b) students trained with MSE loss on binary CIFAR-100. The NTK matrices are computed on $2^{12}$ test examples.}
    \label{fig:ntk-condition-number}
\end{figure}

In contrast to these static targets, the complexity of the online teacher predictions smoothly increases, and is significantly smaller for most of the epochs.
To account for softness differences of the various targets, we consider plotting the adjusted supervision complexity normalized by the target norm $\nbr{\bs{Y'}}_2$.
As shown in \cref{fig:normalzied-sup-res-complexity-lenet-test,fig:normalzied-sup-res-complexity-resnet-test}, the normalized complexity of offline and online teacher predictions is smaller compared to the dataset labels, indicating a better alignment with top eigenvectors of the corresponding NTKs.
Importantly, we see that the predictions of an online teacher have significantly lower normalized complexity in the critical early stages of training.

\begin{figure}[!t]
    \centering
    \begin{subfigure}{0.45\textwidth}
        \includegraphics[width=\textwidth]{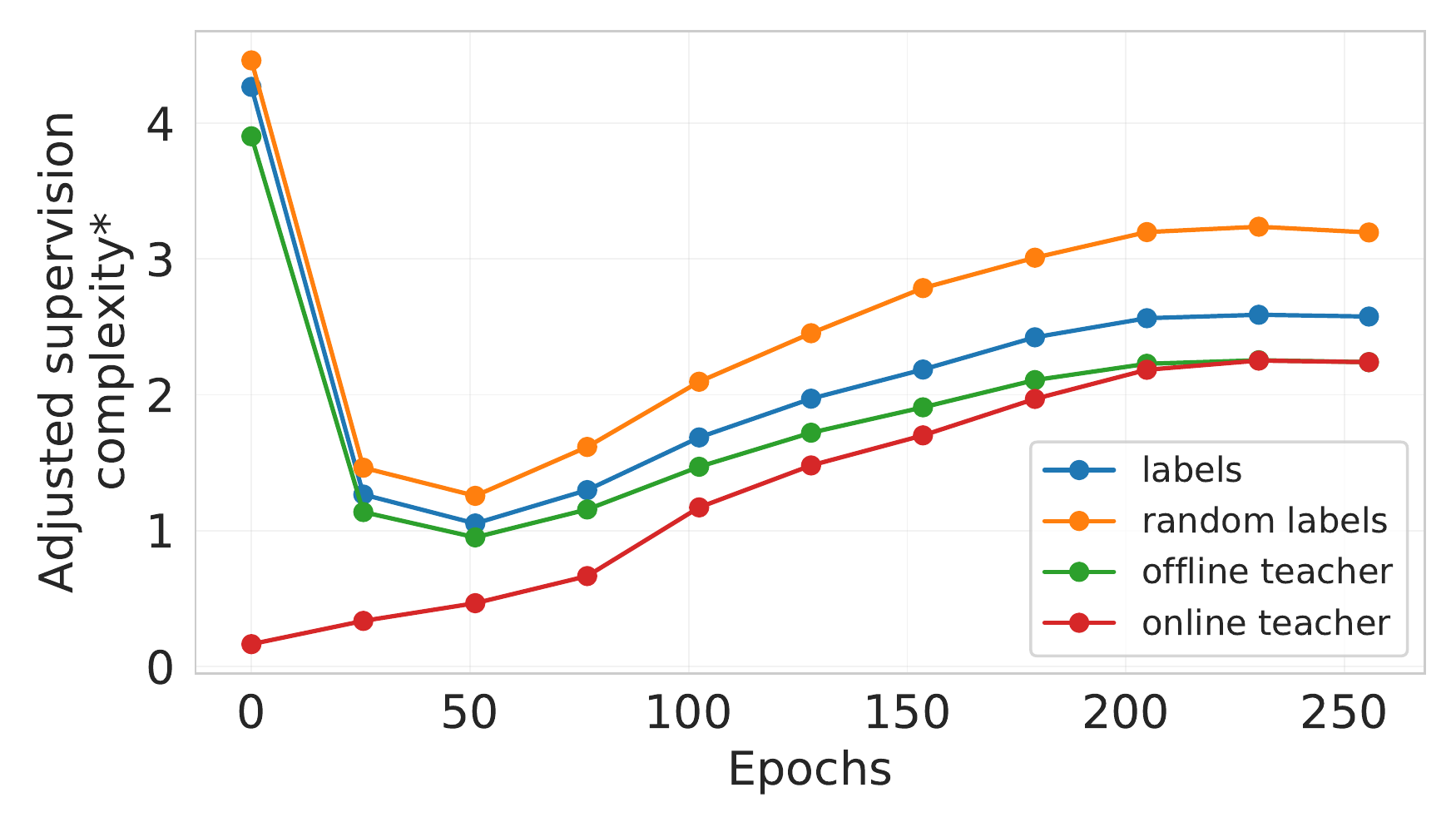}
    \end{subfigure}
    \hspace{2em}
    \begin{subfigure}{0.45\textwidth}
        \includegraphics[width=\textwidth]{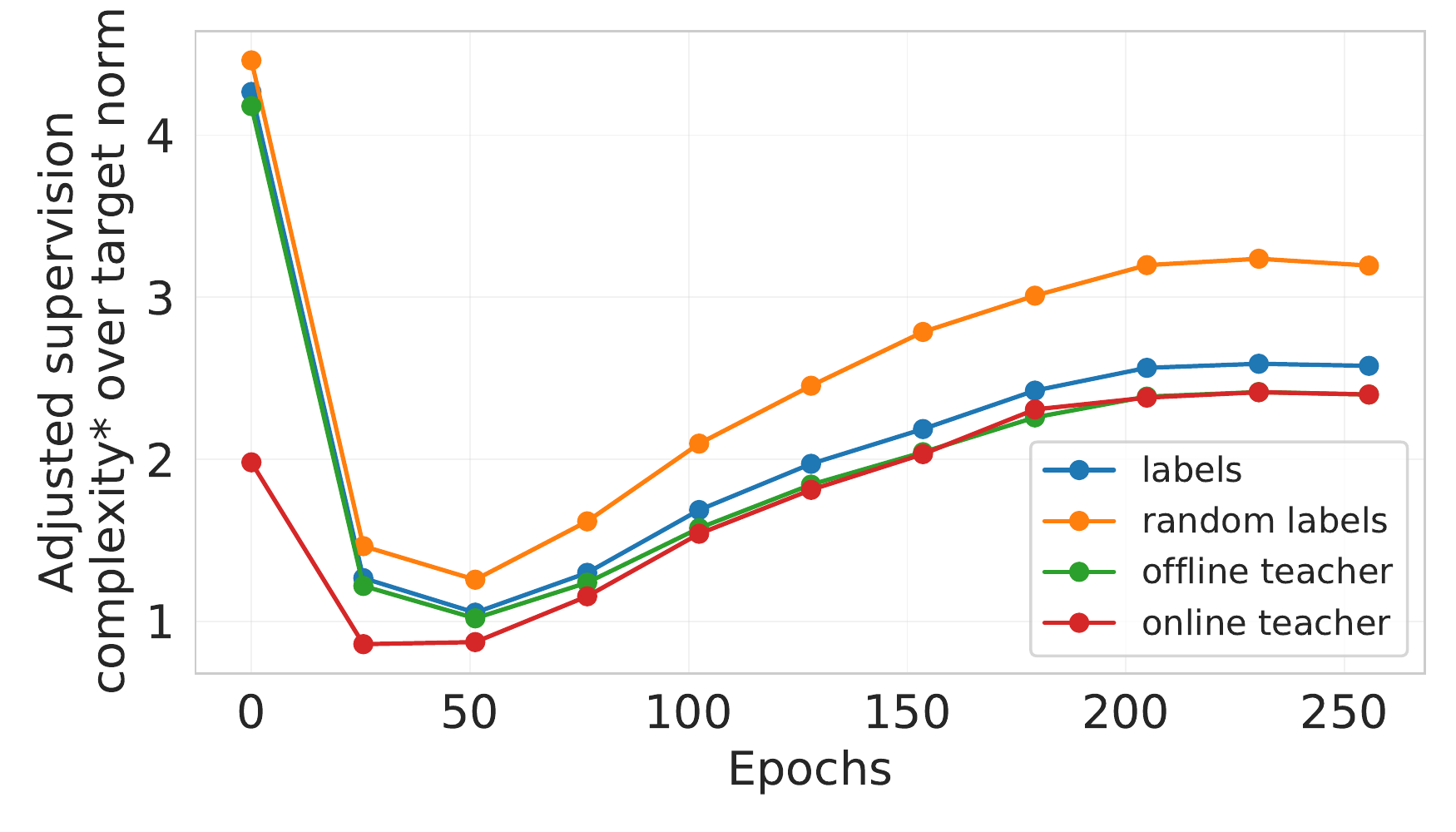}
    \end{subfigure}
    \caption{Adjusted supervision complexities* of various targets with respect to a LeNet-5x8 network at different stages of its training. The experimental setup of the left and right plots matches that of \cref{fig:sup-res-complexity-lenet-test} and \cref{fig:normalzied-sup-res-complexity-lenet-test} respectively.}
    \label{fig:sup-non-res-complexity-comparison-test}
\end{figure}

We compare the adjusted supervision complexities* of random labels, dataset labels, and predictions of an offline and online ResNet-56 teacher predictions with respect to various checkpoints of the LeNet-5x8 network.
The results presented in \cref{fig:sup-non-res-complexity-comparison-test} are remarkably similar to the results with adjusted supervision complexity (\cref{fig:sup-res-complexity-lenet-test} and \cref{fig:normalzied-sup-res-complexity-lenet-test}).
We therefore, focus only on adjusted supervision complexity of \cref{eq:adjusted-res-complexity} when comparing various targets.
The only other experiment where we compute adjusted supervision complexities* (i.e., without subtracting the current predictions from labels) is presented in \cref{fig:ntk-alignment-with-labels}, where the goal is to demonstrate that training labels become aligned with the training NTK matrix over the course of training.

\begin{figure}[t]
    \centering
    \begin{subfigure}{0.45\textwidth}
        \includegraphics[width=\textwidth]{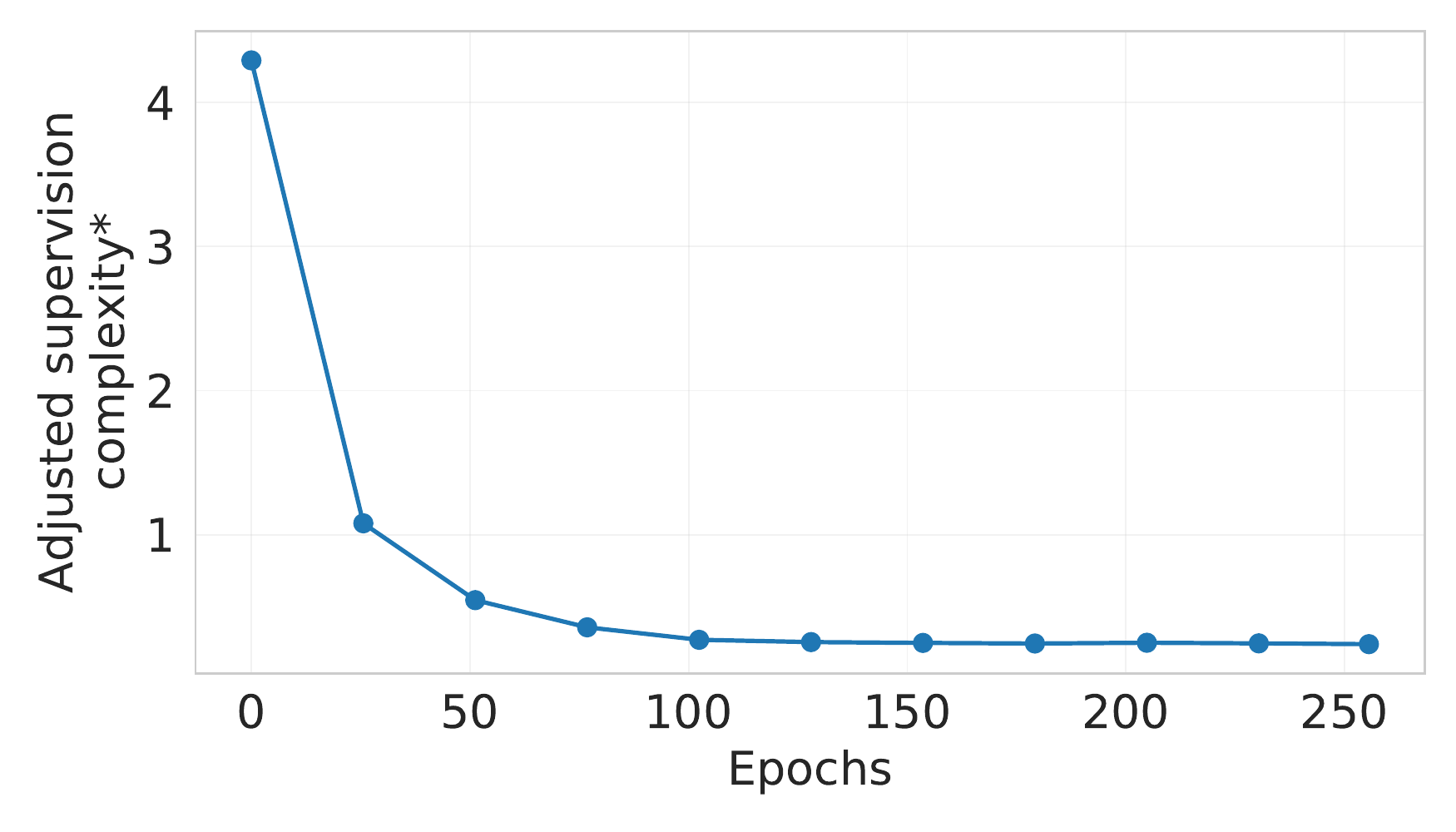}
        \caption{LeNet-5x8 student}
    \end{subfigure}
    \hspace{2em}
    \begin{subfigure}{0.45\textwidth}
        \includegraphics[width=\textwidth]{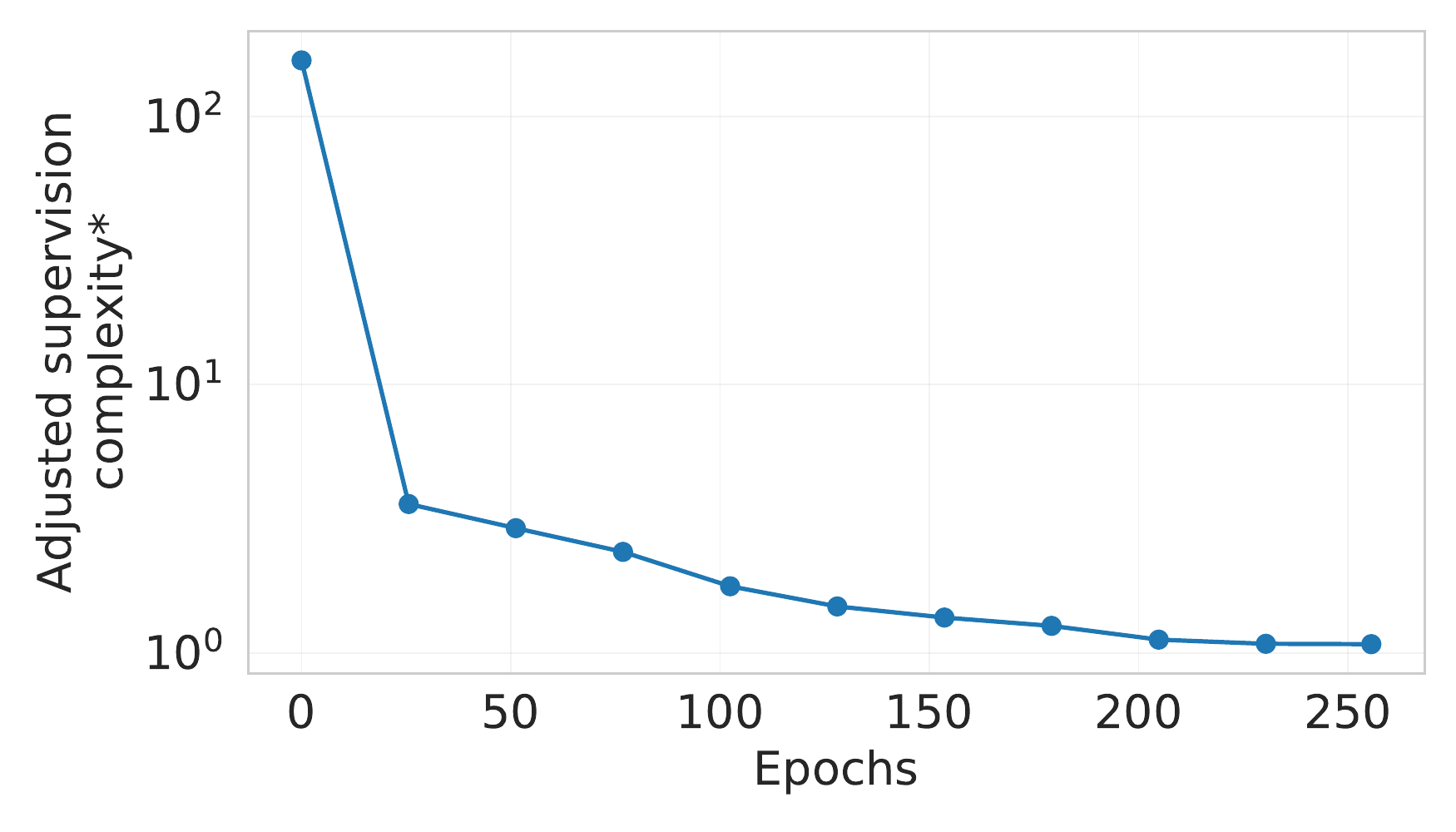}
        \caption{ResNet-20 student}
    \end{subfigure}
    \caption{Adjusted supervision complexity* of dataset labels measured on a subset of $2^{12}$ \emph{training} examples of binary CIFAR-100. Complexities are measured with respect to either a LeNet-5x8 (on the left) or ResNet-20 (on the right) models trained with MSE loss and without knowledge distillation. Note that the plot on the right is in logarithmic scale.}
    \label{fig:ntk-alignment-with-labels}
\end{figure}

\paragraph{On early stopped teachers.} \citet{cho2019efficacy} observe that sometimes offline KD works better with early stopped teachers.
Such teachers have worse accuracy and perhaps results in a smaller student margin, but they also have a significantly smaller supervision complexity (see \cref{fig:sup-res-complexity-comparison-test}), which provides a possible explanation for this phenomenon.

\paragraph{Effect of temperature scaling.}
As discussed earlier, higher temperature makes the teacher predictions softer, decreasing their norm.
This has a large effect on supervision complexity (\cref{fig:temperature-effect}).
Even when one controls for the norm of the predictions, the complexity still decreases (\cref{fig:temperature-effect}).

\begin{figure}[!t]
    \centering
    \begin{subfigure}{0.45\textwidth}
        \includegraphics[width=\textwidth]{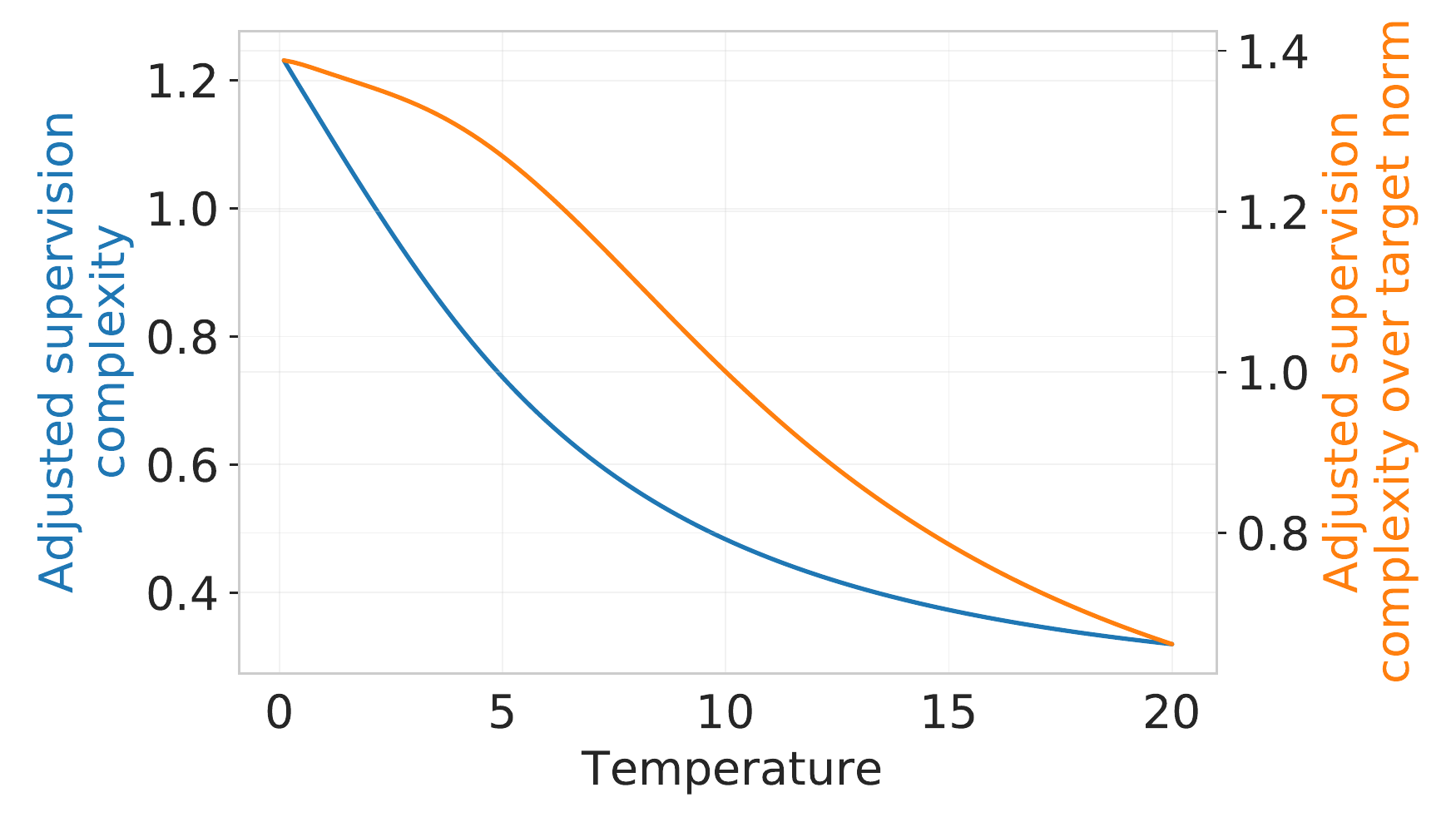}
        \caption{{\centering Temperature effect}}
        \label{fig:temperature-effect}
    \end{subfigure}
    \hspace{2em}
    \begin{subfigure}{0.45\textwidth}
        \includegraphics[width=\textwidth]{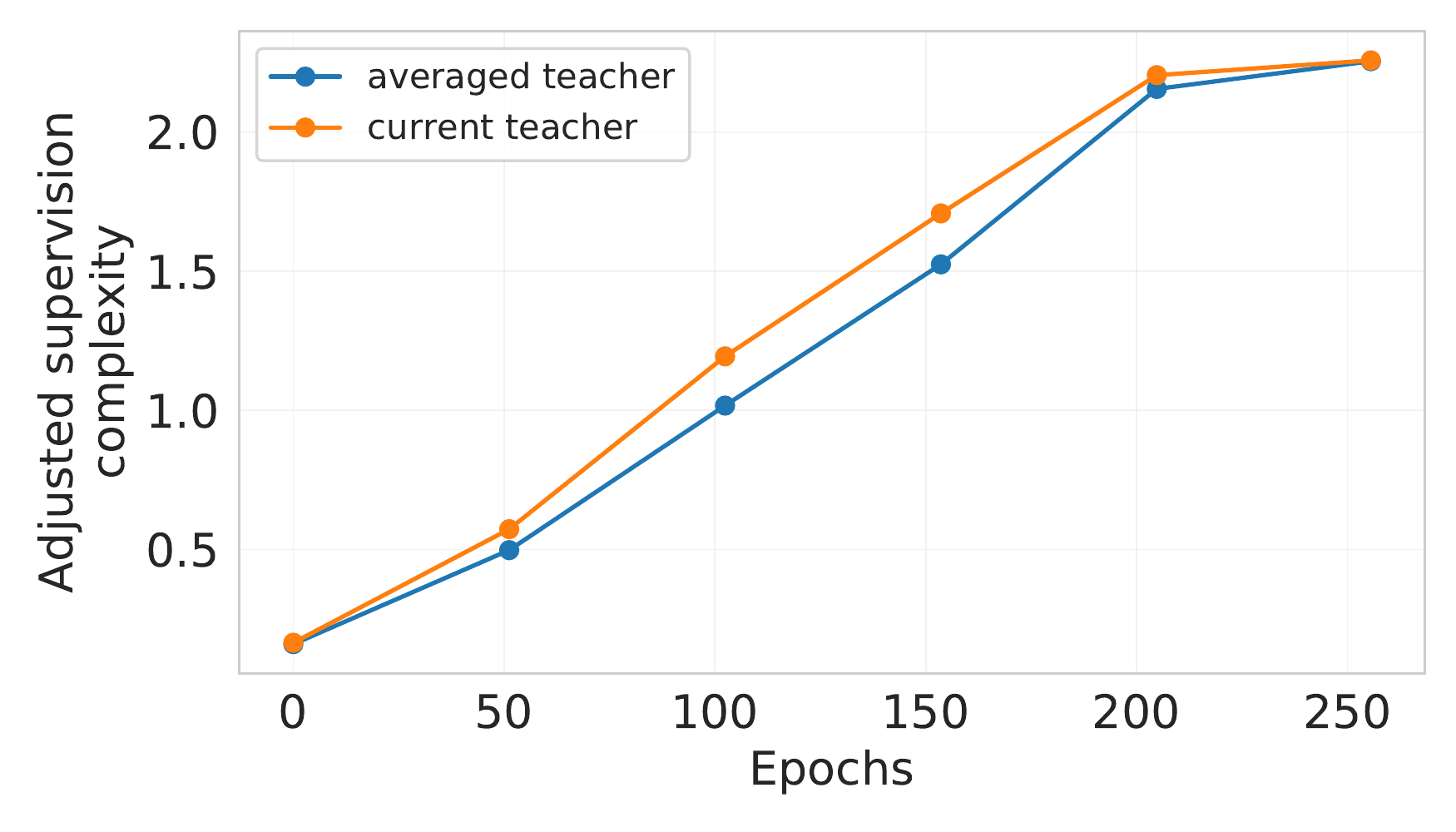}
        \caption{{\centering Teacher averaging}}
        \label{fig:teacher-averaging}
    \end{subfigure}
    \caption{Adjusted supervision complexity for various targets. \emph{One the left:} The effect of temperature on the supervision complexity of an offline teacher for a LeNet-5x8 after training for 25 epochs. \emph{On the right:} The effect of averaging teacher predictions.}
\end{figure}

\paragraph{Average teacher complexity.} \citet{ren2022better} observe that teacher predictions fluctuate over time, and showe that using exponentially averaged teachers improves knowledge distillation. \cref{fig:teacher-averaging} demonstrates that the supervision complexity of an online teacher predictions is always slightly larger than that of the average of predictions of teachers of the last 10 preceding epochs.

\paragraph{Teaching students with weak inductive biases.}
As we saw earlier, a fully trained teacher can have predictions as complex as random labels for a weak student at initialization.
This low alignment of student NTK and teacher predictions can result in memorization.
In contrast, an early stopped teacher captures simple patterns and has a better alignment with the student NTK, allowing the student to learn these patterns in a generalizable fashion.
This feature learning improves the student NTK and allows learning more complex patterns in future iterations.
We hypothesize that this is the mechanism that allows online distillation to outperform offline distillation in some cases.

\paragraph{NTK similarity.}
Remarkably, we observe that across all of our experiments, the final test accuracy of the student is strongly correlated with the similarity of final teacher and student NTKs (see \cref{fig:ntk-matching-comparision,fig:ntk-matching-comparision-1-appendix,fig:ntk-matching-comparision-2-appendix}).
This cannot be explained by better matching the teacher predictions.
In fact, we see that the final fidelity 
(the rate of classification agreement of a teacher-student pair) measured on training set has no clear relationship with test accuracy.
Furthermore, we see that online KD results in better NTK transfer without an explicit regularization loss enforcing such transfer.

\begin{figure}[!t]
    \centering
    \begin{subfigure}{0.3\textwidth}
        \includegraphics[width=\textwidth]{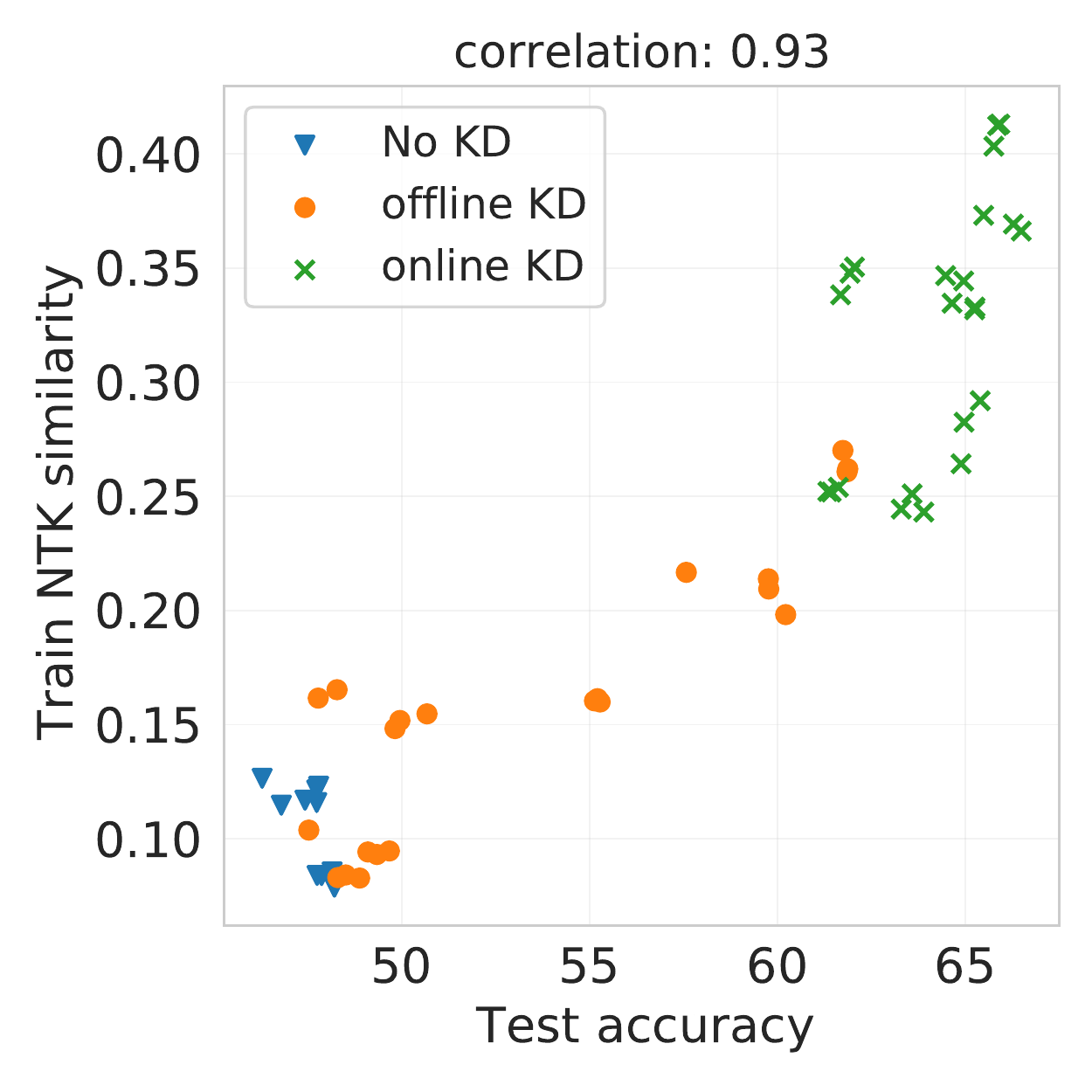}
        \caption{{\centering LeNet-5x8 student}}
    \end{subfigure}
    \hspace{1em}
    \begin{subfigure}{0.3\textwidth}
        \includegraphics[width=\textwidth]{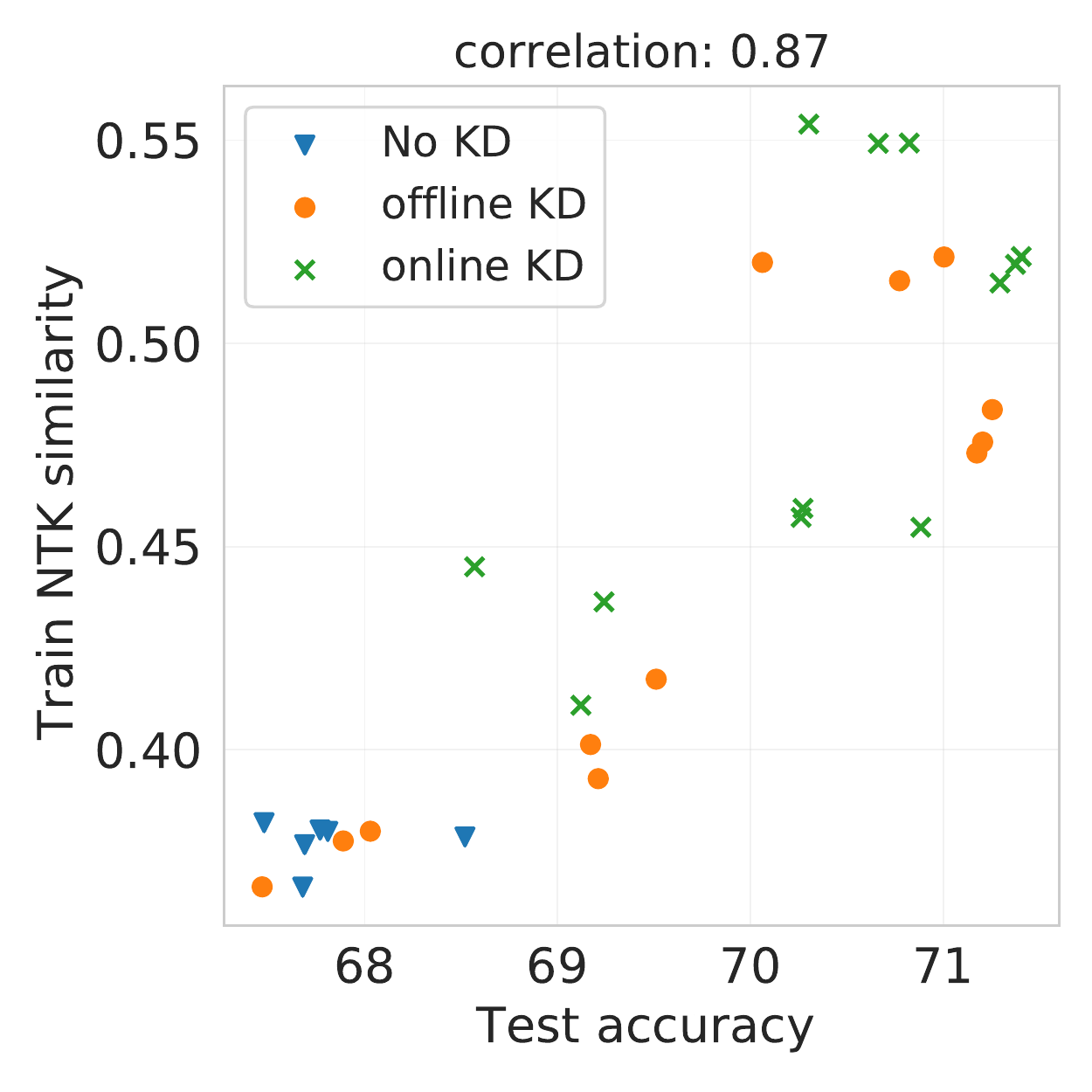}
        \caption{{\centering ResNet-20 student}}
    \end{subfigure}
    \hspace{1em}
    \begin{subfigure}{0.3\textwidth}
        \includegraphics[width=\textwidth]{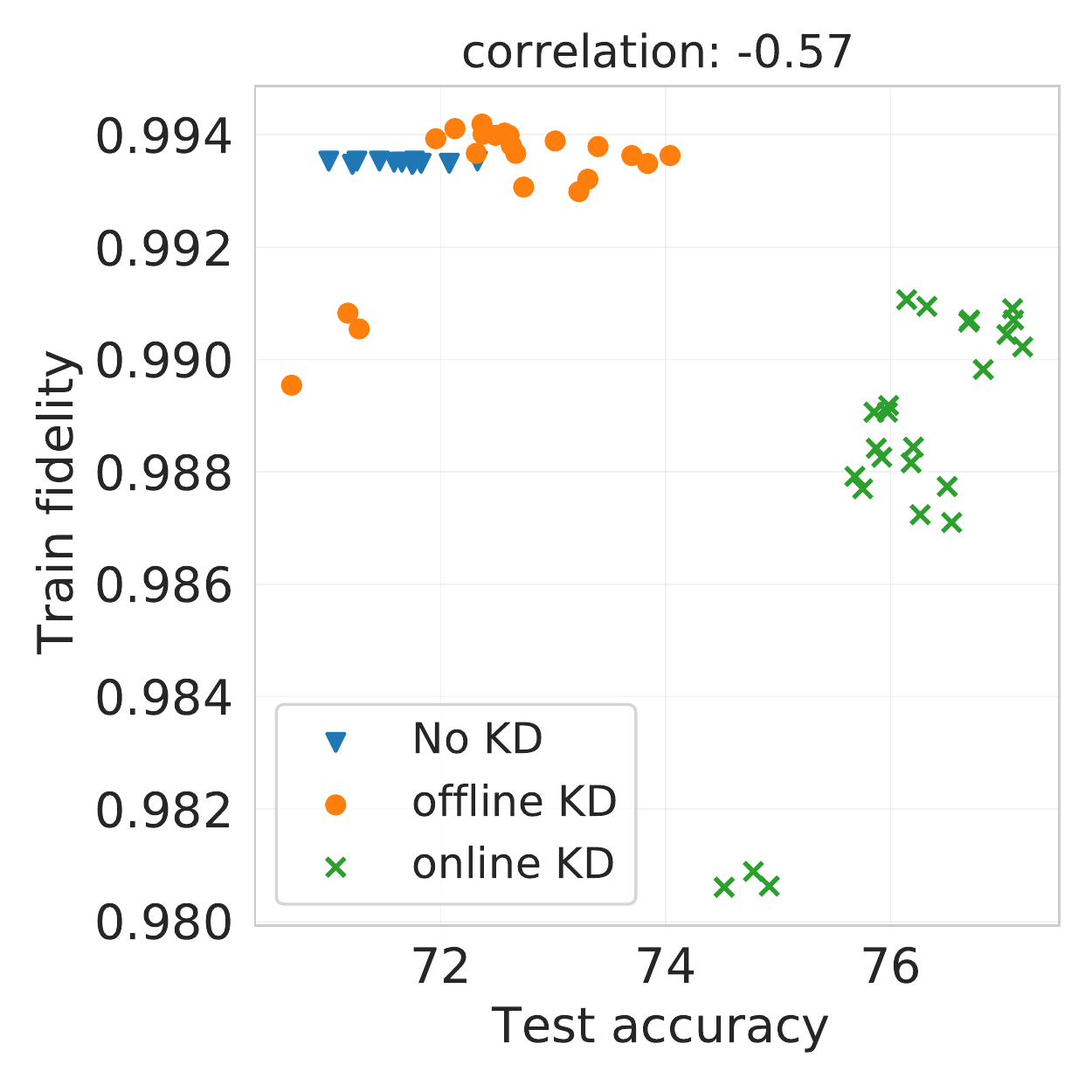}
        \caption{{\centering LeNet-5x8 student}}
    \end{subfigure}
    \caption{Relationship between test accuracy, train NTK similarity, and train fidelity for CIFAR-100 students training with either ResNet-56 teacher (panels (a) and (c)) or ResNet-110 (panel (b)).}
    \label{fig:ntk-matching-comparision}
\end{figure}

\paragraph{The effect of frequency of teacher checkpoints.}
As mentioned earlier, we used one teacher checkpoint per epoch so far. While this served our goal of establishing efficacy of online distillation, this choice is prohibitive for large teacher networks.
To understand the effect of the frequency of teacher checkpoints, we conduct an experiment on CIFAR-100 with ResNet-56 and LeNet-5x8 student with varying frequency of teacher checkpoints. In particular, we consider checkpointing the teacher once in every $\mathset{1, 2, 4, 8, 16, 32, 64, 128}$ epochs.
The results presented in \cref{fig:cifar100-ckpt-frequency} show that reducing the teacher checkpointing frequency to once in 16 epochs results in only a minor performance drop for online distillation with $\tau=4$.

\begin{figure}
    \centering
    \small
    \begin{subfigure}[h]{0.49\textwidth}
    \scalebox{1}{
    \begin{tabular}{@{}lcc@{}}
        \toprule
        {\bf Teacher update period} & \multicolumn{2}{c}{\bf Online KD} \\
        {\bf in epochs} & $\tau = 1$ & $\tau = 4$ \\
        \midrule
        1 (the default value) & 61.9 $\pm$ 0.2 & 66.1 $\pm$ 0.4 \\
        2   & 61.5 $\pm$ 0.4 & 66.0 $\pm$ 0.3  \\
        4   & 61.4 $\pm$ 0.2 & 65.6 $\pm$ 0.2  \\
        8   & 60.0 $\pm$ 0.3 & 65.4 $\pm$ 0.0  \\
        16  & 59.3 $\pm$ 0.6 & 65.4 $\pm$ 0.0  \\
        32  & 56.9 $\pm$ 0.0 & 64.1 $\pm$ 0.4  \\
        64  & 55.5 $\pm$ 0.4 & 62.8 $\pm$ 0.7  \\
        128 & 51.4 $\pm$ 0.5 & 61.3 $\pm$ 0.1  \\
        \bottomrule
    \end{tabular}%
    }
    \end{subfigure}
    \begin{subfigure}[h]{0.49\textwidth}
        \includegraphics[width=0.9\textwidth]{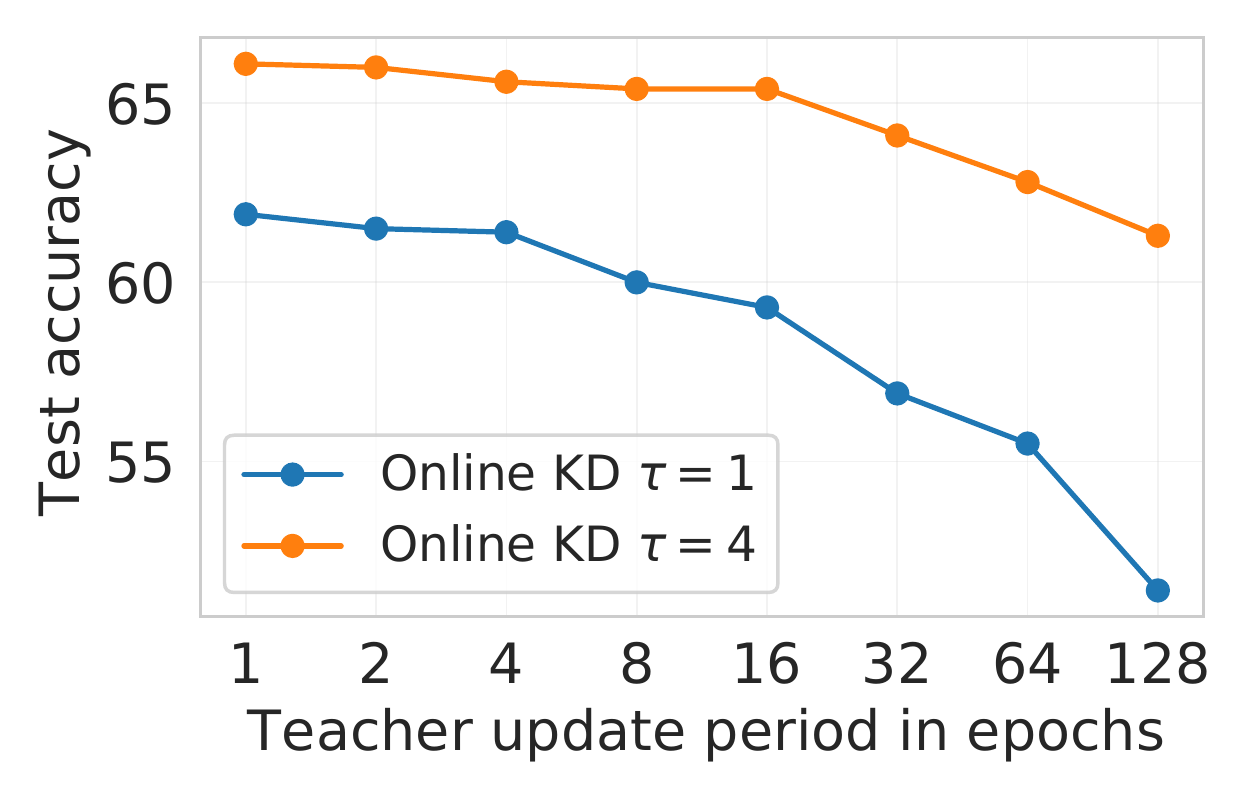}
    \end{subfigure}
    \caption{Online KD results for a LeNet-5x8 student on CIFAR-100 with varying frequency of a ResNet-56 teacher checkpoints.}
    \label{fig:cifar100-ckpt-frequency}
\end{figure}

\paragraph{On label supervision in KD.}
So far in all distillation methods dataset labels were not used as an additional source of supervision for students.
However, in practice it is common to train a student with a convex combination of knowledge distillation and standard losses: $(1 - \alpha) \LL_\mathrm{ce} + \alpha \LL_\mathrm{kd-ce}$.
To verify that the choice of $\alpha=1$ does not produce unique conclusions regarding efficacy of online distillation, we do experiments on CIFAR-100 with varying values of $\alpha$.
The results presented in \cref{tbl:cifar100-kd-alpha} confirm our main conclusions on online distillation.
Furthermore, we observe that picking $\alpha=1$ does not result in significant degradation of student performance.

\begin{table}[!t]
    \centering
    \small
    \caption{Knowledge distillation results on CIFAR-100 with varying loss mixture coefficient $\alpha$.}
    \label{tbl:cifar100-kd-alpha}
    % \resizebox{\linewidth}{!}{%
    % \scalebox{1}{
    \begin{tabular}{@{}llcccccc@{}}
    \toprule
    {\bf Setting} & {\bf $\alpha$} & {\bf No KD} & \multicolumn{2}{c}{\bf Offline KD} & \multicolumn{2}{c}{\bf Online KD} & {\bf Teacher} \\
     & &  & $\tau = 1$ & $\tau = 4$ & $\tau = 1$ & $\tau = 4$ & \\
    \midrule
    \multirow{5}{*}{{\centering \begin{tabular}{c}ResNet-56 $\rightarrow$\\LeNet-5x8\end{tabular}}} & 0.2 & \multirow{5}{*}{47.3 $\pm$ 0.6} & 47.6 $\pm$ 0.7 & 57.6 $\pm$ 0.2 & 54.3 $\pm$ 0.7 & 59.0 $\pm$ 0.6 & \multirow{5}{*}{72.0}\\
    & 0.4 & & 48.9 $\pm$ 0.3 & 58.9 $\pm$ 0.4 & 56.7 $\pm$ 0.5 & 62.5 $\pm$ 0.2 & \\
    & 0.6 & & 49.4 $\pm$ 0.5 & 59.7 $\pm$ 0.0 & 61.1 $\pm$ 0.0 & 65.3 $\pm$ 0.2 & \\
    & 0.8 & & 49.8 $\pm$ 0.1 & 60.1 $\pm$ 0.1 & 62.0 $\pm$ 0.1 & 65.9 $\pm$ 0.2 & \\
    & 1.0 & & 50.1 $\pm$ 0.4 & 59.9 $\pm$ 0.2 & 61.9 $\pm$ 0.2 & 66.1 $\pm$ 0.4 & \\
    \midrule
    \multirow{5}{*}{{\centering \begin{tabular}{c}ResNet-56 $\rightarrow$\\ResNet-20\end{tabular}}} & 0.2 & \multirow{5}{*}{67.7 $\pm$ 0.5} & 67.9 $\pm$ 0.3 & 70.3 $\pm$ 0.3 & 68.2 $\pm$ 0.3 & 70.3 $\pm$ 0.1 & \multirow{5}{*}{72.0}\\
    & 0.4 & & 67.9 $\pm$ 0.1 & 71.0 $\pm$ 0.2 & 68.7 $\pm$ 0.2 & 71.4 $\pm$ 0.2 & \\
    & 0.6 & & 68.1 $\pm$ 0.3 & 71.3 $\pm$ 0.1 & 69.6 $\pm$ 0.4 & 71.5 $\pm$ 0.2 & \\
    & 0.8 & & 68.3 $\pm$ 0.2 & 71.4 $\pm$ 0.4 & 69.8 $\pm$ 0.3 & 71.1 $\pm$ 0.3 & \\
    & 1.0 & & 68.2 $\pm$ 0.3 & 71.6 $\pm$ 0.2 & 69.6 $\pm$ 0.3 & 71.4 $\pm$ 0.3 & \\
    \bottomrule
    \end{tabular}%
    % }
\end{table}

%% file: sup-complexity/related-work.tex
\section{Related work}
\label{sec:sup-complexity-related}
The key contributions of this work are the demonstration of the role of supervision complexity in student generalization, and the establishment of online knowledge distillation as a theoretically grounded and effective method.
Both supervision complexity and online distillation have a number of relevant precedents in the literature that are worth comment.

\paragraph{Transferring knowledge beyond logits.}
In the seminal works of~\citet{Bucilla:2006} and \citet{hinton2015distilling} transferred ``knowledge'' is in the form of output probabilities.
Later works suggest other notions of ``knowledge`` and other ways of transferring knowledge~\citep{gou2021knowledge}.
These include 
activations of intermediate layers~\citep{Romero15fitnets:hints}, 
attention maps~\citep{zagoruyko2017paying}, 
classifier head parameters~\citep{chen2022knowledge}, and 
various notions of example similarity~\citep{passalis2018learning, park2019relational, tung2019similarity, Tian2020Contrastive, he2021feature}.
Transferring teacher NTK matrix belongs to this latter category of methods.
\citet{zhang2022learning} propose to transfer a low-rank approximation of a feature map corresponding the teacher NTK.

\paragraph{Non-static teachers.}
Some works on KD consider non-static teachers.
In order to bridge teacher-student capacity gap,~\citet{mirzadeh2020improved} propose to perform a few rounds of distillation with teachers of increasing capacity.
In deep mutual learning~\citep{zhang2018deep, chen2020online}, codistillation~\citep{anil2018large}, and collaborative learning~\citep{guo2020online},
multiple students are trained simultaneously, distilling from each other or from an ensemble. 
In~\citet{zhou2018rocket} and \citet{shi2021prokt}, the teacher and the student are trained together.
In the former they have a common architecture trunk, while in the latter the teacher is penalized to keep its predictions close to the student's predictions.
\citet{jin2019rco} study \textit{route constrained optimization} which is closest to the online distillation in \cref{sec:online-kd}.
They employ a few teacher checkpoints to perform a multi-round knowledge distillation.
\citet{rezagholizadeh-etal-2022-pro} employ a similar procedure but with an annealed temperature that decreases linearly with training time, followed by a phase of training with dataset labels only.
The idea of distilling from checkpoints also appears in \citet{yang2019snapshot}, where a network is trained with a cosine learning rate schedule, simultaneously distilling from the checkpoint of the previous learning rate cycle.
We complement this line of work by highlighting the role of supervision complexity and by demonstrating that online distillation can be very powerful for students with weak inductive biases.

\paragraph{Fundamental understanding of distillation.}
The effects of temperature, teacher-student capacity gap, optimization time, data augmentations, and other training details is non-trivial~\citep{cho2019efficacy,beyer2022knowledge,stanton2021does}.
It has been hypothesized and shown to some extent that teacher soft predictions capture class similarities, which is beneficial for the student~\citep{hinton2015distilling, furlanello2018born, tang2020understanding}.
\citet{yuan2020revisiting} demonstrate that this softness of teacher predictions also has a regularization effect, similar to label smoothing.
\citet{menon2021statistical} argue that teacher predictions are sometimes closer to the Bayes classifier than the hard labels of the dataset, reducing the variance of the training objective.
The vanilla knowledge distillation loss also introduces some optimization biases.
\citet{mobahi2020self} prove that for kernel methods with RKHS norm regularization, self-distillation increases regularization strength, resulting in smaller norm RKHS norm solutions.

\citet{phuong19understand} prove that in a self-distillation setting, deep linear networks trained with gradient flow converge to the projection of teacher parameters into the data span, effectively recovering teacher parameters when the number of training points is large than the number of parameters.
They derive a bound of the transfer risk that depends on the distribution of the acute angle between teacher parameters and data points.
This is in spirit related to supervision complexity as it measures an “alignment ” between the distillation objective and data.
\citet{ji2020knowledge} extend this results to linearized neural networks, showing that the quantity $\Delta_z^\top \bs{K}^{-1} \Delta_z$, where $\Delta_z$ is the logit change during training, plays a key role in estimating the bound.
The resulting bound is qualitatively different compared to ours, and the $\Delta_z^\top \bs{K}^{-1} \Delta_z$ becomes ill-defined for hard labels.

\paragraph{Supervision complexity.}
The key quantity in our work is supervision complexity $\bs{Y}^\top K^{-1}\bs{Y}$.
\citet{cristianini2001kernel} introduced a related quantity $\bs{Y}^\top K \bs{Y}$ called \emph{kernel-target alignment} and derived a generalization bound with it for expected Parzen window classifiers.
As an easy-to-compute proxy to supervision complexity,~\citet{deshpande2021linearized} use kernel-target alignment for model selection in transfer learning.
\citet{ortiz2021can} demonstrate that when NTK-target alignment is high, learning is faster and generalizes better.
\citet{arora2019fine} prove a generalization bound for overparameterized two-layer neural networks with NTK parameterization trained with gradient flow.
Their bound is approximately $\sqrt{\bs{Y}^\top (K^\infty)^{-1}\bs{Y}} / \sqrt{n}$, where $K^\infty$ is the \emph{expected NTK matrix} at a random initialization.
Our bound of \cref{thm:margin-bound-label-complexity} can be seen as a generalization of this result for all kernel methods, including linearized neural networks of any depth and sufficient width, with the only difference of using the empirical NTK matrix.
\citet{belkin2018understand} warns that bounds based on RKHS complexity of the learned function can fail to explain the good generalization capabilities of kernel methods in presence of label noise.

%% file: sup-complexity/conclusion.tex
We presented a treatment of knowledge distillation through the lens of supervision complexity.
We formalized how the student generalization is controlled by three key quantities:
the teacher's accuracy,
the student's margin with respect to the teacher labels,
and the supervision complexity of the teacher labels under the student's kernel.
This motivated an online distillation procedure that gradually increases the complexity of the targets that the student fits.
In the broader context, our results highlight the important role of data in generalization.

There are several potential directions for future work.
\emph{Adaptive temperature} scaling for online distillation, where the teacher predictions are smoothened so as to ensure low target complexity, is one such direction.
Another avenue is to explore alternative ways to smoothen teacher prediction besides temperature scaling; e.g., can one perform \emph{sample-dependent} scaling?
There is large potential for improving online KD by making more informed choices for the frequency and positions of teacher checkpoints, and controlling how much the student is trained in between teacher updates.
Finally, while we demonstrated that online distillation results in a better alignment with the teacher's NTK matrix, understanding \emph{why} this happens is an open and interesting problem.

%% file: label-noise/sections/proofs.tex
\subsection{Proof of \texorpdfstring{\cref{thm:fano}}{TEXT}}

\begin{proof}
For each example we consider the following Markov chain:
\begin{equation}
Y_i \rightarrow \left[\begin{matrix} \bs{X} \\ \bs{Y} \end{matrix}\right] \rightarrow \left[\begin{matrix} X_i \\ W \end{matrix}\right] \rightarrow \widehat{Y}_i.
\end{equation}
In this setup, Fano's inequality gives a lower bound for the error probability:
\begin{equation}
H(E_i) + P(E_i = 1) \log \left(\lvert \YY \rvert - 1\right) \ge H(Y_i \mid X_i, W),
\label{eq:original-fano}
\end{equation}
which can be written as:
\begin{equation}
P(E_i = 1) \ge \frac{H(Y_i \mid X_i, W) - H(E_i)}{ \log \left(\lvert \YY \rvert - 1\right)}.
\end{equation}
Summing this inequality for $i=1,\ldots,n$ we get
\begin{align}
\sum_{i=1}^n \mathbb{P}(E_i = 1) &\ge \frac{\sum_{i=1}^n\left(H(Y_i \mid X_i, W) - H(E_i)\right)}{ \log \left(\lvert \YY \rvert - 1\right)}\\
&\ge \frac{\sum_{i=1}^n\left(H(Y_i \mid \bs{X}, W) - H(E_i)\right)}{ \log \left(\lvert \YY \rvert - 1\right)}\\
&\ge \frac{H(\bs{Y} \mid \bs{X}, W) - \sum_{i=1}^n H(E_i)}{ \log \left(\lvert \YY \rvert - 1\right)}.
\end{align}
The correctness of the last step follows from the fact that total correlation is always non-negative~\citep{cover}:
\begin{equation}
\sum_{i=1}^n H(Y_i \mid \bs{X}, W) - H(\bs{Y} \mid \bs{X}, W) = \mathrm{TC}(\bs{Y} \mid \bs{X}, W) \ge 0.
\end{equation}
Finally, using the fact that $H(\bs{Y} \mid \bs{X}, W) = H(\bs{Y} \mid \bs{X}) - I(W; \bs{Y} \mid \bs{X})$, we get that the desired result:
\begin{equation}
\mathbb{E}\left[\sum_{i=1}^n E_i\right] \ge \frac{H(\bs{Y} \mid \bs{X}) - I(W ; \bs{Y} \mid \bs{X}) - \sum_{i=1}^n H(E_i)}{ \log \left(\lvert \YY \rvert - 1\right)}.
\label{eq:lower-bound-num-errors}
\end{equation}
\end{proof}

\subsection{Proof of \texorpdfstring{\cref{prop:gradient_capacity}}{TEXT}}

\begin{proof}
Given that $\epsilon_t$ and $\mu_t$ are independent, let us bound the expected L2 norm of $G_t$:
\begin{align}
\mathbb{E}\left[ G_t^T G_t \right] &= \mathbb{E}\left[ (\epsilon_t + \mu_t)^T (\epsilon_t + \mu_t) \right]\\
&=\mathbb{E}\left[\epsilon_t^T \epsilon_t\right] + \mathbb{E}\left[\mu_t^T\mu_t\right]\\
&\le d \sigma_q^2  + L^2.
\end{align}
Among all random variables $V \in \bR^d$ with $\mathbb{E}[V^T V] \le C$, the Gaussian distribution $\mathcal{N}\left(0, \frac{C}{d} I_d\right)$ has the largest entropy, given by $\frac{d}{2}\log\left(\frac{2\pi e C}{d}\right)$.
Therefore,
\begin{align}
H(G_t) \le \frac{d}{2}\log\left(\frac{2\pi e (d \sigma_q^2 + L^2)}{d}\right).
\end{align}
With this we can upper bound the $I(G_t ; \bs{Y} \mid \bs{X}, G_{<t})$ as follows:
\begin{align}
I(G_t ; \bs{Y} \mid \bs{X}, G_{<t}) &= H(G_t \mid \bs{X}, G_{<t}) - H(G_t \mid \bs{X}, \bs{Y}, G_{<t})\\
&= H(G_t \mid \bs{X}, G_{<t}) - H(\epsilon_t)\\
&\le \frac{d}{2}\log\left(\frac{2\pi e (d \sigma_q^2 + L^2)}{d}\right)\numberthis\label{eq:max-ent} - \frac{d}{2}\log\left(2\pi e \sigma_q^2\right)\\
&= \frac{d}{2}\log\left(1 + \frac{L^2}{d\sigma_q^2}\right).
\end{align}
\end{proof}
Note that the proof will work for arbitrary $\epsilon_t$ that has zero mean and independent components, where the L2 norm of each component is bounded by $\sigma_q^2$.
This holds because in such cases $H(\epsilon_t) \le \frac{d}{2} \log(2\pi e \sigma^2_q)$ (as Gaussians have highest entropy for fixed L2 norm) and the transition of \cref{eq:max-ent} remains correct.
Therefore, the same result holds when $\epsilon_t$ is sampled from a product of univariate zero-mean Laplace distributions with scale parameter $\sigma_q/\sqrt{2}$ (which makes the second moment equal to $\sigma_q^2$).

A similar result has been derived by \citet{pensia2018generalization} (lemma 5) to bound $I(W_t ; (X_t, Y_t) \mid W_{t-1})$.

%% file: unique-info/sections/proofs.tex
\subsection{Proof of \texorpdfstring{\cref{prop:unique-info-upper-bound}}{TEXT}}
\begin{proof}
Note that by definition $\forall S=s, i\in[n],\ Q_{W|S=s}=P_{W|S=s} \ll P_{W|S_{-i}=s_{-i}}$.
By the assumption, we also have that $\forall S=s,i\in[n],\ P_{W|S_{-i}=s_{-i}} \ll Q_{W|S=s_{-i}}$.
By transitivity, $\forall S=s, i\in[n],\ P_{W|S=s} \ll  Q_{W|S=s_{-i}}$.
\begin{align*}
     &\KL{P_{W|S}}{Q_{W|S_{-i}}} - \KL{P_{W|S_{-i}}}{Q_{W|S_{-i}}}\\
     &\hspace{2em}=\E_{P_{Z_i}}\E_{P_{W|S}}\sbr{\log \frac{\text{d} P_{W|S}}{\text{d} Q_{W|S_{-i}}}} - \E_{P_{W|S_{-i}}}\sbr{\log \frac{\text{d} P_{W|S_{-i}}}{\text{d} Q_{W|S_{-i}}}}\numberthis\\
     &\hspace{2em}= \E_{P_{Z_i}}\E_{P_{W|S}}\sbr{\log \frac{\text{d} P_{W|S}}{\text{d} Q_{W|S_{-i}}}} - \E_{P_{Z_i}}\E_{P_{W|S}}\sbr{\log \frac{\text{d} P_{W|S_{-i}}}{\text{d} Q_{W|S_{-i}}}}\numberthis\\
     &\hspace{2em}= \E_{P_{Z_i}}\E_{P_{W|S}}\sbr{\log \frac{\text{d} P_{W|S}}{\text{d} Q_{W|S_{-i}}} - \log \frac{\text{d} P_{W|S_{-i}}}{\text{d} Q_{W|S_{-i}}}}\numberthis\\
     &\hspace{2em}= \E_{P_{Z_i}}\E_{P_{W|S}}\sbr{\log \frac{\text{d} P_{W|S}}{\text{d} P_{W|S_{-i}}}}\numberthis\\
     &\hspace{2em}=\KL{P_{W|S}}{P_{W|S_{-i}}}.\numberthis
\end{align*}
As KL-divergence is non-negative, we get that $\KL{P_{W|S}}{Q_{W|S_{-i}}} \ge \KL{P_{W|S}}{P_{W|S_{-i}}}$.
\end{proof}

\subsection{Proof of \texorpdfstring{\cref{prop:equality-of-non-smooth-and-smooth-SIs}}{TEXT}}

\begin{proof}
Assuming $\Lambda(w)$ is approximately constant around $w^*$, the steady-state distributions of \cref{eq:SGD-SDE-plain} is a Gaussian distribution with mean $w^*$ and covariance $\Sigma$ such that:
\begin{equation}
H \Sigma + \Sigma H^T = \frac{\eta}{b}\Lambda(w^*),
\end{equation}
where $H = (\nabla_w f_0(\bs{x}) \nabla_w f_0(\bs{x})^T + \lambda I)$ is the Hessian of the loss function~\citep{mandt2017stochastic}.
This can be verified by checking that the distribution $\mathcal{N}(\cdot; w^*, \Sigma)$ satisfies the Fokker-Planck equation (see \cref{sec:sgd-final-dist}).
Having $Q^\mathrm{SGD}_{W|S=s}$ given by the Gaussian density $\mathcal{N}(w ; w^*, \Sigma)$ and $Q^\mathrm{SGD}_{W|S=s_{-i}}$ by the Gaussian density $\mathcal{N}(w; w^*_{-i}, \Sigma_{-i})$, we have that
\begin{align}
\USI(z_i, Q^\mathrm{SGD}) = \frac{1}{2} \left((w^* - w^*_{-i})^T\Sigma_{-i}^{-1}(w^* - w^*_{-i}) + \mathrm{tr}(\Sigma_{-i}^{-1}\Sigma) + \log |\Sigma_{-i} \Sigma^{-1}| - d\right).
\label{eq:appendix-KL-stability-KL-between-Gaussians}
\end{align}
By the assumption that SGD steady-state covariance stays constant after removing an example, i.e. $\Sigma_{-i} = \Sigma$, equation \cref{eq:appendix-KL-stability-KL-between-Gaussians} simplifies to:
\begin{align*}
\KL{Q^\mathrm{SGD}_{W|S=s}}{Q^\mathrm{SGD}_{W|S=s_{-i}}} &= \frac{1}{2} (w^* - w^*_{-i})^T \Sigma^{-1}(w^* - w^*_{-i}). \numberthis\label{eq:appendix-weights-stability-simplified}
\end{align*}
By the definition of the $Q^\mathrm{ERM}$ algorithm and smooth sample information, this is equal to $\USI_\Sigma(z_i, Q^\mathrm{ERM})$.
\end{proof}

%% file: sample-info/sections/proofs.tex
%%%%%%% MI lemma %%%%%%%%%
\subsection{Proof of \texorpdfstring{\cref{lemma:mutual-info-lemma}}{TEXT}}
The proof of \cref{lemma:mutual-info-lemma} uses the Donsker-Varadhan inequality and a simple result on the moment generating function of a square of a Gaussian random variable.
\begin{fact}[Donsker-Varadhan inequality, Thm. 5.2.1 of \citet{gray2011entropy}]
Let $P$ and $Q$ be two probability measures defined on the same measurable space $(\Omega, \FF)$, such that $P$ is absolutely continuous with respect to $Q$. Then the Donsker-Varadhan dual characterization of Kullback-Leibler divergence states that
\begin{equation}
    \KL{P}{Q} = \sup_{f} \cbr{\int_\Omega f dP  - \log\int_\Omega e^{f}dQ},
\end{equation}
where $f : \Omega \rightarrow \bR$ is a measurable function, such that both integrals above exist.
\label{fact:mi-donsker-varadhan}
\end{fact}

\begin{lemma}
If $X$ is a $\sigma$-subgaussian random variable with zero mean, then
\begin{equation}
\E e^{\lambda X^2} \le 1 + 8\lambda \sigma^2, \quad \forall \lambda \in \left[0, \frac{1}{4\sigma^2}\right).
\label{eq:subgaussian-square}
\end{equation}
\label{lemma:subgaussian-square}
\end{lemma}
\begin{proof}[Proof of \cref{lemma:subgaussian-square}]
As $X$ is $\sigma$-subgaussian and $\E X = 0$, the $k$-th moment of $X$ can be bounded the following way~\citep{vershynin2018high}:
\begin{align}
    \E\abs{X}^k \le (2 \sigma^2)^{k/2} k \Gamma(k/2),\quad \forall k \in \mathbb{N},
    \label{eq:subgaussian-moments}
\end{align}
where $\Gamma(\cdot)$ is the Gamma function.
Continuing,
\begin{align}
\E e^{\lambda X^2} &= \E\sbr{\sum_{k=0}^\infty \frac{(\lambda X^2)^k}{k!}}\\
&=1 + \sum_{k=1}^\infty\E\rbr{\frac{(\sqrt{\lambda} |X|)^{2k}}{k!}}&&\text{(by Fubini's theorem)}\\
&\le 1 + \sum_{k=1}^\infty \rbr{\frac{(2\lambda \sigma^2)^{k} \cdot 2k \cdot \Gamma(k)}{k!}}&&\text{(by \cref{eq:subgaussian-moments})}\\
&= 1 + 2 \sum_{k=1}^\infty (2\lambda\sigma^2)^k.\label{eq:36}
\end{align}
When $\lambda\le1/(4\sigma^2)$, the infinite sum of \cref{eq:36} converges to a value that is at most twice of the first element of the sum. Therefore
\begin{align}
\E e^{\lambda X^2} \le 1 + 8 \lambda \sigma^2, \quad \forall \lambda \in \left[0, \frac{1}{4\sigma^2}\right).
\end{align}
\end{proof}

\begin{proof}[Proof of \cref{lemma:mutual-info-lemma}]
We use the Donsker-Varadhan inequality for $I(\Phi; \Psi)$ and $\lambda g(\phi, \psi)$, where $\lambda \in \bR$ is any constant:
\begin{align}
I(\Phi; \Psi) &= \KL{P_{\Phi, \Psi}}{P_{\Phi} \otimes P_{\Psi}}&&\text{(by definition)}\\
&\ge \E_{P_{\Phi, \Psi}}\sbr{\lambda g(\Phi, \Psi)} - \log\E_{P_{\Phi}\otimes P_{\Psi}}\sbr{e^{\lambda g(\Phi, \Psi)}}&&\text{(by Fact~\ref{fact:mi-donsker-varadhan})}.
\label{eq:donsker-varadhan}
\end{align}
The subgaussianity of $g(\Phi, \Psi)$ under $P_{\Phi}\otimes P_{\Psi}$ implies that
\begin{equation}
\log\E_{P_{\Phi}\otimes P_{\Psi}}\sbr{\exp\cbr{\lambda\rbr{g(\Phi, \Psi) - \E_{P_{\Phi}\otimes P_{\Psi}}\sbr{g(\Phi, \Psi)}}}} \le \frac{\lambda^2 \sigma^2}{2}, \quad \forall \lambda \in \bR.
\end{equation}
Plugging this into \cref{eq:donsker-varadhan}, we get that
\begin{align}
I(\Phi; \Psi) \ge \lambda\rbr{\E_{P_{\Phi, \Psi}}\sbr{g(\Phi, \Psi)} - \E_{P_\Phi\otimes P_\Psi}\sbr{g(\Phi, \Psi)}} - \frac{\lambda^2 \sigma^2}{2}.
\label{eq:donsker-varadhan-second-step}
\end{align}
Picking $\lambda$ to maximize the right-hand side, we get that
\begin{equation}
I(\Phi; \Psi) \ge \frac{1}{2\sigma^2}\rbr{\E_{P_{\Phi, \Psi}}\sbr{g(\Phi, \Psi)} - \E_{P_\Phi\otimes P_\Psi}\sbr{g(\Phi, \Psi)}}^2,
\end{equation}
which proves the first part of the lemma.

To prove the second part of the lemma, we are going to use Donsker-Varadhan inequality again, but for a different function. Let $\lambda \in \left[0, \frac{1}{4\sigma^2}\right)$ and define
\begin{equation}
    \tilde{g}(\phi,\psi) \triangleq \lambda \rbr{g(\phi,\psi) - \E_{\Psi' \sim P_\Psi} g(\phi, \Psi')}^2.
\end{equation}
By assumption $\E_{P_{\Phi, \Psi}}\sbr{\tilde{g}(\Phi,\Psi)}$ exists.
Note that for each fixed $\phi$, the random variable $g(\phi,\Psi) - \E_{\Psi} g(\phi, \Psi)$ has zero mean and is $\sigma$-subgaussian, by the additional assumptions of the second part of the lemma.
As a result, $\E_{P_{\Phi}\otimes P_\Psi} \exp\rbr{\tilde{g}(\Phi,\Psi)} = \E_{P_{\Phi}}\sbr{\E_{P_{\Psi}} \exp\rbr{\tilde{g}(\Phi,\Psi)}}$ also exists (by \cref{lemma:subgaussian-square}).
Therefore, Donsker-Varadhan is applicable for $\tilde{g}$ and gives the following:
\begin{align*}
I(\Phi; \Psi) &\ge \E_{P_{\Phi, \Psi}}\sbr{\tilde{g}(\Phi, \Psi)} - \log\E_{P_\Phi \otimes P_\Psi}\sbr{\exp\cbr{\tilde{g}(\Phi, \Psi)}}\numberthis\\
&= \lambda \E_{P_{\Phi, \Psi}}\sbr{\rbr{g(\Phi,\Psi) - \E_{\Psi'\sim P_\Psi} g(\Phi, \Psi')}^2}\\
&\quad- \log \E_{P_\Phi \otimes P_\Psi}\exp\cbr{\lambda\rbr{g(\Phi,\Psi) - \E_{\Psi'\sim P_{\Psi}} g(\Phi, \Psi')}^2}\numberthis\\
&\ge \lambda \E_{P_{\Phi, \Psi}}\sbr{\rbr{g(\Phi,\Psi) - \E_{\Psi'\sim P_\Psi} g(\Phi, \Psi')}^2} - \log\rbr{1 + 8 \lambda \sigma^2}  && \text{(by \cref{lemma:subgaussian-square})}\numberthis.
\end{align*}
Picking $\lambda \rightarrow 1 / (4\sigma^2)$, we get
\begin{equation}
  I(\Phi; \Psi) \ge \frac{1}{4\sigma^2} \E_{P_{\Phi, \Psi}}\sbr{\rbr{g(\Phi,\Psi) - \E_{\Psi' \sim P_\Psi} g(\Phi, \Psi')}^2} - \log 3,
\end{equation}
which proves the desired inequality.
To prove the last part of the lemma, we just use the Markov's inequality and combine with this last result:
\begin{align}
P_{\Phi,\Psi}\rbr{\abs{g(\Phi, \Psi) - \E_{\Psi'\sim P_\Psi}g(\Phi,\Psi')} \ge \epsilon} &= P_{\Phi,\Psi}\rbr{\rbr{g(\Phi, \Psi) - \E_{\Psi'\sim P_\Psi}g(\Phi,\Psi')}^2 \ge \epsilon^2}\\
&\le \frac{\E_{P_{\Phi, \Psi}}\rbr{g(\Phi, \Psi) - \E_{\Psi' \sim P_\Psi}g(\Phi,\Psi')}^2}{\epsilon^2}\\
&\le \frac{4 \sigma^2 (I(\Phi; \Psi) +\log 3)}{\epsilon^2}.
\end{align}
\end{proof}

%%%%%%% Xu-Raginsky generalization %%%%%%%%%
\subsection{Proof of \texorpdfstring{\cref{thm:generalized-xu-raginsky}}{TEXT}}
Before proceeding to the proof of \cref{thm:generalized-xu-raginsky}, we prove a simple lemma that will be used also in the proofs of \cref{thm:cmi-sharpened-generalized} and \cref{thm:f-cmi-bound}.

\begin{lemma}
Let $X$ and $Y$ be independent random variables. If $g$ is a measurable function such that $g(x,Y)$ is $\sigma$-subgaussian and $\E g(x,Y) = 0$ for all $x \in \XX$, then $g(X,Y)$ is also $\sigma$-subgaussian.
\label{lemma:zero-mean-sub-gaussianity}
\end{lemma}
\begin{proof}[Proof of \cref{lemma:zero-mean-sub-gaussianity}]
As $\E g(X,Y) = 0$, we have that
\begin{align}
\mathbb{E}_{X,Y}\exp\cbr{t \rbr{g(X,Y) - \mathbb{E}_{X,Y}g(X,Y)}} &=  \mathbb{E}_{X,Y}\exp\cbr{t g(X,Y)}\\
&\hspace{-15em}=\E_{X}\sbr{\mathbb{E}_{Y}\exp\cbr{t g(X,Y)}}&&\text{\hspace{-10em}(by independence of $X$ and $Y$)}\\
&\hspace{-15em}\le\E_{X}e^{t^2 \sigma^2}&&\text{\hspace{-10em}(by subgaussianity of $g(x, Y)$)}\\
&\hspace{-15em}=e^{t^2 \sigma^2}.
\end{align}
\end{proof}

\begin{proof}[Proof of \cref{thm:generalized-xu-raginsky}]
Let us fix a value of $U$. Conditioning on $U=u$ keeps the distribution of $W$ and $S$ intact, as $U$ is independent of $W$ and $S$. Let us set $\Phi = W, \Psi = J_u$,  and 
\begin{equation}
g(w, J_u) = \frac{1}{m}\sum_{i=1}^m \rbr{\ell(w, s_{u_i}) - \E_{Z'\sim P_Z} \ell(w, Z')}.
\end{equation}
Note that for each value of $w$, the random variable $g(w, J_u)$ is $\frac{\sigma}{\sqrt{m}}$-subgaussian, as it is a sum of $m$ i.i.d. $\sigma$-subgaussian random variables. Furthermore, $\forall w,\ \E g(w, J_u) = 0$. These two statements together and \cref{lemma:zero-mean-sub-gaussianity} imply that $g(W, J_u)$ is also $\frac{\sigma}{\sqrt{m}}$-subgaussian under $P_W \otimes P_{J_u}$.
Therefore, with these choices of $\Phi, \Psi$, and $g$, \cref{lemma:mutual-info-lemma} gives that
\begin{equation}
\abs{\E_{P_{S,W}}\sbr{\frac{1}{m}\sum_{i=1}^m \ell(W,S_{u_i}) - \E_{Z'\sim P_Z} \ell(W, Z')}} \le \sqrt{\frac{2\sigma^2}{m} I(W; J_u)}.
\end{equation}
Taking expectation over $u$ on both sides, then swapping the order between expectation over $u$ and absolute value (using Jensen's inequality), we get
\begin{equation}
\abs{\E_{P_{S,W,U}}\sbr{\frac{1}{m}\sum_{i=1}^m \ell(W,S_{u_i}) - \E_{Z'\sim P_Z} \ell(W, Z')}} \le \E_{P_U}\sqrt{\frac{2\sigma^2}{m} I^U(W; J_u)}.
\end{equation}
This proves the first part of the theorem as the left-hand side is equal to the absolute value of the expected generalization gap, $\abs{\E_{P_{S,W}}\sbr{R(W) - r_S(W)}}$.

The second part of \cref{lemma:mutual-info-lemma} gives that
\begin{equation}
\E_{P_{S,W,U}}\rbr{\frac{1}{m}\sum_{i=1}^m \ell(W,S_{u_i}) - \E_{Z'\sim P_Z} \ell(W, Z')}^2 \le \frac{4\sigma^2}{m}\rbr{I(W; J_u) + \log 3}.
\label{eq:85}
\end{equation}
When $u = [n]$, ~\cref{eq:85} becomes
\begin{equation}
\E_{P_{S,W}}\rbr{R(W) - r_S(W)}^2 \le \frac{4\sigma^2}{n}\rbr{I(W; S) + \log 3},
\end{equation}
proving the second part of the theorem.
\end{proof}

\begin{remark}
Note that in case of $m=1$ and $u=\mathset{i}$, ~\cref{eq:85} becomes
\begin{equation}
\E_{P_{S,W}}\rbr{\ell(W,Z_i) - \E_{Z'\sim P_Z} \ell(W, Z')}^2 \le 4\sigma^2\rbr{I(W; Z_i) + \log 3}.
\label{eq:86}
\end{equation}
Unfortunately, this result is not useful for bounding $\E_{P_{W,S}}\rbr{R(W) - r_S(W)}^2$, as for large $n$ the $\log 3$ term will likely dominate over $I(W; Z_i)$.
\end{remark}

%%%%%%% Generalized Xu & Raginsky bound, choice of m %%%%%%%%%
\subsection{Proof of \texorpdfstring{\cref{prop:xu-raginsky-generalization-choice-of-m}}{TEXT}}
Before we prove \cref{prop:xu-raginsky-generalization-choice-of-m}, we establish two useful lemmas that will be helpful also in the proofs of \cref{prop:cmi-sharpened-generalized-choice-of-m}, \cref{prop:fcmi-choice-of-m}, \cref{thm:generalized-xu-raginsky-stability}, \cref{thm:cmi-sharpened-generalized-stability} and \cref{thm:f-cmi-bound-stability}.

%%%%%%%% Stability lemma %%%%%%%%%%%
\begin{lemma}
Let $\Psi = (\Psi_1,\ldots,\Psi_n)$ be a collection of $n$ independent random variables and $\Phi$ be another random variable defined on the same probability space. Then
\begin{equation}
    \forall i \in [n],\ I(\Phi; \Psi_i) \le I(\Phi; \Psi_i \mid \Psi_{-i}),
\end{equation}
and
\begin{equation}
    I(\Phi; \Psi) \le \sum_{i=1}^n I(\Phi; \Psi_i \mid \Psi_{-i}).
\end{equation}
\label{lemma:stability-lemma}
\end{lemma}
\begin{proof}[Proof of \cref{lemma:stability-lemma}]
First, for all $i \in [n]$,
\begin{align}
    I(\Phi; \Psi_i \mid \Psi_{-i}) &= I(\Phi; \Psi_i) - I(\Psi_i; \Psi_{-i}) + I(\Psi_i; \Psi_{-i} \mid \Phi)&&\text{(chain rule of MI)}\\
    &= I(\Phi; \Psi_i) + I(\Psi_i; \Psi_{-i} \mid \Phi)&&\text{($\Psi_i \indep \Psi_{-i}$)}\\
    &\ge I(\Phi; \Psi_i)&&\text{(nonnegativity of MI)}.
\end{align}
Second,
\begin{align}
I(\Phi; \Psi) &= \sum_{i=1}^n I(\Phi; \Psi_i \mid \Psi_{<i})\\
&= \sum_{i=1}^n\rbr{ I(\Phi; \Psi_i \mid \Psi_{<i},\Psi_{>i}) + I(\Psi_i; \Psi_{>i} \mid \Psi_{<i}) - I(\Psi_i; \Psi_{>i} \mid \Psi_{<i}, \Phi)}\\
&= \sum_{i=1}^n\rbr{ I(\Phi; \Psi_i \mid \Psi_{<i},\Psi_{>i}) - I(\Psi_i; \Psi_{>i} \mid \Psi_{<i}, \Phi)}\\
&\le \sum_{i=1}^n I(\Phi; \Psi_i \mid \Psi_{<i},\Psi_{>i})\\
&= \sum_{i=1}^n I(\Phi; \Psi_i \mid \Psi_{-i}).
\end{align}
The first two equalities above use the chain rule of mutual information, while the third one uses the independence of $\Psi_1,\ldots,\Psi_n$.
The inequality of the fourth line relies on the nonnegativity of mutual information.
\end{proof}
The quantity $\sum_{i=1}^n I(\Phi; \Psi_i \mid \Psi_{-i})$ is also known as erasure information~\citep{verdu2008information}.

%%%%%%% Mutual information Han's inequality %%%%%%%%%
\begin{lemma}
Let $S=(Z_1,\ldots,Z_n)$ be a collection of $n$ independent random variables and $\Phi$ be an arbitrary random variable defined on the same probability space. Then for any subset $u' \subseteq \{1,2,\ldots,n\}$ of size $m+1$ the following holds:
\begin{equation}
    I(\Phi; J_{u'}) \ge \frac{1}{m}\sum_{k \in u'} I(\Phi; S_{u' \setminus \mathset{k}}).
\end{equation}
\label{lemma:mi-hans-ineq}
\end{lemma}
\begin{proof}[Proof of \cref{lemma:mi-hans-ineq}]
\begin{align}
(m+1) I(\Phi; J_{u'}) &= \sum_{k \in u'} I(\Phi; S_{u'\setminus \mathset{k}}) + \sum_{k \in u'} I(\Phi; Z_k \mid S_{u'\setminus \mathset{k}})&&\text{(chain-rule of MI)}\\
&\ge \sum_{k \in u'} I(\Phi; S_{u'\setminus \mathset{k}}) + I(\Phi; J_{u'})&&\text{(second part of \cref{lemma:stability-lemma})}.
\end{align}
\end{proof}

% Generalized Xu & Raginsky bound, choice of m
\begin{proof}[Proof of \cref{prop:xu-raginsky-generalization-choice-of-m}]
By \cref{lemma:mi-hans-ineq} with $\Phi=W$, for any subset $u'$ of size $m+1$ the following holds:
\begin{equation}
I(W; J_{u'}) \ge \frac{1}{m}\sum_{k \in u'} I(W; S_{u' \setminus \mathset{k}}).
\end{equation}
Therefore,
\begin{align}
\phi\rbr{\frac{1}{m+1} I(W; J_{u'})} &\ge \phi\rbr{\frac{1}{m(m+1)}\sum_{k \in u'} I(W; S_{u' \setminus \mathset{k}})}\\
&\ge \frac{1}{m+1}\sum_{k \in u'}\phi\rbr{\frac{1}{m} I(W; S_{u' \setminus \mathset{k}})}.&&\text{(by Jensen's inequality)}
\end{align}
Taking expectation over $u'$ on both sides, we have that
\begin{align}
\E_{P_{U'}}\phi\rbr{\frac{1}{m+1}I^{U'}(W; J_{u'})} &\ge \E_{P_{U'}}\sbr{\frac{1}{m+1}\sum_{k \in U'}\phi\rbr{\frac{1}{m} I^{U'}(W; S_{U' \setminus \mathset{k}})}}\\
&= \sum_{u} \alpha_u \phi\rbr{\frac{1}{m}I(W; J_u)}.
\end{align}
For each subset $u$ of size $m$, the coefficient $\alpha_u$ is equal to
\begin{equation}
    \alpha_u = \frac{1}{\Choose{n}{m+1}}\cdot\frac{1}{m+1}\cdot(n-m) = \frac{1}{\Choose{n}{m}}.
\end{equation}
Therefore
\begin{equation}
    \sum_{u} \alpha_u \phi\rbr{\frac{1}{m}I(W; J_u)} = \E_{P_{U}}\phi\rbr{\frac{1}{m+1}I^{U}(W; J_{u})}.
\end{equation}
\end{proof}

%%%%%%%% Xu-Raginsky generalized stability %%%%%%%%
\subsection{Proof of \texorpdfstring{\cref{thm:generalized-xu-raginsky-stability}}{TEXT}}
\begin{proof}
Using \cref{lemma:stability-lemma} with $\Phi = W$ and $\Psi = S$, we get that 
\begin{equation}
    I(W; Z_i) \le I(W; Z_i \mid Z_{-i}) \text{\quad and \quad } I(W; S) \le \sum_{i=1}  I(W; Z_i \mid Z_{-i}).
\end{equation}
Plugging these upper bounds into \cref{thm:generalized-xu-raginsky} completes the proof.
\end{proof}

%%%%%%% Generalized CMI %%%%%%%%%
\subsection{Proof of \texorpdfstring{\cref{thm:cmi-sharpened-generalized}}{TEXT}}
\begin{proof}
Let us consider the joint distribution of $W$ and $J$ given $\tilde{Z} = \tilde{z}$ and $U = u$: $P_{W,J|\tilde{Z}=\tilde{z},U=u}$.
Let $\Phi=W$, $\Psi=J_u$, and
\begin{equation}
    g(\phi,\psi) = \frac{1}{m}\sum_{i=1}^m \rbr{\ell(\phi,(\tilde{z}_{u})_{i,\psi_i}) - \ell(\phi,(\tilde{z}_{u})_{i,\bar{\psi}_i})}.
\end{equation}
Note that by our assumption, for any $w \in \WW$ each summand of $g(w, J_u) \in [-1, +1]$, hence is a $1$-subgaussian random variable under $P_{J_u|\tilde{Z}=\tilde{z},U=u}=P_{J_u}$.
Furthermore, each of these summands has zero mean. As the average of $m$ independent and zero-mean $1$-subgaussian variables is $\frac{1}{\sqrt{m}}$-subgaussian, then $g(w, J_u)$ is $\frac{1}{\sqrt{m}}$-subgaussian for each $w\in \WW$.
Additionally, $\forall w\in\WW,\ \E g(w, J_u) = 0$. Therefore, by \cref{lemma:zero-mean-sub-gaussianity}, $g(W, J_u)$ is $\frac{1}{\sqrt{m}}$-subgaussian under $P_{W|\tilde{Z}=\tilde{z},U=u} \otimes P_{J_u | \tilde{Z}=\tilde{z}, U=u}$.
With these choices of $\Phi, \Psi$, and $g(\phi,\psi)$, we use \cref{lemma:mutual-info-lemma}.
First,
\begin{align}
    \E_{P_{W, J_u | \tilde{Z}=\tilde{z}, U=u}}\sbr{g(W, J_u)} &= \E_{P_{W, J_u|\tilde{Z}=\tilde{z}, U=u}}\sbr{\frac{1}{m}\sum_{i=1}^m \rbr{\ell(W,(\tilde{z}_{u})_{i,(J_u)_i}) - \ell(W,(\tilde{z}_{u})_{i,(\bar{J}_u)_i})}}.
\end{align}
Second,
\begin{align}
    \E_{P_{W|\tilde{Z}=\tilde{z},U=u} \otimes P_{J_u | \tilde{Z}=\tilde{z}, U=u}} \sbr{g(W, J_u)} = 0.
\end{align}
Therefore, \cref{lemma:mutual-info-lemma} gives
\begin{equation}
\abs{\E_{P_{W, J_u | \tilde{Z}=\tilde{z}, U=u}}\sbr{\frac{1}{m}\sum_{i=1}^m \rbr{\ell(W,(\tilde{z}_{u})_{i,(J_u)_i}) - \ell(W,(\tilde{z}_{u})_{i,(\bar{J}_u)_i})}}} \le \sqrt{\frac{2}{m}I^{\tilde{Z}=\tilde{z}, U=u}(W; J_U)}.
\end{equation}
Taking expectation over $u$ on both sides and using Jensen's inequality to switch the order of absolute value and expectation of $u$, we get
\begin{equation}
\abs{\E_{P_{U}}\E_{P_{W, J | \tilde{Z}=\tilde{z}, U}}\sbr{\frac{1}{m}\sum_{i=1}^m \rbr{\ell(W,(\tilde{z}_{U})_{i,(J_U)_i}) - \ell(W,(\tilde{z}_{U})_{i,(\bar{J}_U)_i})}}} \le \E_{P_U}\sqrt{\frac{2}{m}I^{\tilde{Z}=\tilde{z}, U}(W; J_U)},
\end{equation}
which reduces to
\begin{equation}
\abs{\E_{P_{W, J | \tilde{Z}=\tilde{z}}}\sbr{\frac{1}{n}\sum_{i=1}^n \rbr{\ell(W,\tilde{z}_{i,J_i}) - \ell(W,\tilde{z}_{i,\bar{J}_i})}}} \le \E_{P_U}\sqrt{\frac{2}{m}I^{\tilde{Z}=\tilde{z}, U}(W; J_U)}.
\label{eq:cmi-generalized-bound-for-fixed-z}
\end{equation}
This can be seen as bounding the expected generalization gap for a fixed $\tilde{z}$.
Taking expectation over $\tilde{z}$ on both sides, and then using Jensen's inequality to switch the order of absolute value and expectation of $\tilde{z}$, we get
\begin{equation}
\abs{\E_{P_{\tilde{Z}, W, J}}\sbr{\frac{1}{n}\sum_{i=1}^n \rbr{\ell(W,\tilde{Z}_{i,J_i}) - \ell(W,\tilde{Z}_{i,\bar{J}_i})}}} \le \E_{\tilde{Z}, U}\sqrt{\frac{2}{m}I^{\tilde{Z}, U}(W; J_U)}.
\end{equation}
Finally, noticing that left-hand side is equal to the absolute value of the expected generalization gap, $\abs{\E_{P_{S,W}}\sbr{R(W) - r_S(W)}}$, completes the proof of the first part of this theorem.

When $u = [n]$, applying the second part of \cref{lemma:mutual-info-lemma} gives
\begin{equation}
\E_{P_{W,J|\tilde{Z}=\tilde{z}}}\rbr{\frac{1}{n}\sum_{i=1}^n \rbr{\ell(W,\tilde{z}_{i,J_i}) - \ell(W,\tilde{z}_{i,\bar{J}_i})}}^2 \le \frac{4}{n} (I^{\tilde{Z}=\tilde{z}}(W; J) + \log 3).
\end{equation}
Taking expectation over $\tilde{z}$, we get
\begin{equation}
\underbrace{\E_{P_{\tilde{Z},W,J}}\rbr{\frac{1}{n}\sum_{i=1}^n \rbr{\ell(W,\tilde{Z}_{i,J_i}) - \ell(W,\tilde{Z}_{i,\bar{J}_i})}}^2}_B \le \frac{4}{n} \E_{P_{\tilde{Z}}}(I^{\tilde{Z}}(W; J) + \log 3).
\end{equation}
Continuing,
\begin{align}
\E_{P_{W,S}}\rbr{R(W) - r_S(W)}^2 &= \E_{P_{\tilde{Z}, W, J}}\rbr{\frac{1}{n}\sum_{i=1}^n \ell(W,\tilde{Z}_{i,J_i}) - \E_{Z'\sim P_Z}\ell(W,Z')}^2\\
&\hspace{-4em}\le 2B + 2\E_{P_{\tilde{Z}, W, J}}\rbr{\frac{1}{n}\sum_{i=1}^n \ell(W,\tilde{Z}_{i,\bar{J}_i}) - \E_{Z'\sim P_Z}\ell(W,Z')}^2\\
&\hspace{-4em}= 2B + 2\E_{P_{\tilde{Z}, W, J}}\rbr{\frac{1}{n}\sum_{i=1}^n \rbr{\ell(W,\tilde{Z}_{i,\bar{J}_i}) - \E_{Z'\sim P_Z}\ell(W,Z')}}^2\\
&\hspace{-4em}= 2B + 2\E_{P_{\tilde{Z}_{\bar{J}}, W}}\rbr{\frac{1}{n}\sum_{i=1}^n \rbr{\ell(W,(\tilde{Z}_{\bar{J}})_i) - \E_{Z'\sim P_Z}\ell(W,Z')}}^2\\
&\hspace{-4em}= 2B + 2\E_{P_W}\E_{P_{\tilde{Z}_{\bar{J}} \mid W}}\rbr{\frac{1}{n}\sum_{i=1}^n \rbr{\ell(W,(\tilde{Z}_{\bar{J}})_i) - \E_{Z'\sim P_Z}\ell(W,Z')}}^2.\label{eq:117}
\end{align}

Note that as $\tilde{Z}_{\bar{J}}$ is independent of $W$, conditioning on $W$ does not change its distribution, implying that its components stay independent of each other.
For each fixed value $W=w$ the inner part of the outer expectation in \cref{eq:117} becomes
\begin{equation}
\E_{P_{\tilde{Z}_{\bar{J}}}}\rbr{\frac{1}{n}\sum_{i=1}^n \rbr{w,(\tilde{Z}_{\bar{J}})_i) - \E_{Z'\sim P_Z}\ell(w,Z')}}^2,
\end{equation}
which is equal to
\begin{equation}
\E_{P_{Z'_1,Z'_2,\ldots,Z'_n}}\rbr{\frac{1}{n}\sum_{i=1}^n \rbr{\ell(w,Z'_i) - \E_{Z'\sim P_Z}\ell(w,Z')}}^2,
\label{eq:120}
\end{equation}
where $Z'_1,\ldots,Z'_n$ are $n$ i.i.d. samples from $P_Z$.
The expression in \cref{eq:120} is simply the variance of the average of $n$ i.i.d $[0,1]$-bounded random variables. Hence, it can be bounded by $1/(4n)$.
Connecting this result with \cref{eq:117}, we get
\begin{align}
\E_{P_{S,W}}\rbr{R(W) - r_S(W)}^2 &\le 2B + \frac{1}{2n}\\
&\le 2\E_{P_{\tilde{Z}}}\sbr{\frac{4}{n} (I^{\tilde{Z}}(W; J) + \log 3)} + \frac{1}{2n}\\
&\le \E_{P_{\tilde{Z}}}\sbr{\frac{8}{n} (I^{\tilde{Z}}(W; J) + 2)}.
\end{align}
\end{proof}

%%%%%%% Generalized CMI: choice of m %%%%%%%%%
\subsection{Proof of \texorpdfstring{\cref{prop:cmi-sharpened-generalized-choice-of-m}}{TEXT}}
The proof follows that of \cref{prop:xu-raginsky-generalization-choice-of-m} but with conditioning on $\tilde{Z}$.

%%%%%%%% CMI sharpened generalized stability %%%%%%%%
\subsection{Proof of \texorpdfstring{\cref{thm:cmi-sharpened-generalized-stability}}{TEXT}}
\begin{proof} For a fixed $\tilde{z}$, using \cref{lemma:stability-lemma} we get that 
\begin{equation}
    I^{\tilde{Z}=\tilde{z}}(W; J_i) \le I^{\tilde{Z}=\tilde{z}}(W; J_i \mid J_{-i}),
\end{equation}
and
\begin{equation}
    I^{\tilde{Z}=\tilde{z}}(W; J) \le \sum_{i=1}^n I^{\tilde{Z}=\tilde{z}}(W; J_i \mid J_{-i}).
\end{equation}
Using these upper bounds in \cref{thm:cmi-sharpened-generalized} proves the theorem.
\end{proof}

%%%%% f-CMI generalization bound proof %%%%%
\subsection{Proof of \texorpdfstring{\cref{thm:f-cmi-bound}}{TEXT}}
\begin{proof}
First, with a slight abuse of notation, we will use $\ell(\bs{\widehat{y}}, \bs{y})$ to denote the average loss between a collection of predictions $\bs{\widehat{y}}$ and collection of labels $\bs{y}$.
Let us consider the joint distribution of $\tilde{F}_u$ and $J_u$ given $\tilde{Z}=\tilde{z}$ and $U=u$.
We are going to use \cref{lemma:mutual-info-lemma} for $P_{\tilde{F}_u,J_u | \tilde{Z}=\tilde{z}, U=u}$ with $\Phi=\tilde{F}_u$, $\Psi=J_u$,
\begin{align}
    g(\phi,\psi) &= \ell(\phi_\psi,(\tilde{y}_u)_{\psi}) - \ell(\phi_{\bar{\psi}},(\tilde{y}_u)_{\bar{\psi}})\\
    &=\frac{1}{m}\rbr{\sum_{i=1}^m \ell(\phi_{i,\psi_i},(\tilde{y}_u)_{i,\psi_i}) - \ell(\phi_{i,\bar{\psi}_i},(\tilde{y}_u)_{i,\bar{\psi}_i})}.
\end{align}
The function $g(\phi,\psi)$ computes the generalization gap measured on pairs of the examples specified by subset $u$, assuming that predictions are given by $\phi$ and the training/test set split is given by $\psi$.
Note that by our assumption, for any $\phi$ each summand of $g(\phi,J_u)$ is a $1$-subgaussian random variable under $P_{\tilde{F}_u,J_u | \tilde{Z}=\tilde{z}, U=u}$.
Furthermore, each of these summands has zero mean. 
As the average of $m$ independent and zero-mean $1$-subgaussian variables is $\frac{1}{\sqrt{m}}$-subgaussian, then $g(\phi, J_u)$ is $\frac{1}{\sqrt{m}}$-subgaussian for each possible $\phi$.
Additionally, $\forall \phi \in \mathcal{K}^{m\times2},\ \E_{J_u} g(\phi, J_u) = 0$.
By \cref{lemma:zero-mean-sub-gaussianity}, $g(\tilde{F}_u, J_u)$ is $\frac{1}{\sqrt{m}}$-subgaussian under $P_{\tilde{F}_u | \tilde{Z}=\tilde{z},U=u} \otimes P_{J_u | \tilde{Z}=\tilde{z},U=u}$.
Hence, these choices of $\Phi, \Psi$, and $g(\phi,\psi)$ satisfy the assumptions of \cref{lemma:mutual-info-lemma}. 
We have that
\begin{align}
    \E_{\tilde{F}_u, J_u| \tilde{Z}=\tilde{z},U=u} g(\tilde{F}_u, J_u) &= \E_{\tilde{F}_u, J_u| \tilde{Z}=\tilde{z},U=u} \sbr{\ell((\tilde{F}_u)_{J_u},(\tilde{y}_u)_{J_u})- \ell((\tilde{F}_u)_{\bar{J}_u},(\tilde{y}_u)_{\bar{J}_u}},
\end{align}
and
\begin{align}
    \E_{P_{\tilde{F}_u | \tilde{Z}=\tilde{z},U=u} \otimes P_{J_u | \tilde{Z}=\tilde{z},U=u}} \sbr{g(\tilde{F}_u, J_u)} = 0.
\end{align}
Therefore, the first part of \cref{lemma:mutual-info-lemma} gives
\begin{equation}
\abs{\E_{\tilde{F}_u, J_u| \tilde{Z}=\tilde{z},U=u}\sbr{\ell((\tilde{F}_u)_{J_u},(\tilde{y}_u)_{J_u})- \ell((\tilde{F}_u)_{\bar{J}_u},(\tilde{y}_u)_{\bar{J}_u}}} \le \sqrt{\frac{2}{m}I^{\tilde{Z}=\tilde{z}}(\tilde{F}_u; J_u)}.
\end{equation}
Taking expectation over $U$ on both sides, and then using Jensen's inequality to swap the order of absolute value and expectation of $U$, we get
\begin{equation}
\abs{\E_{U}\E_{U, \tilde{F}_U, J_U| \tilde{Z}=\tilde{z},U}\sbr{\ell((\tilde{F}_U)_{J_U},(\tilde{y}_U)_{J_U})- \ell((\tilde{F}_U)_{\bar{J}_U},(\tilde{y}_U)_{\bar{J}_U}}} \le \E_{U}\sqrt{\frac{2}{m}I^{\tilde{Z}=\tilde{z},U}(\tilde{F}_U; J_U)}.
\end{equation}
which reduces to
\begin{equation}
\abs{\E_{\tilde{F}, J | \tilde{Z}=\tilde{z}}\sbr{\ell(\tilde{F}_J,\tilde{y}_J)- \ell(\tilde{F}_{\bar{J}},\tilde{y}_{\bar{J}})}} \le \E_{U}\sqrt{\frac{2}{m}I^{\tilde{Z}=\tilde{z},U}(\tilde{F}_U; J_U)}.
\label{eq:fcmi-generalized-bound-for-fixed-z}
\end{equation}
This can be seen as bounding the expected generalization gap for a fixed $\tilde{z}$.
Taking expectation over $\tilde{Z}$ on both sides and using Jensen's inequality to switch the order of absolute value and expectation of $\tilde{Z}$, we get
\begin{equation}
\abs{\E_{\tilde{Z}, J, \tilde{F}}\sbr{\ell(\tilde{F}_J,\tilde{Y}_J)- \ell(\tilde{F}_{\bar{J}},\tilde{Y}_{\bar{J}})}} \le \E_{\tilde{Z},U}\sqrt{\frac{2}{m}I^{\tilde{Z},U}(\tilde{F}_U; J_U)}.
\end{equation}
Noticing that the left-hand side is equal to $\abs{\E_{S,R}\sbr{R(f) - r_S(f)}}$ completes the proof of the first part of the theorem.

When $u = [n]$, applying the second part of \cref{lemma:mutual-info-lemma} gives
\begin{equation}
\E_{\tilde{F}, J | \tilde{Z}=\tilde{z}}\rbr{\ell(\tilde{F}_J,\tilde{y}_J)- \ell(\tilde{F}_{\bar{J}},\tilde{y}_{\bar{J}})}^2 \le \frac{4}{n} \rbr{I^{\tilde{Z}=\tilde{z}}(\tilde{F}; J) + \log 3}.
\end{equation}
Taking expectation over $\tilde{Z}$, we get
\begin{equation}
\underbrace{\E_{\tilde{Z}, J, \tilde{F}}\rbr{\ell(\tilde{F}_J,\tilde{Y}_J)- \ell(\tilde{F}_{\bar{J}},\tilde{Y}_{\bar{J}})}^2}_B \le \E_{\tilde{Z}}\sbr{\frac{4}{n} \rbr{I^{\tilde{Z}}(\tilde{F}; J) + \log 3}}.
\end{equation}
Continuing,
\begin{align}
\E_{S,R}\rbr{R(f) - r_S(f)}^2
&= \E_{\tilde{Z}, J, R}\rbr{\ell(\tilde{F}_J,\tilde{Y}_J) - \E_{Z'\sim P_Z}\ell(f(\tilde{Z}_J,X',R),Y')}^2\\
&\hspace{-10em}\le 2B + 2\E_{\tilde{Z}, J, R}\rbr{\ell(\tilde{F}_{\bar{J}},\tilde{Y}_{\bar{J}}) - \E_{Z'\sim P_Z}\ell(f(\tilde{Z}_J,X',R),Y')}^2\\
&\hspace{-10em}= 2B + 2\E_{\tilde{Z},J,R}\rbr{\frac{1}{n}\sum_{i=1}^n\rbr{\ell(\tilde{F}_{i,\bar{J}_i},\tilde{Y}_{i,\bar{J}_i}) - \E_{Z'\sim P_Z}\ell(f(\tilde{Z}_J,X',R),Y')}}^2\\
&\hspace{-10em}= 2B + 2\E_{\tilde{Z}_J,R,\tilde{Z}_{\bar{J}}}\rbr{\frac{1}{n}\sum_{i=1}^n\rbr{\ell(f(\tilde{Z}_J,\tilde{X}_{\bar{J}},R)_i,(\tilde{Y}_{\bar{J}})_i) - \E_{Z'\sim P_Z}\ell(f(\tilde{Z}_J,X',R),Y')}}^2\\
&\hspace{-10em}= 2B + 2\E_{\tilde{Z}_J,R}\E_{\tilde{Z}_{\bar{J}} \mid \tilde{Z}_J,R}\rbr{\frac{1}{n}\sum_{i=1}^n\rbr{\ell(f(\tilde{Z}_J,\tilde{X}_{\bar{J}},R)_i,(\tilde{Y}_{\bar{J}})_i) - \E_{Z'\sim P_Z}\ell(f(\tilde{Z}_J,X',R),Y')}}^2.
\label{eq:136}
\end{align}

Note that as $\tilde{Z}_{\bar{J}}$ is independent of $(\tilde{Z}_J, R)$, conditioning on $(\tilde{Z}_J, R)$ does not change its distribution, implying that its components stay independent of each other.
For each fixed values $\tilde{Z}_J=s$ and $R=r$, the inner part of the expectation in \cref{eq:136} becomes
\begin{equation}
\E_{\tilde{Z}_{\bar{J}}}\rbr{\frac{1}{n}\sum_{i=1}^n\rbr{\ell(f(s,\tilde{X}_{\bar{J}},r)_i,(\tilde{Y}_{\bar{J}})_i) - \E_{Z'\sim P_Z}\ell(f(s,X',r),Y')}}^2,
\end{equation}
which is equal to
\begin{equation}
\E_{Z'_1,Z'_2,\ldots,Z'_n}\rbr{\frac{1}{n}\sum_{i=1}^n\rbr{\ell(f(s,X'_i,Y'_i) - \E_{Z'_i}\ell(f(s,X'_i,r),Y'_i)}}^2,
\label{eq:138}
\end{equation}
where $Z'_1,\ldots,Z'_n$ are $n$ i.i.d. samples from $P_Z$.
The expression in \cref{eq:138} is simply the variance of the average of $n$ i.i.d $[0,1]$-bounded random variables. Hence, it can be bounded by $1/(4n)$.
Connecting this result with \cref{eq:136}, we get
\begin{align}
\E_{S,R}\rbr{R(f) - r_S(f)}^2 &\le 2B + \frac{1}{2n}\\
&\le 2\E_{\tilde{Z}}\sbr{\frac{4}{n} \rbr{I^{\tilde{Z}}(\tilde{F}; J) + \log 3}} + \frac{1}{2n}\\
&\le \E_{\tilde{Z}}\sbr{\frac{8}{n} \rbr{I^{\tilde{Z}}(\tilde{F}; J) + 2}}.
\end{align}
\end{proof}

%%%%%%% f-CMI: choice of m %%%%%%%%%
\subsection{Proof of \texorpdfstring{\cref{prop:fcmi-choice-of-m}}{TEXT}}
\begin{proof}
The proof closely follows that of \cref{prop:cmi-sharpened-generalized-choice-of-m}. The only important difference is that $\tilde{F}_u$ depends on $u$, while $W$ does not.

Let us fix a value of $\tilde{Z}$ and consider the conditional joint distribution $P_{\tilde{F}, J | \tilde{Z}}$.
If we fix a subset $u'$ of size $m+1$, set $\Phi=\tilde{F}_{u'}$, and use \cref{lemma:mi-hans-ineq} under $P_{\tilde{F}, J | \tilde{Z}}$, we get
\begin{align}
I^{\tilde{Z}}(\tilde{F}_{u'}; J_{u'}) &\ge \frac{1}{m}\sum_{k \in u'} I^{\tilde{Z}}\rbr{\tilde{F}_{u'}; J_{u' \setminus \mathset{k}}}\\
&\ge \frac{1}{m}\sum_{k \in u'} I^{\tilde{Z}}\rbr{\tilde{F}_{u'\setminus\mathset{k}}; J_{u' \setminus \mathset{k}}},
\end{align}
as removing predictions on pair $k$ can not increase the mutual information.
Therefore,
\begin{align}
\phi\rbr{\frac{1}{m+1} I^{\tilde{Z}}(\tilde{F}_{u'}; J_{u'})} &\ge \phi\rbr{\frac{1}{m(m+1)}\sum_{k \in u'} I^{\tilde{Z}}(\tilde{F}_{u'\setminus\mathset{k}}; J_{u' \setminus \mathset{k}})}\\
&\hspace{-6em}\ge \frac{1}{m+1}\sum_{k \in u'}\phi\rbr{\frac{1}{m} I^{\tilde{Z}}(\tilde{F}_{u'\setminus\mathset{k}}; J_{u' \setminus \mathset{k}})}.&&\text{\hspace{-8em}(by Jensen's inequality)}
\end{align}
Taking expectation over $U'$ on both sides, we have
\begin{align}
\E_{U'}\phi\rbr{\frac{1}{m+1} I^{\tilde{Z}, U'}\rbr{\tilde{F}_{U'}; J_{U'}}} &\ge \E_{U'}\sbr{\frac{1}{m+1}\sum_{k \in U'}\phi\rbr{\frac{1}{m} I^{\tilde{Z},U'}(\tilde{F}_{U'\setminus\mathset{k}}; J_{U' \setminus \mathset{k}})}}\\
&= \sum_{u} \alpha_u \phi\rbr{\frac{1}{m}I^{\tilde{Z}}(\tilde{F}_{u}; J_u)}.
\end{align}
For each subset $u$ of size $m$, the coefficient $\alpha_u$ is equal to
\begin{equation}
    \alpha_u = \frac{1}{\Choose{n}{m+1}}\cdot\frac{1}{m+1}\cdot(n-m) = \frac{1}{\Choose{n}{m}}.
\end{equation}
Therefore
\begin{equation}
    \sum_{u} \alpha_u \phi\rbr{\frac{1}{m}I^{\tilde{Z}}(\tilde{F}_{u}; J_u)} = \E_{U}\phi\rbr{\frac{1}{m} I^{\tilde{Z},U}(\tilde{F}_{U}; J_U)}.
\end{equation}
\end{proof}

%%%%%%%% f-CMI statbility %%%%%%%%
\subsection{Proof of \texorpdfstring{\cref{thm:f-cmi-bound-stability}}{TEXT}}
\begin{proof}
Let us fix $\tilde{Z}=\tilde{z}$.
Setting $\Phi=\tilde{F}_i$, $\Psi = J$, and using the first part of \cref{lemma:stability-lemma} under $P_{\tilde{F}, J | \tilde{Z}=\tilde{z}}$, we get that
\begin{equation}
    I^{\tilde{Z}=\tilde{z}}(\tilde{F}_i; J_i) \le I^{\tilde{Z}=\tilde{z}}(\tilde{F}_i; J_i \mid J_{-i}).
\end{equation}
Next, setting $\Phi=\tilde{F}$, $\Psi=J$, and using the second part of \cref{lemma:stability-lemma} under $P_{\tilde{F}, J | \tilde{Z}=\tilde{z}}$, we get that
\begin{align}
    I^{\tilde{Z}=\tilde{z}}(\tilde{F}; J) \le \sum_{i=1}^n I^{\tilde{Z}=\tilde{z}}(\tilde{F}; J_i \mid J_{-i}).
\end{align}
Using these upper bounds in \cref{thm:f-cmi-bound} proves this theorem.
\end{proof}

%%%%% VC dimension %%%%%%%
\subsection{Proof of \texorpdfstring{\cref{thm:vc}}{TEXT}}
\begin{proof}
Let $k$ denote the number of distinct values $\tilde{F}$ can take by varying $J$ and $R$ for a fixed $\tilde{Z}=\tilde{z}$. Clearly, $k$ is not more than the growth function of $\HH$ evaluated at $2n$.
Applying the Sauer-Shelah lemma~\citep{sauer1972density,shelah1972combinatorial}, we get that
\begin{equation}
    k \le \sum_{i=0}^d\rbr{\begin{matrix}2n\\i\end{matrix}}.
\end{equation}
The Sauer-Shelah lemma also states that if $2n > d+1$ then
\begin{equation}
\sum_{i=0}^d\rbr{\begin{matrix}2n\\i\end{matrix}} \le \rbr{\frac{2en}{d}}^d.
\end{equation}
If $2n \le d+1$, one can upper bound $k$ by $2^{2n} \le 2^{d+1}$.
Therefore
\begin{equation}
    k \le \max\mathset{2^{d+1}, \rbr{\frac{2en}{d}}^d}.
\end{equation}
Finally, as a $\tilde{F}$ is a discrete variable with $k$ states,
\begin{align}
    \ofcmi(f,\tilde{z}) \le H(\tilde{F} \mid \tilde{Z}=\tilde{z}) \le \log(k).
\end{align}
\end{proof}

%%%%%%% Stability to KL  %%%%%%%% 
\subsection{Proof of \texorpdfstring{\cref{prop:stability-to-kl}}{TEXT}}
\begin{proof}
The proof below uses the independence of $J_1,\ldots,J_n$ and the convexity of KL divergence, once for the first and once for the second argument.
\begin{align*}
I(f(\tilde{z}_J,\bs{x'},R); J_i \mid J_{-i}) &= \KL{f(\tilde{z}_J,\bs{x'},R) | J}{f(\tilde{z}_J,\bs{x'},R) | J_{-i}}\numberthis\\
&\hspace{-8em}=\frac{1}{2}\KL{f(\tilde{z}_{J^{i\leftarrow 0}},\bs{x'},R) | J_{-i}}{f(\tilde{Z}_J,\bs{x'},R) | J_{-i}}\\
&\hspace{-8em}\quad+\frac{1}{2}\KL{f(\tilde{z}_{J^{i\leftarrow 1}},\bs{x'},R) | J_{-i}}{f(\tilde{Z}_J,\bs{x'},R) | J_{-i}}\numberthis\\
&\hspace{-8em}=\frac{1}{2}\KL{f(\tilde{z}_{J^{i\leftarrow 0}},\bs{x'},R) | J_{-i}}{\frac{1}{2}\rbr{f(\tilde{z}_{J^{i\leftarrow 0}},\bs{x'},R) + f(\tilde{z}_{J^{i \leftarrow 1}},\bs{x'},R)} | J_{-i}}\\
&\hspace{-8em}\quad+\frac{1}{2}\KL{f(\tilde{z}_{J^{i\leftarrow 1}},\bs{x'},R) | J_{-i}}{\frac{1}{2}\rbr{f(\tilde{z}_{J^{i\leftarrow 0}},\bs{x'},R) + f(\tilde{z}_{J^{i \leftarrow 1}},\bs{x'},R)} | J_{-i}}\numberthis\\
&\hspace{-8em}\le\frac{1}{4}\KL{f(\tilde{z}_{J^{i\leftarrow 0}},\bs{x'},R) | J_{-i}}{f(\tilde{z}_{J^{i\leftarrow 0}},\bs{x'},R)| J_{-i}}\\
&\hspace{-8em}\quad+ \frac{1}{4}\KL{f(\tilde{z}_{J^{i\leftarrow 0}},\bs{x'},R) | J_{-i}}{f(\tilde{z}_{J^{i\leftarrow 1}},\bs{x'},R)| J_{-i}}\\
&\hspace{-8em}\quad+\frac{1}{4}\KL{f(\tilde{z}_{J^{i\leftarrow 1}},\bs{x'},R) | J_{-i}}{f(\tilde{z}_{J^{i\leftarrow 0}},\bs{x'},R)| J_{-i}}\\
&\hspace{-8em}\quad+\frac{1}{4}\KL{f(\tilde{z}_{J^{i\leftarrow 1}},\bs{x'},R) | J_{-i}}{f(\tilde{z}_{J^{i\leftarrow 1}},\bs{x'},R)| J_{-i}}\numberthis\\
&\hspace{-8em}=\frac{1}{4}\KL{f(\tilde{z}_{J^{i\leftarrow 0}},\bs{x'},R) | J_{-i}}{f(\tilde{z}_{J^{i\leftarrow 1}},\bs{x'},R)| J_{-i}}\\
&\hspace{-8em}\quad+\frac{1}{4}\KL{f(\tilde{z}_{J^{i\leftarrow 1}},\bs{x'},R) | J_{-i}}{f(\tilde{z}_{J^{i\leftarrow 0}},\bs{x'},R)| J_{-i}}.\numberthis
\end{align*}
\end{proof}

\subsection{Proof of \texorpdfstring{\cref{thm:f-cmi-deterministic-stability}}{TEXT}}
\begin{proof}
Given a deterministic algorithm $f$, we consider the algorithm that adds Gaussian noise to the predictions of $f$:
\begin{equation}
f_\sigma(s,x,R) = f(s,x) + \xi(s,x),
\end{equation}
where $\xi(s,x) \sim \mathcal{N}(0,\sigma^2 I_k)$.
The function $f_\sigma$ is constructed in a way that the noise terms are independent for each possible combination of $s$ and $x$.
This can be achieved by viewing $R$ as an infinite collection of independent Gaussian variables, one of which is selected for each possible combination of $s$ and $x$.

Let us consider the random subsample setting and let $Z' \sim P_Z$ be a test example independent of $\tilde{Z}, J$ and randomness $R$.
First we relate the generalization gap of $f_\sigma$ to that of $f$:
\begin{align*}
&\abs{\E_{S}\sbr{R(f) - r_S(f)}}\\
&\quad=
\abs{\E_{\tilde{Z},J,R,Z}\sbr{\ell(f_\sigma(\tilde{Z}_J,X',R), Y')  - \frac{1}{n}\sum_{i=1}^n\ell(f_\sigma(\tilde{Z}_J,\tilde{X}_{i,J_i},R), \tilde{Y}_{i,J_i})}}\numberthis\\
&\quad=
\abs{\E_{\tilde{Z},J,R,Z}\sbr{\ell(f(\tilde{Z}_J,X') + \xi'(\tilde{Z}_J, X'), Y')  - \frac{1}{n}\sum_{i=1}^n\ell(f(\tilde{Z}_J,\tilde{X}_{i,J_i}) + \xi(\tilde{Z}_J, \tilde{X}_{i, J_i}), \tilde{Y}_{i,J_i})}}\numberthis\\
&\quad=
\abs{\E_{\tilde{Z},J,R,Z}\sbr{\ell(f(\tilde{Z}_J,X'), Y') + \Delta'  - \frac{1}{n}\sum_{i=1}^n\rbr{\ell(f(\tilde{Z}_J,\tilde{X}_{i,J_i}), \tilde{Y}_{i,J_i}) + \Delta_i}}},\numberthis
\label{eq:171}
\end{align*}
where
\begin{align*}
    \Delta' &= \ell(f(\tilde{Z}_J,X') + \underbrace{\xi(\tilde{Z}_J, X'), Y')}_{\triangleq \xi'} - \ell(f(\tilde{Z}_J,X'), Y'),\\
    \Delta_i &= \ell(f(\tilde{Z}_J,\tilde{X}_{i,J_i}) + \underbrace{\xi(\tilde{Z}_J, \tilde{X}_{i, J_i})}_{\triangleq \xi_i}, \tilde{Y}_{i,J_i}) - \ell(f(\tilde{Z}_J,\tilde{X}_{i,J_i}), \tilde{Y}_{i,J_i}).
\end{align*}
As $\ell(\widehat{y},y)$ is $\gamma$-Lipschitz in its first argument, $\abs{\Delta'} \le \gamma \nbr{\xi'}$ and $\abs{\Delta_i} \le \gamma \nbr{\xi_i}$.
Connecting this to \cref{eq:171} we get
\begin{align}
\abs{\E_{S,R}\sbr{R(f_\sigma) - r_S(f_\sigma)}} &\ge \abs{\E_{S}\sbr{R(f) - r_S(f)}} - \gamma \E \nbr{\xi'} - \frac{\gamma}{n}\sum_{i=1}^n\E\nbr{\xi_i}\\
&=\abs{\E_{S}\sbr{R(f) - r_S(f)}} - 2 \sqrt{d} \gamma \sigma.\label{eq:174}
\end{align}
Similarly, we relate the expected squared generalization gap of $f_\sigma$ to that of $f$:
\begin{align*}
&\E_{S,R}\rbr{R(f_\sigma) - r_S(f_\sigma)}^2\\
&\quad=\E_{\tilde{Z},J,R}\rbr{\E_{Z'\sim P_Z}\sbr{\ell(f(\tilde{Z}_J,X'), Y') + \Delta'}  - \frac{1}{n}\sum_{i=1}^n\rbr{\ell(f(\tilde{Z}_J,\tilde{X}_{i,J_i}), \tilde{Y}_{i,J_i}) + \Delta_i}}^2\numberthis\\
&\quad= \E_{S}\rbr{R(f) - r_S(f)}^2 + \E_{\tilde{Z},J, R}\rbr{\E_{Z'\sim P_Z}[\Delta'] - \frac{1}{n}\sum_{i=1}^n \Delta_i}^2\\
&\quad\quad+2\E_{\tilde{Z},J,R}\sbr{\rbr{R(f) - r_S(f)}\rbr{\E_{Z'\sim P_Z}[\Delta'] - \frac{1}{n}\sum_{i=1}^n \Delta_i}}\numberthis\\
&\quad\ge \E_{S}\rbr{R(f) - r_S(f)}^2\\
&\quad\quad - 2\E_{\tilde{Z},J,R}\sbr{\abs{R(f) - r_S(f)}\abs{\E_{Z'\sim P_Z}[\Delta'] - \frac{1}{n}\sum_{i=1}^n \Delta_i}}\numberthis\\
&\quad= \E_{S}\rbr{R(f) - r_S(f)}^2\\
&\quad\quad - 2\E_{\tilde{Z},J}\sbr{\abs{R(f) - r_S(f)}\E_R\sbr{\abs{\E_{Z'\sim P_Z}[\Delta'] - \frac{1}{n}\sum_{i=1}^n \Delta_i}}}\numberthis\label{eq:182}.
\end{align*}
As
\begin{align}
\E_R\sbr{\abs{\E_{Z'\sim P_Z}[\Delta'] - \frac{1}{n}\sum_{i=1}^n \Delta_i}} &\le \E_{R} \sbr{\E_{Z'\sim P_Z}\abs{\Delta'}} + \frac{1}{n}\sum_{i=1}^n\E_{R}\abs{\Delta_i}\\
&\le \E_{R} \sbr{\E_{Z'\sim P_Z}\sbr{\gamma \nbr{\xi'}}} + \frac{1}{n}\sum_{i=1}^n\E_{R}\sbr{\gamma \nbr{\xi}}\\
&= 2\gamma \sqrt{d} \sigma,
\end{align}
we can write \cref{eq:182} as 
\begin{align*}
\E_{S,R}\rbr{R(f_\sigma) - r_S(f_\sigma)}^2 &\ge \E_{S}\rbr{R(f) - r_S(f)}^2 - 4\gamma \sqrt{d} \sigma \E_{\tilde{Z},J}\sbr{\abs{R(f) - r_S(f)}}\numberthis\\
&\quad\ge \E_{S}\rbr{R(f) - r_S(f)}^2 - 4\gamma \sqrt{d} \sigma \sqrt{\E_{S}\rbr{R(f) - r_S(f)}^2},\numberthis\label{eq:188}
\end{align*}
where the second line follows from Jensen's inequality ($(\E \abs{X})^2 \le \E X^2$).
Summarizing, \cref{eq:174} and \cref{eq:188} relate expected generalization gap and expected squared generalization gap of $f_\sigma$ to those of $f$.

\paragraph{Bounding expected generalization gap of $f$.}
\begin{align*}
&\abs{\E_{S}\sbr{R(f) - r_S(f)}} \\
&\quad\le \abs{\E_{S,R}\sbr{R(f_\sigma) - r_S(f_\sigma)}} + 2 \sqrt{d} \gamma \sigma &&\text{\hspace{-10em}(by \cref{eq:174})}\numberthis\\
&\quad\le \frac{1}{n}\sum_{i=1}^n\E_{\tilde{Z}} \sqrt{2 I^{\tilde{Z}}(f_\sigma(\tilde{Z}_J,\tilde{X}_i,R); J_i \mid J_{-i})} + 2 \sqrt{d} \gamma \sigma &&\text{\hspace{-10em}(by \cref{thm:f-cmi-bound-stability})}\numberthis\\
&\quad\le\frac{1}{n}\sum_{i=1}^n\E_{\tilde{Z}} \sqrt{\begin{matrix}\frac{1}{2}\KLdis{\tilde{Z}}{f_\sigma(\tilde{Z}_{J^{i\leftarrow 1}},\tilde{X}_i,R) | J_{-i}}{f_\sigma(\tilde{Z}_{J^{i\leftarrow 0}},\tilde{X}_i,R) | J_{-i}}\\ \quad+ \frac{1}{2}\KLdis{\tilde{Z}}{f_\sigma(\tilde{Z}_{J^{i\leftarrow 0}},\tilde{X}_i,R) | J_{-i}}{f_\sigma(\tilde{Z}_{J^{i\leftarrow 1}},\tilde{X}_i,R)| {J_{-i}}}\end{matrix}} + 2 \sqrt{d} \gamma \sigma\numberthis\\
&\quad=\frac{1}{n}\sum_{i=1}^n\E_{\tilde{Z}} \sqrt{\frac{1}{2 \sigma^2} \E_{J_{-i}}\nbr{f(\tilde{Z}_{J^{i\leftarrow 0}},\tilde{X}_i) - f(\tilde{Z}_{J^{i\leftarrow 1}},\tilde{X}_i)}^2_2 } + 2 \sqrt{d} \gamma \sigma\numberthis\\
&\quad\le\frac{1}{n}\sum_{i=1}^n\sqrt{\frac{1}{2 \sigma^2} \E_{\tilde{Z},J_{-i}}\nbr{f(\tilde{Z}_{J^{i\leftarrow 0}},\tilde{X}_i) - f(\tilde{Z}_{J^{i\leftarrow 1}},\tilde{X}_i)}^2_2 } + 2 \sqrt{d} \gamma \sigma\numberthis\\
&\quad\le\sqrt{\frac{\beta^2}{\sigma^2}} + 2 \sqrt{d} \gamma \sigma &&\text{\hspace{-15em}(by $\beta$ self-stability of $f$)}\numberthis.
\end{align*}
Picking $\sigma^2 = \frac{\beta}{2 \sqrt{d} \gamma}$, we get
\begin{equation}
\abs{\E_{S}\sbr{R(f) - r_S(f)}} \le 2^{\frac{3}{2}} d^{\frac{1}{4}}\sqrt{\gamma  \beta}.
\end{equation}

\paragraph{Bounding expected squared generalization gap of $f$.}
For brevity below, let $G\triangleq \E_{S}\rbr{R(f) - r_S(f)}^2$.
Starting with \cref{eq:188}, we get
\begin{align}
G &\le \E_{S,R}\rbr{R(f_\sigma) - r_S(f_\sigma)}^2 + 4\gamma \sqrt{d} \sigma \sqrt{G}\\
&\le \frac{8}{n} \rbr{\E_{\tilde{Z}}\sbr{\sum_{i=1}^n I^{\tilde{Z}}(f_\sigma(\tilde{Z}_J, \tilde{X}, R); J_i \mid J_{-i})} + 2} + 4\gamma \sqrt{d} \sigma \sqrt{G}&&\text{\hspace{-9em}(by \cref{thm:f-cmi-bound-stability})}\\
&\le \frac{16}{n} + \frac{8}{n}\sum_{i=1}^n\rbr{\begin{matrix}\frac{1}{4}\KLdis{\tilde{Z}}{f_\sigma(\tilde{Z}_{J^{i\leftarrow 1}},\tilde{X},R) | J_{-i}}{f_\sigma(\tilde{Z}_{J^{i\leftarrow 0}},\tilde{X},R) | J_{-i}}\\ \quad+ \frac{1}{4}\KLdis{\tilde{Z}}{f_\sigma(\tilde{Z}_{J^{i\leftarrow 0}},\tilde{X},R) | J_{-i}}{f_\sigma(\tilde{Z}_{J^{i\leftarrow 1}},\tilde{X},R)| {J_{-i}}}\end{matrix}} + 4\gamma \sqrt{d} \sigma \sqrt{G}\\
&=\frac{16}{n} + \frac{8}{n}\sum_{i=1}^n E_{\tilde{Z},J}\sbr{\frac{1}{4\sigma^2} \nbr{f(\tilde{Z}_{J^{i\leftarrow 0}},\tilde{X}) - f(\tilde{Z}_{J^{i\leftarrow 1}},\tilde{X})}^2_2} + 4\gamma \sqrt{d} \sigma \sqrt{G}\\
&\le\frac{16}{n} + \frac{2}{\sigma^2}\rbr{2 \beta^2 + n \beta_1^2 + n \beta_2^2} + 4\gamma \sqrt{d} \sigma \sqrt{G}.
\end{align}
The optimal $\sigma$ is given by
\begin{equation}
\sigma = \rbr{\frac{2\beta^2 + n\beta_1^2 + n \beta_2^2}{\gamma \sqrt{G} \sqrt{d}}}^{\frac{1}{3}},
\end{equation}
and gives
\begin{equation}
    G \le \frac{16}{n} + 6 d^{\frac{1}{3}} \gamma^{\frac{2}{3}} \rbr{2\beta^2 + n\beta_1^2 + n \beta_2^2}^{\frac{1}{3}} G^{\frac{1}{3}}.
\end{equation}
We discuss 2 cases.
\begin{itemize} 
\item[(i)] $\frac{16}{n} \ge 6 d^{\frac{1}{3}} \gamma^{\frac{2}{3}} \rbr{2\beta^2 + n\beta_1^2 + n \beta_2^2}^{\frac{1}{3}} G^{\frac{1}{3}}$. In this case $G \le \frac{32}{n}$.
\item[(ii)] $\frac{16}{n} < 6 d^{\frac{1}{3}} \gamma^{\frac{2}{3}} \rbr{2\beta^2 + n\beta_1^2 + n \beta_2^2}^{\frac{1}{3}} G^{\frac{1}{3}}$. In this case, we have
\begin{align}
    G \le 12 d^{\frac{1}{3}} \gamma^{\frac{2}{3}} \rbr{2\beta^2 + n\beta_1^2 + n \beta_2^2}^{\frac{1}{3}} G^{\frac{1}{3}},
\end{align}
which simplifies to
\begin{align}
    G \le 12^{\frac{3}{2}}\sqrt{d}\gamma\sqrt{2\beta^2 + n\beta_1^2 + n \beta_2^2}.
\end{align}
\end{itemize}
Combining these cases we can write that
\begin{align}
    G &\le \max\mathset{\frac{32}{n},12^{\frac{3}{2}}\sqrt{d}\gamma\sqrt{2\beta^2 + n\beta_1^2 + n \beta_2^2}}\\
    &\le \frac{32}{n} + 12^{\frac{3}{2}}\sqrt{d}\gamma\sqrt{2\beta^2 + n\beta_1^2 + n \beta_2^2}.
\end{align}
\end{proof}

\begin{remark}
The bounds of this theorem work even when $\YY = \mathcal{K} = [a,b]^k$ instead of $\bR^k$.
To see this, we first clip the noisy predictions to be in $[a,b]^k$:
\begin{equation}
f^c_\sigma(z,x)_i \triangleq \mathrm{clip}(f_\sigma(z,x), a, b)_i, \quad \forall i \in [k].
\end{equation}
Inequalities \cref{eq:174} and \cref{eq:188} that relate the expected generalization gap and expected squared generalization gap of $f_\sigma$ to those of $f$ stay true when replacing $f_\sigma$ with $f_\sigma^c$.
Furthermore, by data processing inequality, mutual informations measured with $f^c_\sigma$ can always be upper bounded by the corresponding mutual informations informations measures with $f_\sigma$.
Therefore, generalization bounds that hold for $f_\sigma$ will also for $f_\sigma^c$, allowing us to follow the exact same proofs above.
\end{remark}

\begin{remark}
In the construction of $f_\sigma$ we used Gaussian noise with zero mean and $\sigma^2 I$ covariance matrix.
A natural question arises whether a different type of noise would give better bounds.
Inequalities \cref{eq:174} and \cref{eq:188} only use the facts that noise components are independent, have zero-mean and $\sigma^2$ variance.
Therefore, if we restrict ourselves to noise distributions with independent components, each of which has zero mean and $\sigma^2$ variance, then the best bounds will be produced by noise distributions that result in the smallest KL divergence of form $\KLdis{\tilde{Z}}{f_\sigma(\tilde{Z}_{J^{i\leftarrow 1}},x',R) | J_{-i}}{f_\sigma(\tilde{Z}_{J^{i\leftarrow 0}},x',R) | J_{-i}}$.
An informal argument below hints that the Gaussian distribution might be the optimal choice of the noise distribution when for fixed $\tilde{Z}=\tilde{z}$, $f_\sigma(\tilde{z}_{J^{i\leftarrow 1}},x',R)$ and $f_\sigma(\tilde{z}_{J^{i\leftarrow 0}},x',R)$ are close to each other.

Let us fix $\sigma^2$ and consider two means $\mu_1 < \mu_2 \in \bR$. Let $\FF=\mathset{p(x; \mu) \mid \mu \in \bR}$ be a family of probability distributions with one mean parameter $\mu$, such that every distribution of it has variance $\sigma^2$ and KL divergences between members of $\FF$ exist. 
Let $X_1 \sim p(x, \mu_1)$ and $X_2 \sim p(x,\mu_2)$.
We are interested in finding such a family $\FF$ that $\KL{X}{Y}$ is minimized. For small $\mu_2 - \mu_1$, we know that
\begin{equation}
    \KL{P_X}{P_Y} \approx \frac{1}{2}(\mu_2 - \mu_1) \mathcal{I}(\mu_1) (\mu_2 - \mu_1),
    \label{eq:kl-fisher-approx}
\end{equation}
where $\mathcal{I}(\mu)$ is the Fisher information of $p(x; \mu)$.
Furthermore, let $\widehat{\mu_1} \triangleq X$. As $\E \widehat{\mu_1} = \mu_1$ and $\\var[\widehat{\mu_1}] = \sigma^2$, the Cramer-Rao bound gives
\begin{equation}
    \sigma^2 = \\var[\widehat{\mu_1}] \ge \frac{1}{\mathcal{I}(\mu_1)}.
\end{equation}
This gives us the following lower bound on the KL divergence between $X$ and $Y$:
\begin{equation}
    \KL{P_X}{P_Y} \gtrapprox \frac{1}{2\sigma^2}(\mu_2 - \mu_1)^2,
\end{equation}
which is matched by the Gaussian distribution.
\end{remark}

%% file: limitations/proofs.tex
\subsection{Proof of \texorpdfstring{\cref{thm:g2-weak-sample-wise-bound}}{TEXT}}
\begin{proof}
Let $\bar{J} \triangleq (1-J_1,\ldots,1 - J_n)$ be the negation of $J$. We have that
\begin{align}
\E_{P_{W,S}}\sbr{\rbr{R(W) - r_S(W)}^2} &= \E_{P_{W,S}}\sbr{\rbr{\frac{1}{n}\sum_{i=1}^n \ell(W,Z_i) - \E_{Z'\sim P_Z}\ell(W,Z')}^2}\\
&\hspace{-9em}=\E_{\tilde{Z},J,W}\sbr{\rbr{\frac{1}{n}\sum_{i=1}^n \ell(W,\tilde{Z}_{i, J_i}) - \E_{Z'\sim P_Z}\ell(W,Z')}^2}\\
&\hspace{-9em}=\E_{\tilde{Z},J,W}\sbr{\rbr{\frac{1}{n}\sum_{i=1}^n \ell(W,\tilde{Z}_{i, J_i}) - \E_{Z'\sim P_Z}\ell(W,Z') + \frac{1}{n}\sum_{i=1}^n \ell(W,\tilde{Z}_{i, \bar{J}_i}) - \frac{1}{n}\sum_{i=1}^n \ell(W,\tilde{Z}_{i, \bar{J}_i})}^2}\\
&\hspace{-9em}\le 2\underbrace{\E_{\tilde{Z},J,W}\sbr{\rbr{\frac{1}{n}\sum_{i=1}^n \ell(W,\tilde{Z}_{i, J_i}) - \frac{1}{n}\sum_{i=1}^n \ell(W,\tilde{Z}_{i, \bar{J}_i})}^2}}_{B}\\
&\hspace{-9em}\quad+2\E_{\tilde{Z},J,W}\sbr{\rbr{\frac{1}{n}\sum_{i=1}^n \ell(W,\tilde{Z}_{i, \bar{J}_i}) - \E_{Z'\sim P_Z}\ell(W,Z')}^2}\\
&\hspace{-9em}= 2B + 2\E_{\tilde{Z}, J, W}\sbr{\rbr{\frac{1}{n}\sum_{i=1}^n \rbr{\ell(W,\tilde{Z}_{i,\bar{J}_i}) - \E_{Z'\sim P_Z}\ell(W,Z')}}^2}\\
&\hspace{-9em}= 2B + 2\E_{W}\E_{\tilde{Z},J | W}\sbr{\rbr{\frac{1}{n}\sum_{i=1}^n \rbr{\ell(W,\tilde{Z}_{i,\bar{J}_i}) - \E_{Z'\sim P_Z}\ell(W,Z')}}^2}.\label{eq:test-concentration}
\end{align}
For any fixed $w\in\mathcal{W}$, the terms $\ell(W, \tilde{Z}_{i,\bar{J_i}})$ are independent of each other under $P_{\tilde{Z}, J | W=w}$.
Furthermore, $W$ and $\tilde{Z}_{i, \bar{J}_i}$ are independent.
Therefore, the average in \cref{eq:test-concentration} is an average of $n$ i.i.d. random variables with zero mean. The variance of this average is at most $\frac{1}{4n}$.
Hence,
\begin{equation}
    \E_{P_{W,S}}\sbr{\rbr{R(W) - r_S(W)}^2} \le 2B + \frac{1}{2n}.\label{eq:bounding-g2-first-step}
\end{equation}
Let us bound $B$ now:
\begin{align}
B &= \E_{\tilde{Z},J,W}\sbr{\rbr{\frac{1}{n}\sum_{i=1}^n\rbr{ \ell(W,\tilde{Z}_{i, J_i}) - \ell(W,\tilde{Z}_{i, \bar{J}_i})}}^2}\\
    &= \frac{1}{n^2}\sum_{i=1}^n \E_{\tilde{Z},J,W}\sbr{\rbr{\ell(W,\tilde{Z}_{i, J_i}) - \ell(W,\tilde{Z}_{i, \bar{J}_i})}^2} \\
    &\quad+ \frac{1}{n^2}\sum_{i \neq k}\E_{\tilde{Z},J,W}\sbr{\rbr{\ell(W,\tilde{Z}_{i, J_i}) - \ell(W,\tilde{Z}_{i, \bar{J}_i})}\rbr{\ell(W,\tilde{Z}_{k, J_k}) - \ell(W,\tilde{Z}_{k, \bar{J}_k})}}\\
    &\le \frac{1}{n} + \E_{\tilde{Z},J,W}\sbr{\frac{1}{n}\sum_{i=1}^n\rbr{\rbr{\ell(W,\tilde{Z}_{i, J_i}) - \ell(W,\tilde{Z}_{i, \bar{J}_i})}\frac{1}{n}\sum_{k \neq i} \rbr{\ell(W,\tilde{Z}_{k, J_k}) - \ell(W,\tilde{Z}_{k, \bar{J}_k})}}}.\label{eq:bounding-g2-second-step}
\end{align}
Let us consider a fixed $i\in[n]$. Then
\begin{align*}
    &\E_{\tilde{Z},J,W}\sbr{\rbr{\ell(W,\tilde{Z}_{i, J_i}) - \ell(W,\tilde{Z}_{i, \bar{J}_i})}\frac{1}{n}\sum_{k \neq i} \rbr{\ell(W,\tilde{Z}_{k, J_k}) - \ell(W,\tilde{Z}_{k, \bar{J}_k})}}\\
    &\quad\quad=\E_{\tilde{Z}}\E_{J_i, W|\tilde{Z}}\underbrace{\E_{J_{-i} | J_i, W,  \tilde{Z}}\sbr{\rbr{\ell(W,\tilde{Z}_{i, J_i}) - \ell(W,\tilde{Z}_{i, \bar{J}_i})}\frac{1}{n}\sum_{k \neq i} \rbr{\ell(W,\tilde{Z}_{k, J_k}) - \ell(W,\tilde{Z}_{k, \bar{J}_k})}}}_{f(W,J_i,\tilde{Z})}.\label{eq:bounding-g2-third-step}\numberthis
\end{align*}
Note that $f(w,j_i,\tilde{z}) \in [-1, +1]$ for all $w\in\mathcal{W}, j \in \mathset{0,1}^n$, and $\tilde{z} \in \ZZ^{n\times 2}$. Therefore, by \cref{lemma:mutual-info-lemma}, for any value of $\tilde{Z}$
\begin{equation}
\E_{W,J_i|\tilde{Z}}\sbr{f(W, J_i, \tilde{Z})} \le \sqrt{2 I^{\tilde{Z}}(W; J_i)} + \E_{W\mid \tilde{Z}}\E_{J_i\mid \tilde{Z}}\sbr{f(W, J_i, \tilde{Z})}.\label{eq:bounding-g2-dv}
\end{equation}
It is left to notice that for any $w\in\mathcal{W}$, $\E_{J_i | \tilde{Z}}\sbr{f(w, J_i \mid \tilde{Z})} = 0$, as under $P_{J_i | \tilde{Z}}$ the term $\ell(w,\tilde{Z}_{i, J_i})-\ell(w,\tilde{Z}_{i,\bar{J}_i})$ has zero mean.
Therefore, \cref{eq:bounding-g2-dv} reduces to
\begin{equation}
\E_{W,J_i|\tilde{Z}}\sbr{f(W, J_i, \tilde{Z})} \le \sqrt{2 I^{\tilde{Z}}(W; J_i)}.\label{eq:bounding-g2-dv-reduced}
\end{equation}
Putting together \cref{eq:bounding-g2-first-step}, \cref{eq:bounding-g2-second-step}, \cref{eq:bounding-g2-third-step} and \cref{eq:bounding-g2-dv-reduced}, we get that
\begin{equation}
    \E_{P_{W,S}}\sbr{\rbr{R(W) - r_S(W)}^2} \le \frac{5}{2n} + \frac{2}{n}\sum_{i=1}^n \E_{\tilde{Z}}\sqrt{2 I^{\tilde{Z}}(W; J_i)}.
\end{equation}
\end{proof}

\subsection{Proof of \texorpdfstring{\cref{thm:ecmi-sample-wise-exp-gen-gap}}{TEXT}}
\begin{proof}
Let $\tilde{\Lambda} \in [0, 1]^{n\times 2}$ be the losses on examples of $\tilde{Z}$:
\begin{equation}
    \tilde{\Lambda}_{i, c} = \ell(W, \tilde{Z}_{i, c}),\ \ \forall i\in [n], c\in\mathset{0,1}.
\end{equation}
Let $\bar{J} \triangleq (1-J_1,\ldots,1 - J_n)$ be the negation of $J$.
We have that
\begin{align}
    \E_{P_{W,S}}\sbr{r_S(W) - R(W)} &= \frac{1}{n}\sum_{i=1}^n \E\sbr{\ell(W,Z_i) - \E_{Z'\sim P_Z}\ell(W,Z')}\\
    &= \frac{1}{n}\sum_{i=1}^n \E\sbr{\ell(W,Z_{i, J_i}) - \ell(W, Z_{i, \bar{J}_i})}\\
    &=\frac{1}{n}\sum_{i=1}^n \E\sbr{\tilde{\Lambda}_{i,J_i} - \tilde{\Lambda}_{i, \bar{J}_i}}.
\end{align}
If we use \cref{lemma:mutual-info-lemma} with $\Phi = \tilde{\Lambda}_i$, $\Psi=J_i$, and $f(\Phi,\Psi) = \tilde{\Lambda}_{i,J_i} - \tilde{\Lambda}_{i, \bar{J}_i}$, we get that
\begin{align}
    \abs{\E\sbr{\tilde{\Lambda}_{i,J_i} - \tilde{\Lambda}_{i, \bar{J}_i}} - \E_{\tilde{\Lambda}_i}\E_{J_i}\sbr{\tilde{\Lambda}_{i,J_i} \tilde{\Lambda}_{i, \bar{J}_i}}} \le \sqrt{2 I(\tilde{\Lambda}_i, J_i)}
    % &=\sqrt{2 I(\tilde{\Lambda}_i, J_i)}.
\end{align}
This proves the first part of the theorem, as $\E_{\tilde{\Lambda}_i}\E_{J_i}\sbr{\tilde{\Lambda}_{i,J_i} \tilde{\Lambda}_{i, \bar{J}_i}} = 0$.
The second part can be proven by first conditioning on $\tilde{Z}_i$:
\begin{equation}
    \E\sbr{\tilde{\Lambda}_{i,J_i} - \tilde{\Lambda}_{i, \bar{J}_i}} = \E_{\tilde{Z}_i}\E_{\tilde{\Lambda}_i, J_i |\tilde{Z}_i}\sbr{\tilde{\Lambda}_{i,J_i} - \tilde{\Lambda}_{i, \bar{J}_i}},
\end{equation}
and then applying the lemma to upper bound $\abs{\E_{\tilde{\Lambda}_i, J_i |\tilde{Z}_i}\sbr{\tilde{\Lambda}_{i,J_i} - \tilde{\Lambda}_{i, \bar{J}_i}}}$ with $\sqrt{2I^{\tilde{Z_i}}(\tilde{\Lambda}_i;J_i)}$.
Finally, the third part can be proven by first conditioning on $\tilde{Z}$:
\begin{equation}
    \E\sbr{\tilde{\Lambda}_{i,J_i} - \tilde{\Lambda}_{i, \bar{J}_i}} = \E_{\tilde{Z}}\E_{\tilde{\Lambda}_i, J_i |\tilde{Z}}\sbr{\tilde{\Lambda}_{i,J_i} - \tilde{\Lambda}_{i, \bar{J}_i}},
\end{equation}
and then applying the lemma to upper bound $\abs{\E_{\tilde{\Lambda}_i, J_i |\tilde{Z}}\sbr{\tilde{\Lambda}_{i,J_i} - \tilde{\Lambda}_{i, \bar{J}_i}}}$ with $\sqrt{2I^{\tilde{Z}}(\tilde{\Lambda}_i;J_i)}$.
\end{proof}

\begin{remark} As $\tilde{Z}$ and $J_i$ are independent, we have that
\begin{equation}
    I\rbr{\ell(W, \tilde{Z}_i); J_i} \le \E_{\tilde{Z}_i}\sbr{I^{\tilde{Z}_i}\rbr{\ell(W, \tilde{Z}_i); J_i}} \le \E_{\tilde{Z}}\sbr{I^{\tilde{Z}}\rbr{\ell(W, \tilde{Z}_i); J_i}}.
\end{equation}
However, if we consider expected square root of disintegrated mutual informations (as in the bound of this theorem), then this relation might not be true.
\end{remark}

\subsection{Proof of \texorpdfstring{\cref{thm:g2-pairwise-CMI}}{TEXT}}
\begin{proof}
It is enough to prove only \cref{eq:g2-pairwise-ecmi-strongest}, \cref{eq:g2-pairwise-ecmi-strong}, and \cref{eq:g2-pairwise-ecmi-weak}. Inequality \cref{eq:g2-pairwise-cmi-strong} can be derived from \cref{eq:g2-pairwise-ecmi-strong} using data processing inequality, while \cref{eq:g2-pairwise-cmi-weak} can be derived from \cref{eq:g2-pairwise-ecmi-weak}.

As in the proof of \cref{thm:g2-weak-sample-wise-bound}, 
\begin{align}
    \E_{P_{W,S}}\sbr{\rbr{R(W) - r_S(W)}^2} \le 2\underbrace{\E_{\tilde{Z},J,W}\sbr{\rbr{\frac{1}{n}\sum_{i=1}^n \ell(W,\tilde{Z}_{i, J_i}) - \frac{1}{n}\sum_{i=1}^n \ell(W,\tilde{Z}_{i, \bar{J}_i})}^2}}_{B} + \frac{1}{2n}.
\end{align}
where $\bar{J} \triangleq (1-J_1,\ldots,1 - J_n)$ and 
\begin{equation}
   B \le \frac{1}{n} + \frac{1}{n^2}\sum_{i \neq k}\underbrace{\E_{\tilde{Z},J,W}\sbr{\rbr{\ell(W,\tilde{Z}_{i, J_i}) - \ell(W,\tilde{Z}_{i, \bar{J}_i})}\rbr{\ell(W,\tilde{Z}_{k, J_k}) - \ell(W,\tilde{Z}_{k, \bar{J}_k})}}}_{C_{i,k}}.
\end{equation}
Let $\tilde{\Lambda} \in [0, 1]^{n\times 2}$ be the losses on examples of $\tilde{Z}$:
\begin{equation}
    \tilde{\Lambda}_{i, c} = \ell(W, \tilde{Z}_{i, c}),\ \ \forall i\in [n], c\in\mathset{0,1}.
\end{equation}
Then we can write
\begin{align}
    C_{i,k} = \E_{\tilde{\Lambda}_i,\tilde{\Lambda}_k,J_i, J_k}\sbr{\rbr{\tilde{\Lambda}_{i, J_i} - \tilde{\Lambda}_{i, \bar{J}_i}}\rbr{\tilde{\Lambda}_{k, J_k} - \tilde{\Lambda}_{k, \bar{J}_k}}}.
\end{align}
and use \cref{lemma:mutual-info-lemma} to arrive at:
\begin{align}
    C_{i,k} &\le \E_{\tilde{\Lambda}_i,\tilde{\Lambda}_k}\E_{J_i, J_k}\sbr{\rbr{\tilde{\Lambda}_{i, J_i} - \tilde{\Lambda}_{i, \bar{J}_i}}\rbr{\tilde{\Lambda}_{k, J_k} - \tilde{\Lambda}_{k, \bar{J}_k}}} + \sqrt{2 I(\tilde{\Lambda}_i, \tilde{\Lambda}_k; J_i, J_k)}\\
    &=\sqrt{2 I(\tilde{\Lambda}_i, \tilde{\Lambda}_k; J_i, J_k)},
\end{align}
where the last equality holds as for any fixed value of $(\tilde{\Lambda}_i, \tilde{\Lambda}_k)$ the difference terms $\rbr{\tilde{\Lambda}_{i, J_i} - \tilde{\Lambda}_{i, \bar{J}_i}}$ and $\rbr{\tilde{\Lambda}_{k, J_k} - \tilde{\Lambda}_{k, \bar{J}_k}}$ are independent and have zero mean under $P_{J_i, J_k}$.

To derive \cref{eq:g2-pairwise-ecmi-strong}, one can condition on $\tilde{Z}_i, \tilde{Z}_k$:
\begin{equation}
    C_{i,k} = \E_{\tilde{Z}_i,\tilde{Z}_k}\E_{\tilde{\Lambda}_i,\tilde{\Lambda}_k,J_i, J_k | \tilde{Z}_i,\tilde{Z}_k}\sbr{\rbr{\tilde{\Lambda}_{i, J_i} - \tilde{\Lambda}_{i, \bar{J}_i}}\rbr{\tilde{\Lambda}_{k, J_k} - \tilde{\Lambda}_{k, \bar{J}_k}}},
\end{equation}
and then apply \cref{lemma:mutual-info-lemma} for the inner expectation.
Similarly, \cref{eq:g2-pairwise-ecmi-weak} can be derived by conditioning on $\tilde{Z}$ and then applying \cref{lemma:mutual-info-lemma}.
\end{proof}

\begin{remark}
Wit the expectation inside the square root, inequalities \cref{eq:g2-pairwise-ecmi-strongest}, \cref{eq:g2-pairwise-ecmi-strong}, and \cref{eq:g2-pairwise-ecmi-weak} would be in non-decreasing order. With the expectation outside the square root, this relation might not be true.
\end{remark}

%% file: sup-complexity/proofs.tex
\subsection{Proof of \texorpdfstring{\cref{prop:kernel-method-solution-norm}}{TEXT}}
\begin{proof}
As $K$ is a full rank matrix almost surely, then with probability 1 there exists a vector $\bs{\alpha}\in\bR^n$, such that $K \bs{\alpha} = \bs{Y}$.
Consider the function $f(x) = \sum_{i=1}^n \alpha_i k(X_i, x) \in \HH$. Clearly, $f(X_i) = Y_i,\ \forall i \in [n]$. Furthermore, $\nbr{f}^2_\HH = \bs{\alpha}^\top K\bs{\alpha} = \bs{Y}^\top K^{-1} \bs{Y}$.
The existence of such $f \in \HH$ with zero empirical loss and the assumptions on the loss function imply that any optimal solution of problem \cref{eq:kernel-method-general} has a norm at most $\bs{Y}^\top K^{-1} \bs{Y}$.
% \footnote{For $[0,1]$-bounded loss functions, it also holds that $\nbr{f^*}_\HH^2 \le 2/\lambda$.
% This is not of direct relevance to us, as we will be interested in cases with small $\lambda > 0$.}
\end{proof}

\subsection{Proof of \texorpdfstring{\cref{thm:margin-bound-label-complexity}}{TEXT}}

To prove \cref{thm:margin-bound-label-complexity} we need the following definition of Rademacher complexity~\citep{mohri2018foundations}.

\begin{definition}[Rademacher complexity] Let $\GG$ be a family of functions from $\ZZ$ to $\bR$, and $Z_1,\ldots,Z_n$ be $n$ i.i.d. examples from a distribution $P_Z$ on $\ZZ$. Then, the \emph{empirical Rademacher complexity} of $\GG$ with respect to $(Z_1,\ldots,Z_n)$ is defined as
\begin{equation}
    \widehat{\mathfrak{R}}_n(\GG) = \E_{\sigma_1,\ldots,\sigma_n}\sbr{\sup_{g\in\GG} \frac{1}{m}\sum_{i=1}^n \sigma_i g(Z_i)},
\end{equation}
where $\sigma_i$ are independent Rademacher random variables (i.e., uniform random variables taking values in $\mathset{-1, 1}$).
The \emph{Rademacher complexity} of $\GG$ is then defined as
\begin{equation}
    \mathfrak{R}_n(\GG) = \E_{Z_1,\ldots,Z_n}\sbr{ \widehat{\mathfrak{R}}_n(\GG)}.
\end{equation}    
\end{definition}

\begin{remark}
    Shifting the hypothesis class $\GG$ by a constant function does not change the empirical Rademacher complexity:
    \begin{align}
     \widehat{\mathfrak{R}}_n\rbr{\mathset{f + g : g \in \GG}} &= \E_{\sigma_1,\ldots,\sigma_n}\sbr{\sup_{g\in\GG} \frac{1}{m}\sum_{i=1}^n \sigma_i \rbr{f(Z_i) + g(Z_i)}}\\
     &= \E_{\sigma_1,\ldots,\sigma_n}\sbr{\frac{1}{m}\sum_{i=1}^n \sigma_i f(Z_i) + \sup_{g\in\GG} \frac{1}{m}\sum_{i=1}^n \sigma_i g(Z_i)}\\
     &= \E_{\sigma_1,\ldots,\sigma_n}\sbr{\sup_{g\in\GG} \frac{1}{m}\sum_{i=1}^n \sigma_i g(Z_i)} = \widehat{\mathfrak{R}}_n(\GG).
    \end{align}
    \label{remark:shfiting-rademacher}
\end{remark}

Given the kernel classification setting described in \cref{subsec:target-complexity-scalar-output}, we first prove a slightly more general variant of a classical generalization gap bound in~\citet[Theorem 21]{bartlett2002rademacher}.

\begin{lemma}
Assume $\sup_{x \in \XX} k(x,x) < \infty$. Fix any constant $M > 0$.
Then with probability at least $1-\delta$, every function $f \in \HH$ with $\nbr{f}_\HH \le M$ satisfies
\begin{align*}
    P_{X,Y}(Y f(X) \le 0) &\le \frac{1}{n}\sum_{i=1}^n \phi_\gamma(\sign\rbr{Y_i} f(X_i)) + \frac{2M}{\gamma n} \sqrt{\trace\rbr{K}} + 3\sqrt{\frac{\ln\rbr{2/\delta}}{2n}}.\numberthis 
\end{align*}
\label{lemma:margin-bound-classical}
\end{lemma}
\begin{proof}[Proof of \cref{lemma:margin-bound-classical}]
Let $\FF = \mathset{f \in \HH : \nbr{f} \le M}$ and consider the following class of functions:
\begin{equation}
    \GG = \mathset{(x,y) \mapsto \phi_\gamma(\sign(y) f(x)): f \in \FF}.
\end{equation}
By the standard Rademacher complexity  classification generalization bound~\citep[Theorem 3.3]{mohri2018foundations}, for any $\delta > 0$, with probability at least $1-\delta$, the following holds for all $f \in \FF$:
\begin{equation}
    \E_{X,Y}\sbr{\phi_\gamma(\sign(Y) f(X))} \le \frac{1}{n}\sum_{i=1}^n \phi_\gamma(\sign(Y_i) f(X_i)) + 2 \widehat{\mathfrak{R}}_n(\GG) + 3 \sqrt{\frac{\log(2/\delta)}{2n}}.
\end{equation}
Therefore, with probability at least $1-\delta$, for all $f \in \FF$
\begin{equation}
P_{X,Y}\rbr{Y f(X) \le 0} \le \frac{1}{n}\sum_{i=1}^n \phi_\gamma(\sign(Y_i) f(X_i)) + 2 \widehat{\mathfrak{R}}_n(\GG) + 3 \sqrt{\frac{\log(2/\delta)}{2n}}.
\end{equation}

To finish the proof, we upper bound $\widehat{\mathfrak{R}}_n(\GG)$:
\begin{align}
    \widehat{\mathfrak{R}}_n(\GG) &= \E_{\sigma_1,\ldots,\sigma_n}\sbr{\sup_{g \in \GG}\frac{1}{n}\sum_{i=1}^n \sigma_i g(X_i, Y_i)}\\
    &=\E_{\sigma_1,\ldots,\sigma_n}\sbr{\sup_{f \in \FF}\frac{1}{n}\sum_{i=1}^n \sigma_i \phi_\gamma\rbr{\sign(Y_i) f(X_i)}}\\
    &\le \frac{1}{\gamma} \E_{\sigma_1,\ldots,\sigma_n}\sbr{\sup_{f \in \FF}\frac{1}{n}\sum_{i=1}^n \sigma_i \sign(Y_i) f(X_i)}\\
    &=\frac{1}{\gamma} \E_{\sigma_1,\ldots,\sigma_n}\sbr{\sup_{f \in \FF}\frac{1}{n}\sum_{i=1}^n \sigma_i f(X_i)}\\
    &=\frac{1}{\gamma}\widehat{\mathfrak{R}}_n(\FF),
\end{align}
where the third line is due to \citet{ledoux1991probability}.
By Lemma 22 of \citet{bartlett2002rademacher}, we thus conclude that
\begin{equation}
    \widehat{\mathfrak{R}}_n(\FF) \le \frac{M}{n} \sqrt{\trace\rbr{K}}.
\end{equation}
\end{proof}

\begin{proof}[Proof of \cref{thm:margin-bound-label-complexity}]
To get a generalization bound for $f^*$ it is tempting to use \cref{lemma:margin-bound-classical} with $M = \nbr{f^*}$.
However, $\nbr{f^*}$ is a random variable depending on the training data and is an invalid choice for the constant $M$. 
This issue can be resolved by paying a small logarithmic penalty.

For any $M \ge M_0 = \left\lceil \frac{\gamma \sqrt{n}}{2 \sqrt{\kappa}}\right \rceil$ the bound of \cref{lemma:margin-bound-classical} is vacuous.
Let us consider the set of integers $\mathcal{M} = \mathset{1, 2, \ldots, M_0}$ and write \cref{lemma:margin-bound-classical} for each element of $\mathcal{M}$ with $\delta / M_0$ failure probability.
By union bound, we have that with probability at least $1-\delta$, all instances of \cref{lemma:margin-bound-classical} with $M$ chosen from $\mathcal{M}$ hold simultaneously.

If $\bs{Y}^\top K^{-1} \bs{Y} \ge M_0$, then the desired bound holds trivially, as the right-hand side becomes at least 1.
Otherwise, we set $M = \left\lceil \sqrt{\bs{Y}^\top K^{-1} \bs{Y}} \right\rceil \in \mathcal{M}$ and consider the corresponding part of the union bound.
We thus have that with at least $1 - \delta$ probability, every function $f \in \FF$ with $\nbr{f} \le M$ satisfies
\begin{equation}
    P_{X,Y}(Y f(X) \le 0) \le \frac{1}{n}\sum_{i=1}^n \phi_\gamma(\sign\rbr{Y_i} f(X_i)) + \frac{2M}{\gamma n} \sqrt{\trace\rbr{K}} + 3\sqrt{\frac{\ln\rbr{2 M_0/\delta}}{2n}}. 
\end{equation}
As by \cref{prop:kernel-method-solution-norm} any optimal solution $f^*$ has norm at most $\sqrt{\bs{Y}^\top K^{-1} \bs{Y}}$ and $M \le \sqrt{\bs{Y}^\top K^{-1} \bs{Y}} + 1$, we have with probability at least $1-\delta$,
\begin{align*}
    P_{X,Y}(Y f^*(X) \le 0) &\le \frac{1}{n}\sum_{i=1}^n \phi_\gamma(\sign{\rbr{Y_i}} f^*(X_i)) + \frac{2 \sqrt{\bs{Y}^\top K^{-1} \bs{Y}} + 2}{\gamma n} \sqrt{\trace\rbr{K}}+3\sqrt{\frac{\ln\rbr{2 M_0/\delta}}{2n}}. 
\end{align*}
\end{proof}

\subsection{Proof of \texorpdfstring{\cref{thm:margin-bound-label-complexity-vectorcase}}{TEXT}}
The proof of \cref{thm:margin-bound-label-complexity-vectorcase} follows closely that of \cref{thm:margin-bound-label-complexity}.
We first introduce some concepts related to vector-valued positive semidefinite kernels.

For the given matrix-valued positive definite kernel  $k : \XX \times \XX \rightarrow \bR^{d \times d}$, input $x \in \XX$, and $a\in\bR^d$, let $k_x a = k(\cdot, x)a$ be the function from $\XX$ to $\bR^d$ defined the following way:
\begin{equation}
    k_x a (x') = k(x', x) a, \text{ for all } x' \in \XX.
\end{equation}
With any such kernel $k$ there is a unique vector-valued RKHS $\HH$ of functions from $\XX$ to $\bR^d$.
This RKHS is the completion of $\mathrm{span}\mathset{k_x a : x \in \XX, a\in \bR^d}$, with the following inner product:
\begin{equation}
    \left\langle \sum_{i=1}^n k_{x_i} a_i, \sum_{j=1}^m k_{x'_j} a'_j \right\rangle_\HH = \sum_{i=1}^n\sum_{j=1}^m a_i^\top k(x_i, x'_j) a'_j.
\end{equation}
For any $f \in \HH$, the norm $\nbr{f}_\HH$ is defined as $\sqrt{\langle f, f\rangle_\HH}$.
Therefore, if $f(x) = \sum_{i=1}^n  k_{x_i} a_i$ then
\begin{align}
\langle f, f \rangle^2_\HH &= \sum_{i,j=1}^n a_i^\top k(x_i, x_j) a_j\\
    &= \bs{a}^\top K \bs{a},
\end{align}
where $\bs{a}^\top = (a_1^\top, \ldots, a_n^\top) \in \bR^{n d}$ and
\begin{equation}
    K = \left[ \begin{matrix}
        k(x_1, x_1) & \cdots & k(x_1, x_n)\\
        \cdots & \cdots & \cdots \\
        k(x_n, x_1) & \cdots & k(x_n, x_n)\\
    \end{matrix} \right] \in \bR^{nd \times nd}.
\end{equation}

Suppose $\mathset{(X_i, Y_i)}_{i\in[n]}$ are $n$ i.i.d. examples sampled from some probability distribution on $\XX \times \YY$, where $\YY \subset \bR^d$.
As in the binary classification case, we consider the regularized kernel problem \cref{eq:kernel-method-general}.
Let $\bs{Y}^\top = (Y_1^\top, \ldots, Y_n^\top)$ be the concatenation of targets.
The following proposition is the analog of \cref{prop:kernel-method-solution-norm} in this vector-valued setting.
\begin{proposition}
Assume $K$ is full rank almost surely.
Assume also $\ell(y, y') \ge 0, \forall y, y' \in \YY,$ and $\ell(y,y) = 0, \forall y\in\YY$. Then, with probability 1, for any solution $f^*$ of \cref{eq:kernel-method-general}, we have that
\begin{equation}
    \nbr{f^*}_\mathcal{H}^2 \le \bs{Y}^\top K^{-1} \bs{Y}.
\end{equation}
\label{prop:kernel-method-solution-norm-vectorcase}
\end{proposition}
\begin{proof}[Proof of \cref{prop:kernel-method-solution-norm-vectorcase}]
With probability 1, the kernel matrix $K$ is full rank. Therefore, there exists a vector $\bs{a}^\top = (a_1^\top, \ldots, a_n^\top)\in\bR^{nd}$, with $a_i \in \bR^d$, such that $K \bs{a} = \bs{Y}$. Consider the function $f(x) = \sum_{i=1}^n k_{X_i}a_i \in \HH$. Clearly, $f(X_i) = Y_i,\ \forall i \in [n]$. Furthermore,
\begin{align}
\nbr{f}^2_\HH &= \bs{a}^\top K \bs{a}\\
              &= \bs{Y}^\top K^{-1} \bs{Y}.
\end{align}
The existence of such $f(x) \in \HH$ with zero empirical loss and assumptions on the loss function imply that any optimal solution of problem \cref{eq:kernel-method-general} has a norm at most $\bs{Y}^\top K^{-1} \bs{Y}$.
\end{proof}

\begin{proof}[Proof of \cref{thm:margin-bound-label-complexity-vectorcase}]
Consider the class of functions $\FF = \mathset{f \in \HH : \nbr{f} \le M}$ for some $M > 0$.
By Theorem 2 of \citet{kuznetsov2015rademacher}, for any $\gamma > 0$ and $\delta > 0$, with probability at least $1-\delta$, the following bound holds for all $f \in \FF$:\footnote{Note that their result is in terms of Rademacher complexity rather than empirical Rademacher complexity. The variant we use can be proved with the same proof, with a single modification of bounding $R(g)$ with empirical Rademacher complexity of $\tilde{\GG}$ using Theorem 3.3 of \citet{mohri2018foundations}.}
\begin{equation}
P_{X,Y}(\rho_f (X,Y) \le 0) \le \frac{1}{n}\sum_{i=1}^n \ind{\rho_{f}(X_i, Y_i) \le \gamma} + \frac{4d}{\gamma} \widehat{\mathfrak{R}}_n(\tilde{\FF}) + 3\sqrt{\frac{\log(2/\delta)}{2n}},    
\end{equation}
where $\tilde{\FF} = \mathset{(x,y) \mapsto f(x)_y : f\in\FF, y\in[d]}$.
Next we upper bound the empirical Rademacher complexity of $\tilde{\FF}$:
\begin{align}
    \widehat{\mathfrak{R}}_n(\tilde{\FF}) 
    &= \E_{\sigma_1,\ldots,\sigma_n}\sbr{\sup_{y \in [d], h \in \HH, \nbr{h}\le M}\frac{1}{n}\sum_{i=1}^n \sigma_i h(X_i)_y}\\
    &= \E_{\sigma_1,\ldots,\sigma_n}\sbr{\sup_{y \in [d], h \in \HH, \nbr{h}\le M}\frac{1}{n}\sum_{i=1}^n \sigma_i h(X_i)^\top \mathbf{y}} && \hspace{-5.5em}\text{($\bf{y}$ is the one-hot enc. of $y$)}\\
    &= \E_{\sigma_1,\ldots,\sigma_n}\sbr{\sup_{y \in [d], h \in \HH, \nbr{h}\le M}\left\langle h, \frac{1}{n}\sum_{i=1}^n \sigma_i k_{X_i}\mathbf{y}\right\rangle_\HH}&&\hspace{-5.5em}\text{(reproducing property)}\\
    % &\le \E_{\sigma_1,\ldots,\sigma_n}\sbr{\sup_{y \in [d], h \in \HH, \nbr{h}\le M}\nbr{h}_\HH \nbr{\frac{1}{n}\sum_{i=1}^n \sigma_i k_{X_i}\mathbf{y}}_\HH
    &\le\frac{M}{n} \E_{\sigma_1,\ldots,\sigma_n}\sbr{\sup_{y \in [d]} \nbr{\sum_{i=1}^n \sigma_i k_{X_i}\mathbf{y}}_\HH}&&\hspace{-4.5em}\text{(Cauchy-Schwarz)}\\
    &=\frac{M}{n} \sqrt{\E_{\sigma_1,\ldots,\sigma_n}\sbr{\sup_{y \in [d]} \nbr{\sum_{i=1}^n \sigma_i k_{X_i}\mathbf{y}}^2_\HH}}&&\hspace{-5.5em}\text{(Jensen's inequality)}\\
    &\le \frac{M}{n} \sqrt{\E_{\sigma_1,\ldots,\sigma_n}\sbr{\sum_{y=1}^d \nbr{\sum_{i=1}^n \sigma_i k_{X_i}\mathbf{y}}^2_\HH}}\\
    &= \frac{M}{n} \sqrt{\sum_{y=1}^d\E_{\sigma_1,\ldots,\sigma_n}\sbr{\sum_{i=1}^n \nbr{\sigma_i k_{X_i}\mathbf{y}}^2_\HH + \sum_{i\neq j}\left\langle \sigma_i k_{X_i}\mathbf{y}, \sigma_j k_{X_j}\mathbf{y}\right\rangle}}\\
    &= \frac{M}{n} \sqrt{\sum_{y=1}^d\E_{\sigma_1,\ldots,\sigma_n}\sbr{\sum_{i=1}^n \nbr{\sigma_i k_{X_i}\mathbf{y}}^2_\HH}} &&\text{\hspace{-5.5em}(independence of $\sigma_i$)}\\
    &=\frac{M}{n} \sqrt{\sum_{y=1}^d\sbr{\sum_{i=1}^n \mathbf{y}^\top k(X_i, X_i) \mathbf{y}}}\\
    &=\frac{M}{n}\sqrt{\trace\rbr{K}}.
\end{align}
The proof is concluded with the same reasoning of the proof of \cref{thm:margin-bound-label-complexity}.
\end{proof}

\subsection{Proof of \texorpdfstring{\cref{prop:distillation-error-wrt-dataset-labels}}{TEXT}}
\begin{proof}
We have that 
\begin{align*}
    P_{X,Y^*} (Y^* f^*(X) \le 0) &= P_{X,Y^*} (Y^* f^*(X) \le 0 \wedge Y^* g(X) \le 0)\\
    &\quad + P_{X,Y^*} (Y^* f^*(X) \le 0 \wedge Y^* g(X) > 0)\numberthis\\
    &\le P_{X,Y^*} (Y^* g(X) \le 0) + P_{X} (g(X) f^*(X) \le 0).\numberthis
\end{align*}
The rest follows from bounding $P_{X}(g(X) f^*(X) \le 0)$ using \cref{thm:margin-bound-label-complexity}.
\end{proof}

%% file: unique-info/sections/derivations.tex
%%%%%%%%%%%%%%%%%%%%%%%%%%%%%%%%%%%%%%%%%%%%%%%%%%%%%%%%%%%
%                 SGD NOISE COVARIANCE
%%%%%%%%%%%%%%%%%%%%%%%%%%%%%%%%%%%%%%%%%%%%%%%%%%%%%%%%%%%
\subsection{SGD noise covariance}\label{sec:sgd-noise-cov}
Assume we have $n$ examples and the batch size is $b$. Let $g_i \triangleq \nabla_w \ell_i(w)$, $g\triangleq \frac{1}{n}\sum_{i=1}^n g_i$, and $\tilde{G}\triangleq\frac{1}{b}\sum_{i=1}^b g_{k_i}$, where $k_i$ are sampled independently from $\{1,\ldots,n\}$ uniformly at.
Then
\begin{align}
\cov[\tilde{G},\tilde{G}] &= \mathbb{E}\left[\left(\frac{1}{b}\sum_{i=1}^b g_{k_i} - g\right)\left(\frac{1}{b}\sum_{i=1}^b g_{k_i} - g\right)^T\right]\\
&= \frac{1}{b^2} \sum_{i=1}^b \mathbb{E}\left[(g_{k_i} - g)(g_{k_i} - g)^T\right]\\
&= \frac{1}{bn} \sum_{i=1}^n (g_{i} - g)(g_{i} - g)^T\\
&= \frac{1}{b} \left(\frac{1}{n}\left(\sum_{i=1}^n g_i g_i^T\right) -  gg^T\right).
\end{align}
We denote the per-sample covariance, $b \cdot\cov\left[\tilde{G}, \tilde{G}\right]$ with $\Lambda(w)$:
\begin{equation}
   \Lambda(w)=\frac{1}{n}\left(\sum_{i=1}^n g_i g_i^T\right) -  gg^T.
\end{equation}
We can see that whenever the number of samples times number of outputs is less than number of parameters ($n k < d$), then $\Lambda(w)$ will be rank deficient.
Also, note that if we add weight decay to the total loss then covariance $\Lambda(w)$ will not change, as all gradients will be shifted by the same vector.

%%%%%%%%%%%%%%%%%%%%%%%%%%%%%%%%%%%%%%%%%%%%%%%%%%%%%%%%%%%
%                 SGD FINAL DISTRIBUTION COVARIANCE
%%%%%%%%%%%%%%%%%%%%%%%%%%%%%%%%%%%%%%%%%%%%%%%%%%%%%%%%%%%
\subsection{Steady-state covariance of SGD}\label{sec:sgd-final-dist}
In this section we verify that the normal distribution $\mathcal{N}(\cdot; w^*, \Sigma)$, with $w^*$ being the global minimum of the regularization MSE loss and covariance matrix $\Sigma$ satisfying the continuous Lyapunov equation $\Sigma H + H \Sigma = \frac{\eta}{b}\Lambda(w^*)$, is the steady-state distribution of the stochastic differential equation of \cref{eq:SGD-SDE-plain}.
We assume that (a) $\Lambda(w)$ is constant in a small neighborhood of $w^*$ and (b) the steady-state distribution is unique.
% \begin{equation*}
% \frac{\partial p(x, t)}{\partial t} = -\sum_{i=1}^n \frac{\partial}{\partial x_i} \left[ \mu_i(x,t)p(x,t)\right] + \sum_{i=1}^n\sum_{j=1}^n \frac{\partial^2}{\partial x_i \partial x_j}\left[ D_{i,j} p(x,t) \right].
% \end{equation*}
% In our notation this will become:
We start with the Fokker-Planck equation:
\begin{equation}
\frac{\partial p(w, t)}{\partial t} = \sum_{i=1}^n \frac{\partial}{\partial w_i} \left[ \eta \nabla_{w_i} \LL(w)p(w,t)\right] + \frac{\eta^2}{2b}\sum_{i=1}^n\sum_{j=1}^n \frac{\partial^2}{\partial w_i \partial w_j}\left[ \Lambda(w)_{i,j} p(w,t) \right].
\end{equation}
If $p(w) = \mathcal{N}(w ; w^*, \Sigma) = \frac{1}{Z} \exp\left\{-\frac{1}{2} (w - w^*)^T \Sigma^{-1} (w - w^*)\right\}$ is the steady-state distribution, then the Fokker-Planck becomes:
\begin{equation}
    0 = \sum_{i=1}^d \frac{\partial}{\partial w_i} \left[ \eta \nabla_{w_i} \LL(w)p(w)\right] + \frac{\eta^2}{2b}\sum_{i=1}^d\sum_{j=1}^d \frac{\partial^2}{\partial w_i \partial w_j}\left[ \Lambda(w)_{i,j} p(w) \right].
\label{eq:fokker-planck-final}
\end{equation}
In the case of MSE loss:
\begin{align}
&\nabla_{w} \LL(w) = \sum_{k=1}^n \nabla f_0(x_k) (f(x_k) - y_k) + \lambda w = \nabla f_0(\bs{x}) (f(\bs{x}) - \bs{y}) + \lambda w,\\
&\nabla^2_{w} \LL(w) = \nabla f_0(\bs{x}) \nabla f_0(\bs{x})^T + \lambda I.
\end{align}
Additionally, for $p(w)$ the following two statements hold:
\begin{align}
&\frac{\partial}{\partial w_i} p(w) = -p(w) \Sigma^{-1}_i(w-w^*),\\
&\frac{\partial^2}{\partial w_i \partial w_j} p(w) = -p(w) \Sigma^{-1}_{i,j} + p(w) \Sigma_j^{-1}(w-w^*) \Sigma_i^{-1}(w-w^*),
\end{align}
where $\Sigma^{-1}_i$ is the $i$-th row of $\Sigma^{-1}$.
Let's compute the first term of \cref{eq:fokker-planck-final}:
\begin{align}
\sum_{i=1}^d\frac{\partial}{\partial w_i} \left[ \nabla_{w_i}\LL(w) p(w) \right] &= \sum_{i=1}^d \left[p(w) \left( \nabla f_0(\bs{x})_i \nabla f_0(\bs{x})_i^T + \lambda w_i\right) - \nabla_{w_i}\LL(w) \cdot p(w) \Sigma_i^{-1}(w-w^*)\right]\\
&\hspace{-8em}= p(w) \trace\left(\nabla f_0(\bs{x}) \nabla f_0(\bs{x})^T + \lambda I\right) - p(w) \sum_{i=1}^d \left(\nabla f_0(\bs{x})_i
(f(\bs{x}) - \bs{y}) + \lambda w_i\right) \Sigma_i^{-1} (w-w^*)\\
&\hspace{-8em}= p(w) \trace(H) - p(w) \left((f(\bs{x}) - \bs{y})^T \nabla f_0(\bs{x})^T + \lambda w^T\right) \Sigma^{-1} (w-w^*)\numberthis\label{eq:fokker-planck-first-term}.
\end{align}
As  $w^*$ is a critical point of $\LL(w)$, we have that $\nabla f_0(\bs{x}) (f_{w^*}(\bs{x}) - \bs{y}) + \lambda w^* = 0$.
Therefore, we can subtract $p(w)\left((f_{w^*}(\bs{x}) - \bs{y})^T \nabla f_0(\bs{x})^T + \lambda (w^*)^T\right) \Sigma^{-1}(w-w^*)$ from \cref{eq:fokker-planck-first-term}:
\begin{align*}
\sum_{i=1}^d\frac{\partial}{\partial w_i} \left[ \nabla_{w_i}\LL(w) p(w) \right] &=\\
&\hspace{-8em}=p(w) \trace(H) - p(w) \left((f(\bs{x}) - f_{w^*}(\bs{x}))^T \nabla f_0(\bs{x})^T  + \lambda (w-w^*)^T\right)\Sigma^{-1} (w-w^*)\numberthis\\
&\hspace{-8em}=p(w) \trace(H) - p(w) (w-w^*)^T  \left(\nabla f_0(\bs{x}) \nabla f_0(\bs{x})^T + \lambda I\right)\Sigma^{-1} (w-w^*)\numberthis\\
&\hspace{-8em}=p(w) \trace(H) - p(w) (w-w^*)^T  H \Sigma^{-1} (w-w^*).\numberthis\label{eq:fokker-planck-first-summand}
\end{align*}

\paragraph{Isotropic case: $\boldsymbol{\Lambda(w) = \sigma^2 I_d}$.}
In the case when $\Lambda(w) = \sigma^2 I_d$, we have
\begin{equation}
\sum_{i,j} \frac{\partial^2}{\partial w_i \partial w_j}\left[ \Lambda(w)_{i,j} p(w) \right] = \sigma^2 \trace(\nabla^2_w p(w)) = -\sigma^2 p(w) \trace(\Sigma^{-1}) + \sigma^2 p(w) \lVert \Sigma^{-1}(w-w^*)\rVert_2^2.
\end{equation}
Putting everything together in the Fokker-Planck we get:
\begin{align*}
&\eta \left(p(w) \trace(H) - p(w) (w-w^*)^T  H \Sigma^{-1} (w-w^*))\right)\\
&\quad+ \frac{\eta^2}{2b} \left(-\sigma^2 p(w) \trace(\Sigma^{-1}) + \sigma^2 p(w) \lVert \Sigma^{-1}(w-w^*)\rVert_2^2\right) = 0.\numberthis
\end{align*}
It is easy to verify that $\Sigma^{-1} = \frac{2b}{\eta \sigma^2} H$ is a valid inverse covariance matrix and satisfies the equation above.
Hence, it is the unique steady-state distribution of the stochastic differential equation.
The result confirms that variance is high when batch size is low or learning rate is large. Additionally, the variance is low along directions of low curvature.

\paragraph*{Non-isotropic case.}
We assume $\Lambda(w)$ is constant around $w^*$ and is equal to $\Lambda$.
This assumption is acceptable to a degree, because SGD converges to a relatively small neighborhood, in which we can assume $\Lambda(w)$ to not change much.
With this assumption,
\begin{align}
\sum_{i,j} \frac{\partial^2}{\partial w_i \partial w_j}\left[ \Lambda_{i,j} p(w) \right] &= \sum_{i,j} \Lambda_{i,j} \left[ -p(w) \Sigma^{-1}_{i,j} + p(w) \Sigma_j^{-1}(w-w^*) \Sigma_i^{-1}(w-w^*) \right]\\
&\hspace{-4em}= -p(w)\trace(\Sigma^{-1}\Lambda) + p(w) \sum_{i,j} \Lambda_{i,j} (\Sigma^{-1} (w-w^*) (w-w^*)^T \Sigma^{-1})_{i,j}\\
&\hspace{-4em}= -p(w)\trace(\Sigma^{-1}\Lambda) + p(w) \trace(\Sigma^{-1} (w-w^*) (w-w^*)^T \Sigma^{-1} \Lambda)\\
&\hspace{-4em}=-p(w)\trace(\Sigma^{-1}\Lambda) + p(w) (w-w^*)^T \Sigma^{-1}\Lambda \Sigma^{-1} (w-w^*))\numberthis\label{eq:fp-second-summand-constant-sigma}.
\end{align}
It is easy to verify that if $\Sigma H + H \Sigma = \frac{\eta}{b} \Lambda$, then terms in \cref{eq:fokker-planck-first-summand,eq:fp-second-summand-constant-sigma} will be negatives of each other up to a constant $\frac{\eta}{2b}$, implying that $p(w)$ satisfies the Fokker-Planck equation.
Note that $\Sigma^{-1} = \frac{2b}{\eta} H\Lambda^{-1}$ also satisfies the Fokker-Planck, but will not be positive definite unless $H$ and $\Lambda$ commute.

%%%%%%%%%%%%%%%%%%%%%%%%%%%%%%%%%%%%%%%%%%%%%%%%%%%%%%%%%%%
%                 SCALABLE INVERSE UPDATE
%%%%%%%%%%%%%%%%%%%%%%%%%%%%%%%%%%%%%%%%%%%%%%%%%%%%%%%%%%%
\subsection{Fast update of NTK inverse after data removal}
\label{sec:inverse-update}
For computing weights or predictions of a linearized network at some time $t$, we need to compute the inverse of the NTK matrix.
To compute the informativeness scores, we need to do this inversion $n$ time, each time with one data point excluded.
In this section, we describe how to update the inverse of NTK matrix after removing one example in $O(n^2k^3)$ time, instead of doing the straightforward $O(n^3k^3)$ computation.
Without loss of generality let's assume we remove last $r$ rows and corresponding columns from the NTK matrix.
We can represent the NTK matrix as a block matrix:
\begin{equation}
    \Theta_0 = \left[\begin{matrix}A_{11} & A_{12} \\ A_{21} & A_{22}\end{matrix}\right].
\end{equation}
The goal is to compute $A_{11}^{-1}$ from $\Theta_0^{-1}$. We start with the block matrix inverse formula:
\begin{equation}
    \Theta_0^{-1} = \left[\begin{matrix}A_{11} & A_{12} \\ A_{21} & A_{22}\end{matrix}\right]^{-1} = \left[\begin{matrix}F_{11}^{-1} & -F_{11}^{-1} A_{12}A_{22}^{-1} \\  -A_{22}^{-1} A_{21} F_{11}^{-1} & F_{22}^{-1}\end{matrix}\right],
    \label{eq:block-matrix-inv}
\end{equation}
where
\begin{align}
F_{11} &= A_{11} - A_{12} A_{22}^{-1} A_{21},\label{eq:f11-def}\\
F_{22} &= A_{22} - A_{21} A_{11}^{-1} A_{12}.
\end{align}
From \cref{eq:f11-def} we have $A_{11} = F_{11} +  A_{12} A_{22}^{-1} A_{21}$.
Applying the Woodbury matrix identity on this we get:
\begin{equation}
A_{11}^{-1} = F_{11}^{-1} - F_{11}^{-1} A_{12} (A_{22} + A_{21} F_{11}^{-1} A_{12})^{-1} A_{21} F_{11}^{-1}\label{eq:ntk_inv_update}.
\end{equation}
This \cref{eq:ntk_inv_update} gives the recipe for computing $A_{11}^{-1}$.
Note that $F_{11}^{-1}$ can be read from $\Theta_0^{-1}$ using \cref{eq:block-matrix-inv}, $A_{12}, A_{21},$ and $A_{22}$ can be read from $\Theta$.
Finally, the complexity of computing $A_{11}^{-1}$ using \cref{eq:ntk_inv_update} is $O(n^2k^3)$ if we remove one example.

%%%%%%%%%%%%%%%%%%%%%%%%%%%%%%%%%%%%%%%%%%%%%%%%%%%%%%%%%%%
%                 WEIGHT DECAY LINARIZED
%%%%%%%%%%%%%%%%%%%%%%%%%%%%%%%%%%%%%%%%%%%%%%%%%%%%%%%%%%%
\subsection{Adding weight decay to linearized neural network training}\label{sec:wd-linearized}
Let us consider the loss function $\LL(w) = \sum_{i=1}^n \ell_i(w) + \frac{\lambda}{2} \lVert w-w_0 \rVert_2^2$.
In this case continuous-time gradient descent is described by the following ODE:
\begin{align}
\dot w(t) &= -\eta \nabla_w f_0(\bs{x}) (f_t(\bs{x}) - \bs{y}) - \eta \lambda (w(t)-w_0)\\
&= -\eta \nabla_w f_0(\bs{x}) (\nabla_w f_0(\bs{x})^T (w(t) - w_0) + f_0(\bs{x}) - \bs{y}) - \eta \lambda (w(t) - w_0)\\
&= \underbrace{-\eta (\nabla_w f_0(\bs{x}) \nabla_w f_0(\bs{x})^T + \lambda I)}_A (w(t)-w_0) + \underbrace{\eta \nabla_w f_0(\bs{x}) \left(-f_0(\bs{x}) + \bs{y} \right)}_b.
\end{align}
Let $\omega(t) \triangleq w(t) - w_0$, then we have
\begin{align}
\dot \omega(t) = A \omega(t) + b.
\end{align}
Since all eigenvalues of $A$ are negative, this ODE is stable and has steady-state 
\begin{align}
\omega^* &= -A^{-1}b\\
&= (\nabla_w f_0(\bs{x}) \nabla_w f_0(\bs{x})^T + \lambda I)^{-1} \nabla_w f_0(\bs{x}) (\bs{y} - f_0(\bs{x})).
\end{align}
The solution $\omega(t)$ is given by:
\begin{align}
\omega(t) &= \omega^* + e^{A t} (\omega_0 - \omega^*)\\
&= (I - e^{A t}) \omega^*.
\label{eq:weight-decay-omega-sol}
\end{align}
Let $\Theta_w \triangleq \nabla_w f_0(\bs{x}) \nabla_w f_0(\bs{x})^T$ and $\Theta_0 \triangleq \nabla_w f_0(\bs{x})^T \nabla_w f_0(\bs{x})$.
If the SVD of $\nabla_w f_0(\bs{x})$ is $U D V^T$, then $\Theta_w = U D U^T$ and $\Theta_0 = V D V^T$.
Additionally, we can extend the columns of $U$ to full basis of $\mathbb{R}^d$ (denoted with $\tilde{U}$) and append zeros to $D$ (denoted with $\tilde{D}$) to write down the eigendecomposition $\Theta_w = \tilde{U} \tilde{D} \tilde{U}^T$.
With this, we have $(\Theta_w + \lambda I)^{-1} = \tilde{U} (\tilde{D} + \lambda I)^{-1} \tilde{U}^T$.
Continuing \cref{eq:weight-decay-omega-sol} we have
\begin{align}
\omega(t) &= (I - e^{A t}) \omega^*\\
&= (I - e^{A t}) (\Theta_w + \lambda I)^{-1} \nabla_w f_0(\bs{x}) (\bs{y} - f_0(\bs{x}))\\
&= (I - e^{A t}) \tilde{U}(\tilde{D} + \lambda I)^{-1} \tilde{U}^T U D V^T (\bs{y} - f_0(\bs{x}))\\
&= \tilde{U}(I - e^{-\eta t (\tilde{D} + \lambda I)})\tilde{U}^T \tilde{U}(\tilde{D} + \lambda I)^{-1} \tilde{U}^T U D V^T (\bs{y} - f_0(\bs{x}))\\
&= \tilde{U}(I - e^{-\eta t (\tilde{D} + \lambda I)})(\tilde{D} + \lambda I)^{-1} I_{d \times nk} D V^T (\bs{y} - f_0(\bs{x}))\\
&= \tilde{U}(I - e^{-\eta t (\tilde{D} + \lambda I)}) \tilde{Z} V^T (\bs{y} - f_0(\bs{x})),
\end{align}
where $\tilde{Z} = \left[\begin{matrix} (D + \lambda I)^{-1} D\\ 0 \end{matrix}\right] \in \mathbb{R}^{d \times nk}$. Denoting $Z \triangleq (D + \lambda I)^{-1} D$ and continuing,
\begin{align}
\omega(t) &= \tilde{U}(I - e^{-\eta t (\tilde{D} + \lambda I)}) \tilde{Z} V^T (\bs{y} - f_0(\bs{x}))\\
&= U (I - e^{-\eta t (D + \lambda I)}) Z V^T (\bs{y} - f_0(\bs{x}))\\
&= U Z (I - e^{-\eta t (D + \lambda I)}) V^T (\bs{y} - f_0(\bs{x}))\\
&= U Z V^T V (I - e^{-\eta t (D + \lambda I)}) V^T (\bs{y} - f_0(\bs{x}))\\
&= U Z V^T (I - e^{-\eta t (\Theta_0 + \lambda I)}) (\bs{y} - f_0(\bs{x}))\\
&= \nabla_w f_0(\bs{x}) (\Theta_0 + \lambda I)^{-1} (I - e^{-\eta t (\Theta_0 + \lambda I)}) (\bs{y} - f_0(\bs{x})).\label{eq:weight-dacey-weights-final-containing}
\end{align}

\paragraph{Solving for outputs.}
Having $w(t)$ derived, we can write down dynamics of $f_t(\bs{x})$ for any $x$:
\begin{align}
f_t(\bs{x}) &= f_0(\bs{x}) + \nabla_w f_0(\bs{x})^T \omega(t)\\
&= f_0(\bs{x}) + \nabla_w f_0(\bs{x})^T \nabla_w f_0(\bs{x}) (\Theta_0 + \lambda I)^{-1} (I - e^{-\eta t (\Theta_0 + \lambda I)}) (\bs{y} - f_0(\bs{x}))\\
&= f_0(\bs{x}) + \Theta_0(x, X) (\Theta_0 + \lambda I)^{-1} (I - e^{-\eta t (\Theta_0 + \lambda I)}) (\bs{y} - f_0(\bs{x})).
\end{align}

%% file: sample-info/sections/experimental-details.tex
\begin{table}[!t]
    \centering
    \small
    \caption{Experimental details for MNIST 4 vs 9 classification in the case of standard training.}
    \begin{tabular}{ll}
    \toprule
    Network & The 5-layer CNN of \cref{tab:4-layer-cnn-arch} with 2 output units.\\
    Optimizer & ADAM with $0.001$ learning rate and $\beta_1=0.9$.\\
    Batch size & 128\\
    Number of examples ($n$) & [75, 250, 1000, 4000]\\
    Number of epochs & 200\\
    Number of samples for $\tilde{Z}$ ($k_1$) & 5\\
    Number of samplings for $S$ for each $\tilde{z}$ ($k_2$) & 30\\
    \bottomrule\\
    \end{tabular}
    \label{tab:mnist4vs9-standard}
\end{table}

\begin{table}[!t]
    \centering
    \small
    \caption{Experimental details for MNIST 4 vs 9 classification in the case of SGLD training.}
    \begin{tabular}{lp{7cm}}
    \toprule
    Network & The 5-layer CNN of \cref{tab:4-layer-cnn-arch} with 2 output units.\\
    Learning rate schedule & Starts at 0.004 and decays by a factor of 0.9 after each 100 iterations.\\
    Inverse temperature schedule & $\min(4000, \max(100, 10 e^{t / 100}))$, where $t$ is the iteration.\\
    Batch size & 100\\ 
    Number of examples ($n$) & 4000\\
    Number of epochs & 40\\
    Number of samples for $\tilde{Z}$ ($k_1$) & 5\\
    Number of samplings for $S$ for each $\tilde{z}$ ($k_2$)& 30\\
    \bottomrule\\
    \end{tabular}
    \label{tab:mnist4vs9-langevin}
\end{table}

\begin{table}[!t]
    \centering
    \small
    \caption{Experimental details for CIFAR-10 classification using fine-tuned ResNet-50 networks.}
    \begin{tabular}{ll}
    \toprule
    Network & ResNet-50 pretrained on ImageNet.\\
    Optimizer & SGD with $0.01$ learning rate and $0.9$ momentum.\\
    Data augmentations & Random horizontal flip and random 28x28 cropping.\\
    Batch size & 64\\
    Number of examples ($n$) & [1000, 5000, 20000]\\
    Number of epochs & 40\\
    Number of samples for $\tilde{Z}$ ($k_1$)& 1\\
    Number of samplings for $S$ for each $\tilde{z}$ ($k_2$) & 40\\
    \bottomrule\\
    \end{tabular}
    \label{tab:cifar10-resnet50}
\end{table}

\begin{table}[!t]
    \centering
    \small
    \caption{Experimental details for the CIFAR-10 classification experiment where a pretrained ResNet-50 is fine-tuned using SGLD.}
    \begin{tabular}{lp{7cm}}
    \toprule
    Network & ResNet-50 pretrained on ImageNet.\\
    Data augmentations & Random horizontal flip and random 28x28 cropping.\\
    Learning rate schedule & Starts at 0.01 and decays by a factor of 0.9 after each 300 iterations.\\
    Inverse temperature schedule & $\min(16000, \max(100, 10 e^{t / 300}))$, where $t$ is the iteration.\\
    Batch size & 64\\ 
    Number of examples ($n$) & 20000\\
    Number of epochs & 16\\
    Number of samples for $\tilde{Z}$ ($k_1$) & 1\\
    Number of samplings for $S$ for each $\tilde{z}$ ($k_2$)& 40\\
    \bottomrule\\
    \end{tabular}
    \label{tab:cifar10-langevin}
\end{table}

%% file: label-noise/sections/additional_results.tex
\paragraph{Mislabeled and confusing examples.}
\cref{fig:mnist-cifar-confusing-more-examples,fig:clothgin1m-confusing-more-examples} present incorrectly labeled or confusing examples for each class in MNIST, CIFAR-10, and Clothing1M datasets.

\begin{figure}
    \centering
    \begin{subfigure}{0.49\textwidth}
    \includegraphics[width=\textwidth]{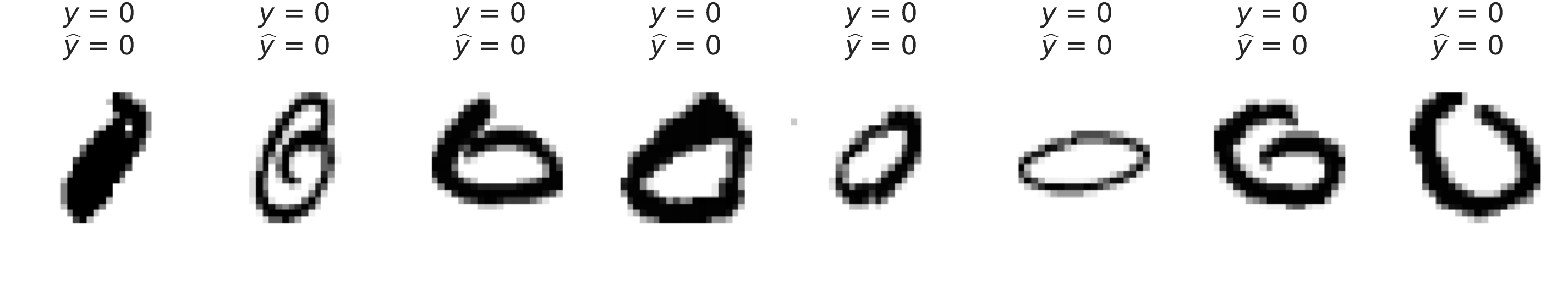}
    \includegraphics[width=\textwidth]{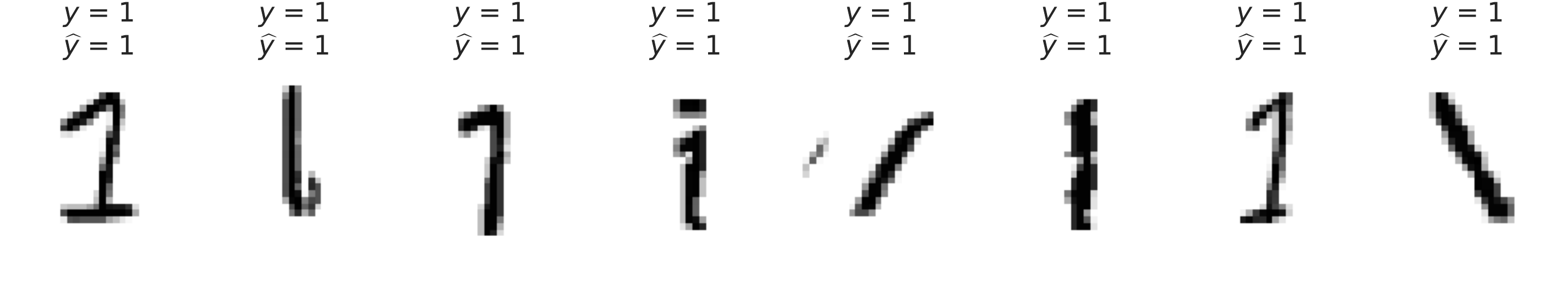}
    \includegraphics[width=\textwidth]{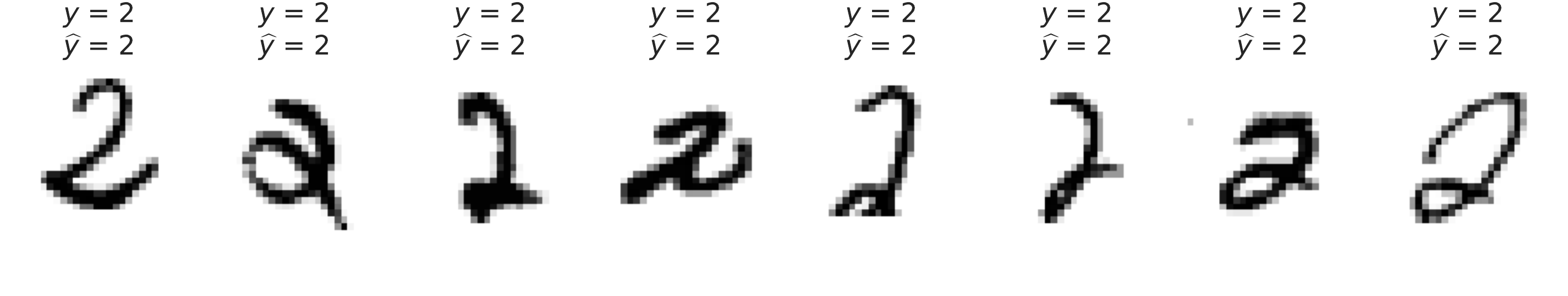}
    \includegraphics[width=\textwidth]{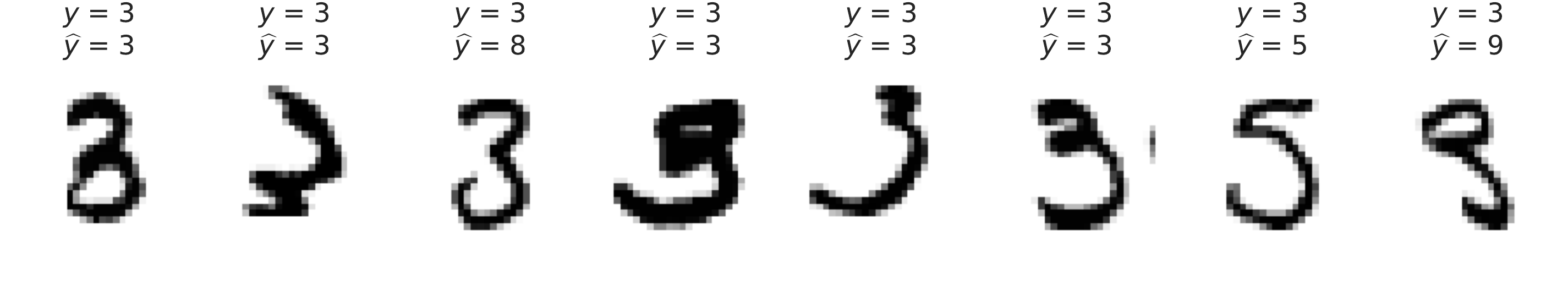}
    \includegraphics[width=\textwidth]{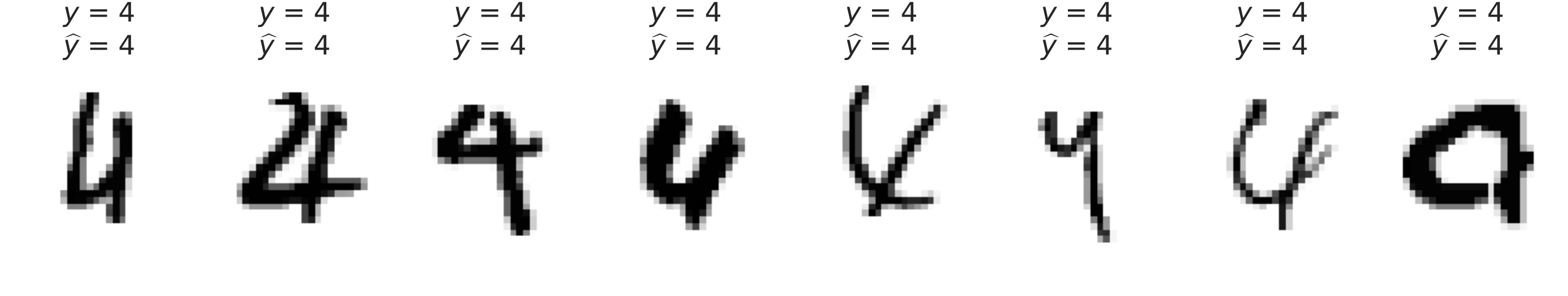}
    \includegraphics[width=\textwidth]{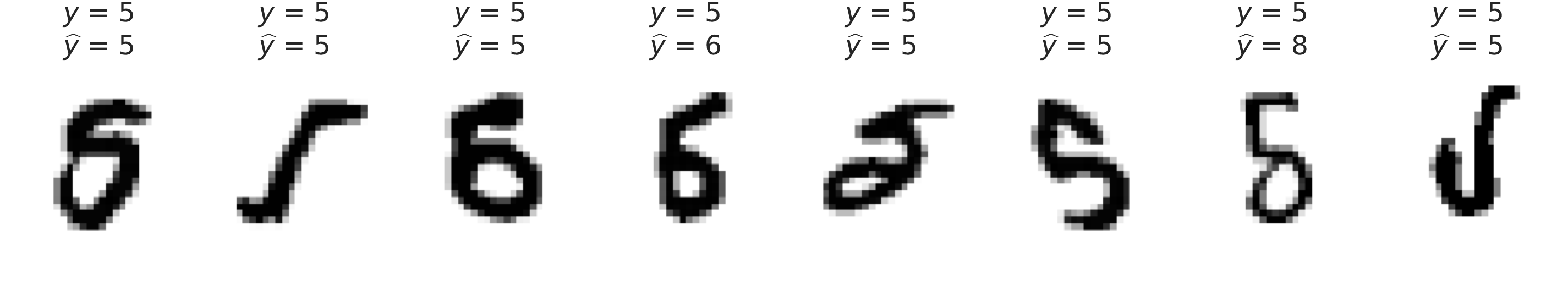}
    \includegraphics[width=\textwidth]{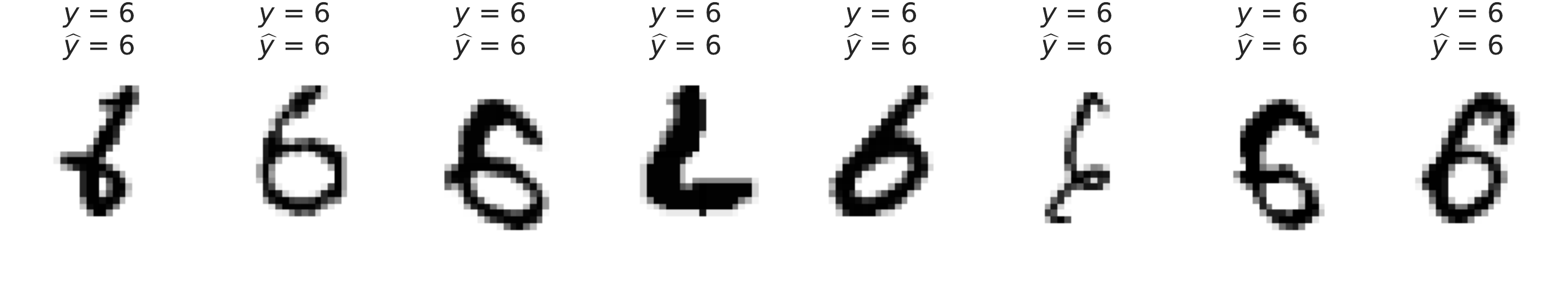}
    \includegraphics[width=\textwidth]{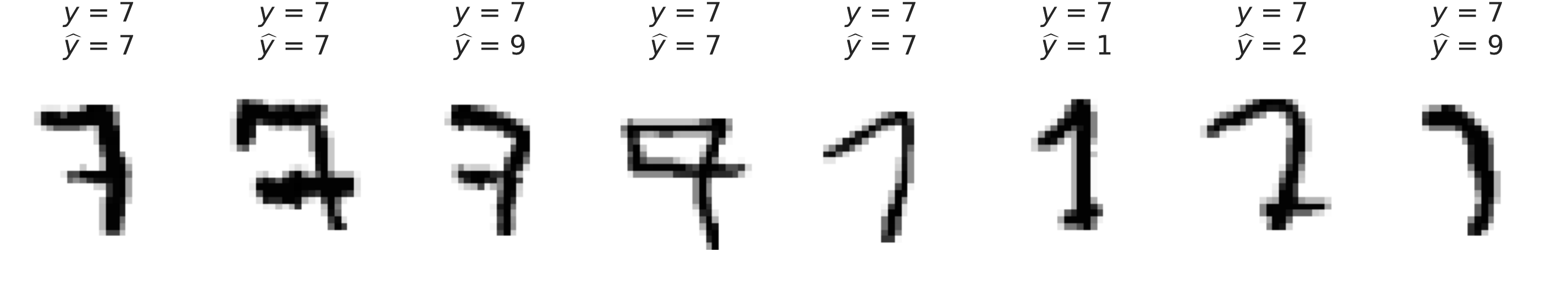}
    \includegraphics[width=\textwidth]{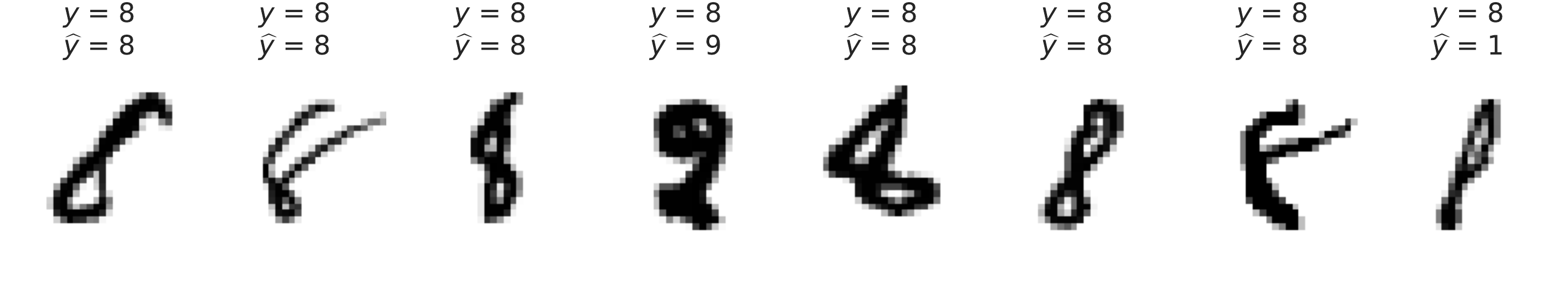}
    \includegraphics[width=\textwidth]{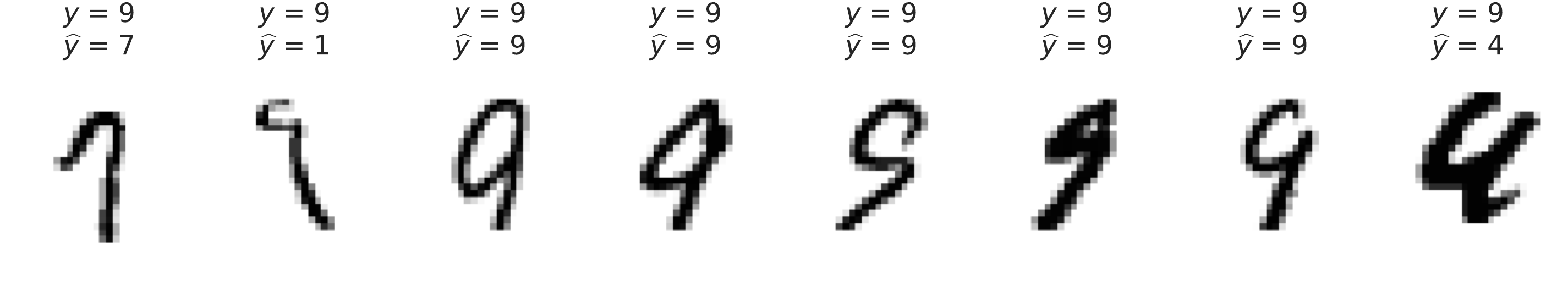}
    \caption{MNIST}
    \end{subfigure}%
    ~
    \begin{subfigure}{0.49\textwidth}
    \includegraphics[width=\textwidth]{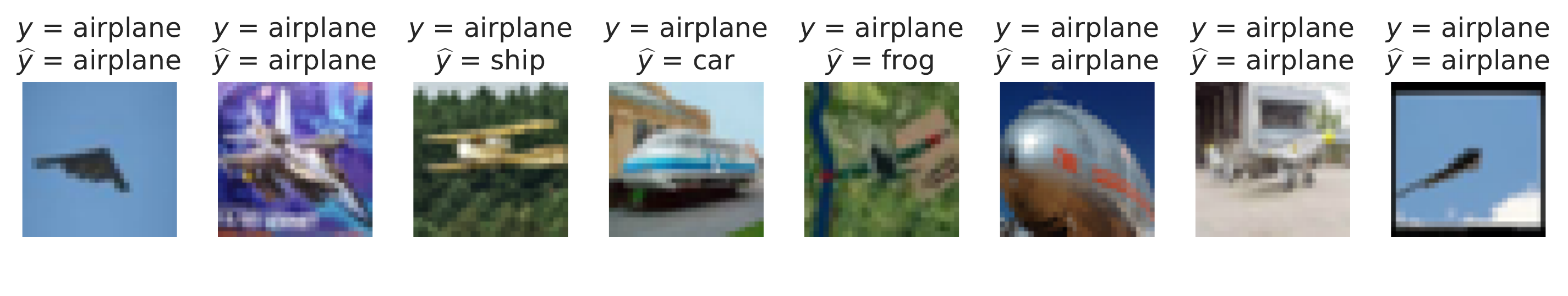}
    \includegraphics[width=\textwidth]{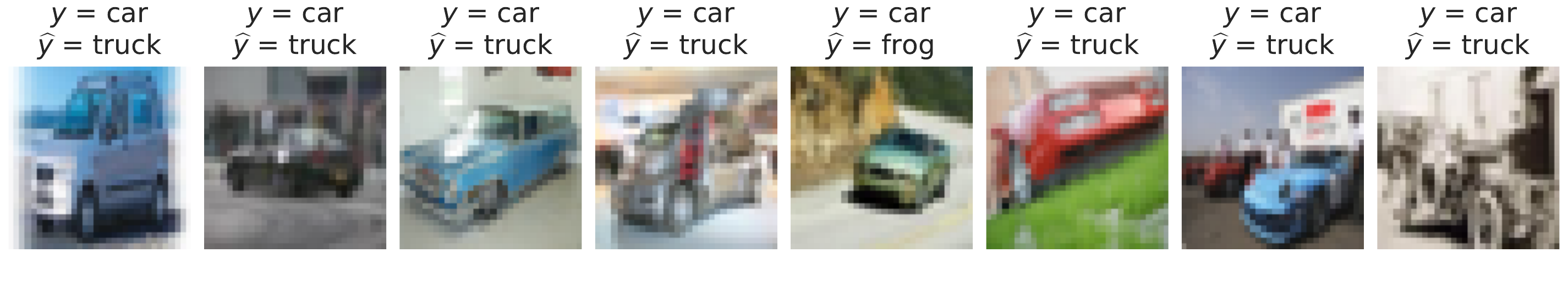}
    \includegraphics[width=\textwidth]{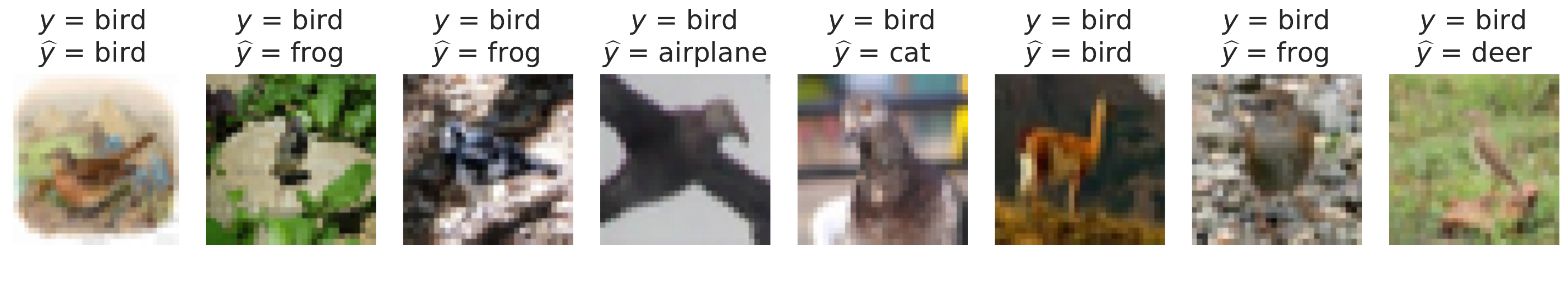}
    \includegraphics[width=\textwidth]{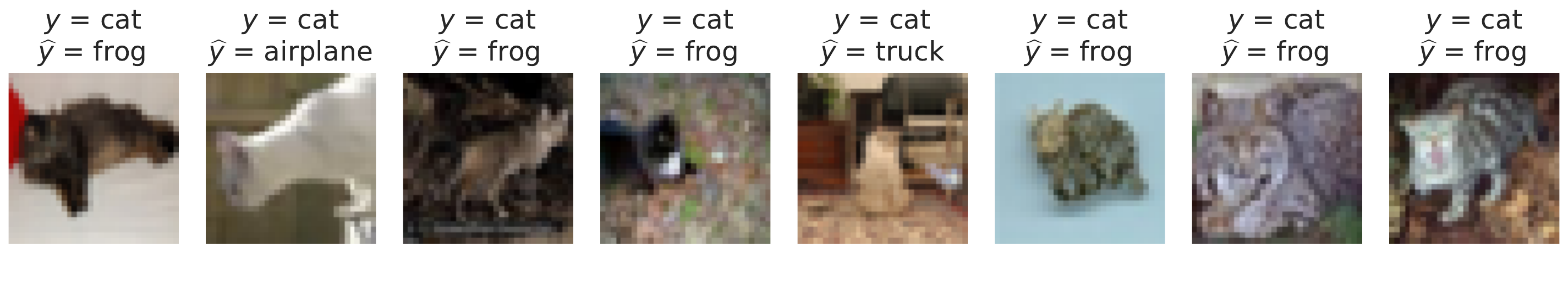}
    \includegraphics[width=\textwidth]{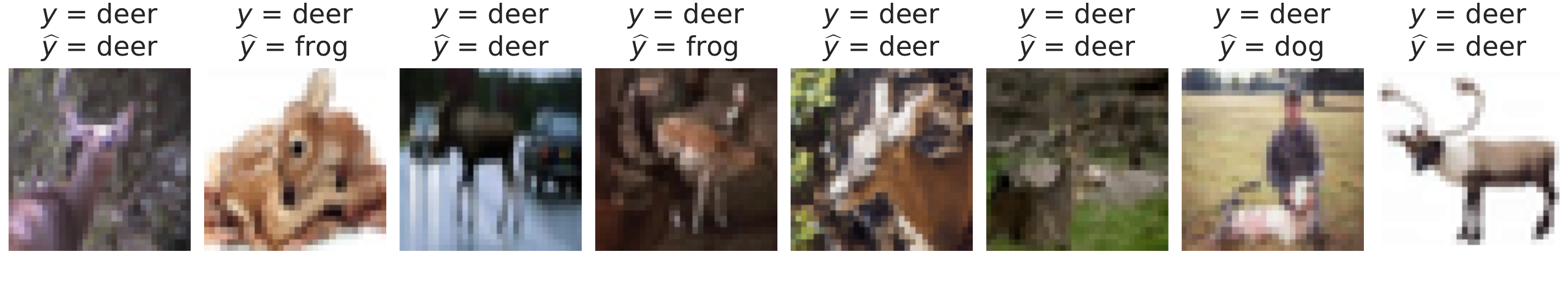}
    \includegraphics[width=\textwidth]{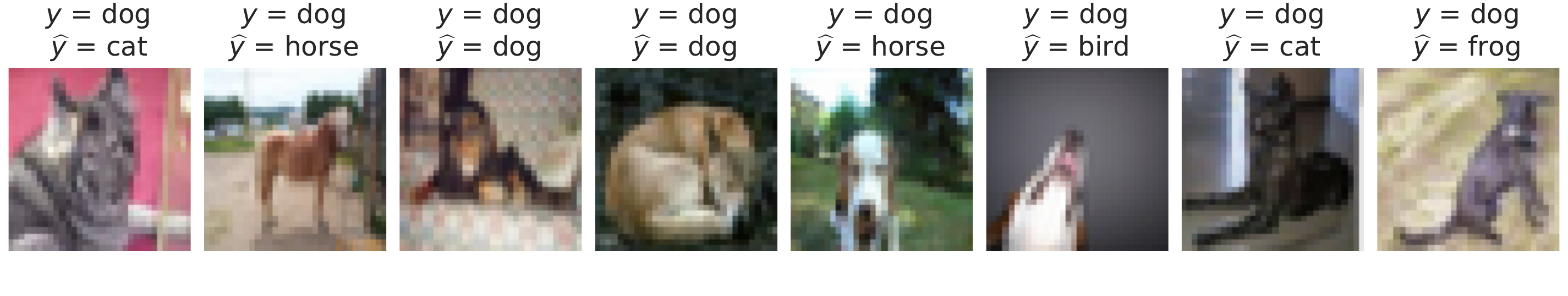}
    \includegraphics[width=\textwidth]{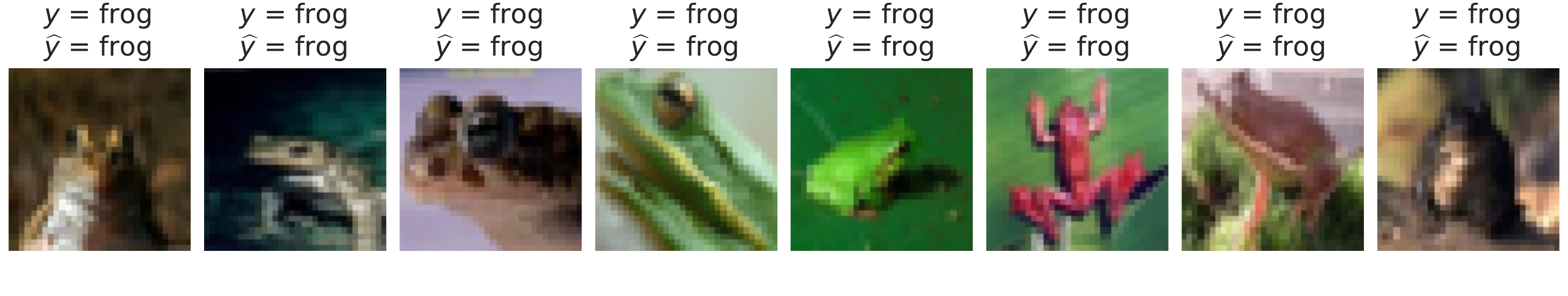}
    \includegraphics[width=\textwidth]{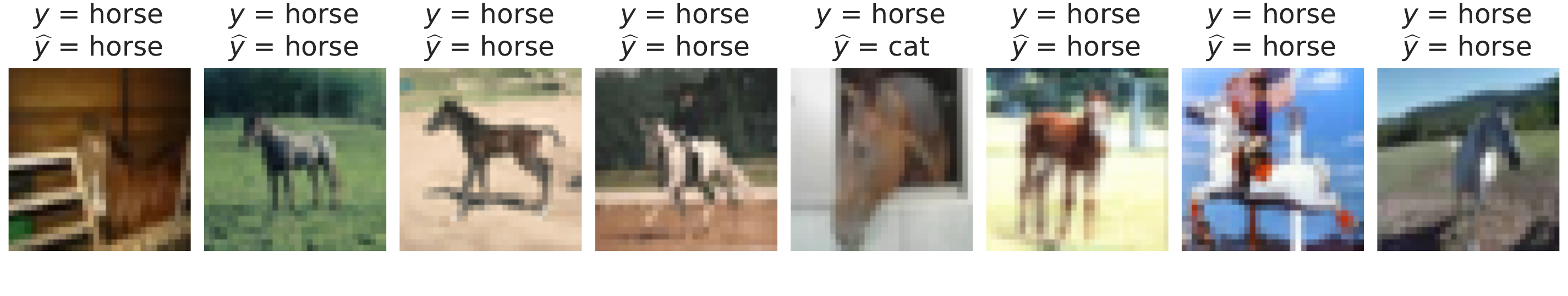}
    \includegraphics[width=\textwidth]{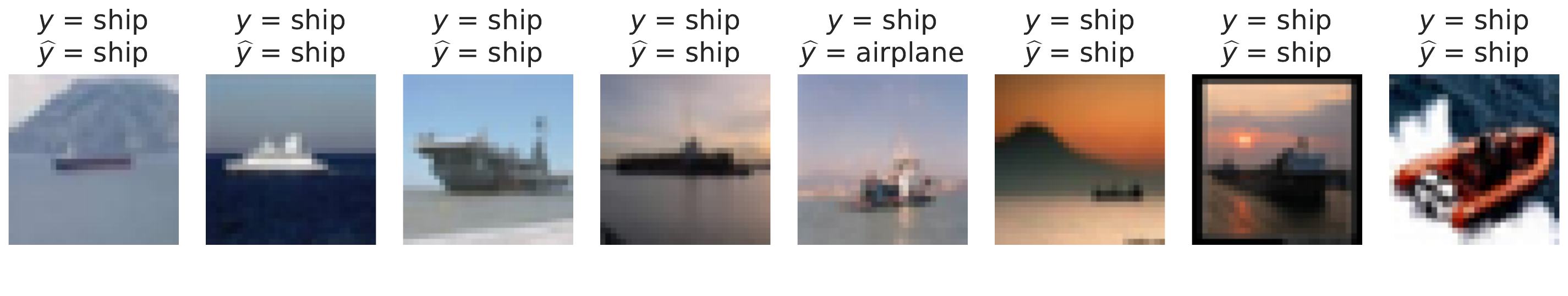}
    \includegraphics[width=\textwidth]{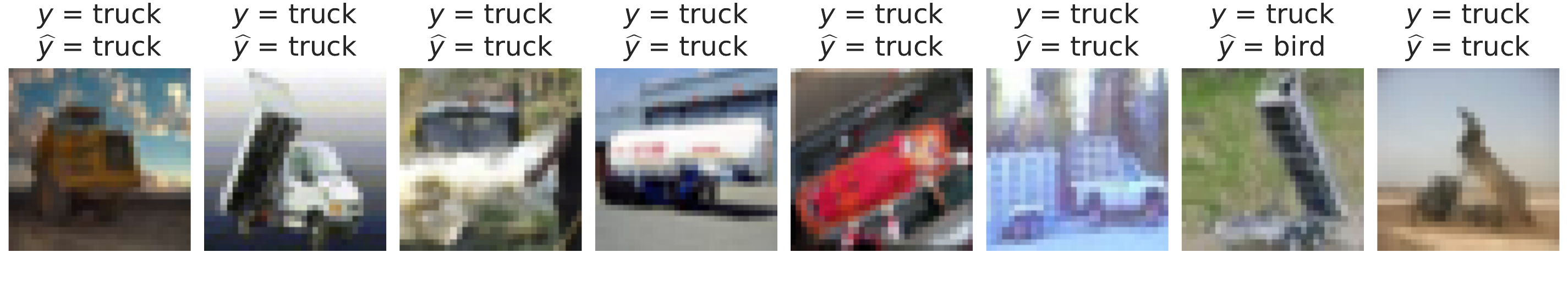}
    \caption{CIFAR-10}
    \end{subfigure}
    \caption{Most confusing 8 labels per class in the MNIST (on the left) and CIFAR-10 (on the right) datasets, according to the distance between predicted and cross-entropy gradients. The gradient predictions are done using the best instances of LIMIT.}
    \label{fig:mnist-cifar-confusing-more-examples}
\end{figure}

\begin{figure}
    \centering
    \begin{subfigure}{0.9\textwidth}
    \includegraphics[width=\textwidth]{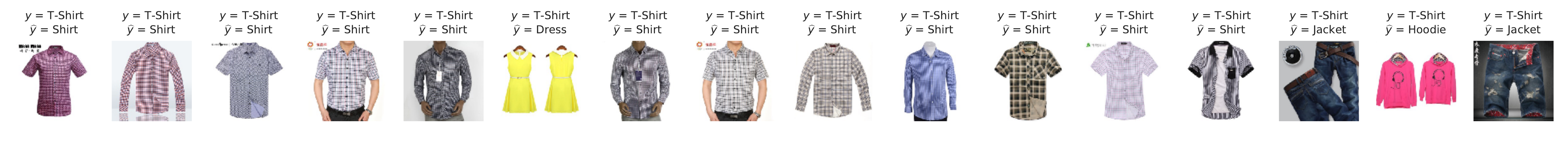}
    \end{subfigure}
    \begin{subfigure}{0.9\textwidth}
    \includegraphics[width=\textwidth]{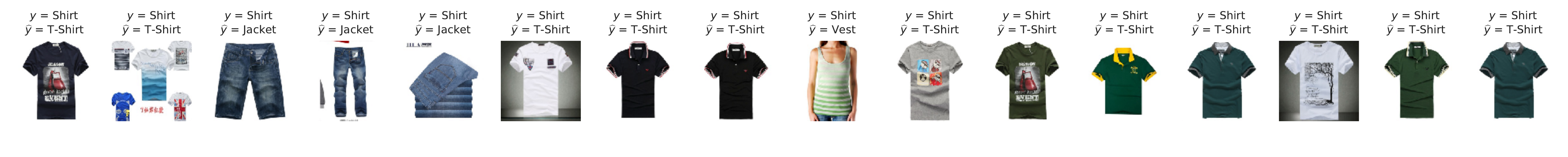}
    \end{subfigure}
    \begin{subfigure}{0.9\textwidth}
    \includegraphics[width=\textwidth]{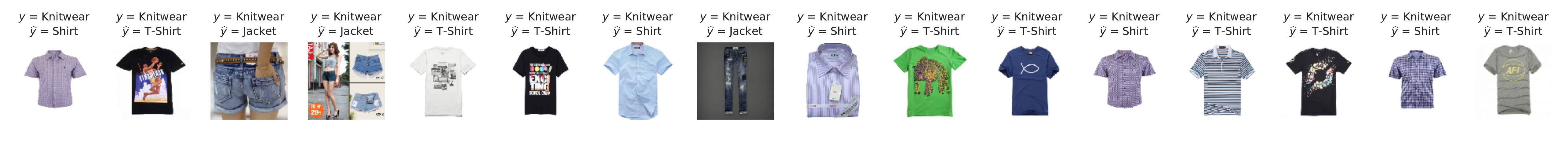}
    \end{subfigure}
    \begin{subfigure}{0.9\textwidth}
    \includegraphics[width=\textwidth]{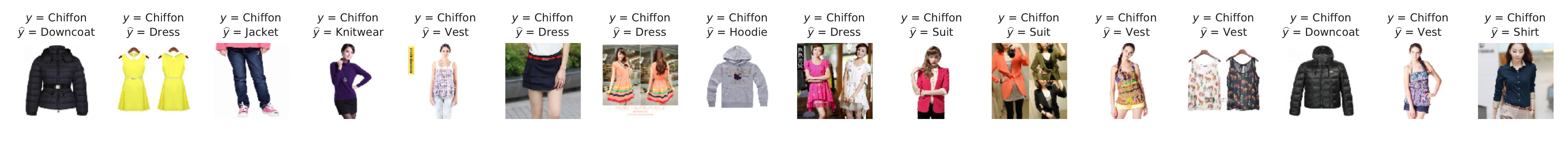}
    \end{subfigure}
    \begin{subfigure}{0.9\textwidth}
    \includegraphics[width=\textwidth]{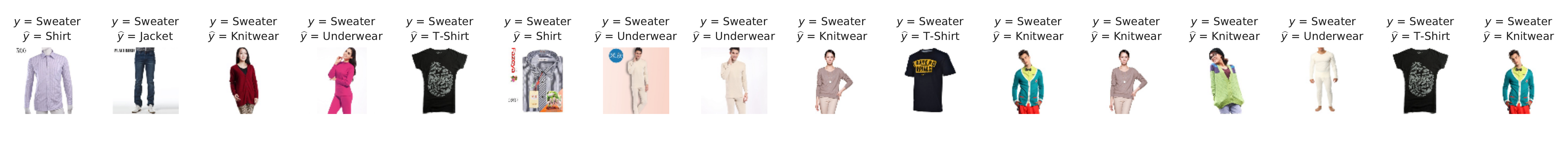}
    \end{subfigure}
    \begin{subfigure}{0.9\textwidth}
    \includegraphics[width=\textwidth]{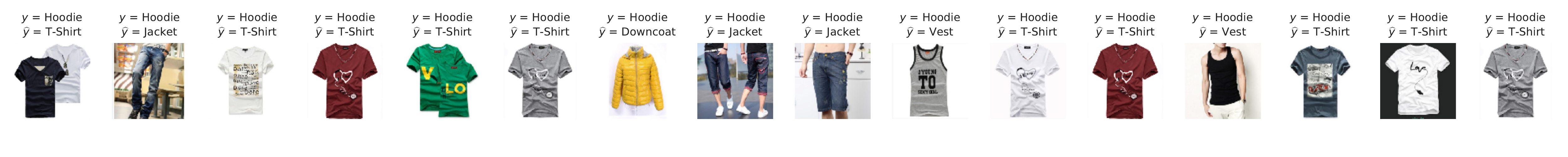}
    \end{subfigure}
    \begin{subfigure}{0.9\textwidth}
    \includegraphics[width=\textwidth]{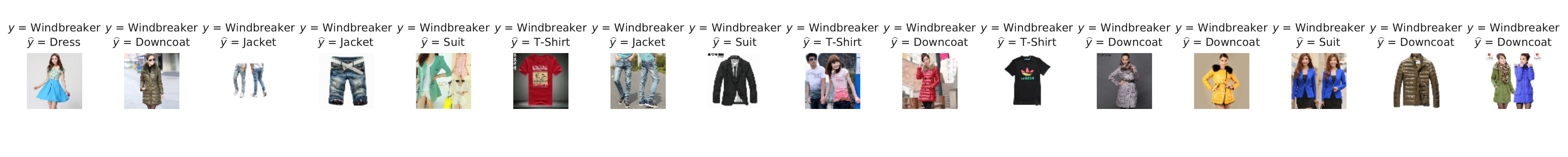}
    \end{subfigure}
    \begin{subfigure}{0.9\textwidth}
    \includegraphics[width=\textwidth]{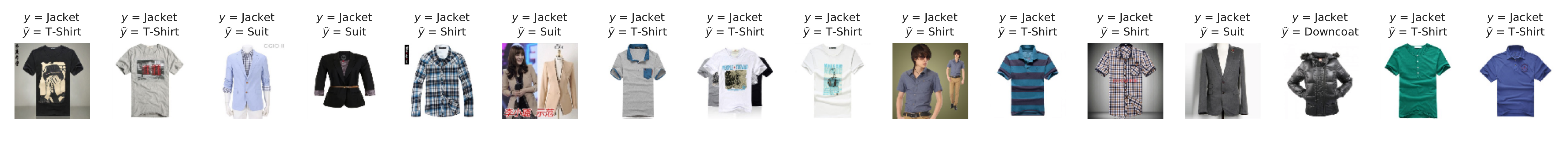}
    \end{subfigure}
    \begin{subfigure}{0.9\textwidth}
    \includegraphics[width=\textwidth]{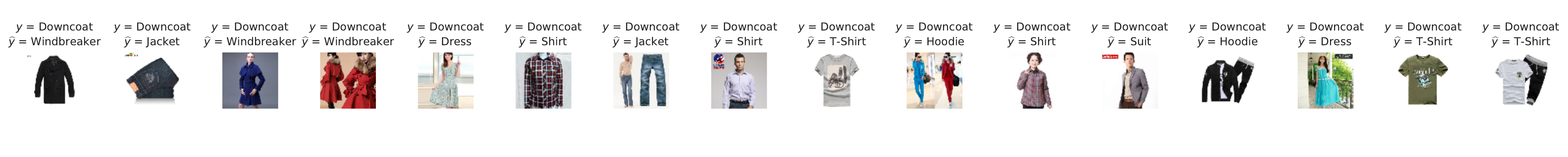}
    \end{subfigure}
    \begin{subfigure}{0.9\textwidth}
    \includegraphics[width=\textwidth]{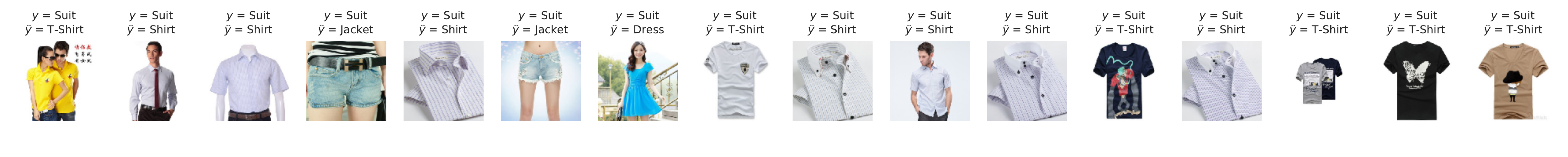}
    \end{subfigure}
    \begin{subfigure}{0.9\textwidth}
    \includegraphics[width=\textwidth]{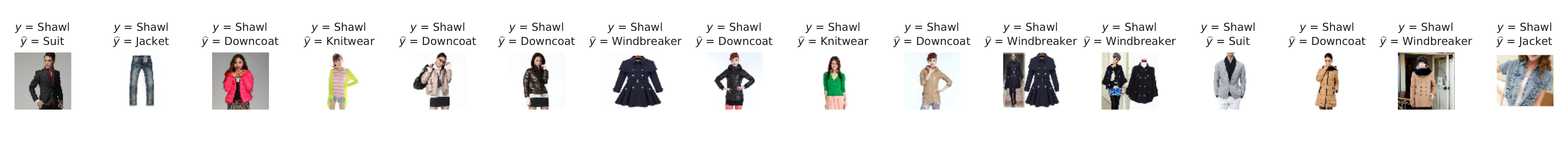}
    \end{subfigure}
    \begin{subfigure}{0.9\textwidth}
    \includegraphics[width=\textwidth]{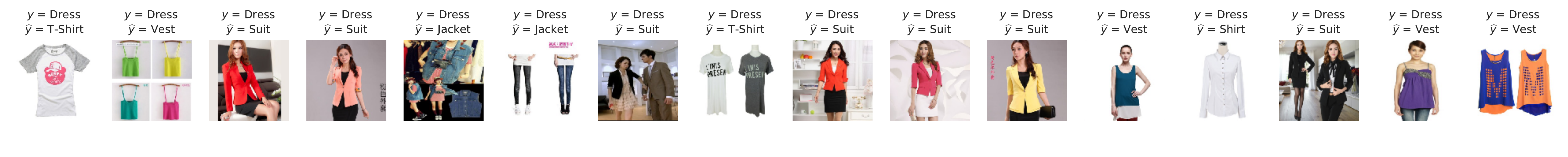}
    \end{subfigure}
    \begin{subfigure}{0.9\textwidth}
    \includegraphics[width=\textwidth]{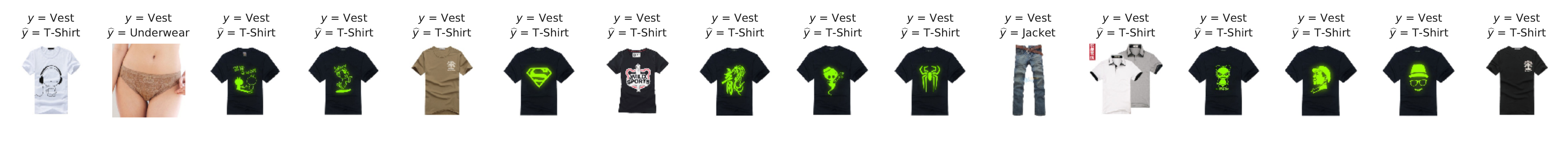}
    \end{subfigure}
    \begin{subfigure}{0.9\textwidth}
    \includegraphics[width=\textwidth]{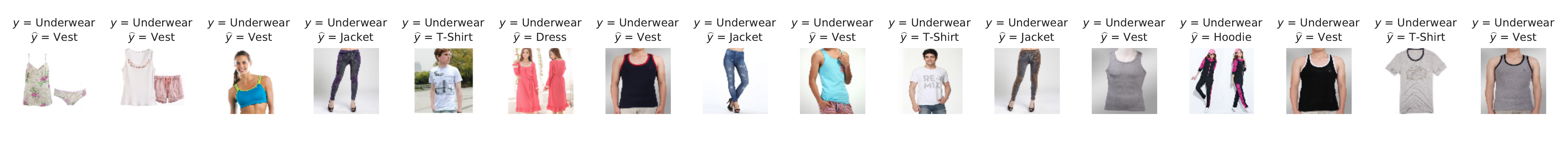}
    \end{subfigure}
    \caption{Most confusing 16 labels per class in the Clothing1M dataset, according to the distance between predicted and cross-entropy gradients. The gradient predictions are done using the best instance of LIMIT.}
    \label{fig:clothgin1m-confusing-more-examples}
\end{figure}

%% file: sup-complexity/additional_results.tex
\paragraph{Teacher-student NTK similarity.}
\cref{fig:ntk-matching-comparision-1-appendix,fig:ntk-matching-comparision-2-appendix} present additional evidence that (a) training NTK similarity of the final student and the teacher is correlated with the final student test accuracy; and (b) that online distillation manages to transfer teacher NTK better.

\begin{figure}[!t]
    \centering
    \begin{subfigure}{0.35\textwidth}
        \includegraphics[width=\textwidth]{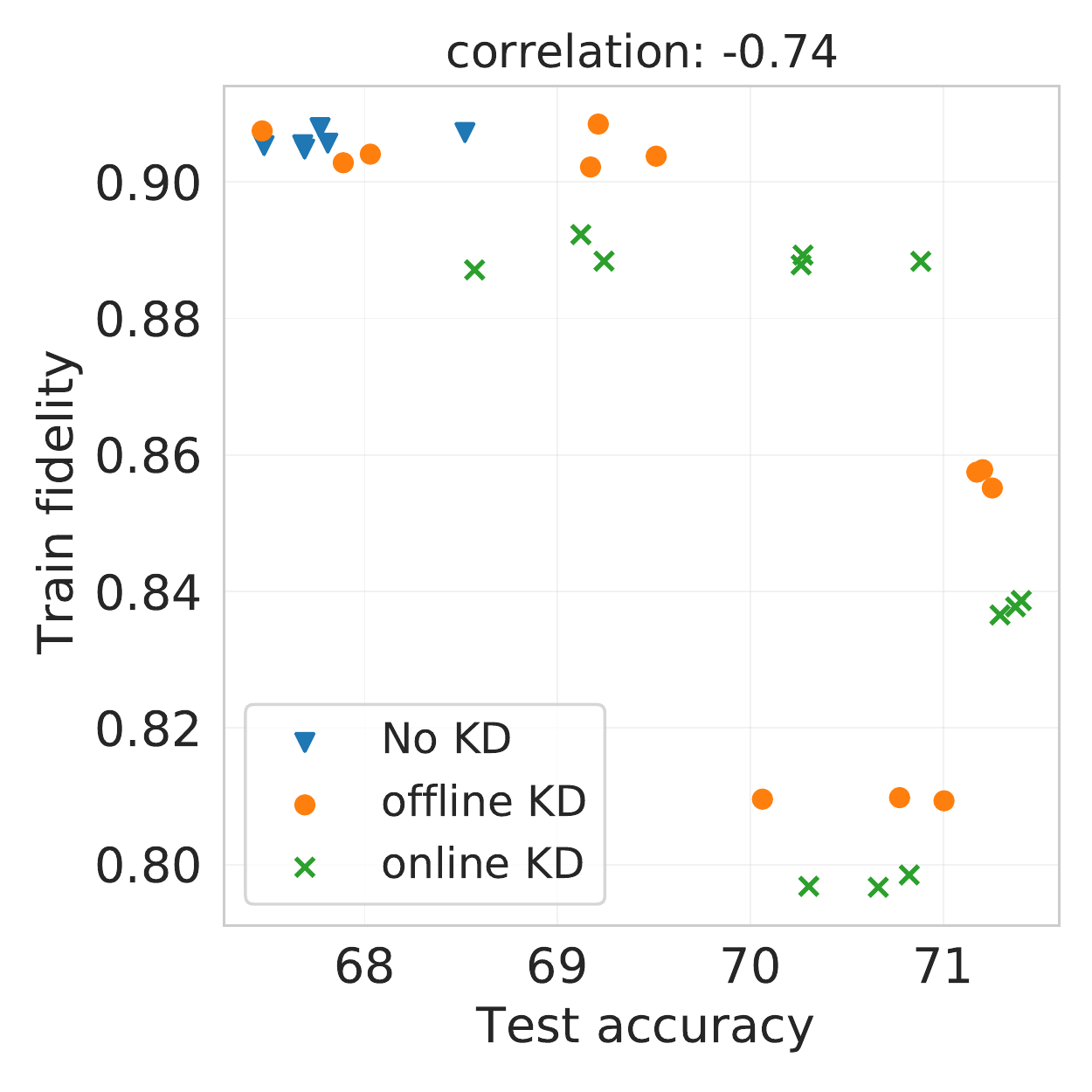}
        \caption{{\centering CIFAR-100, ResNet-20 student, ResNet-110 teacher}}
    \end{subfigure}
    \hspace{4em}
    \begin{subfigure}{0.35\textwidth}
        \includegraphics[width=\textwidth]{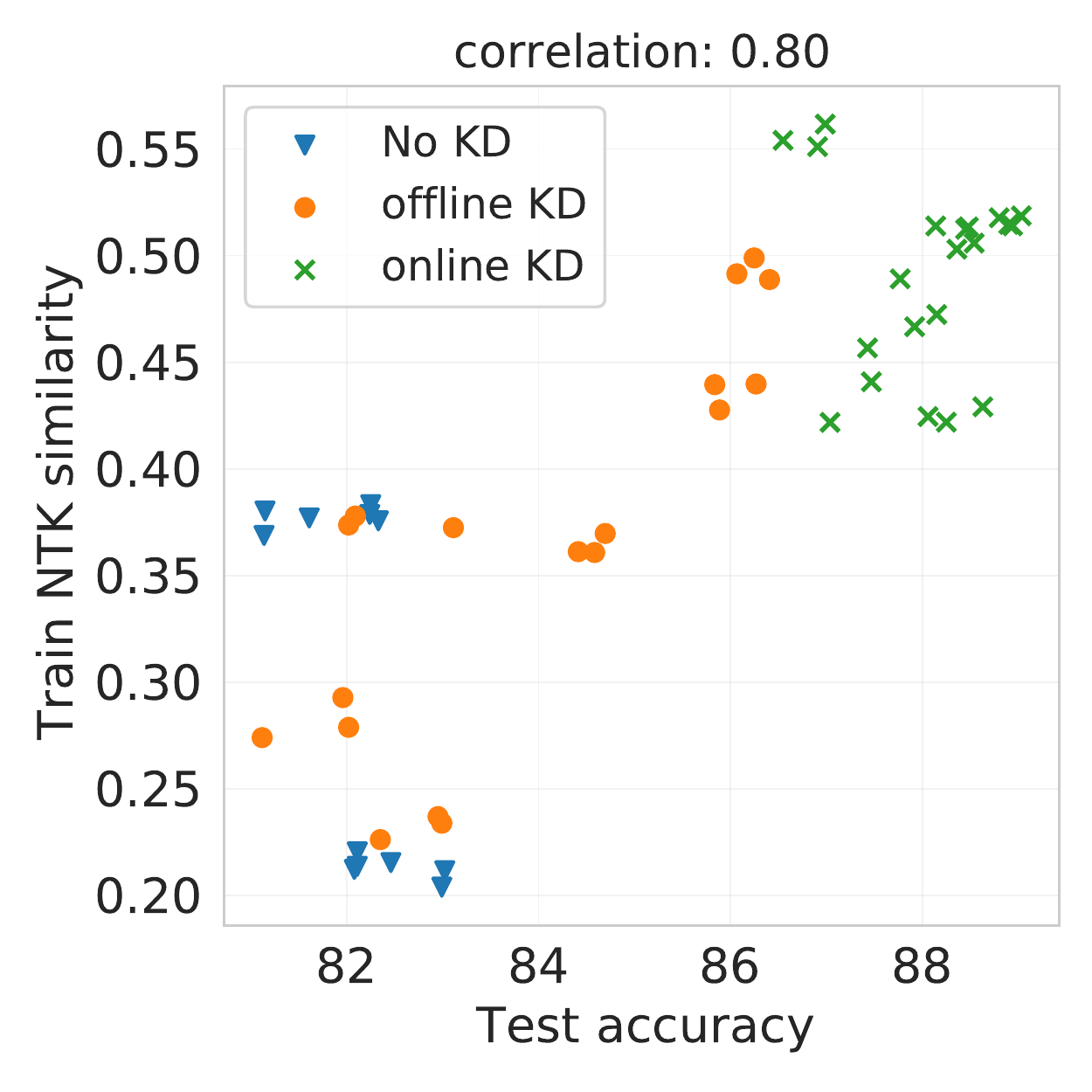}
        \caption{{\centering CIFAR-10, LeNet-5x8 student, ResNet-56 teacher}}
    \end{subfigure}
    \vskip 1em
    \begin{subfigure}{0.35\textwidth}
        \includegraphics[width=\textwidth]{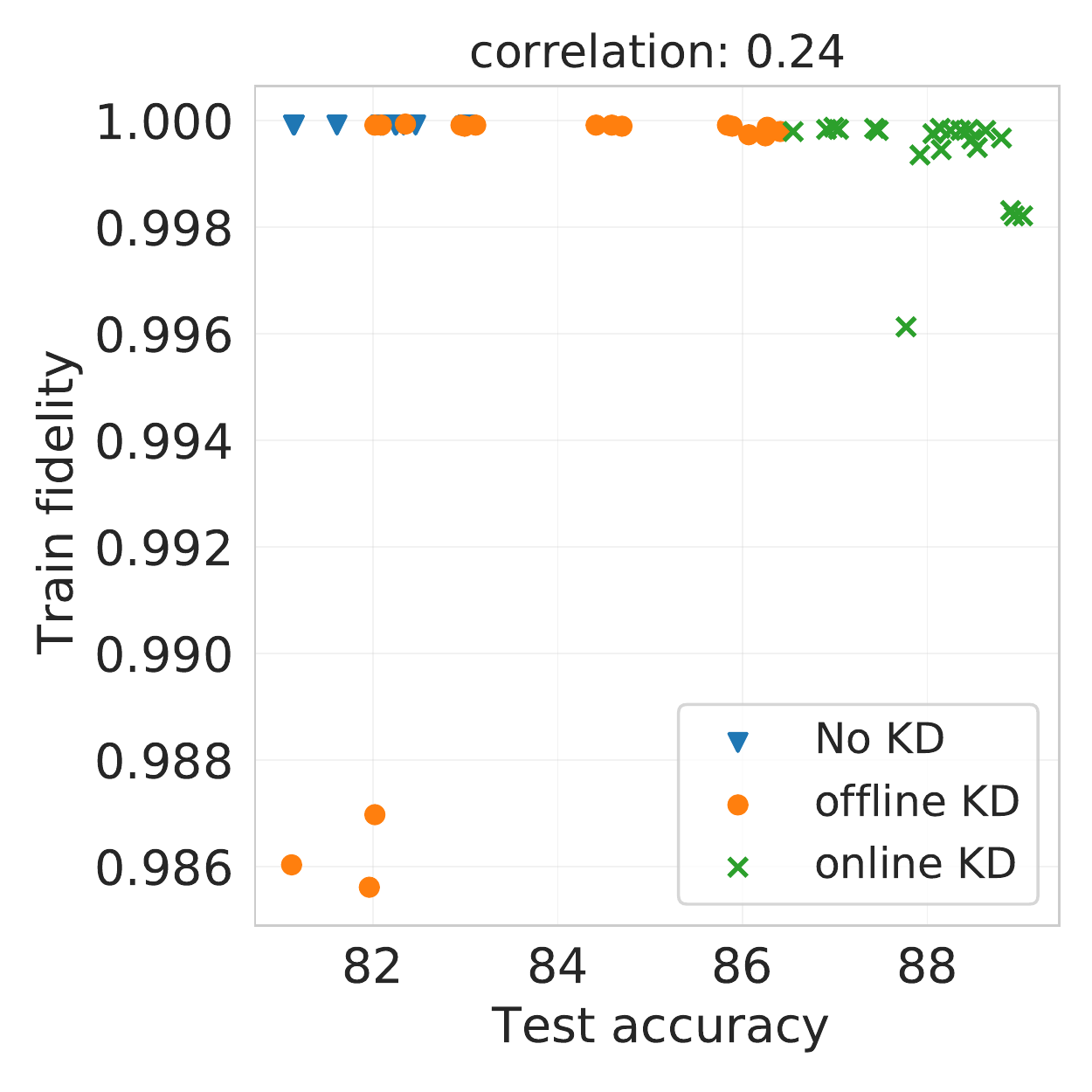}
        \caption{{\centering CIFAR-10, LeNet-5x8 student, ResNet-56 teacher}}
    \end{subfigure}
    \hspace{4em}
    \begin{subfigure}{0.35\textwidth}
        \includegraphics[width=\textwidth]{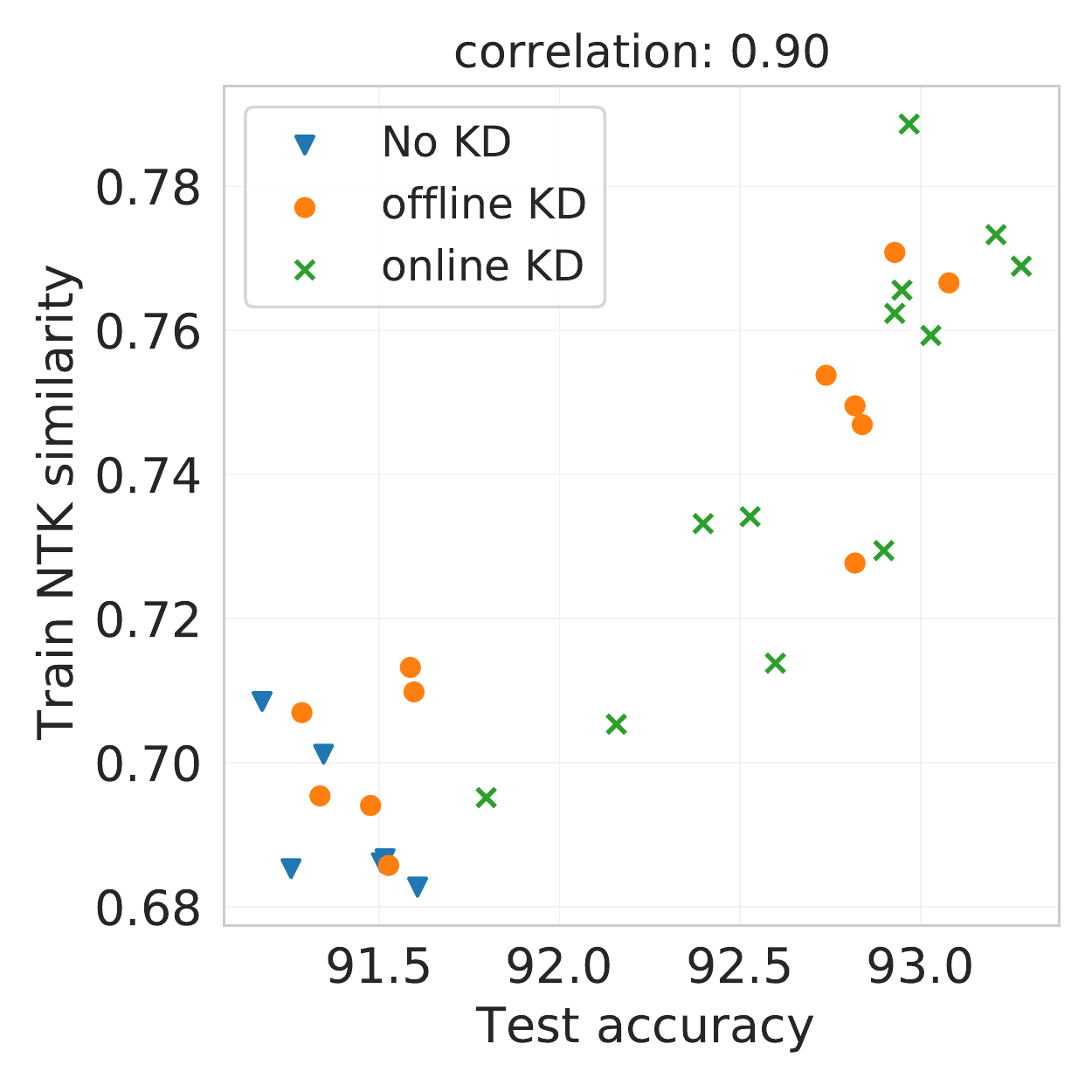}
        \caption{{\centering CIFAR-10, ResNet-20 student, ResNet-101 teacher}}
    \end{subfigure}
    \vskip 1em
    \begin{subfigure}{0.35\textwidth}
        \includegraphics[width=\textwidth]{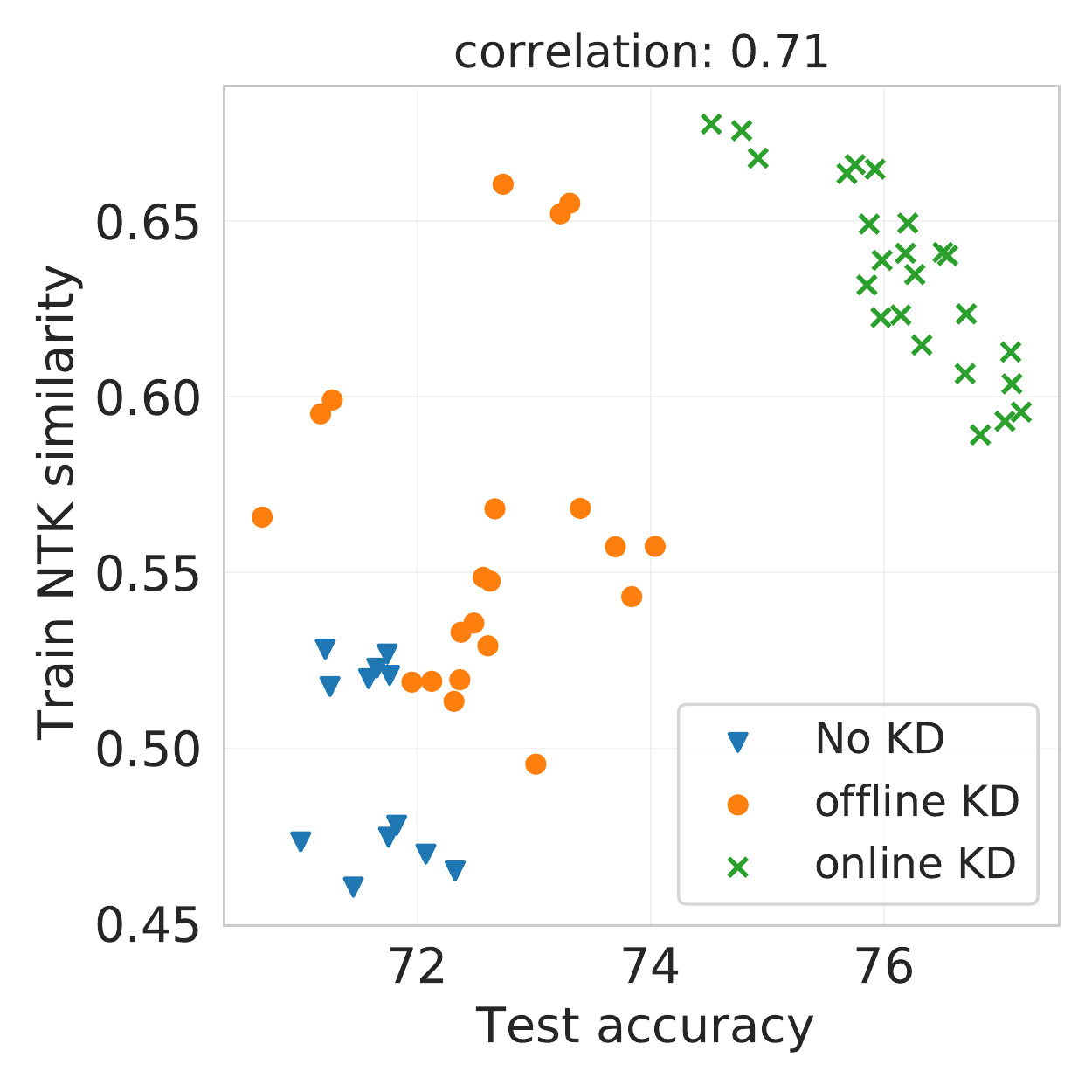}
        \caption{{\centering Binary CIFAR-100, LeNet-5x8 student, ResNet-56 teacher}}
    \end{subfigure}
    \hspace{4em}
    \begin{subfigure}{0.35\textwidth}
        \includegraphics[width=\textwidth]{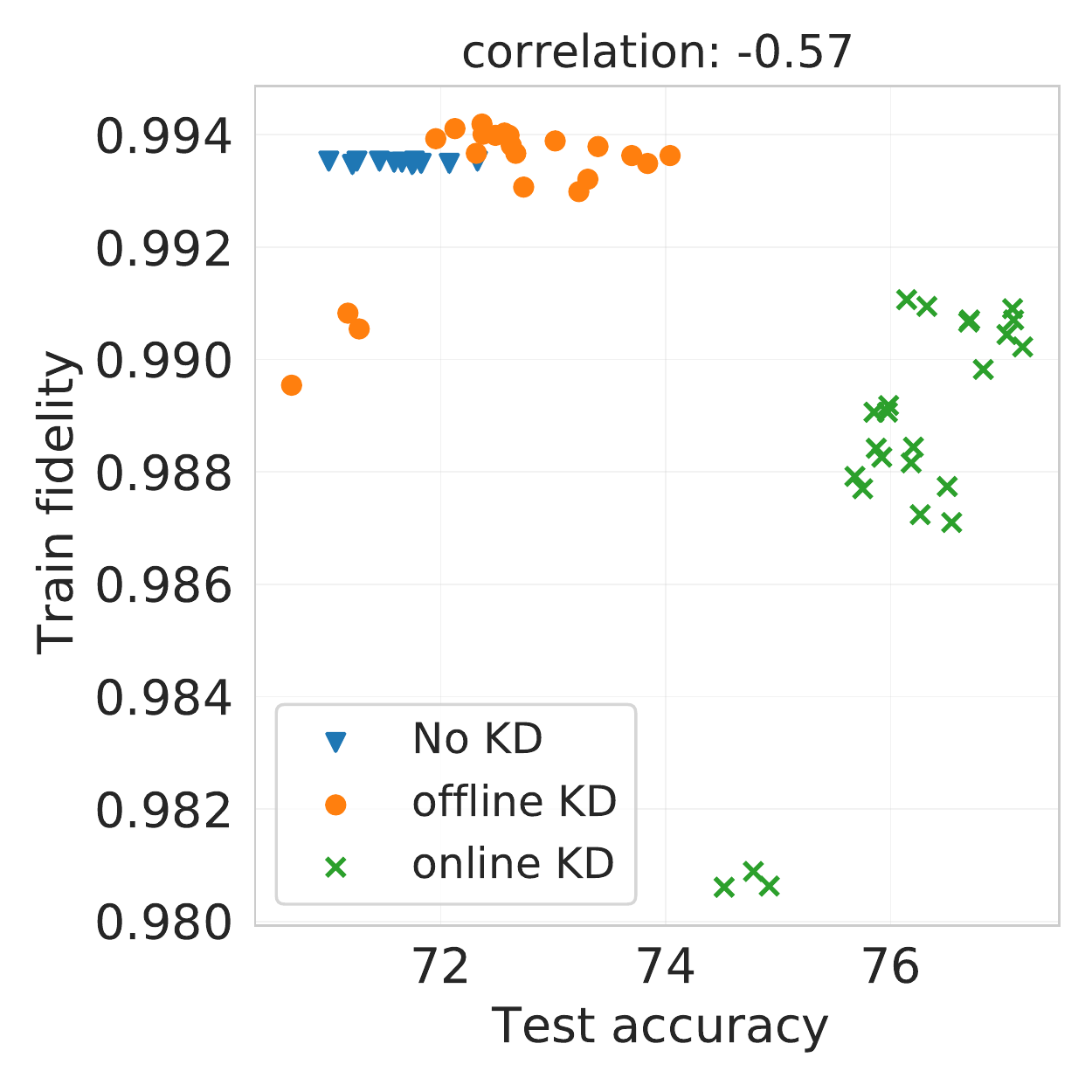}
        \caption{{\centering Binary CIFAR-100, LeNet-5x8 student, ResNet-56 teacher}}
    \end{subfigure}
    \caption{Relationship between test accuracy, train NTK similarity, and train fidelity for various teacher, student, and dataset configurations.}
    \label{fig:ntk-matching-comparision-1-appendix}
\end{figure}

\begin{figure}[!t]
    \centering
    \begin{subfigure}{0.35\textwidth}
        \includegraphics[width=\textwidth]{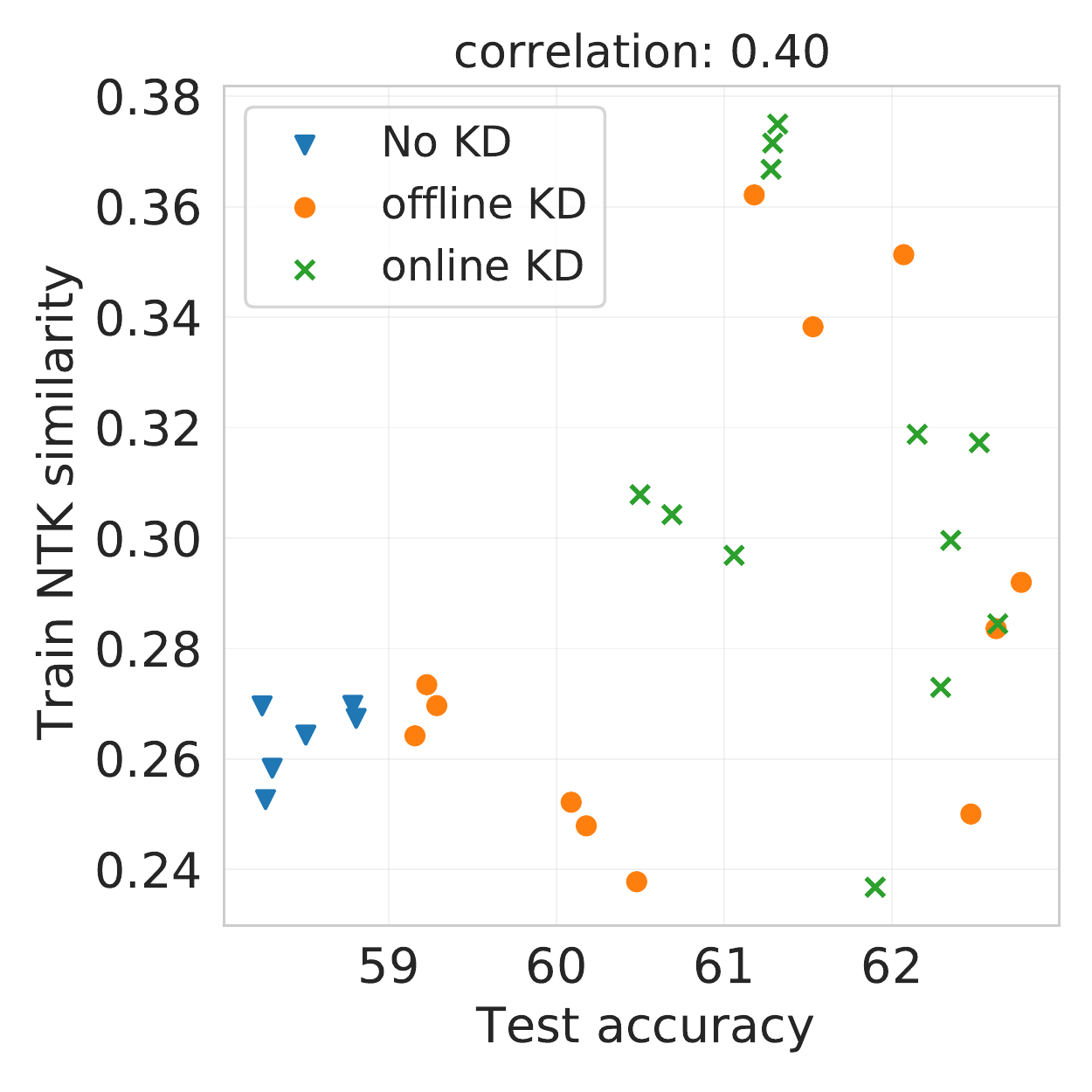}
        \caption{{\centering Tiny ImageNet, MobileNet-V3-35 student, MobileNet-V3-125 teacher}}
    \end{subfigure}
    \hspace{4em}
    \begin{subfigure}{0.35\textwidth}
        \includegraphics[width=\textwidth]{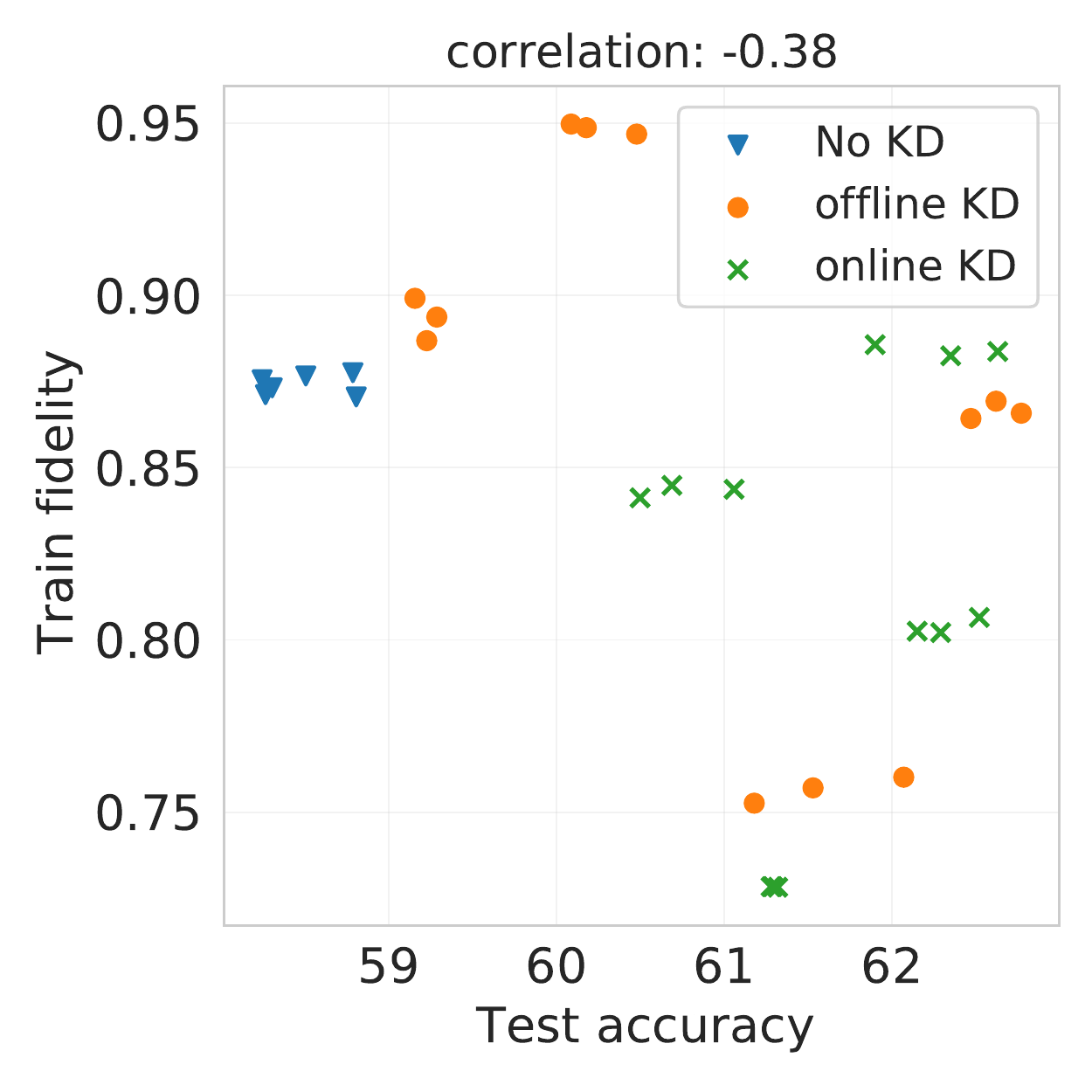}
        \caption{{\centering Tiny ImageNet, MobileNet-V3-35 student, MobileNet-V3-125 teacher}}
    \end{subfigure}
    \vskip 1em
    \begin{subfigure}{0.35\textwidth}
        \includegraphics[width=\textwidth]{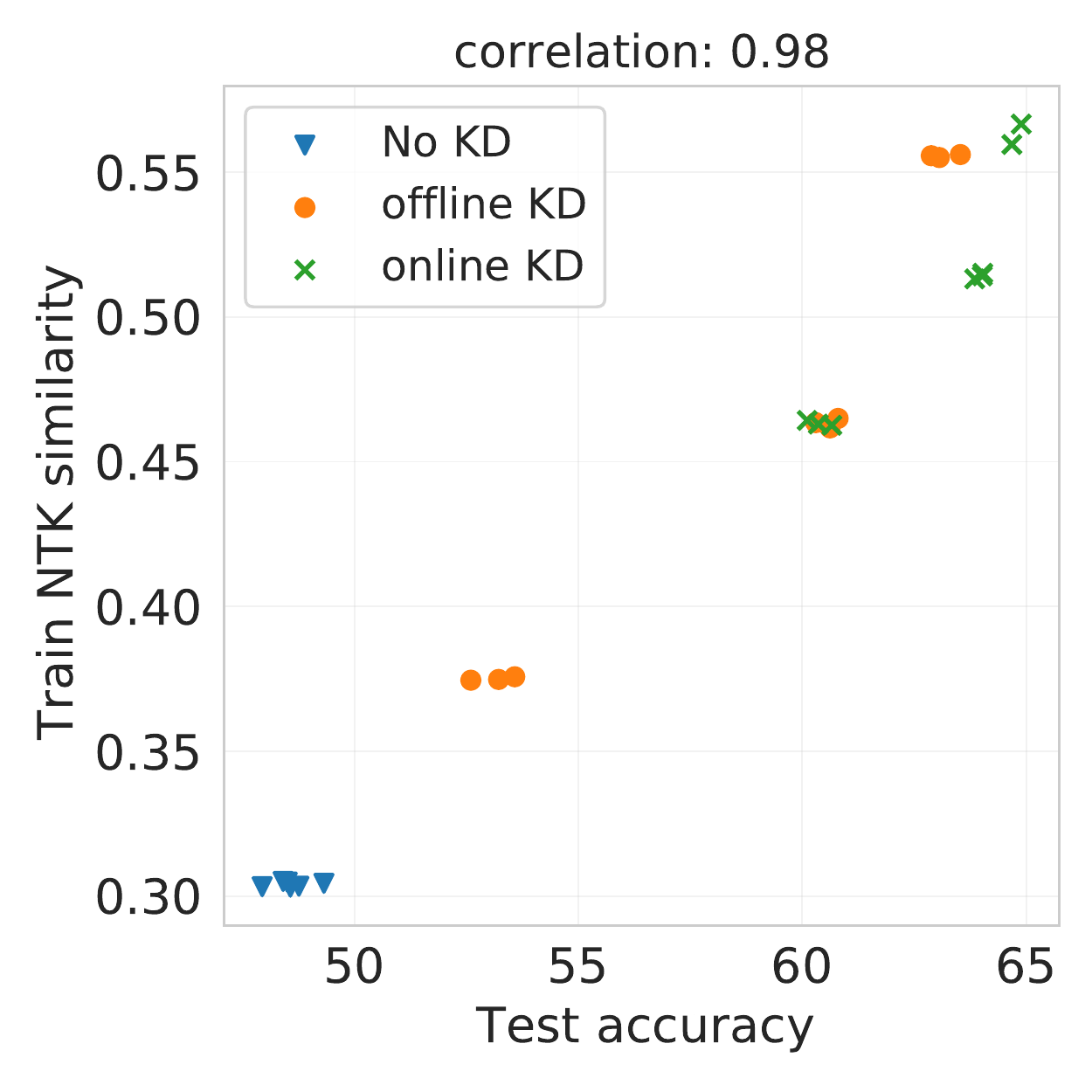}
        \caption{{\centering Tiny ImageNet, VGG-16 student, ResNet-101 teacher}}
    \end{subfigure}
    \hspace{4em}
    \begin{subfigure}{0.35\textwidth}
        \includegraphics[width=\textwidth]{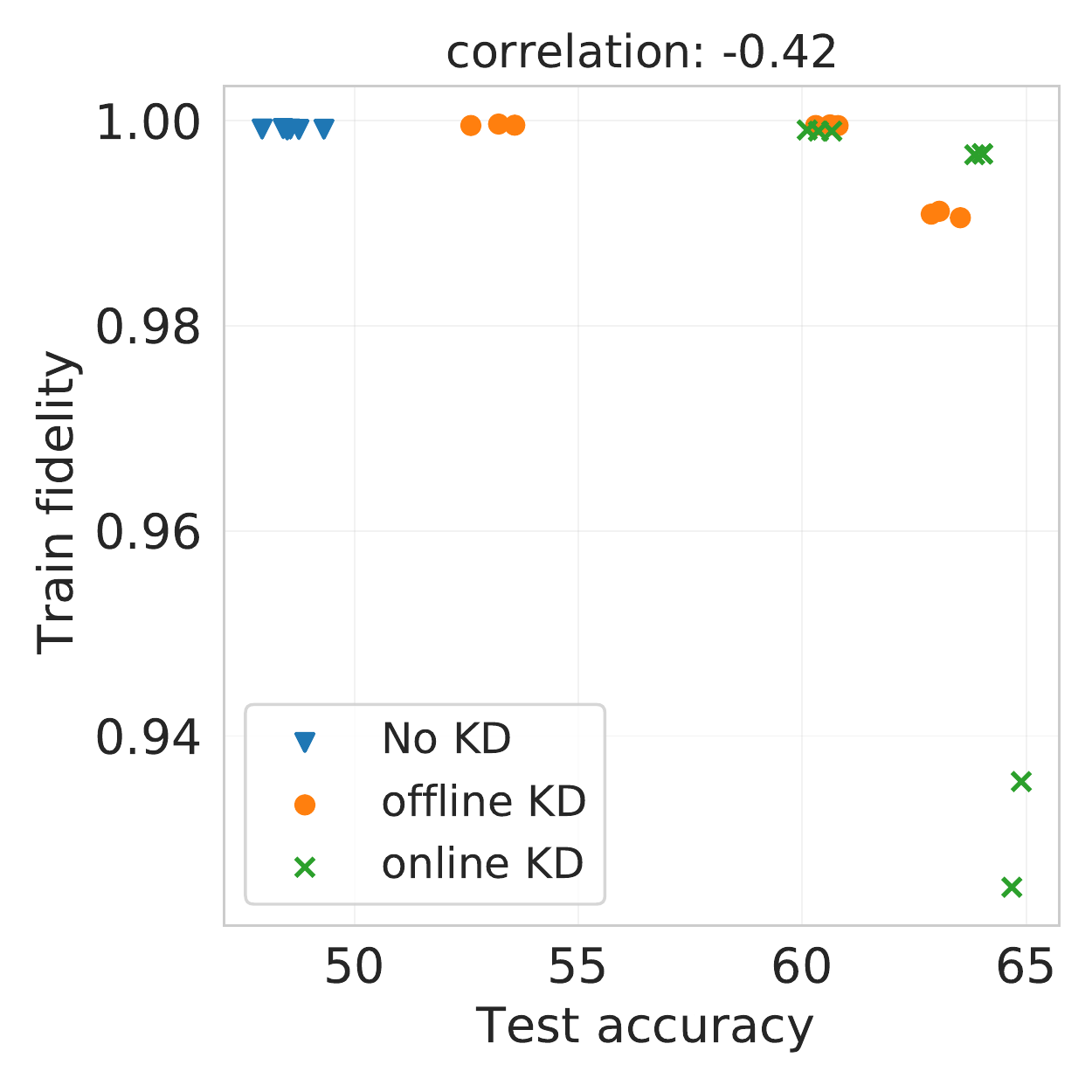}
        \caption{{\centering Tiny ImageNet, VGG-16 student, ResNet-101 teacher}}
    \end{subfigure}
    \vskip 1em
    \begin{subfigure}{0.35\textwidth}
        \includegraphics[width=\textwidth]{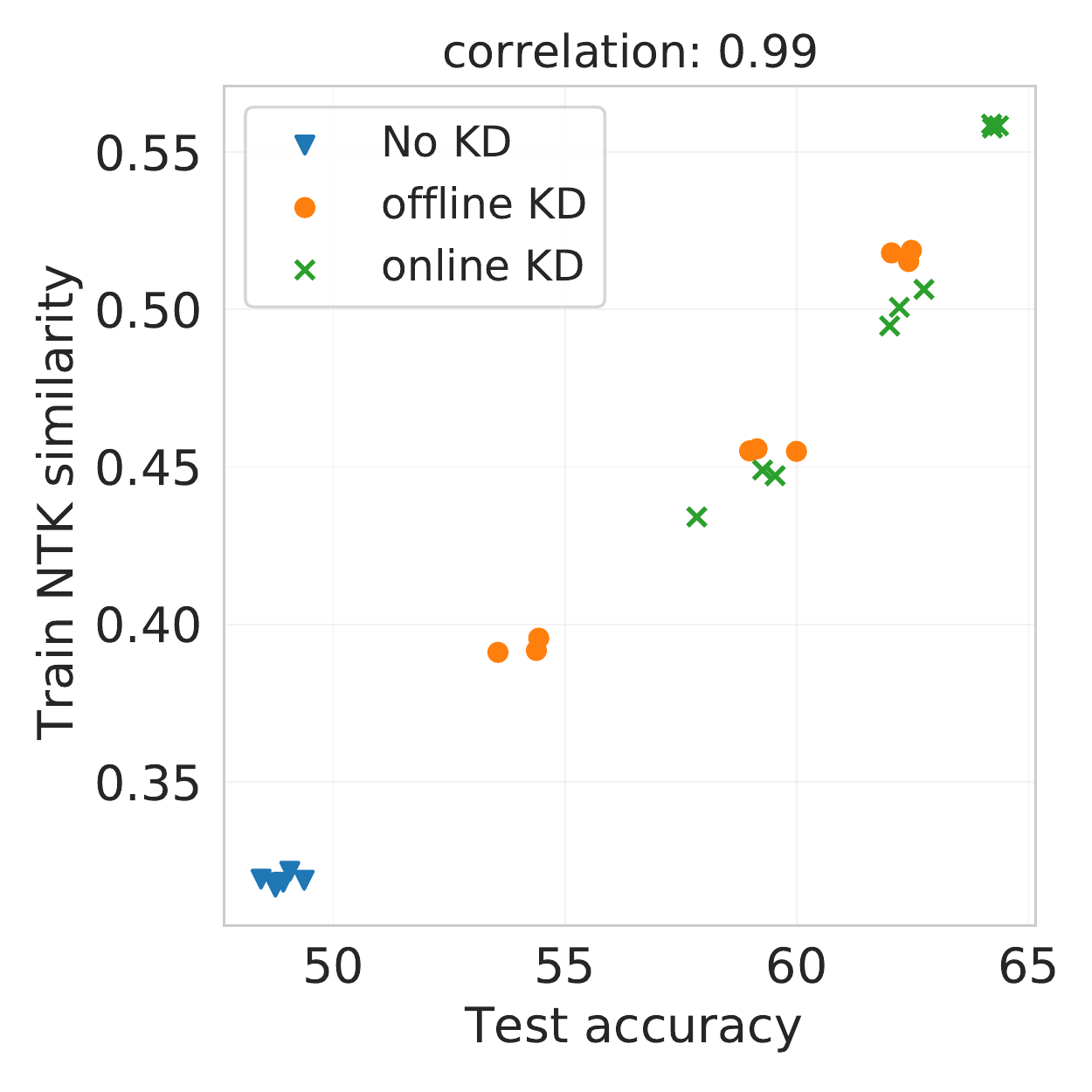}
        \caption{{\centering Tiny ImageNet, VGG-16 student, MobileNet-V3-125 teacher}}
    \end{subfigure}
    \hspace{4em}
    \begin{subfigure}{0.35\textwidth}
        \includegraphics[width=\textwidth]{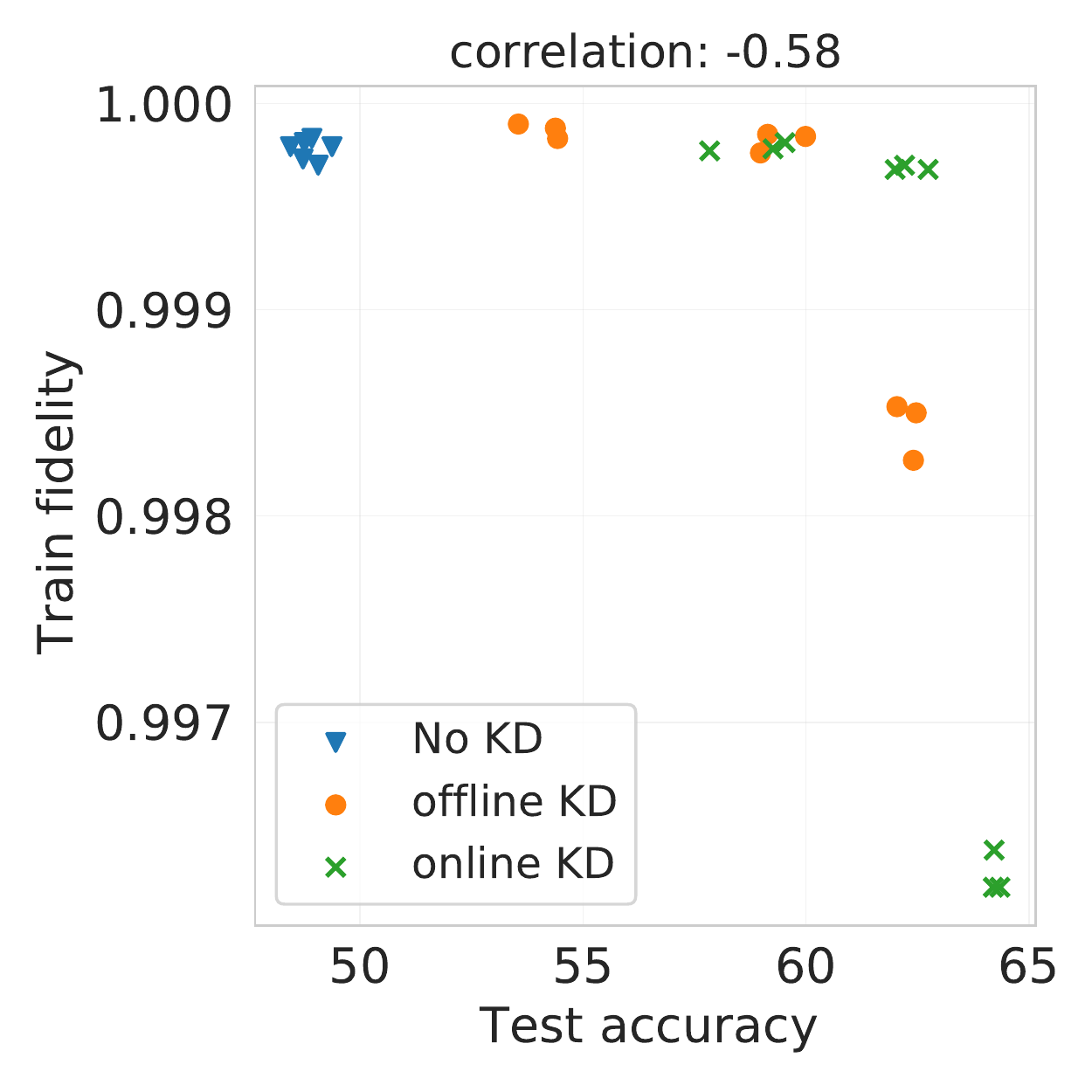}
        \caption{{\centering Tiny ImageNet, VGG-16 student, MobileNet-V3-125 teacher}}
    \end{subfigure}
    \caption{Relationship between test accuracy, train NTK similarity, and train fidelity for various teacher, student, and dataset configurations.}
    \label{fig:ntk-matching-comparision-2-appendix}
\end{figure}